\documentclass[a4paper,11pt]{report} 

\usepackage[fleqn]{amsmath}
\usepackage{amssymb}
\usepackage{amsbsy,color}
 \usepackage{hyperref}

\newtheorem{proposition}{Proposition}[chapter]

\newtheorem{definition}{Definition}[chapter]

 \newenvironment{proof}{\par\noindent{\bf Proof\ }}{\hfill\BlackBox\\[2mm]}
\usepackage[dvips]{graphicx}

 \newcommand{\BEAS}{\begin{eqnarray*}}
\newcommand{\EEAS}{\end{eqnarray*}}
\newcommand{\BEA}{\begin{eqnarray}}
\newcommand{\EEA}{\end{eqnarray}}
\newcommand{\BEQ}{\begin{equation}}
\newcommand{\EEQ}{\end{equation}}
\newcommand{\BIT}{\begin{itemize}}
\newcommand{\EIT}{\end{itemize}}
\newcommand{\BNUM}{\begin{enumerate}}
\newcommand{\ENUM}{\end{enumerate}}
\newcommand{\BA}{\begin{array}}
\newcommand{\EA}{\end{array}}

\newcommand{\Diag}{\mathop{\rm Diag}}

\newcommand{\cov}{{\mathop {\rm cov{}}}}

\newcommand{\argmin}{\mathop{\rm argmin}}

\newcommand{\var}{\mathop{ \rm var}}

\newcommand{\tr}{\mathop{ \rm tr}}
\newcommand{\sign}{\mathop{ \rm sign}}
\newcommand{\idm}{I}
\newcommand{\rb}{\mathbb{R}}
\newcommand{\BlackBox}{\rule{1.5ex}{1.5ex}}  
\newcommand{\lova}{Lov\'asz }

\newcommand{\mysec}[1]{\S\ref{sec:#1}}
\newcommand{\mychap}[1]{Chapter~\ref{chap:#1}}
\newcommand{\eq}[1]{Eq.~(\ref{eq:#1})}
\newcommand{\myfig}[1]{Figure~\ref{fig:#1}}

\def \supp{ { \rm Supp }}

\def \lova{Lov\'asz }

\title{ Learning with Submodular Functions: \\ 
 A Convex Optimization Perspective
}

\author{
Francis Bach \\
INRIA - Ecole Normale Sup\'erieure, Paris, France \\
francis.bach@ens.fr
}

\begin{document}

\maketitle

\begin{abstract}
 Submodular functions are relevant to machine learning for at least two reasons: (1) some problems may be expressed directly as the optimization of submodular functions  and (2) the \lova extension of submodular functions provides a useful set of regularization functions for supervised and unsupervised learning. In this monograph, we present the theory of submodular functions from a convex analysis perspective, presenting tight links between certain polyhedra, combinatorial optimization and convex optimization problems. In particular, we show how submodular function minimization is equivalent to solving a wide variety of convex optimization problems. This allows the derivation of new efficient algorithms for approximate and exact submodular function minimization with  theoretical guarantees and good practical performance.
By listing many examples of submodular functions, we review various applications to machine learning, such as clustering, experimental design, sensor placement, graphical model structure learning or subset selection, as well as a family of structured sparsity-inducing norms that can  be derived and used from submodular functions.
\end{abstract}

\tableofcontents

\chapter{Introduction}
 
Many combinatorial optimization problems may be cast as the minimization of a \emph{set-function}, that is a function defined on the set of subsets of a given base set $V$. Equivalently, they may be defined as functions 
on the vertices of the hyper-cube, i.e,  $\{0,1\}^p$ where $p$ is the cardinality of the base set $V$---they are then often referred to as pseudo-boolean functions~\cite{boros2002pseudo}.  Among these set-functions, submodular functions play an important role, similar to convex functions on vector spaces, as many functions that occur in practical problems turn out to be submodular functions or slight modifications thereof, with applications in many areas  areas of computer science and applied mathematics, such as machine learning~\cite{krause2005near, narasimhan2004pac,kawahara22submodularity,krause2010submodular}, computer vision~\cite{boykov2001fast,hochbaum2001efficient}, operations research~\cite{hochbaum1995strongly,queyranne1995scheduling}, electrical networks~\cite{nara} or~economics~\cite{topkis2001supermodularity}. Since submodular functions may be minimized exactly, and maximized approximately with some guarantees, in polynomial time, they readily lead to efficient algorithms for all the numerous  problems they apply to. They are also appear in several areas of theoretical computer science, such as matroid theory~\cite{schrijver2004combinatorial}.

However, the interest for submodular functions is not limited to discrete optimization problems.
Indeed, the rich structure of submodular functions and their link with convex analysis through the \lova extension~\cite{lovasz1982submodular} and the various associated polytopes makes them particularly adapted to  problems beyond combinatorial optimization, namely as regularizers in signal processing and machine learning problems~\cite{chambolle2009total,bach2010structured}. Indeed, many  continuous optimization problems exhibit an underlying discrete structure (e.g., based on chains, trees or more general graphs), and submodular functions provide an efficient and versatile tool to capture such combinatorial structures.
 
 In this monograph, the theory of submodular functions is presented in a self-contained way, with all results proved from first principles of convex analysis common in machine learning, rather than relying on combinatorial optimization and traditional theoretical computer science concepts such as matroids or flows~(see, e.g.,~\cite{fujishige2005submodular} for a reference book on such approaches).
 Moreover, the algorithms that we present are based on traditional convex optimization algorithms such as the simplex method for linear programming, active set method for quadratic programming, ellipsoid method, cutting planes, and conditional gradient. These will be presented in details, in particular in the context of submodular function minimization and its various continuous extensions.
  A good knowledge of convex analysis is assumed (see, e.g.,~\cite{boyd,borwein2006caa}) and a short review of important concepts is presented in Appendix~\ref{app:convex}---for more details, see, e.g.,~\cite{hiriart1996convex,boyd,borwein2006caa,rockafellar97}.

 \paragraph{Monograph outline.} 
 The monograph is organized in several chapters, which are summarized below (in the table of contents, sections that can be skipped in a first reading are marked with a star$^\ast$):
 
\begin{list}{\labelitemi}{\leftmargin=1.1em}
   \addtolength{\itemsep}{-.0\baselineskip}

\item[(1)]  \textbf{Definitions}:  In \mychap{definitions}, we give the different definitions of submodular functions and of the associated polyhedra, in particular, the base polyhedron and the submodular polyhedron. They are crucial in submodular analysis as many algorithms and models may be expressed naturally using these polyhedra.
 
 \item[(2)] \textbf{\lova extension}:  In \mychap{lova}, we define the \lova extension as an extension from a function defined on $\{0,1\}^p$ to a function defined on $[0,1]^p$ (and then $\rb^p$), and give its main properties. In particular we present  key results in submodular analysis: the \lova extension is convex if and only if the set-function is submodular; moreover, minimizing the submodular set-function $F$ is equivalent to minimizing the \lova extension on $[0,1]^p$. This implies notably that submodular function minimization may be solved in polynomial time. Finally, the link between the \lova extension and the submodular polyhedra through the so-called ``greedy algorithm'' is established: the \lova extension is the support function of the base polyhedron and may be computed in closed form.

\item[(3)] \textbf{Polyhedra}: Associated polyhedra are further studied in \mychap{support}, where support functions and the associated maximizers of linear functions are computed. We also detail the facial structure of such polyhedra, which will be useful when related to the sparsity-inducing properties of the \lova extension in \mychap{relax}.

\item[(4)] \textbf{Convex relaxation of submodular penalties}: 
While submodular functions may be used directly (for minimization of maximization of set-functions), we show in \mychap{relax} how they may be used to penalize supports or level sets of vectors. The resulting mixed combinatorial/continuous optimization problems may be naturally relaxed into convex optimization problems  using the \lova extension.

 \item[(5)] \textbf{Examples}: 
In \mychap{examples}, we present classical examples of submodular functions, together with several applications in machine learning, in particular, cuts, set covers, network flows, entropies, spectral functions and matroids.

 \item[(6)] \textbf{Non-smooth convex optimization}: 
In \mychap{nonsmooth}, we review classical iterative algorithms adapted to the minimization of non-smooth polyhedral functions, such as  subgradient, ellipsoid, simplicial, cutting-planes, active-set, and conditional gradient methods. A particular attention is put on providing when applicable primal/dual interpretations to these algorithms.

\item[(7)] \textbf{Separable optimization - Analysis}: In \mychap{prox}, we consider \emph{separable} optimization problems regularized by the \lova extension $w \mapsto f(w)$, i.e., problems of the form $\min_{w \in \rb^p} \sum_{k \in V} \psi_k(w_k) + f(w)$,  and show how this is equivalent to a sequence of submodular function minimization problems. This is a key theoretical link between combinatorial and convex optimization problems related to submodular functions, that will be used in later chapters.

\item[(8)] \textbf{Separable optimization - Algorithms}: In \mychap{prox-algo}, we present two sets of algorithms for separable optimization problems. The first algorithm is  an exact algorithm which relies on the availability of an efficient submodular function minimization algorithm, while the second set of algorithms are based on existing iterative algorithms for convex optimization, some of which come with online and offline theoretical guarantees. We consider active-set methods (``min-norm-point'' algorithm) and conditional gradient methods.

\item[(9)] \textbf{Submodular function minimization}: In \mychap{sfm}, we present various approaches to submodular function minimization. We present briefly the combinatorial algorithms for exact submodular function minimization, and focus in more depth on the use of specific  {convex}   optimization problems, which can be solved iteratively to obtain approximate or exact solutions for submodular function minimization, with sometimes theoretical guarantees and approximate optimality certificates. We consider the subgradient method, the ellipsoid method,  the simplex algorithm and analytic center cutting planes. We also show how the separable optimization problems from Chapters~\ref{chap:prox} and \ref{chap:prox-algo} may be used for submodular function minimization. These methods are then empirically compared in \mychap{experiments}.

\item[(10)] \textbf{Submodular optimization problems}:
 In \mychap{max-ds}, we present other combinatorial optimization problems which can be partially solved using submodular analysis, such as submodular function maximization and the optimization of differences of submodular functions, and relate  these to non-convex optimization problems on the submodular polyhedra. While these problems typically cannot be  solved in polynomial time, many algorithms come with approximation guarantees based on submodularity.
 
 \item[(11)] \textbf{Experiments}: In \mychap{experiments}, we provide illustrations of the optimization algorithms described earlier, for submodular function minimization, as well as for convex optimization problems (separable or not). The Matlab code for all these experiments may be found at \url{http://www.di.ens.fr/~fbach/submodular/}.
 \end{list}

In  Appendix~\ref{app:convex}, we review relevant notions from convex analysis  (such as Fenchel duality, dual norms, gauge functions, and polar sets), while in Appendix~\ref{app:misc}, we present several results related to submodular functions, such as operations that preserve submodularity.

 Several books and monograph articles already exist on the same topic and the material presented in this monograph rely  on those~\cite{fujishige2005submodular,nara,krause-beyond}. However, in order to present the material in the simplest way, ideas from related research papers have also been used, and a stronger emphasis is put on convex analysis and optimization.

 \paragraph{Notations.} We consider the set $V = \{1,\dots,p\}$, and its power set $2^V$, composed of the $2^p$ subsets of $V$. Given a vector $s \in \rb^p$, $s$ also denotes the modular set-function defined as $s(A) = \sum_{k \in A }s_k$. Moreover, $A \subseteq B$ means that $A$ is a subset of $B$, potentially equal to $B$. 
 We denote
  by $|A|$ the cardinality of the set~$A$, and, for $A \subseteq V = \{1,\dots,p\}$, $1_A \in \rb^p$ denotes the indicator vector of the set $A$. If $w \in \rb^p$, and $ \alpha \in \rb$, then $\{ w \geqslant \alpha\}$ (resp.~$\{ w > \alpha\}$) denotes the subset of $V =\{1,\dots,p\}$ defined as $\{ k \in V, \ w_k \geqslant \alpha \}$ (resp.~$\{ k \in V, \ w_k > \alpha \}$), which we refer to as the weak (resp.~strong) $\alpha$-sup-level sets of $w$. Similarly if $v \in \rb^p$, we denote
 $\{ w \geqslant v \} = \{ k \in V, \ w_k \geqslant v_k \}$.
 
 For $q \in [1,+\infty]$, we denote by $\| w\|_q$ the $\ell_q$-norm of $w$, defined as
 $\| w\|_q = \big( \sum_{k \in V} |w_k|^q \big)^{1/q}$ for $q \in [1,\infty)$ and $\| w\|_\infty = \max_{k \in V} | w_k|$. Finally, we denote by $\rb_+$ the set of non-negative real numbers, by $\rb^\ast$ the set of non-zero real numbers, and by $\rb^\ast_+$ the set of strictly positive real numbers.

 \chapter{Definitions}
\label{chap:definitions}

Throughout this monograph, we consider $V = \{1,\dots,p\}$, $p>0$ and its power set (i.e., set of all subsets) $2^V$, which is of cardinality $2^p$. We also consider a real-valued set-function $F: 2^V \to \rb$ such that $F(\varnothing)=0$. As opposed to the common convention with convex functions (see Appendix~\ref{app:convex}), we do not allow infinite values for the function $F$.

The field of submodular analysis takes its roots in matroid theory, and submodular functions were first seen as extensions of rank functions of matroids (see~\cite{edmonds} and~\mysec{matroids}) and their analysis strongly linked with  special convex polyhedra which we define in \mysec{polyhedra}. After the links with convex analysis  were established~\cite{edmonds,lovasz1982submodular}, submodularity appeared as a central concept in combinatorial optimization.  Like convexity, many models in science and engineering and in particular in machine learning involve submodularity (see \mychap{examples} for many examples). Like convexity, submodularity is usually enough to derive general theories and generic algorithms (but of course some special cases are still of importance, such as min-cut/max-flow problems), which have attractive theoretical and practical properties. Finally, like convexity, there are many areas where submodular functions play a central but somewhat hidden role in combinatorial and convex optimization. For example, in \mychap{relaxation}, we show how many problems in convex optimization involving discrete structured turns out be cast as submodular optimization problems, which then immediately lead to efficient algorithms.

In \mysec{def}, we provide the definition of submodularity and its equivalent characterizations. While submodularity may appear rather abstract, it turns out it come up naturally in many examples. In this chapter, we will only review a few classical examples which will help illustrate our various results. For an extensive list of examples, see \mychap{examples}.
In \mysec{polyhedra}, we define two polyhedra traditionally associated with a submodular function, while in \mysec{polymat}, we consider non-decreasing submodular functions, often referred to as polymatroid rank functions.

\section{Equivalent definitions of submodularity}
\label{sec:def}
Submodular functions may be defined through several equivalent properties, which we now present. Additive measures are the first examples of set-functions, the cardinality being the simplest example. A well known property of the cardinality is that for any two sets $A,B \subseteq V$, then $|A| + |B| =|A\cup B| + |A \cap B|$, which extends to all additive measures. A function is submodular if and only if the previous equality is only an inequality for all subsets $A$ and $B$ of $V$:

\begin{definition} \textbf{(Submodular function)}
\label{def:def}
A set-function $F: 2^V \to \rb$ is   submodular if and only if, for all subsets $A,B \subseteq V$, we have:
$F(A) + F(B) \geqslant F(A\cup B) + F(A \cap B)$.
\end{definition}

Note that if a function is submodular and such that $F(\varnothing)=0$ (which we will always assume), for any two \emph{disjoint} sets $A,B \subseteq V$, then $F(A \cup B) \leqslant F(A) + F(B)$, i.e., submodularity implies sub-additivity (but the converse is not true).

As seen earlier, the simplest example of a submodular function is the cardinality (i.e., $F(A) =|A|$ where $|A|$ is the number of elements of $A$), which is both submodular and supermodular (i.e., its opposite $A \mapsto -F(A)$ is submodular). It turns out that only additive measures have this property of being \emph{modular}.

\begin{proposition}[Modular function]
A set-function $F: 2^V \to \rb$ such that $F(\varnothing)=0$ is modular (i.e., both submodular and supermodular) if and only if there exists $s \in \rb^p$ such that $F(A) = \sum_{k \in A} s_k$.
\end{proposition}
\begin{proof}
For a given $s \in \rb^p$, $A \mapsto \sum_{k \in A} s_k$ is an additive measure and is thus submodular. If $F$ is submodular and supermodular, then it is both sub-additive and super-additive.
This implies that $F(A) = \sum_{k \in A} F(\{k\})$ for all $A \subseteq V$, which defines   a vector $s\in \rb^p$ with $s_k = F(\{k\})$, such that $F(A) = \sum_{k \in A} s_k$.
\end{proof}
From now on, from a vector $s \in \rb^p$, we denote by $s$ the modular set-function defined as $s(A) = \sum_{k \in A} s_k = s^\top 1_A$, where $1_A \in \rb^p$ is the indicator vector of the set $A$. Modular functions essentially play for set-functions the same role as linear functions for continuous functions.

\paragraph{Operations that preserve submodularity.}
From Def.~\ref{def:def}, it is clear that the set of submodular functions is closed under linear combination and multiplication by a positive scalar (like convex functions).

Moreover, like convex functions, several notions of restrictions and extensions may be defined for submodular functions (proofs immediately follow from Def.~\ref{def:def}):
\begin{list}{\labelitemi}{\leftmargin=1.1em}
   \addtolength{\itemsep}{-.0\baselineskip}

\item[--] \textbf{Extension}: given a set $B \subseteq V$, and a submodular function $G:2^B \to \rb$, then the function $F: 2^V \to \rb$ defined as $F(A) = G( B \cap A)$ is submodular.
\item[--] \textbf{Restriction}: given a set $B \subseteq V$, and a submodular function $G:2^V \to \rb$, then the function $F: 2^B\to \rb$ defined as $F(A) = G(A)$ is submodular.
\item[--] \textbf{Contraction}: given a set $B \subseteq V$, and a submodular function $G:2^V \to \rb$, then the function $F: 2^{V \backslash B}\to \rb$ defined as $F(A) = G(A \cup B - G(B)$ is submodular (and such that $G(\varnothing)=0$).
\end{list}
More operations that preserve submodularity are defined in Appendix~\ref{app:preserve}, in particular partial minimization (like for convex functions). Note however, that in general the pointwise minimum or pointwise maximum of submodular functions are not submodular (properties which would be  true
for respectively concave and convex functions).

\paragraph{Proving submodularity.}

Checking the condition in Def.~\ref{def:def} is not always easy in practice; it turns out that  it can be restricted to only certain sets $A$ and $B$, which we  now present.

The following proposition shows that a submodular has the ``diminishing return'' property, and that this is sufficient to be submodular. Thus, submodular functions may be seen as a discrete analog to \emph{concave} functions. However, as shown in \mychap{lova}, in terms of optimization they behave more like \emph{convex} functions (e.g., efficient minimization, duality theory, links with the convex \lova extension).

\begin{proposition} \textbf{(Definition with first-order differences}) 
\label{prop:firstorder}
The set-function
$F$ is submodular if and only if for all $A,B  \subseteq  V$ and $k \in V$, such that $A \subseteq B$ and $k \notin B$, we have
 $$ F(A \cup \{k\}) - F(A) \geqslant  F(B \cup \{k\}) - F(B) .$$
\end{proposition}
\begin{proof} Let $A \subseteq B$, and $k \notin B$; we have
$F(A \cup\{k\}) - F(A) - F(B \cup \{k\}) + F(B) = F(C)+F(D) - F(C \cup D) - F(C \cap D)$ with $C = A \cup \{k\}$ and $D = B$, which shows that the condition is necessary.  To prove the opposite, we assume that the first-order difference condition is satisfied; one can first show that if $A \subseteq B $ and $ C \cap B = \varnothing$, then $F(A \cup C) - F(A) \geqslant F(B \cup C) - F(B)$ (this can be obtained by summing the $m$ inequalities $F(A \cup \{c_1,\dots,c_k\} ) - F(A \cup \{c_1,\dots,c_{k-1}\}) \geqslant
F(B \cup \{c_1,\dots,c_k\} ) - F(B \cup \{c_1,\dots,c_{k-1}\})$ where $C = \{c_1,\dots,c_m\}$).

Then, for any $X,Y \subseteq V$, take $A = X \cap Y$, $C=X \backslash Y$ and $B=Y$ (which implies $A \cup C = X$ and $B \cup C = X \cup Y$) to obtain
$F(X) + F(Y) \geqslant F(X\cup Y) + F(X \cap Y)$, which shows that the condition is sufficient.
\end{proof}

The following proposition gives the tightest condition for submodularity (easiest to show in practice).

\begin{proposition} \textbf{(Definition with second-order differences)}
\label{prop:second}
The set-function
 $F$ is submodular if and only if for all $A \subseteq  V$ and $j,k \in V \backslash A$, we have
 $ F(A \cup \{k\}) - F(A) \geqslant  F(A \cup \{j,k\}) - F(A \cup \{ j\} ) $.
\end{proposition}
\begin{proof} 
This condition is weaker than the one from the previous proposition (as it corresponds to taking $B = A \cup \{j\}$). To prove that it is still sufficient, consider
$A \subseteq V$,  $B = A \cup \{ b_1,\dots,b_s\}$, and $k \in V \backslash B$. We can apply the second-order difference condition to  subsets $A \cup \{b_1,\dots,b_{s-1}\}$, $j=b_{s}$, and sum the $m$ inequalities $F(A \cup \{b_1,\dots,b_{s-1} \} \cup \{k\} ) - F( A\cup \{b_1,\dots,b_{s-1} \} \ )\geqslant F(A \cup \{b_1,\dots,b_{s} \} \cup \{k\} ) - F( A\cup \{b_1,\dots,b_{s} \}  )$, for $s \in \{1,\dots,m\}$, to obtain the condition in Prop.~\ref{prop:firstorder}.
\end{proof}

Note that the set of submodular functions is itself a conic polyhedron   with the facets defined in Prop.~\ref{prop:second}.
In order to show that a given set-function is submodular, there are several possibilities: (a) use Prop.~\ref{prop:second} directly, (b) use the \lova extension (see \mychap{lova}) and show that it is convex,
(c) cast the function as a special case from \mychap{examples} (typically a cut or a flow), or (d) use known operations on submodular functions presented in Appendix~\ref{app:ope}.

Beyond modular functions, we will consider as running examples for the first chapters of this monograph the following submodular functions (which will be studied further in \mychap{examples}):
\begin{list}{\labelitemi}{\leftmargin=1.1em}
   \addtolength{\itemsep}{-.0\baselineskip}
\item[--] \textbf{Indicator function of non-empty sets}: we consider the function $F:2^V \to \rb$ such that $F(A)=0$ if $A = \varnothing $ and $F(A) = 1$ otherwise. By Prop.~\ref{prop:firstorder} or Prop.~\ref{prop:second}, this function is obviously submodular (the gain of adding any element is always zero, except when adding to the empty set, and thus the returns are indeed  diminishing). Note that this function may be written compactly as
$F(A) = \min\{|A|,1\}$ or $F(A) = 1_{ |A| > 0 } = 1_{ A \neq \varnothing}$. Generalizations to all cardinality-based functions will be studied in \mysec{cardinality}.

\item[--] \textbf{Counting elements in a partitions}: Given a partition of $V$ into $m$ sets $G_1,\dots,G_m$, then the function $F$ that counts for a set $A$ the number of elements in the partition which intersects $A$ is submodular. It may be written as $F(A) = \sum_{j=1}^m \min\{| A \cap G_j|,1\}$ (submodularity is then immediate from the previous example and the restriction properties outlined previously). Generalizations to all set covers will be studied in \mysec{covers}.

\item[--] \textbf{Cuts}: given an undirected graph $G=(V,E)$ with vertex set $V$, then the cut function for the set $ A\subseteq V$  is defined as the number of edges between vertices in $A$ and vertices in $V \backslash A$, i.e., 
$F(A) = \sum_{(u,v) \in E} |(1_A)_u - (1_A)_v|$. For each $(u,v) \in E$, then the function  $|(1_A)_u - (1_A)_v|
= 2 \min \{ |A \cap \{ u,v\}|,1 \} - |A \cap \{ u,v\}|$ is submodular (because of operations that preserve submodularity), thus as a sum of submodular functions, it is submodular.
\end{list}

\section{Associated polyhedra}
 \label{sec:polyhedra}

We now define specific polyhedra in $\rb^p$.
These play a crucial role in submodular analysis, as most results  and algorithms in this monograph may be interpreted or proved using such polyhedra.

\begin{definition} \textbf{(Submodular and base polyhedra)}
\label{def:polyhedra}
Let $F$ be a submodular function such that $F(\varnothing)=0$. The submodular polyhedron $P(F)$ and the base polyhedron $B(F)$ are defined as:
\BEAS
P(F) & = &  \{ s \in \rb^p, \ \forall A \subseteq V, s(A) \leqslant F(A) \} \\
B(F) & = &  \{ s \in \rb^p, \ s(V) = F(V), \ \forall A \subseteq V, s(A) \leqslant F(A) \} \\
&
= & P(F) \cap \{ s(V) = F(V) \} .
\EEAS
\end{definition}

These polyhedra are defined as the intersection of hyperplanes $\{s \in \rb^p, \  s(A) \leqslant F(A)\} = \{
s \in \rb^p, \ s^\top 1_A \leqslant f(A)\} = \{ s \leqslant t\}$, whose normals are indicator vectors $1_A$ of subsets $A$ of $V$.
As shown in the following proposition, the submodular polyhedron $P(F)$ has non-empty interior and is unbounded. Note that the other polyhedron (the base polyhedron) will be shown to be non-empty and bounded as a consequence of Prop.~\ref{prop:greedy}. It has empty interior since it is included in the subspace $s(V)=F(V)$. 

For a modular function $F:A \mapsto t(A)$ for $ t \in \rb^p$, then $P(F) = \{ s \in \rb^p, \forall k \in V, \  s_k \leqslant t_k \}$, and it thus isomorphic (up to translation) to the negative orthant. However, for a more general function, $P(F)$ may have more extreme points; see \myfig{poly} for canonical examples with $p=2$ and $p=3$.

\begin{proposition} \textbf{(Properties of submodular polyhedron)}
\label{prop:nonemptyinterior}
Let $F$ be a submodular function such that $F(\varnothing)=0$. If $s \in P(F)$, then for all $t \in \rb^p$, such that $t \leqslant s$ (i.e., $\forall k \in V, \ t_k \leqslant s_k$), we have $t \in P(F)$. Moreover, $P(F)$ has non-empty interior.
\end{proposition}
\begin{proof}
The first part is trivial, since $t \leqslant s$ implies that for all $A \subseteq V$,  $t(A) \leqslant s(A)$. For the second part, given the previous property, we only need to show that $P(F)$ is non-empty, which is true since the constant vector equal to $\min_{ A \subseteq V, \ A \neq \varnothing} \frac{ F(A) }{|A|}$ belongs to $P(F)$.
\end{proof}

\begin{figure}

\begin{center}
\parbox[b]{5cm}{
\includegraphics[scale=.46]{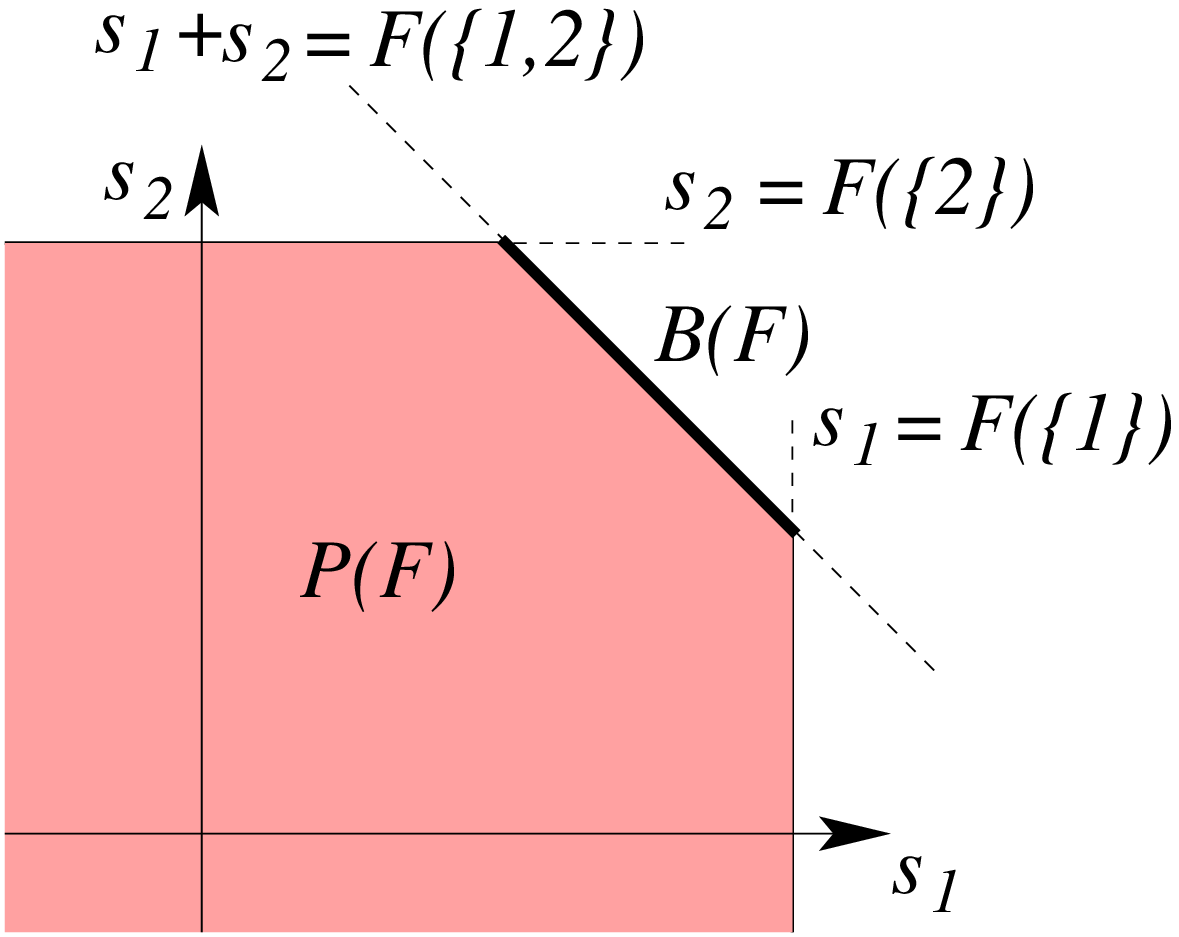}
\vspace*{.25cm}}
\hspace*{.25cm}
\includegraphics[scale=.46]{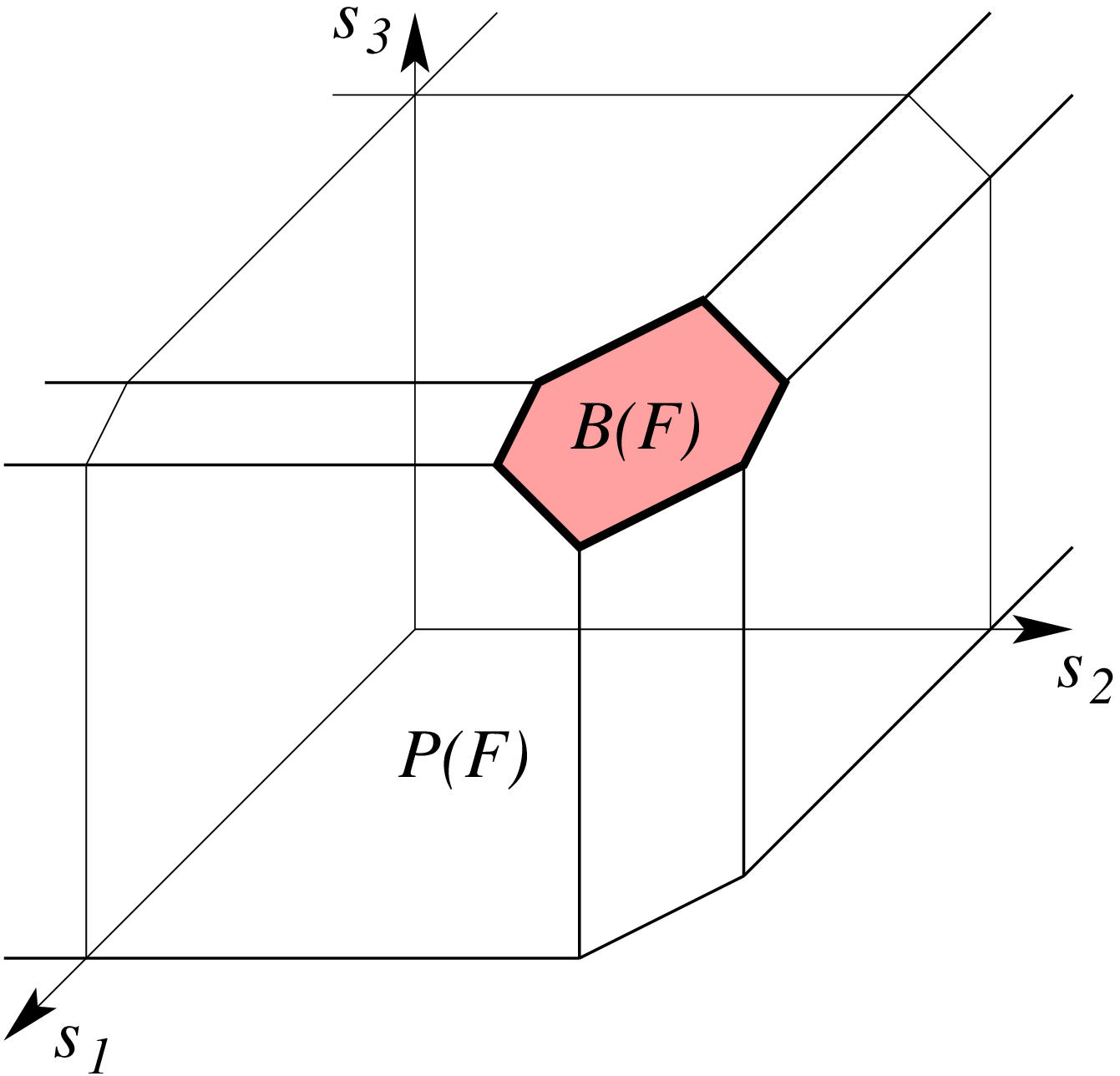}
\end{center}

\vspace*{-.5cm}

\caption{Submodular polyhedron $P(F)$ and base polyhedron $B(F)$ for $p=2$ (left) and $p=3$ (right), for a non-decreasing submodular function (for which $B(F) \subseteq \rb^p_+$, see Prop.~\ref{prop:basepolym}).}
\label{fig:poly}
\end{figure}

\section{Polymatroids (non-decreasing submodular functions)}

\label{sec:polymat}

When the submodular function $F$ is also \emph{non-decreasing}, i.e., when for $A, B \subseteq V$, $A \subseteq B \Rightarrow F(A) \leqslant F(B)$, then the function is often referred to as a \emph{polymatroid rank function} (see related matroid rank functions in \mysec{matroids}). For these functions, as shown in \mychap{support}, the base polyhedron happens to be included in the positive orthant (the submodular function from \myfig{poly} is thus non-decreasing).

Although, the study of polymatroids may seem too restrictive as many submodular functions of interest are not non-decreasing (such as cuts), polymatroids were historically introduced as the generalization of  matroids (which we study in \mysec{matroids}). Moreover, any submodular function may be transformed to a non-decreasing function by adding a modular function:

\begin{proposition} \textbf{(Transformation to non-decreasing functions)}
Let $F$ be a submodular function such that $F(\varnothing)=0$. Let $s \in \rb^p$ defined through $s_k = F(V) - F(V \backslash \{k \})$ for $k \in V$. The function
$G: A \mapsto F(A) - s(A) $ is then submodular and non-decreasing.
\end{proposition}
\begin{proof}
Submodularity is immediate since $A \mapsto -s(A)$ is submodular and adding two submodular functions preserves submodularity. Let $A  \subseteq V $ and $ k \in V \backslash A$. We have:
\BEAS
 & & G(A \cup \{k\} ) - G(A) \\
 & = &  F(A \cup \{k\} ) - F(A) - F(V) +  F(V \backslash \{k \}) \\
& = &  F(A \cup \{k\} ) - F(A) - F( (V \backslash \{k \} ) \cup \{k\} ) +  F(V \backslash \{k \}),
\EEAS
which is non-negative since $A \subseteq V \backslash \{k \} $ (because of Prop.~\ref{prop:firstorder}). This implies that $G$ is non-decreasing.
\end{proof}

The joint properties of submodularity and monotonicity gives rise to a compact characterization of polymatroids~\cite{nemhauser1978analysis}, which we now describe:

\begin{proposition} \textbf{(Characterization of polymatroids)}
\label{prop:char-polym}
Let $F$ by a set-function such that $F(\varnothing)=0$. For any $A \subseteq V$, define for $j \in V$, 
$\rho_j(A) = F(A \cup \{j\}) - F(A)$ the gain of adding element $j$ to the set $A$. The function $F$ is a polymatroid rank function (i.e., submodular and non-decreasing) if and only if  for all $A,B \subseteq V$,
\BEQ
\label{eq:condpolym} 
F(B) \leqslant F(A) + \sum_{ j \in B \backslash A} \rho_j(A).
\EEQ
\end{proposition}
\begin{proof}
If \eq{condpolym} is true, then, if $B \subseteq A$, $B \backslash A = \varnothing$, and thus $F(B) \leqslant F(A)$, which implies monotonicity. We can then apply \eq{condpolym} to $A$ and $B=A \cup \{j,k\}$ to obtain the condition in Prop.~\ref{prop:second}, hence the submodularity.

We now assume that $F$ is non-decreasing and submodular. For any two subsets $A$ and $B$ of $V$, if we enumerate the set $B \backslash A$ as $\{b_1,\dots,b_s\}$, we have
\BEAS
F(B) &\!\!\! \leqslant \!\!\!&  F( B\cup A)  = \sum_{i=1}^s  \big\{ F( A \cup \{ b_1,\dots,b_i\})
-
F( A \cup \{ b_1,\dots,b_{i-1}\}) \big\} \\
& \!\!\! \leqslant \!\!\!&  \sum_{i=1}^s \rho_{b_i} (A)) = 
\sum_{ j \in B \backslash A} \rho_j(A),
\EEAS
which is exactly \eq{condpolym}.
\end{proof}
The last proposition notably shows that each submodular function is upper-bounded by a constant plus a modular function, and these upper-bounds may be enforced to be  tight at any given $A \subseteq V$. This will be contrasted in \mysec{closure} to the other property shown later that modular lower-bounds also exist (Prop.~\ref{prop:greedy}).

\begin{figure}

\begin{center}
\parbox[b]{5cm}{\vspace*{.25cm}
\includegraphics[scale=.44]{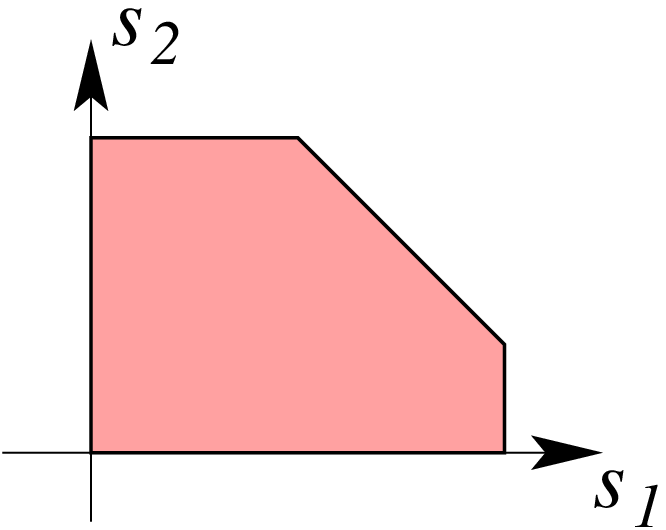}
\vspace*{.025cm}}
\hspace*{.1cm}
\includegraphics[scale=.33]{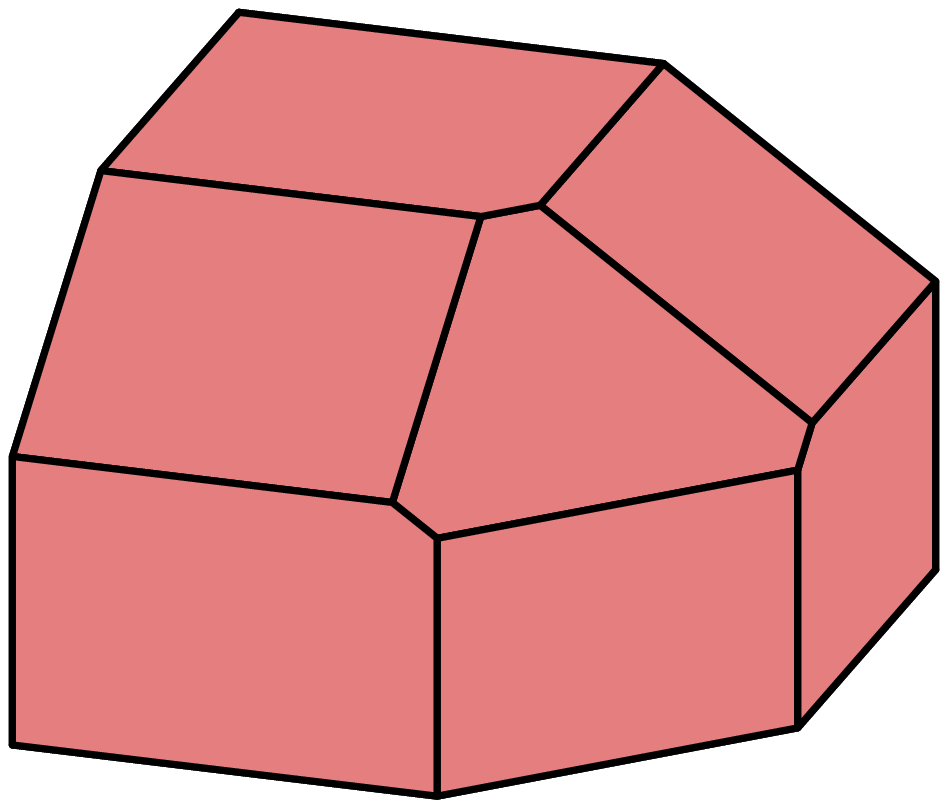}
\end{center}

\vspace*{-.5cm}

\caption{Positive submodular polyhedron $P_+(F)$  for $p=2$ (left) and $p=3$ (right), for a non-decreasing submodular function.}
\label{fig:pospoly}
\end{figure}

\begin{figure}

\begin{center}
\parbox[b]{5cm}{\vspace*{.25cm}
\includegraphics[scale=.44]{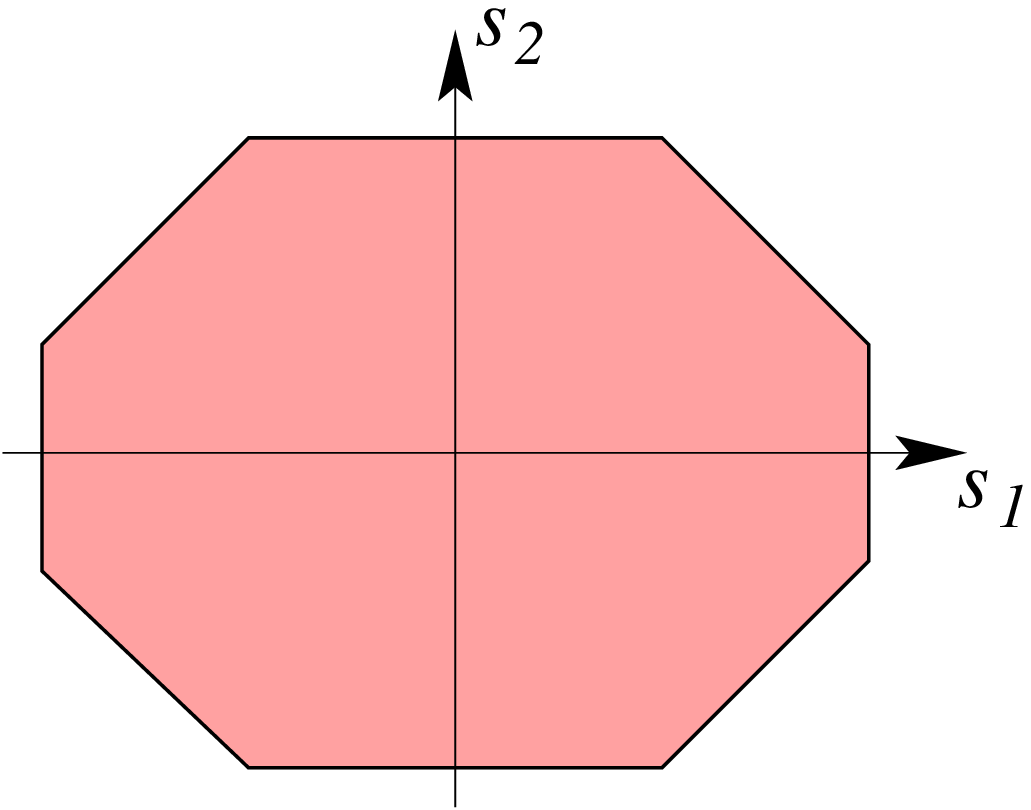}
\vspace*{.025cm}}
\hspace*{.1cm}
\includegraphics[scale=.44]{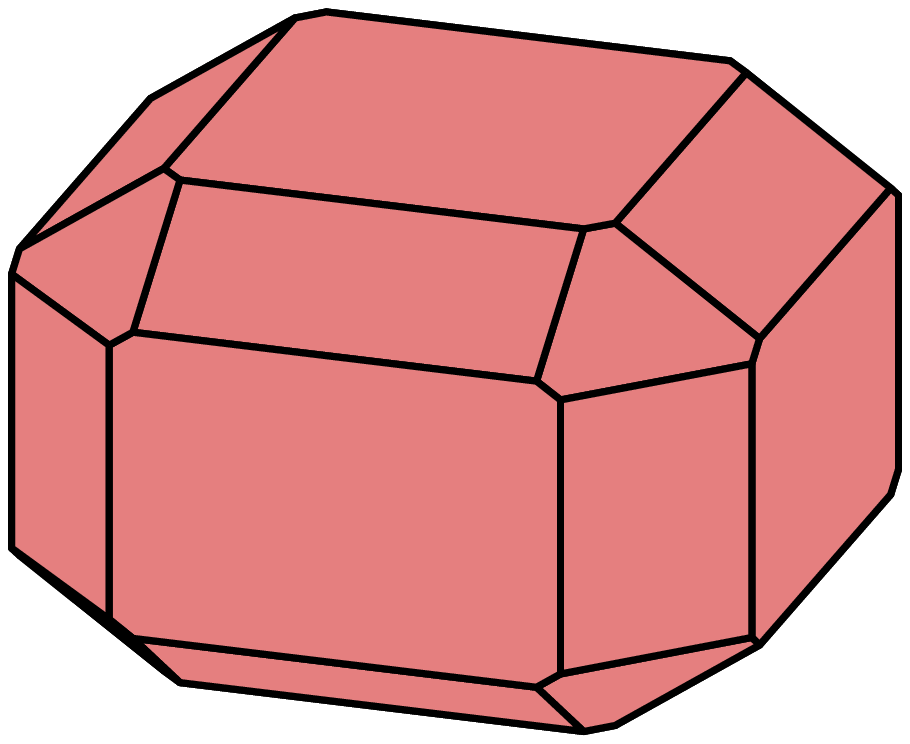}
\end{center}

\vspace*{-.5cm}

\caption{Symmetric submodular polyhedron $|P|(F)$  for $p=2$ (left) and $p=3$ (right), for a non-decreasing submodular function.}
\label{fig:symmpoly}
\end{figure}

\paragraph{Associated polyhedra.}
For polymatroids, we will consider in this monograph two other polyhedra:
the positive submodular polyhedron, which we now define by considering the positive part of the submodular polyhedron (sometimes called the independence polyhedron), and then its symmetrized version, which we refer to as
the symmetric submodular polyhedron. See examples in two and three dimensions in \myfig{pospoly} and \myfig{symmpoly}.

\begin{definition} \textbf{(Positive submodular polyhedron)}
\label{def:polyhedra-pos}
Let $F$ be a non-decreasing submodular function such that $F(\varnothing)=0$. The positive submodular polyhedron $P_+(F)$ is defined as:
\BEAS
P_+(F) & = &  \{ s \in \rb_+^p, \ \forall A \subseteq V, s(A) \leqslant F(A) \}  =  \rb_+^p \cap P(F).
 \EEAS
\end{definition}
The positive submodular polyhedron is the intersection of the submodular polyhedron $P(F)$ with the positive orthant (see \myfig{pospoly}). Note that if $F$ is not non-decreasing, we may still define the positive submodular polyhedron, which is then equal to the submodular polyhedron $P(G)$ associated with the monotone version $G$ of~$F$, i.e.,
$G(A)= \min_{B \supseteq A} F(B)$ (see Appendix~\ref{app:ope} for more details).

\begin{definition} \textbf{(Symmetric submodular polyhedron)}
\label{def:polyhedra-symm}
Let $F$ be a non-decreasing submodular function such that $F(\varnothing)=0$. The submodular polyhedron $|P|(F)$ is defined as:
\BEAS
|P|(F) & = &  \{ s \in \rb^p, \ \forall A \subseteq V, |s|(A) \leqslant F(A) \}  = \{ s  \in \rb^p, \ |s| \in P(F) \}.
\EEAS
\end{definition}

For the cardinality function $F: A \mapsto |A|$,   $|P|(F)$ is exactly the $\ell_\infty$-ball, while for the function $A \mapsto \min\{|A|,1\}$,   $|P|(F)$ is exactly the $\ell_1$-ball. More generally, this polyhedron will turn out to be the unit ball of the dual norm of the norm defined in \mysec{sparse} (see more details and figures in \mysec{sparse}).

\chapter{\lova Extension}
\label{chap:lova}
We first consider a set-function $F$ such that $F(\varnothing)=0$, \emph{which may not be submodular}. 
Every element of the power set $2^V$ may be associated to a vertex of the hypercube $\{0,1\}^p$. Namely, a 
set $A \subseteq V$ may be uniquely identified to the indicator vector $1_A$ (see \myfig{hypercube_2d} and \myfig{hypercube_3d}).

The \lova extension~\cite{lovasz1982submodular}, which is often referred to as the Choquet integral in decision theory~\cite{choquet1953theory,marichal2000axiomatic}, allows the extension of a set-function defined on the vertices of the hypercube $\{0,1\}^p$, to the full hypercube $[0,1]^p$ (and in fact also to the entire space $\rb^p$). As shown in this section, the \lova extension is obtained by cutting the hypercube in $p!$ simplices and defining the \lova extension  by linear interpolation of the values at the  vertices of these simplices.

\begin{figure}

\begin{center}

\hspace*{-.5cm}
 \includegraphics[scale=.8]{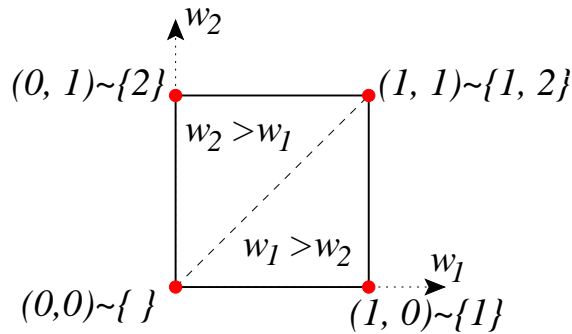} 
  \hspace*{-.5cm}

 \vspace*{-.6cm}
 
   \end{center}

\caption{Equivalence between sets and vertices of the hypercube: every subset $A$ of $V$ may be identified to a vertex of the hypercube, i.e., elements of $\{0,1\}^p$, namely the indicator vector $1_A$ of the set $A$. Illustration in two dimensions ($p=2$). The hypercube is divided in two parts (two possible orderings of $w_1$ and $w_2$).}

\label{fig:hypercube_2d}
\end{figure}

The \lova extension, which we define  in \mysec{lovadef}, allows to draw links between submodular set-functions and regular convex functions, and transfer known results   from convex analysis, such as duality. In particular, we prove in this chapter,  two key results of submodular analysis and its relationship to convex analysis, namely,  (a) that the \lova extension is the support function of the base polyhedron, with a direct relationship through the ``greedy algorithm''~\cite{edmonds} (\mysec{greedy}), and 
 (b) that a set-function is submodular if and only if its \lova extension is convex~\cite{lovasz1982submodular} (\mysec{links}), with additional links between convex optimization and submodular function minimization.

While there are many additional results relating submodularity and convexity through the analysis of properties of the polyhedra defined in \mysec{polyhedra}, these two results are the main building blocks of all the results presented in this monograph (for additional results, see \mychap{support} and~\cite{fujishige2005submodular}).
In particular, in \mychap{relax}, we show how the \lova extension may be used in convex continuous problems arising as convex relaxations of problems having   mixed combinatorial/discrete structures.

\section{Definition}
\label{sec:lovadef}
 We now define the \lova extension of any set-function (not necessarily submodular). For several alternative representations and first properties, see Prop.~\ref{prop:lova}.

\begin{definition} \textbf{(\lova extension)}
\label{def:lovadef}
Given a set-function $F$ such that $F(\varnothing)=0$, the \lova extension $f:\rb^p \to \rb$ is defined as follows; for $w \in \rb^p$, order the components in decreasing order $w_{j_1} \geqslant \cdots \geqslant w_{j_p}$, where $(j_1,\dots,j_p)$ is a permutation, and define $f(w)$ through any of the following equivalent equations:
\BEA
\label{eq:lova1} \!\! f(w) \!\!\! &  = &   \sum_{k=1}^p w_{j_k} \big[ F(\{j_1,\dots,j_k\}) - F(\{j_1,\dots,j_{k-1}\}) \big] ,   \hspace*{.95cm} 
\EEA

\vspace*{-.25cm}

\BEA
\label{eq:lova2}  \!\! f(w) \!\!\!  &  = &  \sum_{k=1}^{p-1} F(\{j_1,\dots,j_k\}) (w_{j_k} - w_{j_{k+1}} ) + F(V)w_{j_p},
\hspace*{1.05cm} \EEA

\vspace*{-.25cm}

\BEA
\label{eq:lova3}\!\! f(w) \!\!\! & = &\!\!\!\! \int_{\min \{ w_1,\dots,w_p \}}^{+\infty} \!\!\!\! F(  \{w \geqslant z\}   ) dz + F(V) \min \{ w_1,\dots,w_p \}, \\
\label{eq:lova4} \!\! f(w) \!\!\! & = &\!\!\!\! \int_{0}^{+\infty} \!\! F(   \{w \geqslant z\}   ) dz + 
 \int_{-\infty}^0 [ F(  \{w \geqslant z\} ) - F(V) ] dz.
 \EEA

\end{definition}
\begin{proof}
To prove that we actually define a function,
one needs to prove that the definitions are independent of the potentially non unique ordering 
$w_{j_1} \geqslant \cdots \geqslant w_{j_p}$, which is trivial from the last formulations in \eq{lova3} and \eq{lova4}. The first and second formulations in \eq{lova1} and \eq{lova2} are equivalent (by integration by parts, or Abel summation formula). To show equivalence with \eq{lova3}, one may notice that 
$z \mapsto F(   \{w \geqslant z\}   ) $ is piecewise constant, with value zero for $z >  w_{j_1} = \max \{w_1,\dots,w_p\}$, and equal to $F(\{j_1,\dots,j_k\}) $ for $ z \in (w_{j_{k+1}},w_{j_k} ) $, $k=\{1,\dots,p-1\}$, and equal to $F(V)$ for $z < w_{j_p}  = \min \{ w_1,\dots,w_p \}$. What happens at break points is irrelevant for integration. Note that in \eq{lova3}, we may replace the integral  $\int_{\min \{ w_1,\dots,w_p \}}^{+\infty}$ by $\int_{\min \{ w_1,\dots,w_p \}}^{\max \{ w_1,\dots,w_p \}}$. 

\vspace*{.15cm}
 
To prove \eq{lova4} from \eq{lova3}, notice that for $\alpha \leqslant \min \{ 0, w_1,\dots,w_p \} $, \eq{lova3}
leads to
\BEAS
f(w)&  = &  \int_{\alpha}^{+\infty} F(   \{w \geqslant z\}   ) dz 
-  \int_{\alpha}^{\min \{ w_1,\dots,w_p \}} F(   \{w \geqslant z\}   ) dz  \\[-.1cm]
& & \hspace*{6cm}
 + F(V) \min \{ w_1,\dots,w_p \}  \\
 & = &   \int_{\alpha}^{+\infty} F(  \{w \geqslant z\}   ) dz 
-  \int_{\alpha}^{\min \{ w_1,\dots,w_p \}} F( V ) dz \\[-.1cm]
& & \hspace*{6cm}
 +   \int_0^{ \min \{ w_1,\dots,w_p \}} F(V) dz  \\[-.2cm]
& = &  
\int_{\alpha}^{+\infty} F(   \{w \geqslant z\}  ) dz -  \int_{\alpha}^{0} F(V) dz,
\EEAS
 and we get the result by letting $\alpha$ tend to $-\infty$.   Note also that in \eq{lova4} the integrands are equal to zero for $z$ large enough.
 \end{proof}

\begin{figure}

\begin{center}

\hspace*{.8cm}
  \includegraphics[scale=.8]{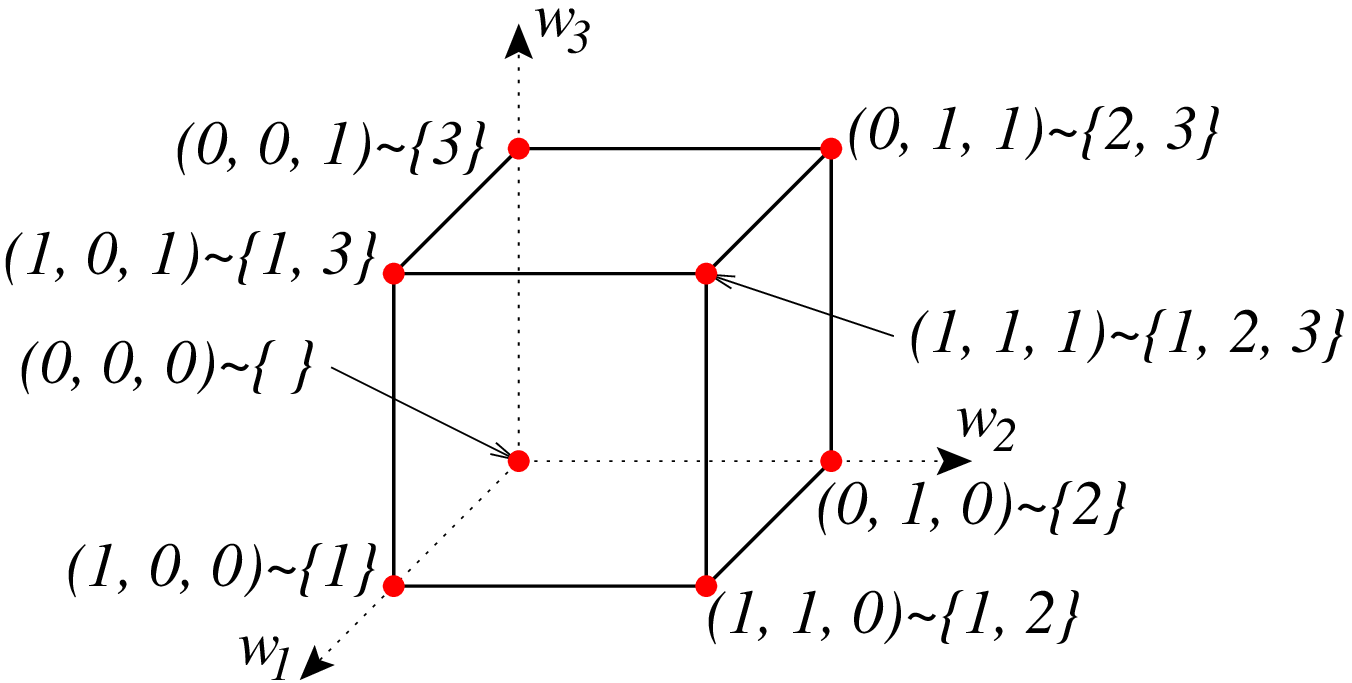} 

\vspace*{.5cm}

  \includegraphics[scale=.8]{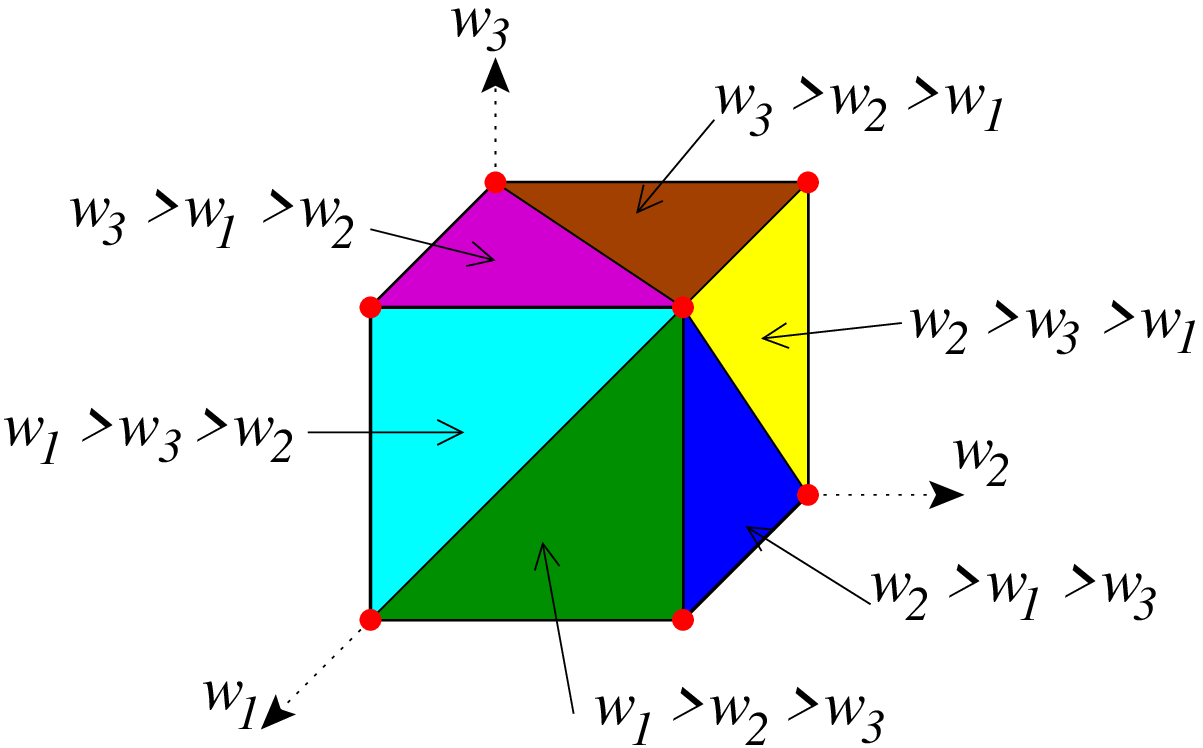} 

\hspace*{-.5cm}

 \vspace*{-.6cm}
 
   \end{center}

\caption{Equivalence between sets and vertices of the hypercube: every subset $A$ of $V$ may be identified to a vertex of the hypercube, i.e., elements of $\{0,1\}^p$, namely the indicator vector $1_A$ of the set $A$. Top: Illustration in three dimensions ($p=3$). Bottom: The hypercube is divided in six parts (three possible orderings of $w_1$, $w_2$ and $w_3$).}

\label{fig:hypercube_3d}
\end{figure}

\paragraph{Modular functions.}
For modular functions $F: A \mapsto s(A)$, with $s \in \rb^p$,  the \lova extension is the linear function $w \mapsto w^\top s$ (as can be seem from \eq{lova1}), hence the importance of modular functions within submodular analysis, comparable to the relationship between linear and convex functions.

\paragraph{Two-dimensional problems.} For $p=2$, we may give several representations
of the \lova extension of a set-function $F$. Indeed, from \eq{lova1}, we obtain
$$
f(w) = \bigg\{ \begin{array}{ll}
F(\{1\}) w_1 + [ F(\{1,2\}) - F( \{1\}) ] w_2 \hspace*{.2cm} \mbox{ if } w_1 \geqslant w_2 \\[.15cm]
F(\{2\}) w_2 + [ F(\{1,2\}) - F( \{2\}) ] w_1 \hspace*{.2cm} \mbox{ if } w_2 \geqslant w_1 ,
\end{array}
$$
which can be written compactly into two different forms:
\BEA
 \label{eq:lova2d}  f(w) & = &
  F(\{1\}) w_1  + F(\{2\}) w_2
   \\
\nonumber& & \hspace*{2cm} - [ F(\{1\}) + F(\{2\}) - F(\{1,2\}) ] \min\{w_1,w_2\}
 \\
\nonumber& = & 
  \frac{1}{2}[ F(\{1\}) + F(\{2\}) - F(\{1,2\}) ]\cdot | w_1 - w_2|  \\
\nonumber& & +\frac{1}{2} [  F(\{1\}) - F(\{2\}) + F(\{1,2\}) ]  \cdot w_1 \\
\nonumber& & 
+ \frac{1}{2}[  - F(\{1\}) + F(\{2\}) + F(\{1,2\}) ] \cdot w_2.
\EEA
This allows an illustration of various propositions in this section (in particular Prop.~\ref{prop:lova}). See also \myfig{submodminimization} for an illustration. Note that for the cut in the complete graph with two nodes, we have $F(\{1,2\}) = 0$ and $F(\{1\}) = F(\{2\}) = 1$, leading to $f(w) = |w_1 - w_2|$.

\paragraph{Examples.}
We have seen that for modular functions $F: A \mapsto s(A)$, then $f(w) = s^\top w$. For the function $A \mapsto \min \{|A|,1\} = 1_{|A| \neq \varnothing}$, then from \eq{lova1}, we have $f(w) = \max_{ k \in V} w_k$. For the function $F:  A \mapsto \sum_{j=1}^m \min\{| A \cap G_j|,1\}$, that counts elements in a partition,  we have $f(w) = \sum_{j=1}^m \max_{ k \in G_j} w_k$, which can be obtained directly from \eq{lova1}, or by combining \lova extensions of sums of set-functions (see property (a) in Prop.~\ref{prop:lova}). For cuts, by combining the results for two-dimensional functions, we obtain $f(w) 
= \sum_{(u,v)\in E} |w_u - w_v|$.

The following proposition details classical properties of the Choquet integral/\lova extension. In particular, property~(f) below implies that the \lova extension is equal to the original set-function on $\{0,1\}^p$ (which can canonically be identified to $2^V$), and hence is indeed an \emph{extension} of $F$. See an illustration in \myfig{submodminimization} for $p=2$.

\begin{proposition} \textbf{(Properties of \lova extension)}
\label{prop:lova}
Let $F$ be any set-function such that $F(\varnothing)=0$. We have:
\\
(a) if $F$ and $G$ are set-functions with \lova extensions $f$ and $g$, then $f+g$ is the \lova extension of  $F+G$, and for all $\lambda \in \rb$, $\lambda f$ is the \lova extension of $\lambda F$,
\\
(b) for $w \in \rb^p_+$,  
$f(w) = \int_{0}^{+\infty} F( \{ w \geqslant z \} ) dz$,
\\
(c) if $F(V)=0$,   for all $w \in \rb^p$, $f(w) = \int_{-\infty}^{+\infty} F( \{ w \geqslant z \} ) dz $,
\\
(d) for all $w \in \rb^p$ and $\alpha \in \rb$,
$f(w + \alpha 1_V) = f(w) + \alpha F(V)$,
\\
(e) the \lova extension $f$ is positively homogeneous,
\\
(f) for all $A \subseteq V$, $F(A) = f(1_A)$,
\\
(g) if $F$ is symmetric (i.e., $\forall A \subseteq V, \ F(A) = F( V \backslash A)$), then $f$ is even,
\\
(h) if $V=A_1 \cup \cdots \cup A_m$ is a partition of $V$, and $w = \sum_{i=1}^m v_i 1_{A_i}$ (i.e., $w$ is constant on each set $A_i$), with $v_1 \geqslant \cdots \geqslant v_m$, then $f(w) = \sum_{i=1}^{m-1} (v_i-v_{i+1}) F(A_1 \cup \cdots \cup A_i) + v_{m} F(V)$,
\\
(i)
if $w \in [0,1]^p$, $f(w)$ is the expectation of $F( \{ w \geqslant x \})$ for $x$  a random variable with uniform distribution in $[0,1]$.
\end{proposition}
\begin{proof}
Properties (a), (b) and (c) are immediate from \eq{lova4} and \eq{lova2}. Properties (d), (e) and (f) are straightforward from \eq{lova2}. If $F$ is symmetric, then $F(V)=F(\varnothing) = 0$, and thus
$f(-w) = \int_{-\infty}^{+\infty} F( \{ -w \geqslant z \} ) dz =
\int_{-\infty}^{+\infty} F( \{ w  \leqslant -z \} ) dz =
\int_{-\infty}^{+\infty} F( \{ w \leqslant z \} ) dz
=
\int_{-\infty}^{+\infty} F( \{ w > z \} ) dz  = f(w)$ (because we may replace strict inequalities by weak inequalities without changing the integral), i.e., $f$ is even. In addition, property (h) is a direct consequence of \eq{lova2}.

Finally, to prove property (i), we simply use property (b) and notice that since all components of $w$ are less than one, then $f(w) = \int_0^1 F( \{ w \geqslant z\} ) dz$, which leads to the desired result.
\end{proof}

\begin{figure}

\begin{center}

 \includegraphics[scale=.7]{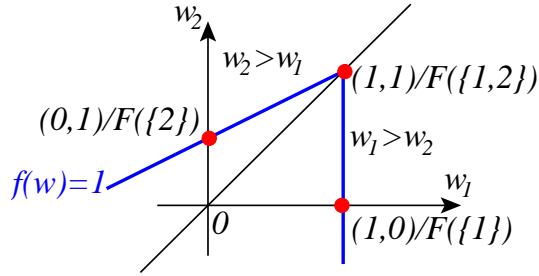} 
 
 \vspace*{-.6cm}
 
   \end{center}

\caption{\lova extension for $V=\{1,2\}$: the function is piecewise affine, with different slopes for $w_1 \geqslant w_2$, with values $F(\{1\}) w_1 + [ F(\{1,2\}) - F(\{1\}) ] w_2$, and for $w_1 \leqslant w_2$, with values $F(\{2\}) w_2 + [ F(\{1,2\}) - F(\{2\}) ] w_1$. The level set $\{w \in \rb^2, f(w) = 1\}$ is displayed in blue, together with points of the form $\frac{1}{F(A)}1_A$. In this example, $F(\{2\}) = 2$, $F(\{1\}) = F(\{1,2\}) = 1$.
}
\label{fig:submodminimization}
\end{figure}

Note that when the function is a cut function (see \mysec{cuts}), then the \lova extension is related to the total variation and property (c) is often referred to as the co-area formula (see~\cite{chambolle2009total} and references therein, as well as \mysec{cuts}).

\paragraph{Linear interpolation on simplices.}  
One may view the definition in Def.~\ref{def:lovadef} in a geometric way. We can cut the set $[0,1]^p$ in $p!$ polytopes, as shown in \myfig{hypercube_2d} and the the bottom plot of \myfig{hypercube_3d}. These small polytopes are parameterized by one of the $p!$ permutations of $p$ elements, i.e., one of  the orderings $\{j_1,\dots,j_p\}$, and are defined as the set of $w \in [0,1]^p$ such that $w_{j_1} \geqslant \cdots \geqslant w_{j_p}$. For a given ordering, the corresponding convex set is the convex hull of the  $p\!+\! 1$ indicator vectors of  sets $A_k = \{j_1,\dots,j_k\}$, for $k \in \{0,\dots,p\}$ (with the convention that $A_0 = \varnothing  $), and any 
$w$ in this polytope may be written as $w =
\sum_{k=1}^{p-1}  (w_{j_k} - w_{j_{k+1}} ) 1_{\{j_1,\dots,j_k\}}  + w_{j_p} 1_V  + ( 1 - w_{j_1}) \times 0$ (which is indeed a convex combination), and thus, the definition of $f(w)$ in \eq{lova2} corresponds exactly to a linear interpolation of the values at the vertices of the polytope
$\{ w \in [0,1]^p, \ w_{j_1}\geqslant \cdots \geqslant w_{j_p} \}$.

\paragraph{Decomposition into modular plus non-negative function.}
Given any submodular function $G$ and an element $t$ of the base polyhedron $B(G)$ defined in Def.~\ref{def:polyhedra}, then the function $F=G-t$ is also submodular, and is such that $F$ is always non-negative and $F(V)=0$. Thus $G$ may be (non uniquely because there are many choices for $t \in B(F)$ as shown in \mysec{greedy}) decomposed as the sum of a modular function $t$ and a submodular function $F$ which is always non-negative and such that $F(V)=0$. Such functions $F$ have interesting \lova extensions. Indeed, for all $w \in \rb^p$, $f(w) \geqslant 0$ and $f(w + \alpha 1_V) = f(w)$. Thus in order to represent the level set $\{w \in \rb^p, \ f(w)=1\}$ (which we will denote $\{f(w)=1\}$), we only need to project onto a subspace orthogonal to $1_V$. In Figure~\ref{fig:symmballs}, we consider a function $F$ which is symmetric (which implies that $F(V) = 0$ and $F$ is non-negative, see more details in \mysec{posi}).
See also \mysec{shaping} for the sparsity-inducing properties of such \lova extensions.

\begin{figure}
\begin{center}

\hspace*{-1.5cm}
  \hspace*{-.5cm} \includegraphics[scale=.42]{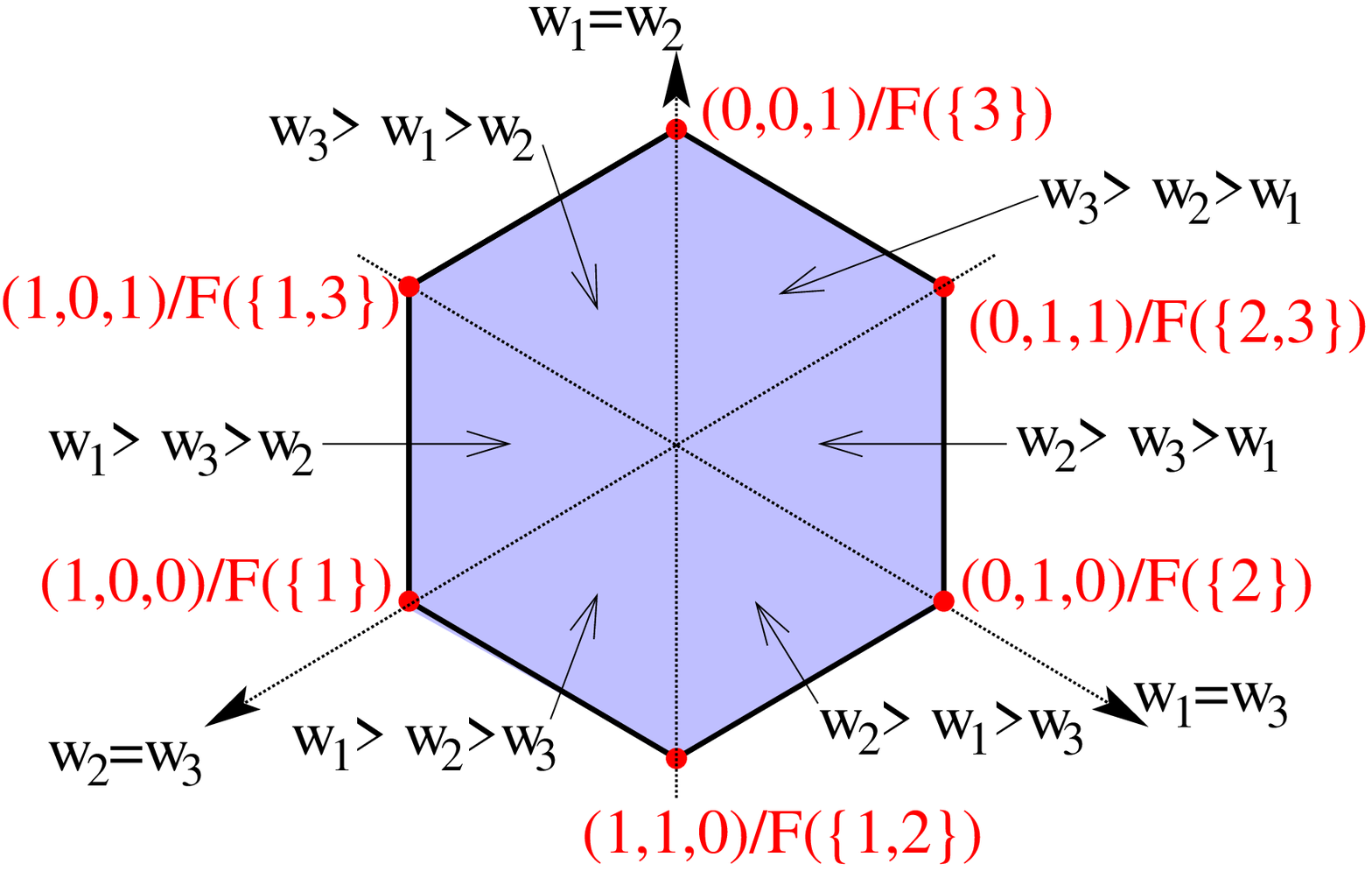}  \hspace*{-.5cm}
\includegraphics[scale=.42]{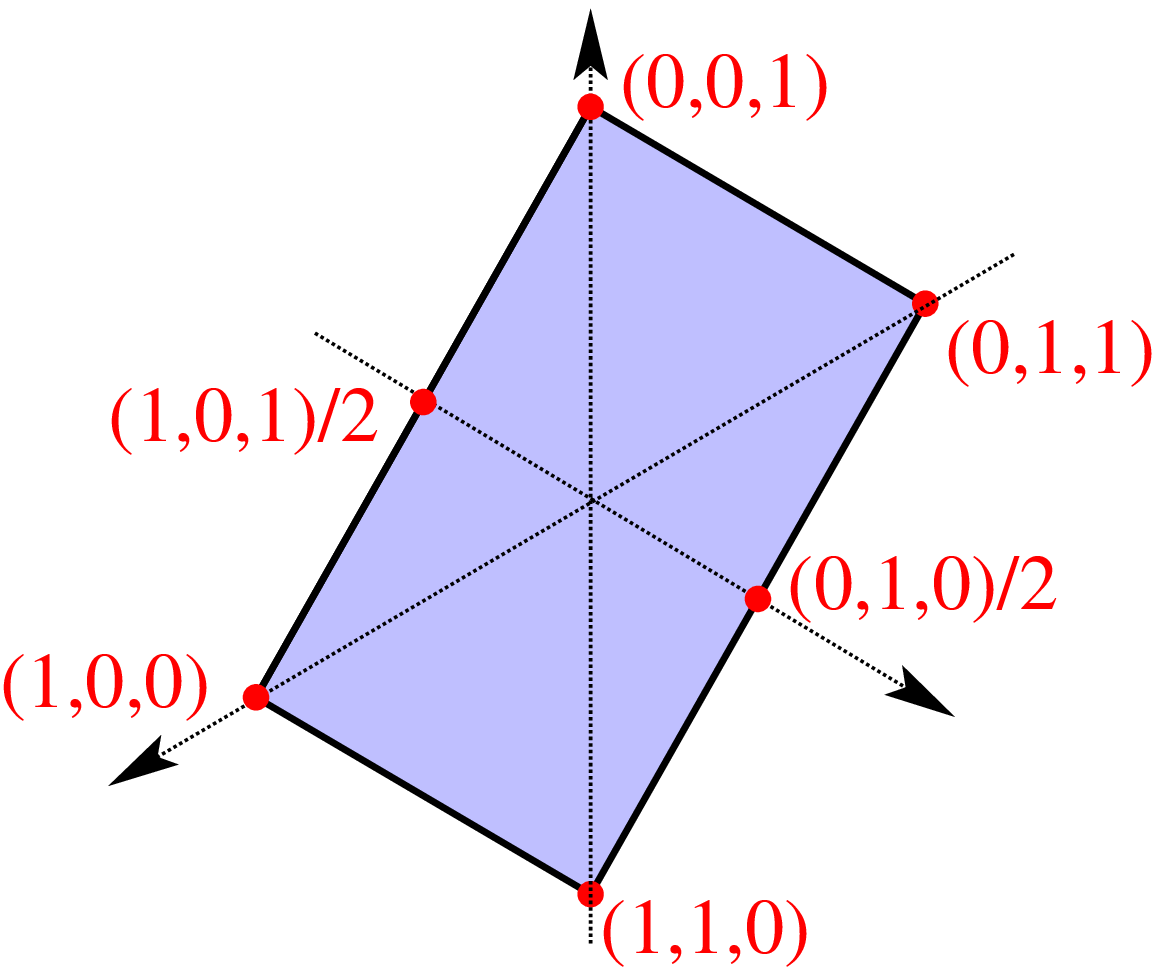}   \hspace*{-1.5cm}
\end{center}

 \vspace*{-.6cm}

\caption{Top: Polyhedral level set of $f$ (projected on the set $ w^\top 1_V=0$), for 2 different submodular symmetric functions of three variables.
The various extreme points cut the space into polygons where the ordering of the components is fixed. Left: $F(A) = 1_{|A| \in \{1,2\}}$ (which is a symmetrized version of $A \mapsto \min\{|A|,1\}$), leading to $f(w) = \max_{k\in \{1,2,3\}} w_k - \min_{k\in \{1,2,3\}} w_k$ (all possible extreme points); note that the polygon need not be symmetric in general. Right: one-dimensional total variation on three nodes, i.e., $F(A) = |1_{ 1 \in A } - 1_{ 2 \in A }| + |1_{ 2 \in A } - 1_{ 3 \in A }|$, leading to $f(w) = |w_1-w_2| + |w_2-w_3|$.
}
\label{fig:symmballs}

  \vspace*{-.25cm}

\end{figure}

\section{Greedy algorithm}
\label{sec:greedy}

The next result relates the \lova extension with the support function\footnote{The support function of a convex set $K$ is obtained by maximizing linear functions $w^\top s$ over $s \in K$, which leads to a convex function of $w$; see definition in Appendix~\ref{app:convex}.} of the submodular polyhedron $P(F)$ or the base polyhedron $B(F)$, which are defined in Def.~\ref{def:polyhedra}. This is the basis for many of the theoretical results and algorithms related to submodular functions. Using convex duality, it shows that maximizing a linear function with non-negative coefficients on the submodular polyhedron may be obtained in closed form, by the so-called ``greedy algorithm'' (see \cite{lovasz1982submodular,edmonds} and \mysec{matroids} for an intuitive explanation of this denomination in the context of matroids), and the optimal value is equal to the value $f(w)$ of the \lova extension.
Note that otherwise, solving a linear programming problem with $2^p-1$ constraints would then be required. This applies to the submodular polyhedron $P(F)$ and to the base polyhedron $B(F)$; note the different assumption regarding the positivity of the components of $w$.
See also Prop.~\ref{prop:optsupport} for a characterization of all maximizers and Prop.~\ref{prop:greedy-positive} for similar results for the positive submodular polyhedron $P_+(F)$ and Prop.~\ref{prop:greedy-indep} for  the symmetric submodular polyhedron $|P|(F)$.

\begin{proposition} \textbf{(Greedy algorithm for submodular and base polyhedra)}
\label{prop:greedy}
\label{prop:greedy-base}
Let $F$ be a submodular function such that $F(\varnothing)=0$.  Let $w \in \rb^p$, with components ordered in decreasing order, i.e., $w_{j_1} \geqslant \cdots \geqslant w_{j_p}  $ and define
$s_{j_k} =  F(\{j_1,\dots,j_k\}) - F(\{j_1,\dots,j_{k-1}\})$. Then $s \in B(F)$ and,
\\
(a) if $w \in \rb_+^p$, $s$ is a maximizer of $\max_{s \in P(F)} w^\top s$; moreover  $\max_{s \in P(F)} w^\top s = f(w)$,
\\
(b) $s$ is a maximizer of $\max_{s \in B(F)} w^\top s$, and $\max_{s \in B(F)} w^\top s = f(w)$.
\end{proposition}
\begin{proof}
	Let $w \in \rb_+^p$. By convex strong duality (which applies because $P(F)$ has non empty interior from Prop.~\ref {prop:nonemptyinterior}), we have, by introducing Lagrange multipliers $\lambda_A \in \rb_+$ for the constraints $s(A) \leqslant F(A)$, $A \subseteq V$, the following pair of convex optimization problems dual to each other:
\BEA
\label{eq:dualll}
\!\!\! \max_{ s \in P(F) } w^\top s \!\!\!\!
 & = &  \max_{s \in \rb^p}  \min_{\lambda_A \geqslant 0, A \subseteq V}  \ 
 \bigg\{  w^\top s - \sum_{A \subseteq V} \lambda_A [ s(A) - F(A) ] \bigg\} \\
\nonumber  & = &  \min_{\lambda_A \geqslant 0, A \subseteq V} \max_{s \in \rb^p} \ 
 \bigg\{  w^\top s - \sum_{A \subseteq V} \lambda_A [ s(A) - F(A) ] \bigg\} \\[-.15cm]
\nonumber & = &  \min_{\lambda_A \geqslant 0, A \subseteq V} \max_{s \in \rb^p} \ 
 \bigg\{  \sum_{A \subseteq V} \lambda_A F(A) + \sum_{k=1}^p s_k \big(
 w_k -   \sum_{A \ni k} \lambda_A
 \big) \bigg\} \\
\nonumber & = & \min_{\lambda_A \geqslant 0, A \subseteq V} \sum_{A \subseteq V} \lambda_A F(A)
 \mbox{ such that } \forall k \in V, \ w_k = \sum_{A \ni k} \lambda_A.
 \EEA
 In the last equality, maximizing with respect to each $s_k \in \rb$ a linear function of $s_k$ introduces the constraint that this linear function has to be zero (otherwise the maximum is equal to $+\infty$).
 If we take the (primal) candidate solution $s$ obtained from  the greedy algorithm, we have $f(w) = w^\top s$ from \eq{lova1}. We now show that $s$ is feasible (i.e., in $P(F)$), as a consequence of the submodularity of~$F$. Indeed, without loss of generality, we assume that $j_k=k$ for all $k \in \{1,\dots,p\}$.  We have for any set $A$:
 \BEAS
 s(A) & \!\!\!= \!\!\!& s^\top 1_A =  \sum_{k=1}^p (1_A)_k s_k \\[-.15cm]
 & \!\!\!= \!\!\! & \sum_{k=1}^p (1_A)_k \big[ F( \{1,\dots,k\})-F(\{1,\dots,k\!-\!1\}) \big] \mbox{ by definition of } s, \\
 & \!\!\! \leqslant \!\!\! & \sum_{k=1}^p (1_A)_k \big[ F( A \cap \{1,\dots,k\})-F( A \cap \{1,\dots,k\!-\!1\}) \big]  \\[-.25cm]
 & & \hspace*{6cm} \mbox{ by submodularity}, \\[-.1cm]
 & \!\!\! = \!\!\! & \sum_{k=1}^p \big[ F( A \cap \{1,\dots,k\})-F( A \cap \{1,\dots,k\!-\!1\}) \big]  \\
 & \!\!\! = \!\!\! & F(A) \mbox{ by telescoping the sums.}
 \EEAS 
%
 
 Moreover, we can define
 dual variables $\lambda_{ \{j_1,\dots,j_k\}} = w_{j_k} - w_{j_{k+1}}$ for $k \in \{1,\dots,p-1\}$ and $\lambda_{V} = w_{j_p}$ with all other $\lambda_A$'s equal to zero. Then they are all non negative (notably because $w \geqslant 0$), and satisfy the constraint  $\forall k \in V, \ w_k = \sum_{A \ni k} \lambda_A
$. Finally, the dual cost function has also value $f(w)$ (from \eq{lova2}). Thus by strong duality (which holds, because $P(F)$ has a non-empty interior), $s$ is an optimal solution, hence property (a). Note that the maximizer $s$ is not unique in general (see Prop.~\ref{prop:optsupport} for a description of the set of solutions).

In order to show (b),  we consider $w \in \rb^p$ (not necessarily with non-negative components); we follow the same proof technique and replace $P(F)$ by $B(F)$, by simply dropping the constraint $\lambda_V \geqslant 0$ in \eq{dualll} (which makes our choice  $\lambda_V =w_{j_p}$ feasible, which could have been a problem since $w$ is not assumed to have nonnegative components). Since the solution obtained by the greedy algorithm satisfies $s(V) = F(V)$, we get a pair of primal-dual solutions, hence the optimality.
\end{proof}

Given the previous proposition that provides a maximizer of linear functions over $B(F)$, we obtain a list of all extreme points of $B(F)$. Note that this also shows that $B(F)$ is a \emph{polytope} (i.e., it is a compact polyhedron).

\begin{proposition} \textbf{(Extreme points of $B(F)$)}
\label{prop:extreme}
The set of extreme points is the set of vectors $s$ obtained as the result of the greedy algorithm from Prop.~\ref{prop:greedy}, for all possible orderings of components of $w$.
\end{proposition}
\begin{proof}
Let $K$ denote the finite set described above. From Prop.~\ref{prop:greedy}, 
$\max_{ s \in K} w^\top s  = \max_{s \in B(F)} w^\top s$. We thus only need to show that for any element of $K$, there exists $w \in \rb^p$ such that the minimizer $w$ is unique. For any ordering $j_1, 
  \cdots , j_p$, we can simply  take any $w \in \rb^p $ such that $w_{j_1} > \cdots > w_{j_p}$. In the proof of Prop.~\ref{prop:greedy}, we may compute the difference between the primal objective value and the dual objective values, which is equal to
$\sum_{k=1}^p (  w_{j_k} - w_{j_{k+1}} ) \big[ F(\{j_1,\dots,j_k\}) - s(\{j_1,\dots,j_k\}) \big]$; it is equal to zero if and only if $s$ is the result of the greedy algorithm for this ordering.
\end{proof}
Note that there are at most $p!$ extreme points, and often less as several orderings may lead to the same vector $s  \in B(F)$.

We end this section, by simply stating the greedy algorithm for the symmetric and positive submodular polyhedron, whose proofs are similar to the proof of Prop.~\ref{prop:greedy} (we define the sign of $a$ as $+1$ if $a>0$, and $-1$ if $a<0$, and zero otherwise; $|w|$ denotes the vector composed of the absolute values of the components of $w$). See also~Prop.~\ref{prop:optsupporttight-positive} and Prop.~\ref{prop:optsupporttight-indep} for a characterization of all maximizers of linear functions.

\begin{proposition} \textbf{(Greedy algorithm for positive submodular polyhedron)}
\label{prop:greedy-positive}
Let $F$ be a submodular function such that $F(\varnothing)=0$ and $F$ is non-decreasing. Let $w \in \rb^p$. A maximizer of $\max_{ s \in P_+(F) } w^\top s$ may be obtained by the following algorithm: order the components of $w$, as $w_{j_1} \geqslant \cdots \geqslant w_{j_p}$ and define
$s_{j_k} =    [ F(\{j_1,\dots,j_k\}) - F(\{j_1,\dots,j_{k-1}\})] $ if $w_{j_k}>0$, and zero otherwise. Moreover,  for all $w \in \rb^p$, $\max_{ s \in P_+(F) } w^\top s = f(w_+)$.
\end{proposition}

\begin{proposition} \textbf{(Greedy algorithm for symmetric submodular polyhedron)}
\label{prop:greedy-indep}
Let $F$ be a submodular function such that $F(\varnothing)=0$ and $F$ is non-decreasing. Let $w \in \rb^p$. A maximizer of $\max_{ s \in |P|(F) } w^\top s$ may be obtained by the following algorithm: order the components of $|w|$, as $|w_{j_1}| \geqslant \cdots \geqslant |w_{j_p}|$ and define
$s_{j_k} =  \sign(w_{j_k}) [ F(\{j_1,\dots,j_k\}) - F(\{j_1,\dots,j_{k-1}\})] $. Moreover,  for all $w \in \rb^p$, $\max_{ s \in |P|(F) } w^\top s = f(|w|)$.
\end{proposition}

 \section{Links between submodularity and convexity}
\label{sec:links}
The next proposition draws precise links between convexity and submodularity, by showing that a set-function $F$ is submodular if and only if its \lova extension $f$ is convex~\cite{lovasz1982submodular}. This is further developed in Prop.~\ref{prop:minsub} where it is shown that, when $F$ is submodular,  minimizing $F$ on $2^V$ (which is equivalent to minimizing $f$ on $\{0,1\}^p$ since $f$ is an extension of $F$) and minimizing $f$ on $[0,1]^p$ are  equivalent.

\begin{proposition} \textbf{(Convexity and submodularity}) 
\label{prop:convexity}
A set-function $F$ is submodular if and only if its \lova extension $f$ is convex.
\end{proposition}
\begin{proof}
We first assume that $f$ is convex.
Let $A,B \subseteq V$. The vector $1_{A \cup B} + 1_{A \cap B} = 1_A + 1_B$ has components equal to $0$ (on $V \backslash (A\cup B)$), $2$ (on $A \cap B$) and $1$ (on $A \Delta B = (A \backslash B)  \cup (B \backslash A) $). Therefore, from property (b) of Prop.~\ref{prop:lova},
$f(1_{A \cup B} + 1_{A \cap B}) =
\int_{0}^2 F(1_{\{w \geqslant z\}}) dz = \int_0^1 F(A\cup B) dz +  \int_1^2 F(A\cap B) dz =
F(A \cup B) + F(A \cap B)$. 
Since $f$ is convex, then by homogeneity, 
$f(1_{A  } + 1_{  B})  \leqslant  f(1_{A  })  +  f(1_{ B}) $, which is equal to 
$ F({A    })  +  F({  B})$, and thus $F$ is submodular.

If we now assume that $F$ is submodular, then by Prop.~\ref{prop:greedy}, for all $w \in \rb^p$, $f(w)$ is a maximum of linear functions, thus, it is convex on $\rb^p$. 
\end{proof}

The next proposition completes Prop.~\ref{prop:convexity} by showing that minimizing the \lova extension on $[0,1]^p$ is equivalent to minimizing it on $\{0,1\}^p$, and hence to minimizing the set-function $F$ on $2^V$  (when $F$ is submodular). 

\begin{proposition} \textbf{(Minimization of submodular functions)}
\label{prop:minlova}
\label{prop:minsub}
\label{prop:min}
Let $F$ be a submodular function and $f$ its \lova extension; then $\min_{A \subseteq V}F(A)  = \min_{ w \in \{0,1\}^p } f(w) = \min_{ w \in [0,1]^p } f(w) $. Moreover, the set of minimizers of $f(w)$ on $[0,1]^p$ is the convex hull of minimizers of $f$ on $\{0,1\}^p$.
\end{proposition}

\begin{proof} Because $f$ is an extension from $\{0,1\}^p$ to $[0,1]^p$ (property (f) from Prop.~\ref{prop:lova}),   
we must have $\min_{A \subseteq V} F(A) =  \min_{ w \in \{0,1\}^p } f(w) \geqslant  \min_{ w \in [0,1]^p } f(w)$. To prove the reverse inequality, we may represent $w \in [0,1]^p$ uniquely through its constant sets and their corresponding values; that is, there exists a unique partition $A_1,\dots,A_m$ of $V$ where $w$ is constant on each $A_i$ (equal to $v_i)$ and $(v_i)$ is a strictly decreasing sequence (i.e., $v_1 > \cdots > v_m$). From property (h) of Prop.~\ref{prop:lova}, we have
\BEAS
f(w) & = &  \sum_{i=1}^{m-1} (v_i-v_{i+1}) F(A_1 \cup \cdots \cup A_i) + v_{m} F(V) \\[-.05cm]
& \geqslant &  \sum_{i=1}^{m-1} (v_i-v_{i+1}) \min_{A \subseteq V} F(A) + v_{m} \min_{A \subseteq V} F(A) \\
& = & v_1 \min_{A \subseteq V} F(A) \geqslant \min_{A \subseteq V} F(A),
\EEAS
where the last inequality is obtained from $v_1 \leqslant 1$ and $\min_{A \subseteq V} F(A)
\leqslant F(\varnothing)=0$. This implies that 
$ \min_{ w \in [0,1]^p } f(w) \geqslant   \min_{A \subseteq V} F(A)$.

There is equality in the previous sequence of inequalities, if and only if
(a) for all $i \in \{1,\dots,m-1\} $, $ F(A_1 \cup \cdots \cup A_i)   = \min_{A \subseteq V} F(A)$,
(b)  $ v_{m} ( F(V)  -  \min_{A \subseteq V} F(A) )=0 $, and
(c) $   (  v_{1} -1 ) \min_{A \subseteq V} F(A)  =0$.
Moreover, we have
$$
w = \sum_{j=1}^{m-1} ( v_j - v_{j+1}) 1_{A_1 \cup \cdots \cup A_j}
+ v_m 1_V + ( 1- v_1) 1_\varnothing .
$$
Thus, $w$ is the convex hull of the indicator vectors of the sets $A_1 \cup \cdots \cup A_j$, for $j \in \{1,\dots,m-1\}$,  of $1_V$ (if $v_m > 0$, i.e., from (b), if $V$ is a minimizer of $F$), and of
$0 = 1_\varnothing $ (if $v_m < 1$, i.e., from (c), if $\varnothing$ is a minimizer of $F$). Therefore, any minimizer $w$ is in the convex hull of indicator vectors of minimizers $A$ of $F$. The converse is true by the convexity of the \lova extension $f$.

See \mychap{sfm} for more details on  submodular function minimization and the structure of minimizers.
 \end{proof}

\paragraph{\lova extension for convex relaxations.}
Given that the \lova extension $f$ of a submodular function is convex, it is natural to study its behavior when used within a convex estimation framework. In \mychap{relax}, we show that it corresponds to the convex relaxation of imposing some structure on supports or level sets of the vector to be estimated.

\chapter{Properties of Associated Polyhedra}
\label{chap:support}

We now study in more details submodular and base polyhedra defined in \mysec{polyhedra}, as well as the symmetric and positive submodular polyhedra defined in \mysec{polymat} for non-decreasing functions.
We first review in \mysec{supportfunction} that the support functions may be computed by the greedy algorithm, but now characterize the set of maximizers of linear functions, from which we deduce a detailed facial structure of the base polytope $B(F)$ in \mysec{faces}. We then study the positive submodular polyhedron $P_+(F)$  and the symmetric submodular polyhedron $|P|(F)$ in \mysec{faces-indep}.

The results presented in this chapter are key to understanding precisely the sparsity-inducing effect of the \lova extension, which we present in details in \mychap{relax}. Note that \mysec{faces} and \mysec{faces-indep} may be skipped in a first reading.
 
\section{Support functions}
\label{sec:supportfunction}
The next proposition completes Prop.~\ref{prop:greedy} by computing the full support function of  $P(F)$ (see~\cite{boyd,borwein2006caa} and Appendix~\ref{app:convex} for definitions of support functions), i.e., computing  $\max_{s \in P(F)} w^\top s$ for all possible $w \in \rb^p$ (with positive and/or negative coefficients). Note the different behaviors for $B(F)$ and $P(F)$.

\begin{proposition} \textbf{(Support functions of associated polyhedra)}
\label{prop:support}
Let $F$ be a submodular function such that $F(\varnothing)=0$. We have:
\\
 (a) for all $w \in \rb^p$, $\max_{ s \in B(F) } w^\top s = f(w)$,
 \\
  (b) if
$w \in \rb_+^p$,  $\max_{ s \in P(F) } w^\top s = f(w)$,
\\
 (c) if  there exists $j$ such that $w_j<0$, then $\sup_{ s \in P(F) } w^\top s = + \infty$.
\end{proposition}
\begin{proof} The only statement left to prove beyond Prop.~\ref{prop:greedy}  is (c): we just need to notice that, for $j$ such that $w_j<0$, we can define $s(\lambda) = s_0 - \lambda \delta_j \in P(F)$ for $\lambda \to + \infty$ and $s_0 \in P(F)$ and that  $w^\top s(\lambda) \to +\infty$. 
\end{proof}

The next proposition shows necessary and sufficient conditions for optimality in the definition of support functions. Note that Prop.~\ref{prop:greedy} gave one example obtained from the greedy algorithm, and that we can now characterize all maximizers. Moreover, note that the maximizer is unique only when $w$ has distinct values, and otherwise, the ordering of the components of $w$ is not unique, and hence, the greedy algorithm may have multiple outputs (and all convex combinations of these are also solutions, and are in fact exactly all solutions, as discussed below the proof of Prop.~\ref{prop:optsupport}). The following proposition essentially shows what is exactly needed for $s \in B(F)$ to be a maximizer. In particular, this is done by showing that for some sets $A \subseteq V$, we must have $s(A)  =F(A)$; such sets are often said \emph{tight} for $s \in B(F)$.
This proposition is key to deriving optimality conditions for the separable optimization problems that we consider in \mychap{prox} and \mychap{prox-algo}.

\begin{proposition} \textbf{(Maximizers of the support function of submodular and base polyhedra)}
\label{prop:optsupporttightSUB}
\label{prop:optsupport}
\label{prop:optsupporttight}
Let $F$ be a submodular function such that $F(\varnothing)=0$. Let $w \in \rb^p$, with unique values $v_1 > \cdots > v_m $, taken at sets $A_1,\dots,A_m$ (i.e., $V = A_1 \cup \cdots \cup A_m$ and $\forall i \in \{1,\dots,m\}, \ \forall k \in A_i, \ w_k = v_i$). Then,  
\\
(a) if $w \in (\rb_+^\ast)^p$ (i.e., with strictly positive components, that is, $v_m >0$),  $s$ is optimal for $\max_{s \in P(F)} w^\top s$ if and only if for all $i=1,\dots,m$,
$s(A_1 \cup \cdots \cup A_i) = F(A_1 \cup \cdots \cup A_i)$,
\\
(b) if  $v_m  =  0$,  $s$ is optimal for $\max_{s \in P(F)} w^\top s$ if and only if for all $i=1,\dots,m-1$,
$s(A_1 \cup \cdots \cup A_i) = F(A_1 \cup \cdots \cup A_i)$,
\\
(c) $s$ is optimal for $\max_{s \in B(F)} w^\top s$ if and only if for all $i=1,\dots,m$,
$s(A_1 \cup \cdots \cup A_i) = F(A_1 \cup \cdots \cup A_i)$.
\end{proposition}
\begin{proof}
We first prove (a). 
Let $B_i = A_1 \cup \cdots \cup A_i$, for $i=1,\dots,m$.
From the optimization problems defined in the proof of Prop.~\ref{prop:greedy}, let $\lambda_V = v_m >0 $, and $\lambda_{B_i} = v_i - v_{i+1}>0$ for $i<m$, with all other $\lambda_A$'s, $A \subseteq V$, equal to zero. Such $\lambda$ is optimal because the dual function is equal to the primal objective  $f(w)$.

Let $s \in P(F)$. We have:
\BEAS
\sum_{A \subseteq V} \lambda_A F(A) &\!\!\! = \!\!\!& v_m F(V) + \sum_{i=1}^{m-1} F(B_i) ( v_i - v_{i+1})   \mbox{ by definition of } \lambda,  \\[-.2cm]
 &\!\!\! = \!\!\!& v_m ( F(V) - s(V) ) + \sum_{i=1}^{m-1} [F(B_i) -s(B_i) ]( v_i - v_{i+1}) \\
 & & \hspace*{1cm} + v_m s(V) + \sum_{i=1}^{m-1} s(B_i) ( v_i - v_{i+1}) \\[-.2cm]
& \!\!\!\geqslant \!\!\! & v_m s(V) + \sum_{i=1}^{m-1} s(B_i) ( v_i - v_{i+1}) = s^\top w.
\EEAS
The last inequality is made possible by the conditions $v_m > 0$ and $v_i > v_{i+1}$.
Thus $s$ is optimal, if and only if the primal objective value $s^\top w$ is equal to the optimal dual objective value $\sum_{A \subseteq V} \lambda_A F(A) $, and thus, if and only if
there is equality in all above inequalities, that is, if and only if
$S(B_i) = F(B_i)$ for all $i \in \{1,\dots,m\}$. 

The proof for (b) follows the same arguments, except that we do not need to ensure that $s(V) =F(V)$, since 
$v_m=0$. Similarly, for (c), where $s(V)=F(V)$ is always  satisfied for $s \in B(F)$, hence we do not need $v_m >0$. \end{proof}
 
 Note that the previous may be rephrased as follows. An element $s \in \rb^p$ is a maximizer of the linear function $w^\top s$ over these polyhedra if and only if certain level sets of $w$ are tight for $s$ (all sup-level sets for $B(F)$, all the ones corresponding to positive values for $P(F)$).
 
 Given $w$ with constant sets $A_1,\dots,A_m$, then the greedy algorithm may be run with $d
 = \prod_{j=1}^m |A_j| !$ possible orderings, as the only constraint is that the elements of $A_j$ are considered before the elements of $A_{j+1}$ leaving $|A_j|!$ possibilities within each  set $A_j$, $j \in \{1,\dots,m\}$. This leads to as most as many extreme points (note that the corresponding extreme points of $B(F)$ may be equal). Since, by Prop.~\ref{prop:extreme}, all extreme points of $B(F)$ are obtained by the greedy algorithm, the set of maximizers defined above is the convex hull of the $d$ potential bases defined by the greedy algorithm, i.e., these are extreme points of the corresponding face of $B(F)$ (see \mysec{faces} for a detailed analysis of the facial structure of $B(F)$).

\section{Facial structure$^\ast$}
\label{sec:faces}

In this section, we describe the facial structure of the base polyhedron. We first review the relevant concepts for convex polytopes.

\paragraph{Face lattice of a convex polytope.} 
We quickly review the main concepts related to convex polytopes. For more details, see~\cite{grunbaum2003convex}. A convex polytope is the convex hull of a finite number of points. It may be also seen as the intersection of finitely many half-spaces (such intersections are referred to as polyhedra and are called polytopes if they are bounded). 

\emph{Faces} of a polytope are sets of maximizers of $w^\top s$ for certain $w \in \rb^p$. Faces are convex sets whose affine hulls are intersections of the hyperplanes defining the half-spaces from the intersection of half-space representation. The dimension of a face is the dimension of its affine hull. The $(p-1)$-dimensional faces are often referred to as \emph{facets}, while zero-dimensional faces are its \emph{vertices}. A natural order may be defined on the set of faces, namely the inclusion order between the sets of hyperplanes defining the face. With this order, the set of faces is a distributive lattice~\cite{davey2002introduction}, with appropriate notions of ``join'' (unique smallest face that contains the two faces) and ``meet'' (intersection of the two faces).

\paragraph{Dual polytope.}
We now assume that we consider a polytope with zero in its interior (this can be done by projecting it onto its affine hull and translating it appropriately). The dual polytope of $C$ is the polar set $C^\circ$ of the polytope $C$, defined as
$C^\circ = \{ w \in \rb^p, \ \forall s \in  C, s^\top w \leqslant 1 \}$
 (see Appendix~\ref{app:convex} for further details). It turns out that faces of $C^\circ$ are in bijection with the faces of $C$, with vertices of $C$ mapped to facets of $C^\circ$ and vice-versa. If $C$ is represented as the convex hull of points $s_i$, $i\in \{1,\dots,m\}$, then the polar of $C$ is defined through the intersection of the half-space $\{ w \in \rb^p, \ s_i^\top w \leqslant 1\}$, for $i=1,\dots,m$. Analyses and algorithms related to  polytopes may always be defined or looked through their dual polytopes. In our situation, we will consider three polytopes: (a) the base polyhedron, $B(F)$, which is included in the hyperplane $\{s \in \rb^p, \ s(V) = F(V) \}$,  for which the dual polytope is the
set $\{w , f(w) \leqslant 1, w^\top 1_V = 0 \}$ (see an example in Figure~\ref{fig:dualpoly}), (b) the positive submodular polyhedron $P_+(F)$, and (c) the symmetric  submodular polytope $|P|(F)$, whose dual polytope is the unit ball of the norm $\Omega_\infty$ defined in \mysec{sparse} (see Figure~\ref{fig:balls} for examples).

\begin{figure}
\begin{center}
 
  \includegraphics[scale=.37]{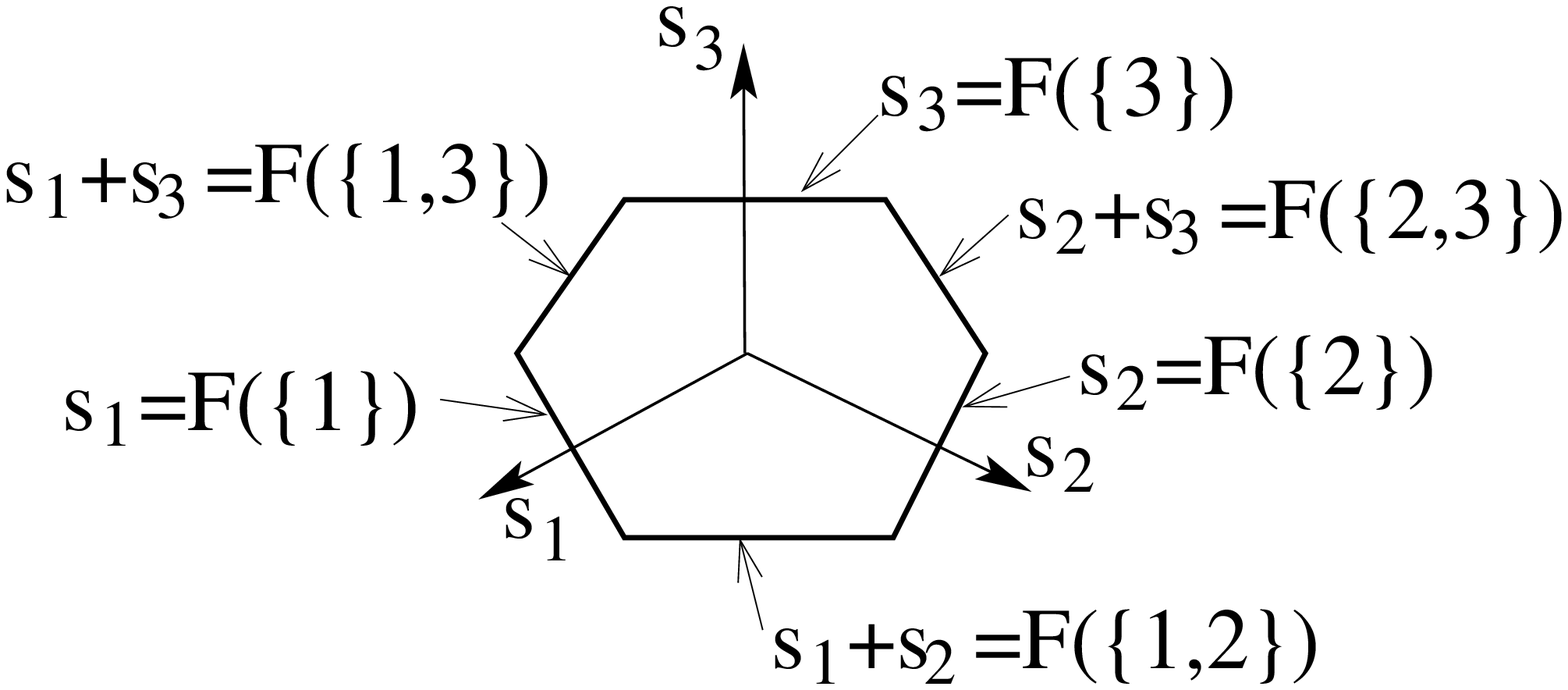}  
  \includegraphics[scale=.42]{ball_1_symm_withcomments_nips.eps}  
  
  \end{center}

 \vspace*{-.4cm}

\caption{(Top) representation of $B(F)$ for $F(A) = 1_{|A| \in \{1,2\}}$ and $p=3$ (projected onto the set $s(V) = F(V)$). (Bottom) associated dual polytope, which is the  $1$-sublevel set of $f$ (projected on the hyperplane $ w^\top 1_V=0$).} 
\label{fig:dualpoly}
  \vspace*{-.25cm}

\end{figure}

\begin{figure}
\begin{center}

\includegraphics[scale=.35]{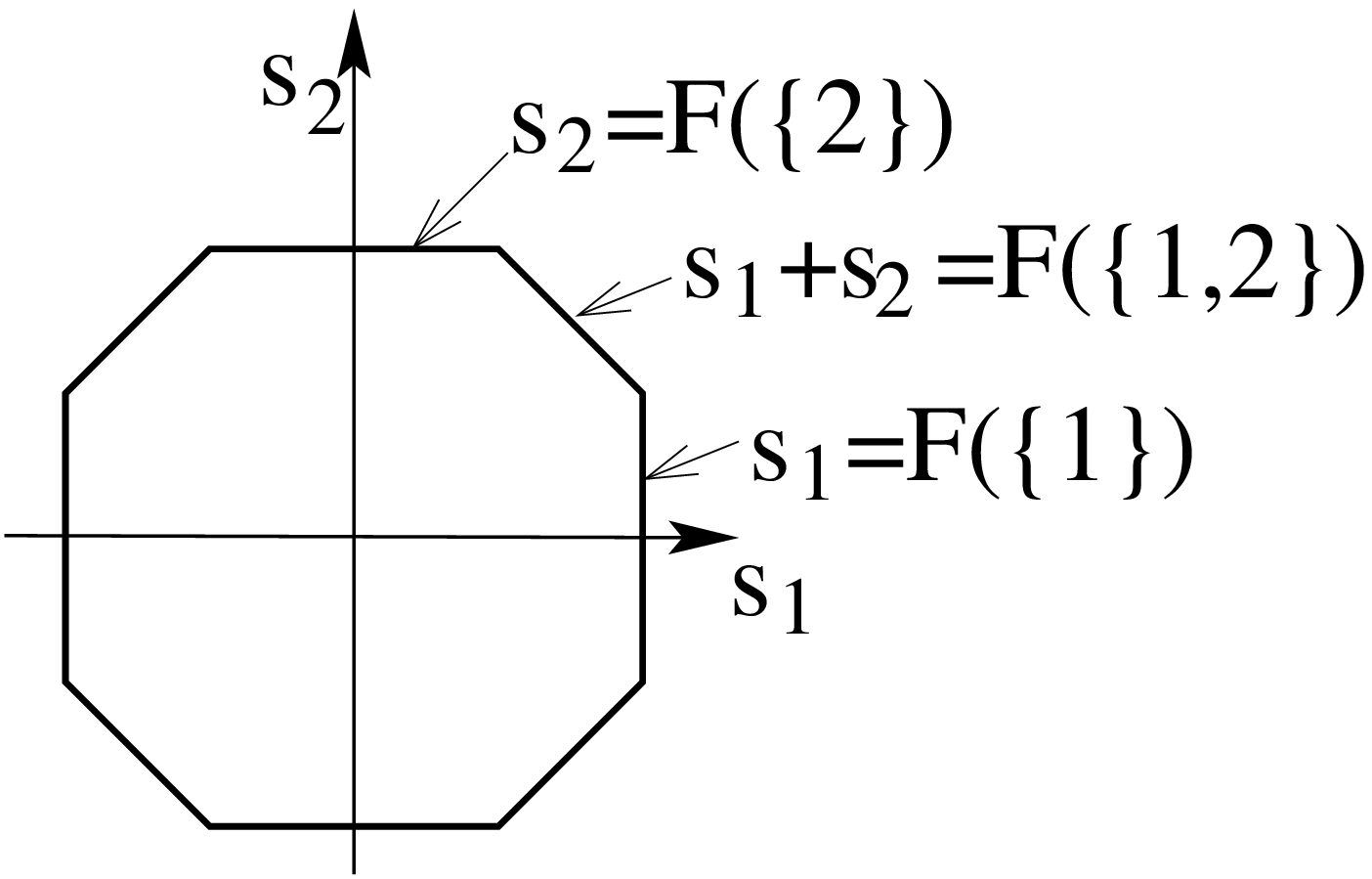} 
 \includegraphics[scale=.35]{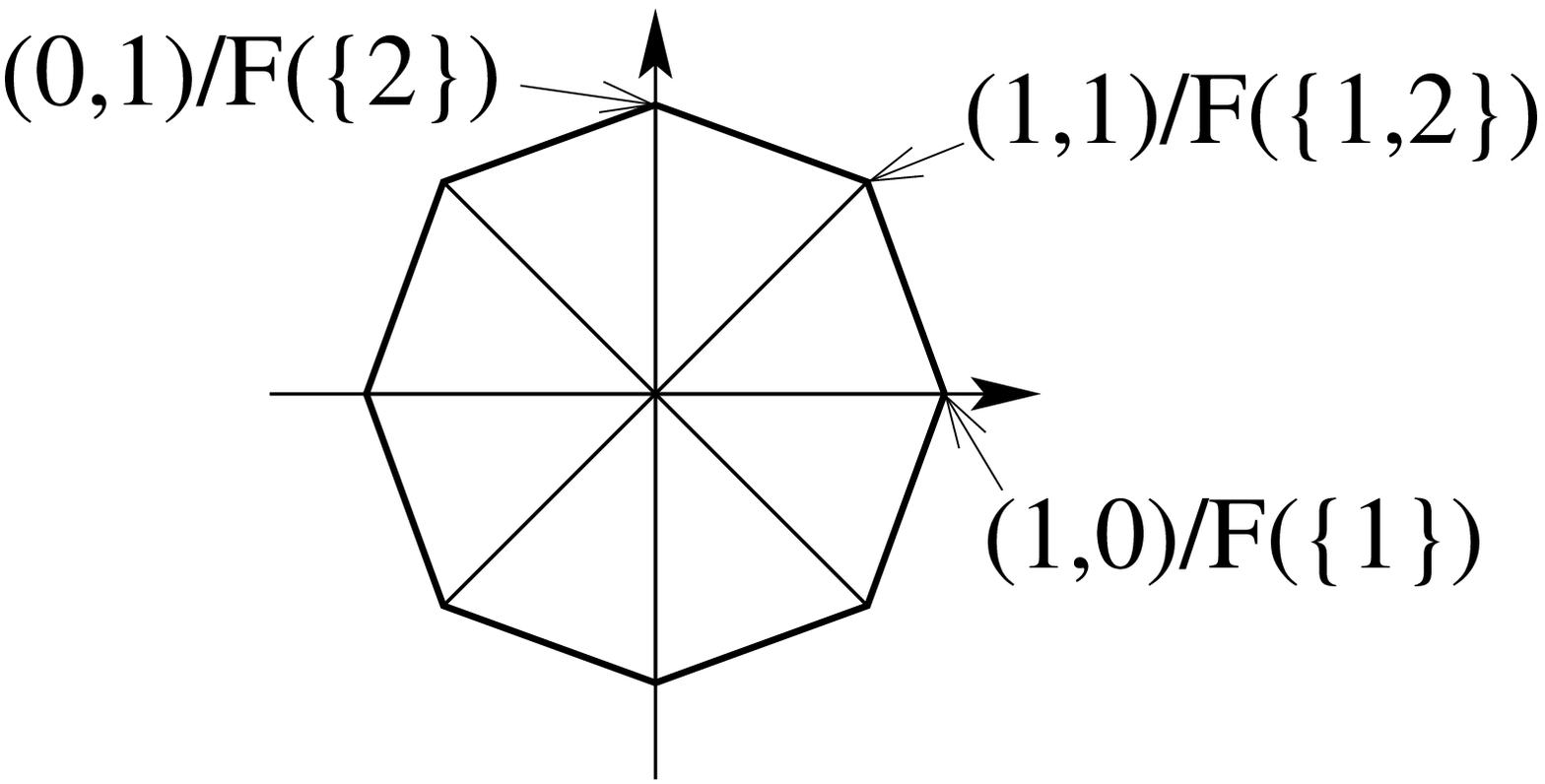}

\vspace*{-.6cm}

\end{center}
\caption{(Left) symmetric submodular polyhedron $|P|(F)$ with its facets. (Right) dual polytope. As shown in \mysec{increasing}, this will be the set of $w \in \rb^p$ such that $f(|w|) \leqslant 1$.
}
\label{fig:balls}
\end{figure}

\paragraph{Separable sets.} In order to study the facial structure, the notion of separable sets is needed; when a set is separable, then the submodular function will decompose a the sum of two submodular functions defined on disjoint subsets. Moreover, any subset $A$ of $V$ may be decomposed uniquely as the disjoint union of inseparable subsets.

\begin{definition} \textbf{(Inseparable set)} Let $F$ be a submodular function such that $F(\varnothing)=0$. A set $A \subseteq V$ is said separable if and only there is a set $B \subseteq A$, such that $B \neq \varnothing$, $B \neq A$ and $F(A) = F(B) + F( A \backslash B)$. If $A$ is not separable, $A$ is said inseparable.
\end{definition}

\begin{proposition} \textbf{(Inseparable sets and function decomposition)}
\label{prop:separable}
Assume $V$ is a separable set for the submodular function $F$, i.e., such that $F(V) = F(A) + F( V \backslash A)$ for a non-trivial subset $A$ of $V$. Then for all $B \subseteq V$, $F(B) = F( B \cap A) + F( B \cap ( V\backslash A) )$.
\end{proposition}
\begin{proof}
If $s \in B(F)$, then we have $F(A) \geqslant s(A) = s(V) - s( V \backslash A) \geqslant F(V) - F( V \backslash A) = F(A)$. This implies that $s(A) = F(A)$ and thus that $B(F)$ can be factorized as $B(F_A) \times B(F^A)$ where $F_A$ is the restriction of $F$ to $A$ and $F^A$ the contraction of $F$ on $A$ (see definition and properties in Appendix~\ref{app:ope}).
Indeed, if $s \in B(F)$, then $s_A \in B(F_A)$ because $s(A)=F(A)$, and $s_{V \backslash A} \in B(F^A)$, because for $B \subseteq V \backslash A$, $s_{V \backslash A}(B) = s(B)  = s( A \cup B) - s(A) \leqslant F( A \cup B) - F(A)$. Similarly, if $s \in B(F_A) \times B(F^A)$, then for all set $B \subseteq V$, $s(B) = s( A \cap B) + S( (V \backslash A) \cap B) \leqslant F(A \cap B) + F( A \cup B) - F(A) \leqslant F(B)$ by submodularity, and $s(A) = F(A)$. This shows that $f(w) =  f_A(w_A) + f^A(w_{V \backslash A} ) $.

Given $B \subseteq V$, we apply the last statement to $w=1_B$ and $w=1_{B \cap ( V \backslash A)}$, to get $F(B) = F( A \cap B) + F( A \cup B) - F(A)$ and $F(B \cap ( V \backslash A) ) = 0 + F( A \cup B) - F(A)$.
We obtain the desired result by taking the difference between the last two equalities.
\end{proof}

\begin{proposition} \textbf{(Decomposition into inseparable sets)}
\label{prop:decseparable}
Let $F$ be a submodular function such that $F(\varnothing)=0$. $V$ may be decomposed uniquely as the disjoint union of non-empty inseparable subsets $A_i$, $i=1,\dots,m$, such that for all $B \subseteq V$, $F(B) = \sum_{i=1}^m F( A_i \cap B)$.
\end{proposition}
\begin{proof}
The existence of such a decomposition is straightforward while the decomposition of $F(B)$ may be obtained from recursive applications of Prop.~\ref{prop:separable}. Given two such decompositions $V = \bigcup_{i=1}^m A_i = \bigcup_{j=1}^s B_i$ of $V$, then 
from  Prop.~\ref{prop:separable}, we have for all $j$, $F(B_j)  = \sum_{i=1}^m F(A_i \cap B_j)$, which implies that  the inseparable set $B_j$ has to be exactly one of the set $A_i$, $i=1,\dots,m$. This implies the unicity.
\end{proof}
Note that by applying the previous proposition to the restriction of $F$ on any set $A$, any set $A$ may be decomposed uniquely as the disjoint union of inseparable sets.

Among the submodular functions we have considered so far, modular functions  of course 
lead to the decomposition of $V$ into a union of singletons. Moreover, for a partition $V = A_1\cup \cdots \cup A_m$, and the function that counts   elements in a partitions, i.e., $F(A) = \sum_{j=1}^m \min\{| A \cap G_j|,1\}$, the decomposition of $V$ is, as expected, $V = A_1\cup \cdots \cup A_m$. 

Finally, the notion of inseparable sets allows to give a  representation of the submodular polyhedron $P(F)$ as the intersection of a potentially smaller number of half-hyperplanes.

\begin{proposition} \textbf{(Minimal representation of $P(F)$)}
\label{prop:minimalPF}
If we denote by $K$ the set of inseparable subsets of $V$. Then,  $P(F) = \{ s\in \rb^p, \ \forall A \in K, s(A) \leqslant F(A) \}$.
\end{proposition}
\begin{proof}
Assume $s \in \{ s\in \rb^p, \ \forall A \in K, s(A) \leqslant F(A) \}$, and let $ B \subseteq V$;
by Prop.~\ref{prop:decseparable}, 
 $B$ can be decomposed into a disjoint union  $A_1 \cup \cdots \cup A_m$
of inseparable sets. Then
$s(B) = \sum_{i=1}^m s(A_i) \leqslant  \sum_{i=1}^m F(A_i) = F(B)$, hence $s \in P(F)$. Note that a consequence of Prop.~\ref{prop:faces}, will be that this set $K$ is the smallest set such that $P(F)$ is the intersection of the hyperplanes defined by $A \in K$.
\end{proof}

\paragraph{Faces of the base polyhedron.}
Given the Prop.~\ref{prop:optsupporttight} that provides the maximizers of $\max_{s \in B(F)} w^\top s$, we may now give necessary and sufficient conditions for characterizing faces of the base polyhedron. We first characterize when the base polyhedron $B(F)$ has non-empty  interior within the subspace $\{s(V) = F(V)\}$.

\begin{proposition} \textbf{(Full-dimensional base polyhedron)}
\label{prop:fulldim}
Let $F$ be a submodular function such that $F(\varnothing)=0$. 
The base polyhedron has non-empty  interior in $\{ s(V) = F(V) \}$ if and only if $V$ is inseparable.
\end{proposition}

\begin{proof}
If $V$ is separable into $A$ and $V \backslash A$, then, by submodularity of $F$,  for all $s \in B(F)$, we  have 
$F(V) = s(V) = s(A) + s(V\backslash A) \leqslant  F(A) + F(V\backslash A)  = F(V)$, which impies that
$s(A) = F(A)$ (and  also $F(V \backslash A) = s(V \backslash A)$). Therefore the base polyhedron is included in the intersection of two distinct affine hyperplanes, i.e., $B(F)$ does not have non-empty  interior in $\{ s(V)=F(V) \}$.

To prove the opposite statement, we proceed by contradiction.
Since $B(F)$ is defined through supporting hyperplanes, it has non-empty interior in $\{ s(V) = F(V)\}$ if it is not contained in any of the supporting hyperplanes.
We thus now assume that $B(F)$ is included in $\{s(A) = F(A)\}$, for $A$ a non-empty strict subset of $V$. Then,
following the same reasoning than in the proof of Prop.~\ref{prop:separable}, $B(F)$ can be factorized as $B(F_A) \times B(F^A)$ where $F_A$ is the restriction of $F$ to $A$ and $F^A$ the contraction of $F$ on $A$ (see definition and properties in Appendix~\ref{app:ope}).

This implies that $f(w) =  f_A(w_A) + f^A(w_{V \backslash A} ) $, which implies that $F(V) = F(A)+F(V \backslash A)$, when applied to $w = 1_{V \backslash A}$, i.e., $V$ is separable.
\end{proof}

We can now detail the facial structure of the base polyhedron, which will be dual to the one of the polyhedron defined by $\{ w \in \rb^p, \ f(w) \leqslant 1, w^\top 1_V = 0 \}$ (i.e., the sub-level set of the \lova extension projected on a subspace of dimension $p-1$). As the base polyhedron $B(F)$ is a polytope in dimension $p-1$ (because it is bounded and contained in the affine hyperplane $\{ s(V) = F(V) \}$), one can define its set of \emph{faces}. As described earlier, faces are the intersections of the polyhedron $B(F)$ with any of its supporting hyperplanes. Supporting hyperplanes are themselves defined as the hyperplanes $\{ s(A) = F(A) \}$ for $A \subseteq V$.
From Prop.~\ref{prop:optsupporttight}, faces   are obtained as the intersection of $B(F)$ with $s(A_1 \cup \cdots \cup A_i) = F( A_1 \cup \cdots  \cup A_i)$ for a partition $V=A_1 \cup \cdots \cup A_m$. Together with Prop.~\ref{prop:fulldim}, we can now provide a characterization of the faces of $B(F)$. See more details on the facial structure of $B(F)$ in~\cite{fujishige2005submodular}.

Since the facial structure is invariant by translation, as done at the end of \mysec{lovadef}, we may translate $B(F)$ by a certain vector $t \in B(F)$, so that $F$ may be taken to be non-negative and such that $F(V)=0$, which we now assume.

\begin{proposition}\textbf{(Faces of the base polyhedron)}
\label{prop:faces}
Let $A_1 \cup \cdots \cup A_m$ be a partition of $V$, such that for all $j \in \{1,\dots,m\}$, $A_j$ is inseparable for the function $G_j: D \mapsto F( A_1 \cup \cdots  \cup A_{j-1} \cup D) - F( A_1 \cup \cdots  \cup A_{j-1})$ defined on subsets of $A_j$. The set of bases $s \in B(F)$ such that for all $j \in \{1,\dots,m\}$, $s(A_1 \cup \cdots \cup A_i) = F( A_1 \cup \cdots  \cup A_i)$ is a face of $B(F)$ with non-empty  interior in the intersection of the $m$ hyperplanes (i.e., the affine hull of the face is exactly the intersection of these $m$ hyperplanes). Moreover, all faces of $B(F)$ may be obtained this way.
\end{proposition}

\begin{proof}
From Prop.~\ref{prop:optsupporttight}, all faces may be obtained with supporting hyperplanes  of the form
$s(A_1 \cup \cdots \cup A_i) = F( A_1 \cup \cdots  \cup A_i)$, $i=1,\dots,m$, for a certain partition $V = A_1 \cup \cdots \cup A_m$. Hovever, among these partitions, only some of them will lead to an affine hull of full dimension~$m$. From Prop.~\ref{prop:fulldim} applied to the submodular function $G_j$, this only happens if $G_j$ has no separable sets. Note that the corresponding face is then exactly equal to the product of base polyhedra $B(G_1) \times \cdots \times  B(G_m)$.
\end{proof}
Note that in the previous proposition, several ordered partitions may lead to the exact same face. The maximal number of full-dimensional faces of $B(F)$ is always less than $2^{p}-2$ (number of non-trivial subsets of $V$), but this number may be reduced in general (see examples in Figure~\ref{fig:symmballs} for the cut function). Moreover, the number of extreme points may also be large, e.g., $p!$ for the submodular function $A \mapsto -|A|^2 $ (leading to the permutohedron~\cite{fujishige2005submodular}).

\paragraph{Dual polytope of $B(F)$.} We now assume that $F(V) = 0$, and that for all non-trivial subsets $A$  of $V$, $F(A)>0$. This implies that $V$ is inseparable for $F$, and thus, by Prop.~\ref{prop:fulldim}, that $B(F)$ has non-empty relative interior in $\{ s(V) = 0 \}$. We thus have a polytope with non-empty interior in a space of dimension $p-1$. We may compute the support function of the polytope in this low-dimensional space. For any $w \in \rb^p$ such that $w^\top 1_V = 0$, then
$\sup_{ s \in B(F) } s^\top w = f(w)$. Thus, the dual polytope is the set of elements $w$ such that $w^\top 1_V = 0$ and the support function is less than one, i.e., 
 $\mathcal{U}  = \{w \in \rb^p , f(w) \leqslant 1, w^\top 1_V = 0 \}$.
 
The faces of $\mathcal{U}$ are   obtained from the faces of $B(F)$ through the relationship defined in Prop.~\ref{prop:optsupport}: that is, given a face of $B(F)$, and all the  partitions of Prop.~\ref{prop:faces} which lead to it, the corresponding face of $\mathcal{U}$ is the closure of the union of all $w$ that satisfies the level set constraints imposed by the different ordered partitions. As shown in \cite{shapinglevelsets}, the different ordered partitions all share the same elements but with a different order, thus inducing  a set of partial constraints between the ordering of the $m$ values $w$ is allowed to take.

An important aspect is that the separability criterion in Prop.~\ref{prop:faces} forbids some level sets from being characteristic of a face. For example, for cuts in an undirected graph, we will show
in \mysec{shaping} that all level sets within a face must be connected components of the graph. When the \lova extension is used as a constraint for a smooth optimization problem, the solution has to be in one of the faces. Moreover, within this face, all other affine constraints are very unlikely to happen, unless the smooth function has some specific directions of zero gradient (unlikely with random data, for some sharper statements, see~\cite{shapinglevelsets}). Thus, when using the \lova extension as a regularizer, only certain level sets are likely to happen, and in the context of cut functions, only connected sets are allowed, which is one of the justifications behind using the total variation (see more details in \mysec{shaping}).

\section{Positive and symmetric submodular polyhedra$^\ast$}
\label{sec:faces-indep}

In this section, we extend the previous results to the positive and symmetric submodular polyhedra, which were defined in \mysec{polymat} for non-decreasing submodular functions. We start with a characterization of such non-decreasing function through the inclusion of the base polyhedron to the postive orthant.

\begin{proposition} \textbf{(Base polyhedron and polymatroids)}
\label{prop:basepolym}
Let $F$ be a submodular function such that $F(\varnothing)=0$.  The function $F$ is non-decreasing, if and only if the base polyhedron is included in the positive orthant $\rb_+^p$.
\end{proposition}
\begin{proof}

A simple  proof uses the representation of the \lova extension as the the support function of $B(F)$. Indeed,  from Prop.~\ref{prop:greedy}, we get
$\min_{s \in B(F) } s_k = 
- \max_{s \in B(F)} (-1_{\{k\}})^\top s = -f(-1_{\{k\}})  = F(V) - F(V \backslash \{k\})$. Thus, $B(F) \subseteq \rb_+^p$ if and only if for all $k \in V$, $F(V) - F(V \backslash \{k\}) \geqslant 0 $. Since, by submodularity, for all $A \subseteq V$ and $k \notin A$, 
$F(A \cup \{k\}) - F(A) \geqslant F(V) - F(V \backslash \{k\}) $, $B(F) \subseteq \rb_+^p$  if and only if $F$ is non-decreasing.
\end{proof}

We now assume that the function $F$ is non-decreasing, and consider the positive and symmetric submodular polyhedra $P_+(F)$ and $|P|(F)$. These two polyhedra are compact and are thus polytopes.
Moreover, $|P|(F)$  is the unit ball of the dual norm $\Omega_\infty^\ast$ defined in \mysec{sparse}. This polytope is polar to the unit ball of $\Omega_\infty$, and it it thus of interest to characterize the facial structure of the symmetric submodular polyhedron\footnote{The facial structure of the positive submodular polyhedron $P_+(F)$ will not be covered in this monograph but results are similar to $B(F)$. We will only provide maximizers of linear functions in Prop.~\ref{prop:optsupporttight-positive}.} $|P|(F)$.

We first derive the same proposition than Prop.~\ref{prop:optsupporttightSUB} for the positive and symmetric submodular polyhedra. For $w\in \rb^p$, $w_+$ denotes the $p$-dimensional vector with components $(w_k)_+ = \max \{w_k,0\}$, and $|w|$ denotes the $p$-dimensional vector with components $|w_k|$.

\begin{proposition}\textbf{(Maximizers of the support function of positive submodular polyhedron)}
\label{prop:optsupporttight-positive}
Let $F$ be a non-decreasing submodular function such that $F(\varnothing)=0$. Let $w \in \rb^p$.
Then $\max_{s \in P_+(F)} w^\top s = f(w_+)$. Moreover, if $w$ has unique values
 $v_1 > \cdots > v_m $, taken at sets $A_1,\dots,A_m$. Then $s$ is optimal for $f(w_+) = \max_{s \in P_+(F)} w^\top s$ if and only if (a) for all $i \in \{1,\dots,m\}$ such that $v_i>0$,
$s(A_1 \cup \cdots \cup A_i) = F(A_1 \cup \cdots \cup A_i)$, and
(b) for all $k \in V$ such that $w_k < 0$, then $s_k = 0$.
\end{proposition}
\begin{proof} The proof follows the same arguments than for Prop.~\ref{prop:optsupporttightSUB}.
Let $d$ be the largest integer such that $v_d>0$. We have, with $B_i = A_1 \cup \cdots \cup A_i$:
\BEAS
f(w_+) & = & v_d F(B_d) + \sum_{i=1}^{d-1} F(B_i) ( v_i - v_{i+1}) \\[-.2cm]
 & = & v_d ( F(B_d) - s(B_d) ) + \sum_{i=1}^{d-1} [F(B_i) -s(B_i) ]( v_i - v_{i+1})   \\[-.2cm]
 & & \hspace*{1cm} + v_d s(B_d) + \sum_{i=1}^{d-1} s(B_i) ( v_i - v_{i+1}) \\[-.2cm]
& \geqslant & v_d s(B_d) + \sum_{i=1}^{m-1} s(B_i) ( v_i - v_{i+1}) = s^\top w_+ \geqslant s^\top w.
\EEAS
We have equality if and only if the components of $s_k$ are zero as soon as the corresponding component of $w_k$ is strictly negative (condition (b)), and $F(B_i) -s(B_i)=0$ for all $i \in \{1,\dots,d-1\}$ (condition (a)). This proves the desired result.
 \end{proof}

\begin{proposition}\textbf{(Maximizers of the support function of symmetric submodular polyhedron)}
\label{prop:optsupporttight-indep}
Let $F$ be a non-decreasing submodular function such that $F(\varnothing)=0$. Let $w \in \rb^p$, with unique values 
for $|w|$, $v_1 > \cdots > v_m  $, taken at sets $A_1,\dots,A_m$. 
Then $\max_{s \in|P|(F)} w^\top s = f(|w|)$.
Moreover $s$ is optimal for $\max_{s \in |P|(F)} w^\top s$ if and only if for all $i$ such that $v_i>0$ (i.e., for all $i \in \{1,\dots,m\}$ except potentially the last one)
$|s|(A_1 \cup \cdots \cup A_i) = F(A_1 \cup \cdots \cup A_i)$, and $w$ and $s$ have the same signs, i.e., for all $k \in V$, $w_k s_k \geqslant 0$. \end{proposition}
\begin{proof} 
We have
$
\max_{ s \in |P|(F) } w^\top s = \max_{ t \in P_+(F) |w|^\top t}
$, where a solution $s$ may be obtained from a solution $t$ as long as $|s|=t$ and $ w \circ s \geqslant 0$. Thus, we may apply Prop.~\ref{prop:optsupporttight-positive}, by noticing that the condition (b) is not applicable because $|w| \in \rb_+^p$. Note that the value of $s_k$ when $w_k=0$ is irrelevant (as long as $s \in B(F)$).
 \end{proof}

Before describing the facial structure of $|P|(F)$, we need the notion of stable sets, which are sets which cannot be augmented without strictly increasing the values of $F$.

\begin{definition} \textbf{(Stable sets)} A set $A \subseteq V $ is said stable for a submodular function $F$, if $A \subseteq B$ and $A \neq B$ implies that $F(A) < F(B)$.
\end{definition}

We can now derive a characterization of the faces of $|P|(F)$ (a similar proposition holds for $P_+(F)$).

\begin{proposition}\textbf{(Faces of the symmetric submodular polyhedron)}
\label{prop:facesNORM}
Let $C$ be a stable set and
let $A_1 \cup \cdots \cup A_m$ be a partition of $C$, such that for all $j \in \{1,\dots,m\}$, $A_j$ is inseparable for the function $G_j: D \mapsto F( A_1 \cup \cdots  \cup A_{j-1} \cup D) - F( A_1 \cup \cdots  \cup A_{j-1})$ defined on subsets of $A_j$, and $\varepsilon \in \{-1,1\}^C$.  The set of   $s \in |P|(F)$ such that for all $j \in \{1,\dots,m\}$, $ (\varepsilon \circ  s ) (A_1 \cup \cdots \cup A_i) = F( A_1 \cup \cdots  \cup A_i)$ is a face of $|P|(F)$ with non-empty  interior in the intersection of the $m$ hyperplanes. Moreover, all faces of $|P|(F)$ may be obtained this way.
\end{proposition}
\begin{proof}
The proof follows the same structure than for Prop.~\ref{prop:faces}, but by applying Prop.~\ref{prop:optsupporttight-indep} instead of Prop.~\ref{prop:optsupporttight}. We consider $w \in \rb^p$, with support $C$, which we decompose into $C = A_1 \cup \cdots \cup A_m$ following the decreasing sequence of constant sets of $|w|$. Denote by $\varepsilon$ the sign vector  of $w$. Following Prop.~\ref{prop:optsupporttight-indep}, the set of maximizers of $s^\top w$ over $s \in |P|(F)$ are such that
$s \in |P|(F)$ and $ (s \circ \varepsilon) ( A_1 \cup \dots \cup A_m) = F( A_1 \cup \dots \cup A_m)$. The set of maximizers in then isomorphic to the product of all $ \varepsilon_{G_j} \circ B(G_i)$ and $|P|(F^C)$ where $F^C: D \mapsto F( C \cup D) - F(C)$ is the contraction of $F$ on $C$. The face has non-empty relative interior, if and only if, (a) all $ \varepsilon_{G_j} \circ B(G_i)$  have non relative empty-interior (hence the condition of inseparability) and (b) $|P|(F^C)$ has non-empty interior. Condition (b) above is equivalent to the function $w_C \mapsto f^C(|w_{V \backslash C}|)$ being a norm. This is equivalent to 
$f^C(|w_{V\backslash C}|) = 0 \Leftrightarrow w_{V\backslash C} = 0 $. Since $f$ is non-decreasing with respect to each of its components, this is equivalent to $f^C( 1_{\{k\}} )>0$ for all $k \in V\backslash C$. Given the extension property of $f^C$, this in turn is equivalent to $F(C \cup \{k\}) - F(C) > 0$, for all $k \in V \backslash C$, i.e., since $F$ is submodular, $F(D)>F(C)$ for all subsets $D$ strictly containing $C$, i.e., $C$ is stable. See also Prop.~\ref{prop:norm} for similar arguments regarding norms.
\end{proof}

The last proposition will have interesting consequences for the use of submodular functions for defining sparsity-inducing norms in \mysec{increasing}. Indeed, the faces of the unit-ball of $\Omega_\infty$ are dual to the ones of the dual ball of $\Omega_\infty^\ast$ (which is exactly $|P|(F)$). 
As a consequence of  Prop.~\ref{prop:optsupporttight-indep}, the set $ C$ in Prop.~\ref{prop:facesNORM} corresponds to the non-zero elements of   $w$ in a face of the unit-ball of $\Omega_\infty$. This implies that all faces of the unit ball of $\Omega_\infty$ will only impose non-zero patterns which are stable sets. See a more precise statement in \mysec{sparse}.

\paragraph{Stable inseparable sets.}
We end the description of the structure of $|P|(F)$ by noting that among the $2^{p}-1$ constraints of the form $\|s_A\|_1 \leqslant F(A)$ defining it, we may restrict the sets $A$ to be stable and inseparable. Indeed, if $\|s_A\|_1 \leqslant F(A)$ for all stable and inseparable sets~$A$, then if a set $B$ is not stable, then we may consider the smallest enclosing stable set (these are stable by intersection, hence the possibility of defining such a smallest enclosing stable set) $C$, and we have
$\| s_B \|_1 \leqslant \| s_C \|_1$, and $F(B) = F(C)$, which implies $\|s_B\|_1 \leqslant F(B)$. We thus need to show that $\| s_C \|_1 \leqslant F(C)$ only for stable sets $C$. If the set $C$ is separable into $C = D_1 \cup \cdots \cup D_m$, where all $D_i$, $i=1,\dots,m$ are separable (from Prop.~\ref{prop:decseparable}), they must all be stable (otherwise $C$ would not be), and thus we 
have $\|s_C\|_1 = \| s_{D_1}\|_1 + \cdots + \| s_{D_m}\|_1
\leqslant F(D_1) + \cdots + F(D_m) = F(C)$.

For $F(A) = |A|$, then $|P|(F)$ is the $\ell_\infty$-ball, with all singletons being the stable inseparable sets. For $F(A) =\min\{|A|,1\} = 1_{|A| \neq \varnothing}$, then $|P|(F)$ is the $\ell_1$-ball and $V$ is the only stable inseparable set. See also \myfig{balls-new}, and \myfig{3dballs} in \mysec{increasing}.

\chapter{Convex Relaxation of Submodular Penalties}
\label{chap:relax}
\label{chap:relaxation}
\label{chap:norms}

In this chapter, we show how submodular functions and their \lova extensions are intimately related to various relaxations of combinatorial optimization problems, or problems with a joint discrete and continuous structure.

In particular, we present in \mysec{closure} the theory of convex and concave closures of set-functions: these can be defined for any set-functions and allow convex reformulations of the minimization and maximization of set-functions. It turns out that for submodular functions, the convex closure is exactly the \lova extension, which can be computed in closed form, which is typically not the case for non-submodular set-functions.

In \mysec{sparse}, we introduce the concept of structured sparsity, which corresponds to situations where a vector $w$ has to be estimated, typically a signal or the linear representation of a prediction, and structural assumptions are imposed on $w$. In \mysec{increasing} and \mysec{l2}, we consider imposing that $w$ has many zero components, but with  the additional constraint that some supports are favored. A submodular function will encode that desired behavior. In \mysec{shaping}, we consider a similar approach, but on the level sets of $w$.

\section{Convex and concave closures of set-functions}
\label{sec:closure}

Given any set-function $F:2^V \to \rb$ such that $F(\varnothing)=0$, we may define the convex closure of
of $F$ as the largest function $f:\rb^p \to \rb \cup \{ + \infty\} $ so that (a) $f$ is convex and (b) for all $A \subseteq V$, $f(1_A) \leqslant F(A)$. 

\paragraph{Computation by Fenchel bi-conjugation.}
In this section, given a non-convex function $g$, we will consider several times the task of computing its convex envelope $f$, i.e., its largest convex lower-bound. As explained in Appendix~\ref{app:convex}, a systematic way to obtain $f$ is to compute the Fenchel bi-conjugate.

We thus consider the function $g$ so that $g(1_A) = F(A)$ for all $A \subseteq V$, and $g(w) = +\infty$ for any other $w$ which is not an indicator function of a set $A \subseteq V$ (i.e., $w \notin \{0,1\}^p$). We have for $s \in \rb^p$:
\BEAS
g^\ast(s) & = & \sup_{ w \in \rb^p} w^\top s - g(w)  = \sup_{ w = 1_A, \ A \subseteq V } w^\top s - g(w) \\
& = & \max_{ A \subseteq V } s(A) - F(A) ,
\EEAS
leading to, for any $w \in \rb^p$,
\BEAS
f(w) & = & g^{\ast \ast}(w)  = \sup_{s \in \rb^p} w^\top s - g^\ast(s) \\
& = & \sup_{s \in \rb^p} \Big\{  \min_{ A \subseteq V } FA) - s(A)  + w^\top s   \Big\} \\
& = & \sup_{s \in \rb^p} \ \ \min_{\lambda \geqslant 0, \ \sum_{A \subseteq V} \lambda_A  = 1 } \sum_{A \subseteq V} \lambda_A [ F(A) - s(A) ]  + w^\top s   \\
& = &  \min_{\lambda \geqslant 0, \ \sum_{A \subseteq V}   \lambda_A  = 1 }  \sup_{s \in \rb^p}
\sum_{A \subseteq V} \lambda_A [ F(A) - s(A) ]  + w^\top s   \\
& = &  \min_{\lambda \geqslant 0, \ \sum_{A \subseteq V}   \lambda_A  = 1 }
\sum_{A \subseteq V} \lambda_A  F(A) 
\mbox{ such that } w = \sum_{A \subseteq V} \lambda_A 1_A.
\EEAS
This implies that the domain of $f$ is $[0,1]^p$ (i.e., $f(w) = +\infty$ for $w \notin [0,1]^p$). Moreover, 
since the vectors $1_A$, for $A \subseteq V$ are extreme points of $[0,1]^p$, for any $B \subseteq V$, the only way to express $1_B$ as a combination of indicator vectors $1_A$ is by having $\lambda_B=1$ and all other values $\lambda_A$ equal to zero. Thus
$f(1_B) = F(B)$. That is, the convex closure is always tight at each $1_A$, and $f$ is an extension of $F$ from $\{0,1\}^p$ to $[0,1]^p$. This property is independent from submodularity.

\paragraph{Minimization of set-functions.}
We may relate the minimization of $F$ to the minimization of its convex closure:
\BEAS
\min_{A \subseteq V} F(A)
& =  &\min_{w \in \{0,1\}^p} f(w) \\
& \geqslant & \min_{w \in [0,1]^p} f(w) \\
& = & \min_{w \in [0,1]^p} \min_{\lambda \geqslant 0, \ \sum_{A \subseteq V}   \lambda_A  = 1 }
\sum_{A \subseteq V} \lambda_A  F(A)   \\[-.2cm]
& & \hspace*{3cm}
\mbox{ such that } w = \sum_{A \subseteq V} \lambda_A 1_A, \\[-.15cm]
& = & \min_{w \in [0,1]^p} \min_{\lambda \geqslant 0, \ \sum_{A \subseteq V}   \lambda_A  = 1 }
\sum_{A \subseteq V} \lambda_A  F(A)  \\[-.1cm]
& \geqslant & \min_{A \subseteq V} F(A),
\EEAS
which implies that minimizing the convex closure of $F$ on $[0,1]^p$ is equivalent to minimizing $F$ on $2^V$. See an illustration
in \myfig{closures}.

For submodular functions, it simply turns out that the convex closure is equal to the \lova extension. Hence, it is computable in closed form and amenable to optimization. This fact is in fact exactly shown in the proof of Prop.~\ref{prop:greedy}.

\paragraph{Concave closure.} The concave closure is defined in a similar way, and can be seen to be the opposite of the convex closure of $-F$. 
Note that it cannot be computed in general as this would mean that there are polynomial-time algorithms for submodular function maximization~\cite{calinescu2007maximizing}. However, following Prop.~\ref{prop:char-polym} and its discussion in \mysec{polymat}, one can always find  ``constant plus modular'' upper-bounds which are tight at any given vertex of the hypercube (but this cannot be done at any interior point in general).

\begin{figure}
\begin{center}

\vspace*{-.15cm}

  \hspace*{-.5cm}
 \includegraphics[scale=.6]{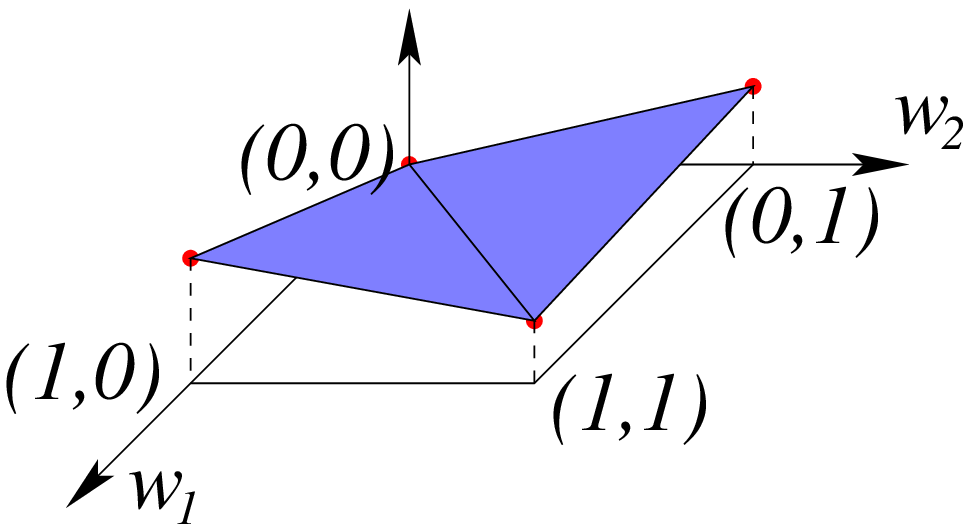} 
 \includegraphics[scale=.6]{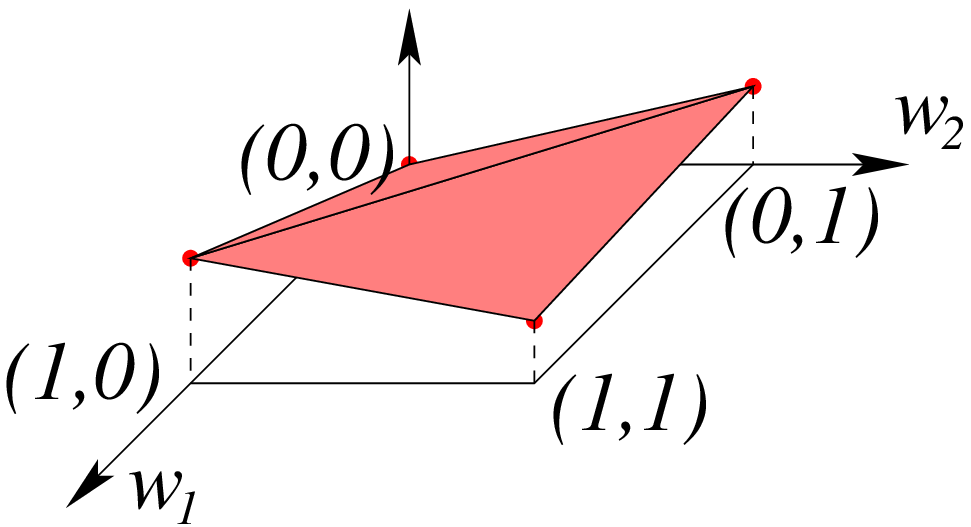} 
  \hspace*{-.5cm}

\end{center}
\caption{Closures of set-functions: (left) convex closure, (right) concave closure. For
a submodular function, the convex closure is the \lova extension, which happens to be positively homogeneous.}
\label{fig:closures}
\end{figure}

\section{Structured sparsity}
\label{sec:sparse}
 
The concept of parsimony is central in many scientific domains. In the context of statistics, signal processing or machine learning, it takes the form of variable or feature selection problems. 

In a supervised learning problem, we aim to predict $n$ responses $y_i \in\rb$, from $n$  observations $x_i \in \rb^p$, for $i \in \{1,\dots,n\}$. In this monograph, we focus on linear predictors of the form $f(x) = w^\top x$, where $w \in \rb^p$ (for extensions to non-linear predictions, see~\cite{bach2008cgl,hkl} and references therein). We consider estimators obtained by the following regularized empirical risk minimization formulation:
\BEQ
\label{eq:erm}
\min_{w \in \rb^p} \frac{1}{n} \sum_{i=1}^n \ell(y_i,w^\top x_i) + \lambda \Omega(w),
\EEQ
where $\ell(y,\hat{y})$ is a loss between a prediction $\hat{y}$ and the true response $y$, and $\Omega$ is a regularizer (often a norm). Typically, the quadratic loss 
$\ell(y,\hat{y}) = \frac{1}{2} (y - \hat{y})^2$ is used for regression problems and the logistic loss $\ell(y,\hat{y}) = \log(1 + \exp( - y \hat{y}) )$ is used for binary classification problems where $y \in \{-1,1\}$ (see, e.g., \cite{Shawe-Taylor2004} and   \cite{hastie} for more complete descriptions of loss functions).

In order to promote sparsity, i.e., to have zeros in the components of $w$, the $\ell_1$-norm is commonly used and, in a least-squares regression framework is referred to as the Lasso 
\cite{Tibshirani1996} in statistics and as basis pursuit~\cite{Chen1998} in signal processing.

Sparse models are commonly used in two situations: First,  to make the model or the prediction more interpretable or cheaper to use, i.e., even if
the underlying problem might not admit sparse solutions, one looks for the best sparse approximation. Second, sparsity  can also be used given prior knowledge that the model should be sparse. In these two situations, reducing parsimony to finding models with low cardinality of their support turns out to be limiting, and structured parsimony has emerged as a fruitful practical extension, with applications to image processing, text processing, bioinformatics or audio processing (see, e.g.,~\cite{cap,jenatton2009structured,huang2009learning,LaurentGuillaumeGroupLasso,kim,jenattonmairal,Mairal10aNIPS,augustin}, a review in~\cite{fot, statscience} and \mychap{examples} for various examples).

For vectors in $w \in \rb^p$, two main types of sparse structures have emerged. The prior which is imposed on a vector $w$ is that $w$ should have either many zeros or many components which are equal to each other. In the former case, structured sparsity aims at enforcing or favoring special sets of patterns of non-zeros, i.e., for the \emph{support} set $\supp(w) = \{ w \neq 0\} = \{ k \in V, \ w_k \neq 0\}$. 
Favoring certain supports may be achieved by adding a penalty to the optimization formulation (such as empirical risk minimization), i.e., choosing for $\Omega(w)$ in \eq{erm}, a function of the support of $w$. In \mysec{increasing}, we show how for submodular functions, the resulting non-continuous problem may be relaxed into a convex optimization problem involving the \lova extension. In the latter case, structured sparsity aims at enforcing or favoring special sublevel sets $\{ w \geqslant \alpha \}  = \{ k \in V, \ w_k \geqslant \alpha\}$ or constant sets $\{ w = \alpha \}  = \{ k \in V, \ w_k = \alpha\}$, for certain $\alpha \in \rb$. Again, specific level sets may be obtained by adding a penalty that is a function of level sets. Convex relaxation approaches  are explored in \mysec{shaping}, for non-negative submodular functions such that $F(V)=0$.

\section{Convex relaxation of combinatorial penalty}
\label{sec:increasing}

Most of the work based on convex optimization and the design of dedicated sparsity-inducing norms has focused mainly on the specific allowed set of sparsity patterns~\cite{cap,jenatton2009structured,LaurentGuillaumeGroupLasso,jenattonmairal}: if $w \in \rb^p$ denotes the predictor we aim to estimate, and $\supp(w)$ denotes its support, then these norms are designed so that penalizing with these norms only leads to supports from a  given family of allowed patterns. 
We can instead follow the direct approach of~\cite{haupt2006signal,huang2009learning} and consider specific penalty functions $F(\supp(w))$ of the support set $\supp(w) = \{ j \in V, \ w_j \neq 0 \}$, which go beyond the cardinality function, but are not limited or designed to only forbid certain sparsity patterns. As first shown in \cite{bach2010structured}, for \emph{non-decreasing} submodular functions,
 these may also lead to restricted sets of supports but their interpretation in terms of an \emph{explicit} penalty on the support leads to additional insights into the behavior of structured sparsity-inducing norms. 
 
 We are thus interested in an optimization problem of the form 
 $$
 \min_{w \in \rb^p} \frac{1}{n}\sum_{i=1}^n \ell(y_i,w^\top x_i) + F(\supp(w)).
 $$
 While direct greedy approaches (i.e., forward selection) to the problem are considered in \cite{haupt2006signal,huang2009learning}, submodular analysis may be brought to bear to provide convex relaxations to the function $w \mapsto  F(\supp(w))$, which extend the traditional link between the $\ell_1$-norm and the cardinality function.

\begin{proposition} \textbf{(Convex relaxation of functions defined through supports)}
\label{prop:relax}
Let $F$ be a non-decreasing submodular function. The function $\Omega_\infty: w \mapsto f(|w|)$ is the convex envelope (tightest convex lower bound) of the function $w \mapsto F(\supp(w))$ on the unit $\ell_\infty$-ball $[-1,1]^p$.
\end{proposition}
\begin{proof}
We use the notation $|w|$ to denote the $p$-dimensional vector composed of the absolute values of the components of $w$.
We denote by $g^\ast$ the Fenchel conjugate (see definition in Appendix~\ref{app:convex}) of $g: w \mapsto F(\supp(w))$ on the domain $\{ w \in \rb^p, \ \| w\|_\infty \leqslant 1\} = [-1,1]^p$, and $g^{\ast \ast}$ its bidual~\cite{boyd}. We only need to show that the Fenchel bidual is equal to the function $w \mapsto f(|w|)$. In order to compute the Fenchel duals, we are going to premultiply vectors $w \in \rb^p$ by an indicator vector $\delta \in \{0,1\}^p$ so that if $w$ has no zero components, then $g(w \circ \delta)$ is equal to $F$ applied to the support of $\delta$, i.e., since $f$ is an extension of $F$, equal to $f(\delta)$ 
(for vectors $a,b \in \rb^p$, we denote by $a \circ b$ the vector obtained by elementwise multiplication of $a$ and $b$).

By definition of the Fenchel conjugacy and of $g$, we have :
\BEAS
 g^\ast(s)  
& = & \sup_{ w \in [-1,1]^p } w^\top s - g(w) \\
& = & \max_{\delta \in \{0,1\}^p } \ \ \sup_{w \in ( [-1,1] \backslash \{0\} ) ^p  }  ( \delta \circ w ) ^\top s - f(\delta).
\EEAS
Moreover, by using separability,
\BEAS
& = & \max_{\delta \in \{0,1\}^p }  \sum_{ j=1}^p  \sup_{w_j \in  [-1,1] \backslash \{0\}  }  w_j \delta_j s_j  - f(\delta) \\
& = & \max_{\delta \in \{0,1\}^p }     \delta  ^\top |s| - f(\delta) 
\mbox{ by maximizing out } w,\\
& = & \max_{\delta \in [0,1]^p }     \delta  ^\top |s| - f(\delta) \mbox{ because of  Prop.~\ref{prop:minlova}}.
\EEAS
Note that the assumption of submodularity is key to applying  Prop.~\ref{prop:minlova} (i.e., equivalence between maximization on the vertices and on the full hypercube).

Thus, for all $w$ such that $\| w\|_\infty \leqslant 1$, 
\BEAS
 g^{\ast \ast}(w)  & = & \max_{s \in \rb^p} s^\top w - g^\ast(s)  \\
& = & \max_{s \in \rb^p}  \min_{\delta \in [0,1]^p } \  s^\top w    -  \delta  ^\top |s| + f(\delta) .
\EEAS
By strong convex duality (which applies because Slater's condition~\cite{boyd} is satisfied), we can invert the ``min'' and ``max''  operations and get:
\BEAS
g^{\ast \ast}(w) 
& = &  \min_{\delta \in [0,1]^p } \max_{s \in \rb^p}   \  s^\top w    -  \delta  ^\top |s| + f(\delta)  \\
& = &  \min_{\delta \in [0,1]^p } \max_{s \in \rb^p}   \ \sum_{j=1}^p \big\{ s_j w_j - \delta_j |s_j| \big\}   + f(\delta) .
\EEAS
We can then maximize in closed form with respect to to each $s_j \in \rb$ to obtain the extra constraint $ |w_j| \leqslant \delta_j$, i.e.:
\BEAS
g^{\ast \ast}(w) 
 & = &  \min_{\delta \in [0,1]^p,  \  \delta \geqslant | w|  }   f(\delta) .
\EEAS
Since  $F$ is assumed non-decreasing, the \lova extension $f$ is non-decreasing with respect to each of its components, which implies that $\min_{\delta \in [0,1]^p,  \  \delta \geqslant | w|  }   f(\delta) = f(|w|)$,
which leads to the desired result. Note that an  alternative proof may be found in \mysec{lprelax}.
\end{proof}
The previous proposition provides a relationship between combinatorial optimization problems---involving functions of the form $w \mapsto F({\rm Supp}(w))$---and convex optimization problems involving the \lova extension. A desirable behavior of a convex relaxation is that some of the properties of the original problem are preserved. In this monograph, we will focus mostly on the allowed set of sparsity patterns (see  below).
For more details about theroretical guarantees and applications of submodular functions to structured sparsity, see~\cite{bach2010structured,shapinglevelsets}. In \mychap{examples}, we consider several examples of submodular functions and present when appropriate how they translate to sparsity-inducing norms.

 \paragraph{Lasso and group Lasso as special cases.}
 
 For the cardinality function $F(A) = |A|$, we have $f(w) = w^\top 1_V$ and thus $\Omega_\infty(w) = \| w\|_1$ and we recover the $\ell_1$-norm, and the classical result that the $\ell_1$-norm $\|w\|_1$ is the convex envelope of the $\ell_0$-pseudo-norm $\| w\|_0 = | { \rm Supp}(w)|$.
 
 For the function $F(A) = \min\{ |A|,1\}$, then we have $f(w) = \max \{w_1,\dots,w_p\}$ and thus $\Omega_\infty(w) = \| w\|_\infty$. This norm is not sparsity-promoting and this is intuively natural since the set-function it corresponds to is constant for all non-empty sets.

We now consider the set-function counting elements in a partitions, i.e., we assume that 
$V$ is partitioned into $m$ sets $G_1,\dots,G_m$,  the function $F$ that counts for a set $A$ the number of elements in the partition which intersects $A$ may be written as $F(A) = \sum_{j=1}^m \min\{| A \cap G_j|,1\}$ and the norm as $\Omega_\infty(w) = \sum_{j=1}^m \| w_{G_j} \|_\infty$. This is the usual $\ell_1$/$\ell_\infty$-norm, a certain form of grouped penalty~\cite{negahban2008joint,yuanlin}. It is known to enforce sparsity at the group level, i.e., variables within a group are selected or discarded simultaneously. This is intuitively natural (and will be made more precise below) given the associated set-function, which, once a variable in a group is selected, does not add extra cost to selecting other variables from the same group.

\paragraph{Structured sparsity-inducing norms and dual balls.}
We now study in more details the properties of the function $\Omega_\infty: w \mapsto f(|w|)$ defined above through a relaxation argument. We first give conditions under which it is a norm, and derive the dual norm.

\begin{proposition} \textbf{(Norm and dual norm)}
\label{prop:norm}
Let $F$ be a submodular function such that $F(\varnothing)=0$ and $F$ is non-decreasing. The function
$\Omega_\infty: w \mapsto f(|w|)$ is a norm if and only if  the values of $F$ on all singletons is strictly positive. Then, the dual norm is equal to $\Omega_\infty^\ast(s)   = \max_{A \subseteq V , \ A \neq \varnothing} \frac{ |s|(A) }{F(A)}
 = \max_{A \subseteq V, \ A \neq \varnothing} \frac{ \| s_A\|_1 }{F(A)} $.
\end{proposition}
\begin{proof}
If $\Omega_\infty$ is a norm, then for all $k \in V$, $F(\{k\}) = \Omega_\infty(1_{\{k\}}) > 0$. Let us now assume that all
 the values of $F$ on all singletons is strictly positive. The positive homogeneity of $\Omega_\infty$ is a consequence of
 property (e) of Prop.~\ref{prop:lova}, while the triangle inequality is a consequence of the convexity of $f$.
Since $F$ is non-decreasing, for all $A \subseteq V$ such that $A \neq \varnothing$, $F(A) \geqslant \min_{k \in V} F(\{k\}) >0$.
From property (b) of Prop.~\ref{prop:lova}, for any $w \in \rb^p$, if $\Omega_\infty(w)=0$, then for all $z > 0$, $F( \{ |w| \geqslant z \})=0$, which implies that $\{ |w| \geqslant z \} = \varnothing$, i.e., $w = 0$. Thus $\Omega_\infty$ is a norm.

We can compute the dual norm
by noticing that for all $w \in \rb^p$,
\BEAS
\Omega_\infty(w) & = & \sup_{ s \in P(F)} s^\top |w| = \sup_{ s\in |P|(F) } s^\top w.
\EEAS
This implies that the unit ball of dual norm is the symmetric submodular polyhedron.
Since $|P|(F) = \{ s \in \rb^p, \ \forall A \subseteq V, \|s_A\|_1 \leqslant F(A) \}$, this implies that
the dual norm is equal to
 $\Omega_\infty^\ast(s)   = \max_{A \subseteq V, \ A \neq \varnothing} \frac{ |s|(A) }{F(A)}
 = \max_{A \subseteq V, \ A \neq \varnothing} \frac{ \| s_A\|_1 }{F(A)} $
 (see Appendix~\ref{app:convex} for more details on polar sets and dual norms).
\end{proof}

The dual norm can be computed efficiently from a sequence of submodular function minimizations (see \mysec{extensions}). Moreover, it may be written as
$ \Omega_\infty^\ast(s) = \max_{ w \in \{-1,0,1\}^p} \frac{ w^\top s} { F( \supp(w))}$.
Thus, the dual ball $|P|(F) = \{s \in \rb^p, \ \Omega_\infty^\ast(s) \leqslant 1 \}$ is naturally characterized by half planes of the form $ {w^\top s} \leqslant {F( {\rm Supp}(w))}   $ for $w \in \{-1,0,1\}^p$. Thus, the unit ball of~$\Omega_\infty$ is the convex hull of the vectors $\frac{1}{F( {\rm Supp}(w))}w$ for the same vectors $w \in \{-1,0,1\}^p$. See \myfig{balls-new} for examples for $p=2$ and \myfig{3dballs} for examples with $p=3$.

A particular feature of the unit ball of $\Omega_\infty$ is that it has faces which are composed of vectors with many zeros, leading to structured sparsity and specific sets of zeros, as shown below. However, as can be seen in Figures~\ref{fig:balls-new} and \ref{fig:3dballs}, there are additional extreme points and faces where many of the components of $|w|$ are equal (e.g., the corners of the $\ell_\infty$-ball). This is due to the fact that the \lova extension is piecewise linear with different linear parts when the orderings of components of $w$ are changing. 
In \mysec{l2}, we show how the extra clustering behavior may be corrected by removing $\ell_\infty$-effects, by the appropriate use of $\ell_q$-norms $q \in (1,\infty)$.
In \mysec{shaping} however, we show how this clustering sparsity-inducing effect may be used to design regularization terms that enforces specific level sets for a vector $w$.

\begin{figure}
\begin{center}

\vspace*{-.15cm}

\hspace*{-1.5cm}
\includegraphics[scale=.3]{ball_1_no.eps} 
\hspace*{-.1cm}
\includegraphics[scale=.3]{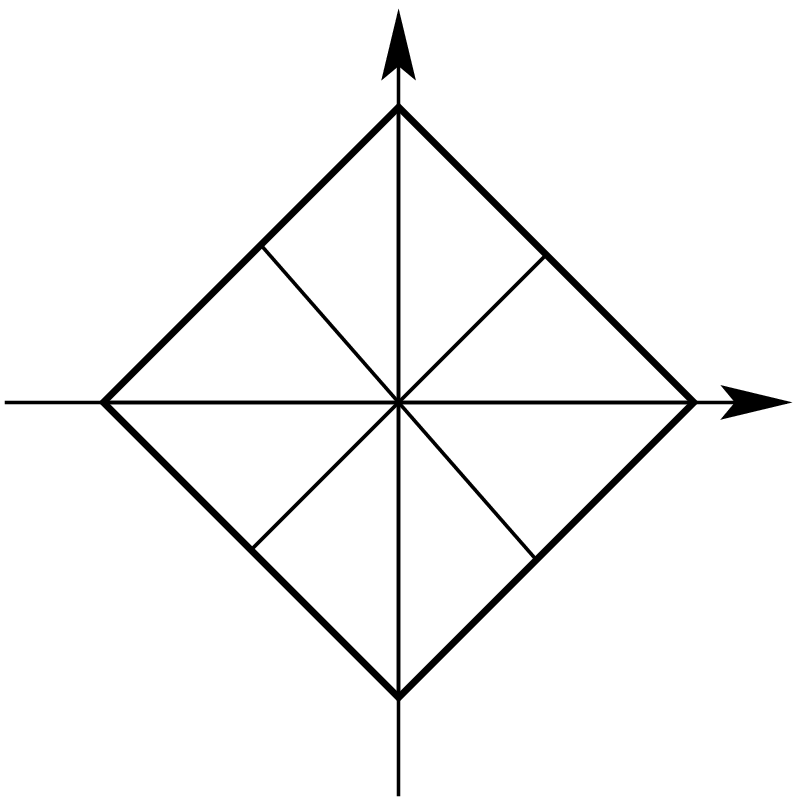}
\hspace*{-.05cm}
\includegraphics[scale=.3]{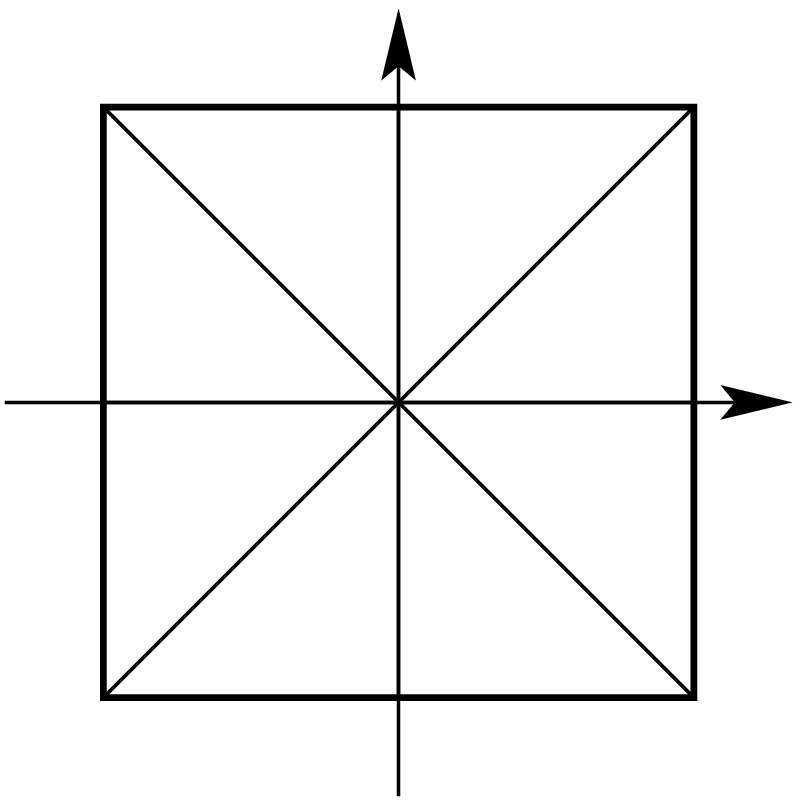}
\hspace*{-.05cm}
\includegraphics[scale=.3]{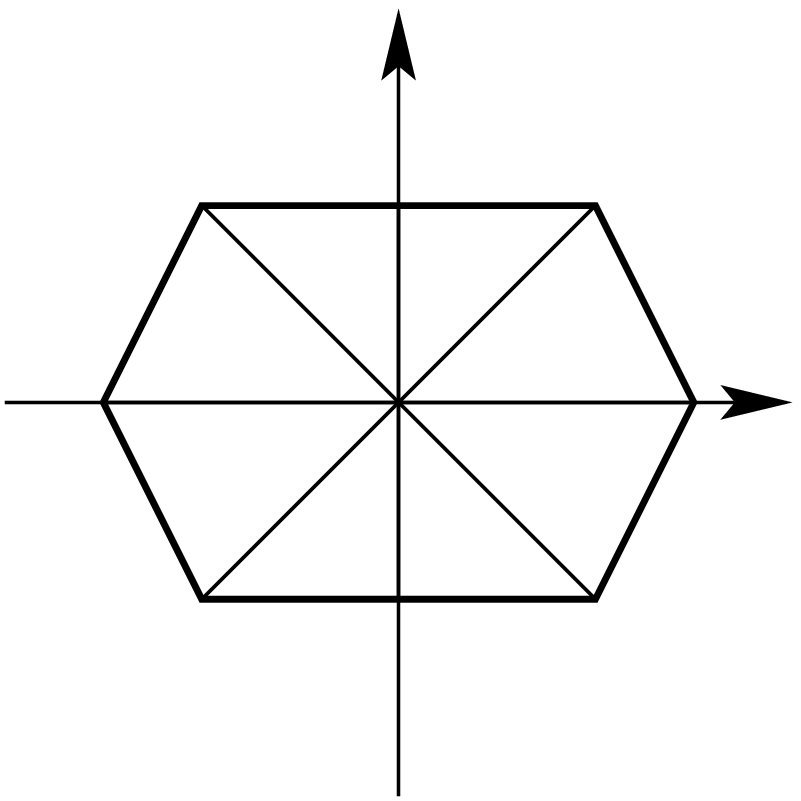} 
\hspace*{-.5cm}

\hspace*{-.5cm}
\includegraphics[scale=.3]{ball_1_no_poly.eps} 
\hspace*{-.5cm}
\includegraphics[scale=.3]{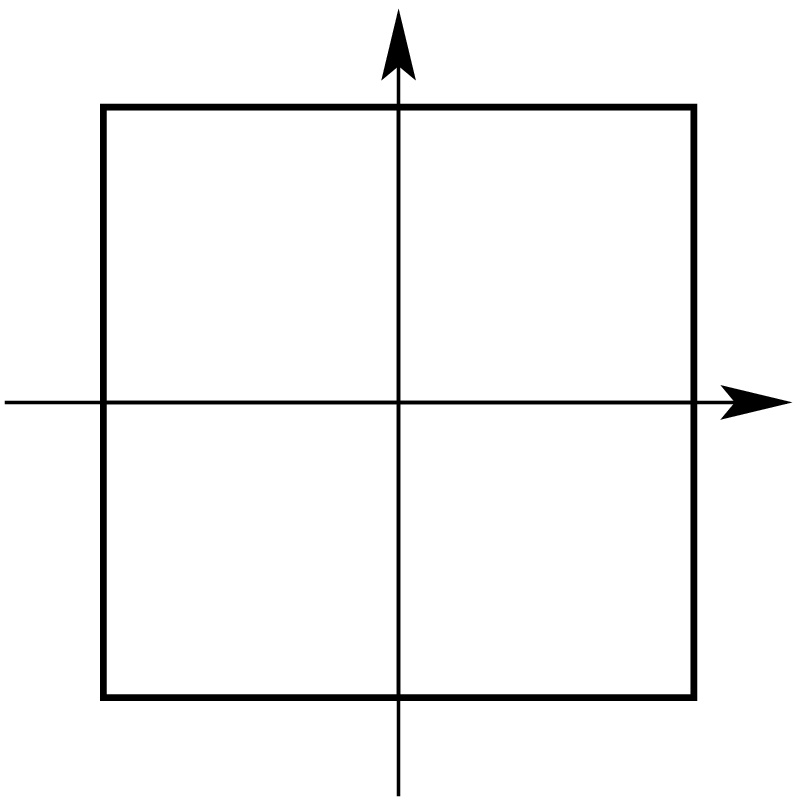}
\hspace*{-.05cm}
\includegraphics[scale=.3]{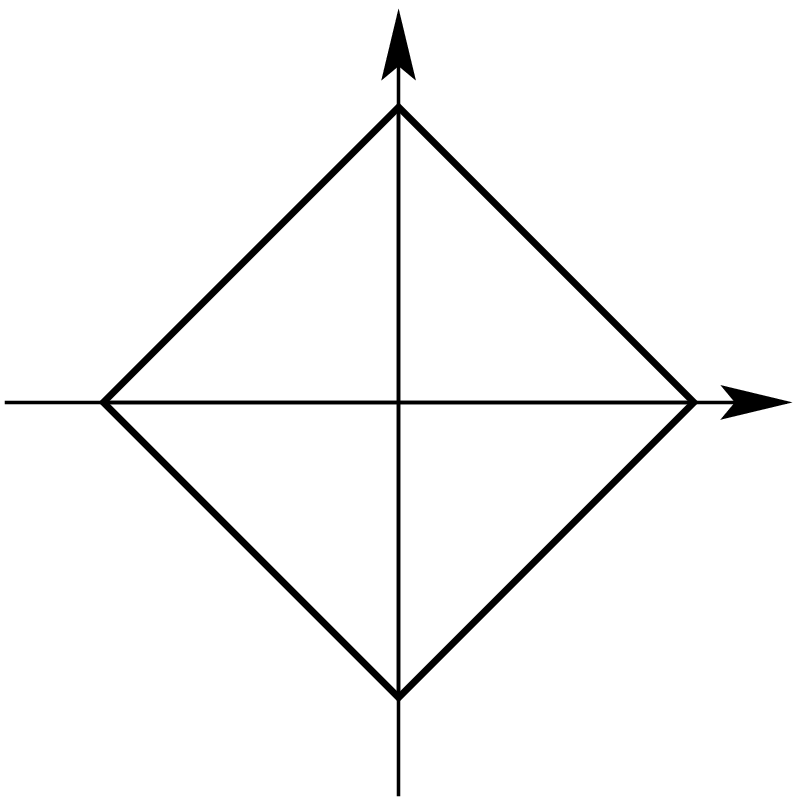}
\hspace*{-.05cm}
\includegraphics[scale=.3]{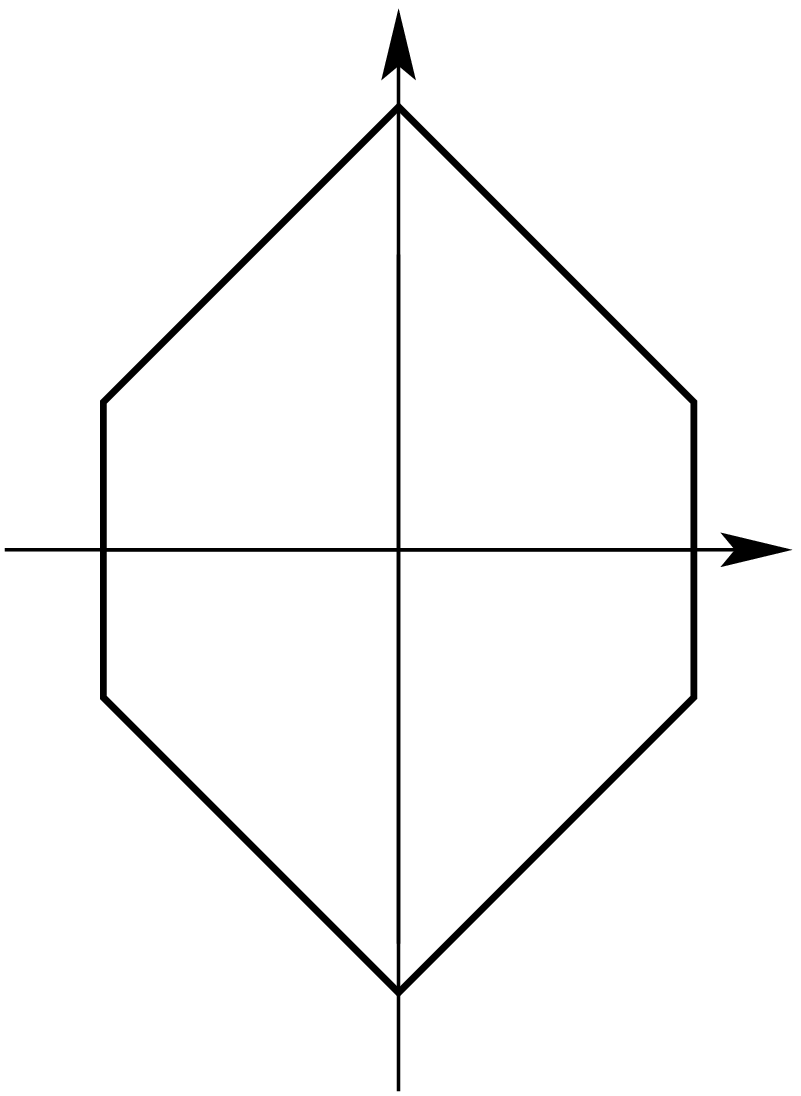} 
\hspace*{-.5cm}

\end{center}
\caption{Polyhedral unit ball of $\Omega_\infty$ (top) with the associated dual unit ball (bottom), for 4 different submodular functions (two variables), with different  sets of extreme points; changing values of $F$ may make some of the extreme points disappear (see the notion of stable sets in \mysec{faces-indep}). From left to right: $F(A) = |A|^{1/2}$ (all possible extreme points), $F(A) = |A|$ (leading to the $\ell_1$-norm),
$F(A) = \min\{|A|,1\}$ (leading to the $\ell_\infty$-norm), $F(A) =  \frac{1}{2} 1_{ \{A \cap \{2\} \neq \varnothing\} } +  1_{ \{A \neq \varnothing\} }  $
(leading to the  structured norm $\Omega_\infty(w) = \frac{1}{2}|w_2 | +   \| w\|_\infty$). Extreme points of the primal balls correspond to full-dimensional faces of the dual ball, and vice-versa.
}
\label{fig:balls-new}
\end{figure}

\begin{figure}

\begin{center}
\hspace*{-.5cm}
\parbox{3.5cm}{\centering \includegraphics[scale=.45]{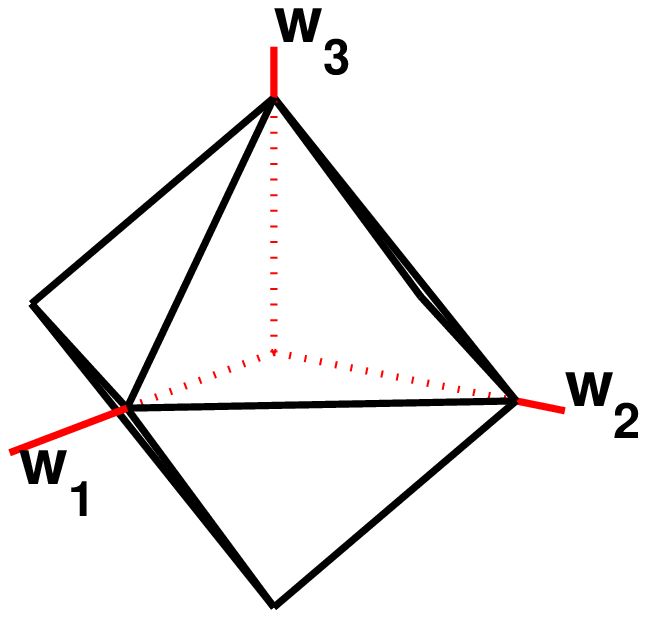} \\
\small $F(A) = |A|$  \\ $\Omega(w) = \|w\|_1$
}
\parbox{4cm}{\vspace*{1cm} 

\centering \small\includegraphics[scale=.3]{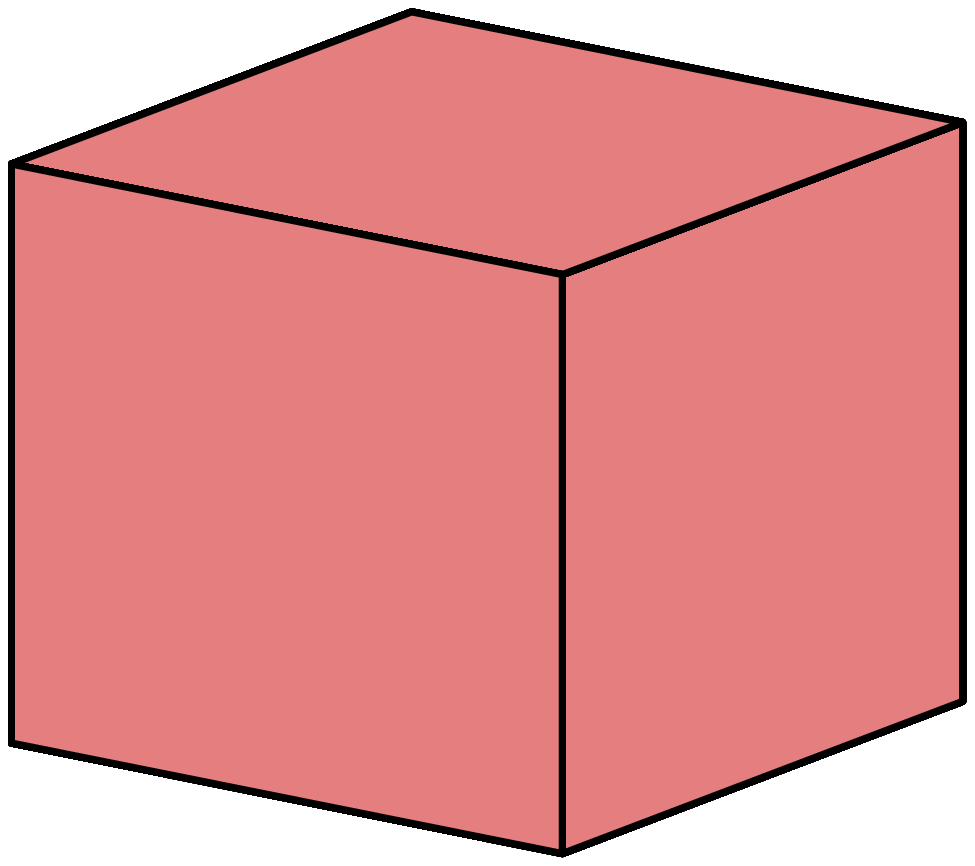} \\
$F(A) = \min\{|A|,1\}$ \\
$\Omega_\infty(w) = \| w\|_\infty$
}
\parbox{4cm}{\vspace*{1cm} 

\centering \small\includegraphics[scale=.38]{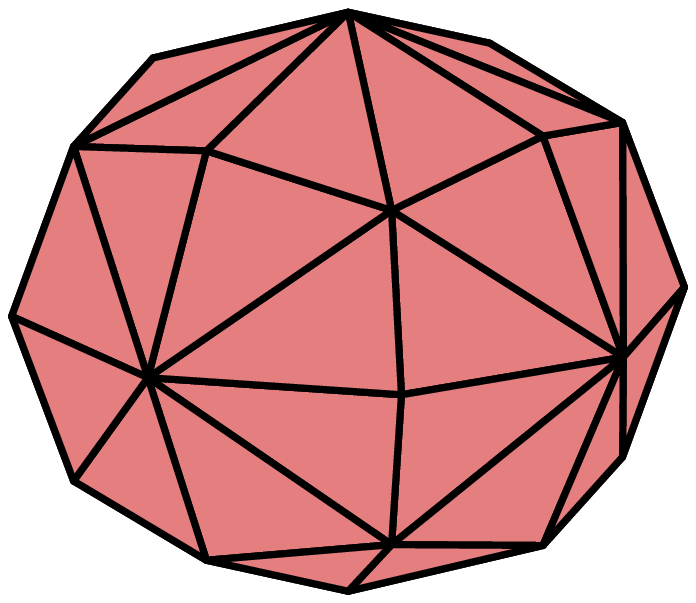} \\
$F(A) = |A|^{1/2}$ \\
all possible extreme points
}
\hspace*{-1cm}

\vspace*{.5cm}

\parbox{6cm}{\centering \includegraphics[scale=.35]{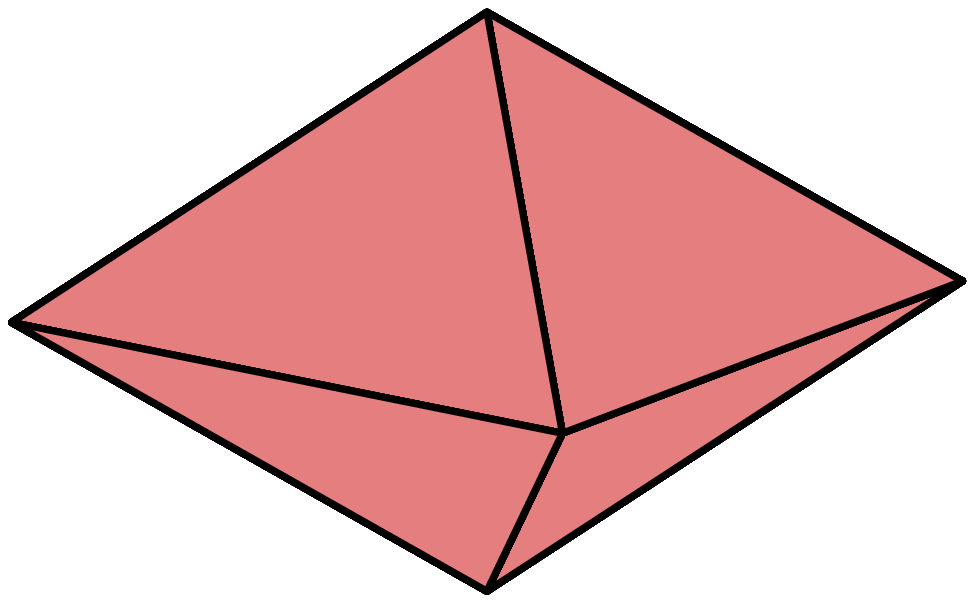} \\
\small
$F(A) =    1_{ \{A \cap \{3\} \neq \varnothing\} } + 1_{ \{A \cap \{1,2\} \neq \varnothing\} }  $ \\  $\Omega_\infty(w) = |w_13 +   \| w_{\{1,2\}} \|_\infty$
}
\parbox{6cm}{\centering \small \includegraphics[scale=.37]{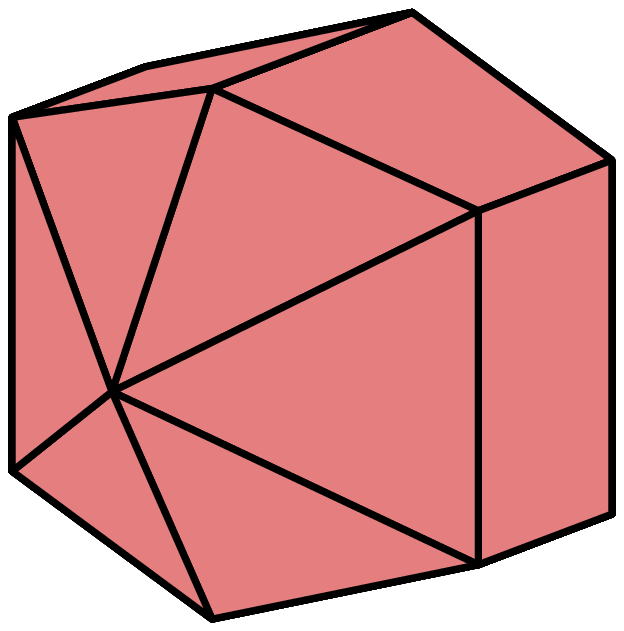} \\
$F(A) =    1_{ \{A \cap \{1,2,3\} \neq \varnothing\} } $ \hspace*{1.5cm} \\
\hspace*{1.5cm} $ + 1_{ \{A \cap \{2,3\} \neq \varnothing\} }  
+ 1_{ \{A \cap \{2\} \neq \varnothing\} }  $ \\[.2cm]
 $\Omega_\infty(w) = \|w \|_\infty + \| w_{\{2,3\}} \|_\infty + |w_2|$
}
\end{center}
  \caption{Unit balls for structured sparsity-inducing norms, with the corresponding submodular functions and the associated norm. \label{fig:3dballs} }

\end{figure}

\paragraph{Sparsity-inducing properties.}
 We now study and characterize which supports can be obtained when regularizing with the norm $\Omega_\infty$. Ideally, we would like the same behavior than $w \mapsto F(\supp(w))$, that is, when this function is used to regularize a continuous objective function, then a stable set is always a solution of the problem, as augmenting unstable sets does not increase the value of $F$, but can only increase the minimal value of the continuous objective function because of an extra variable to optimize upon. It turns out that the same property holds for $\Omega_\infty$ for certain optimization problems.

\begin{proposition} \textbf{(Stable sparsity patterns)}
\label{prop:patterns}
Assume $F$ is submodular and non-decreasing.
Assume  $y \in \rb^n$  has an absolutely continuous density with respect to
the Lebesgue measure and that $X^\top X  \in \rb^{p \times p}$ is invertible. Then the minimizer $\hat{w}$ of 
$ \frac{1}{2n} \| y - X w \|_2^2 + \Omega_\infty(w)$ is unique and, 
 with probability one, its support
$ \supp(\hat{w})$ is a stable set.
\end{proposition}
\begin{proof}
We provide a proof that use Prop.~\ref{prop:optsupporttight-indep} that characterizes maximizers of $s^\top w$ over $s \in |P|(F)$. For an alternative proof based on convex duality, see~\cite{bach2010structured}.
Following the proof of Prop.~\ref{prop:facesNORM}, any $w \in \rb^p$, with support $C$, can be decomposed into $C = A_1 \cup \cdots \cup A_m$ following the strictly decreasing sequence of constant sets of $|w|$. Denote by $\varepsilon$ the sign vector  of $w$. Given $w$, the support $C$, the constant sets $A_i$, and the sign vector are uniquely defined. 

The optimality condition for our optimization problem is the existence of $s \in |P|(F)$ such that
$\frac{1}{n} X^\top X w - \frac{1}{n} X^\top y + s = 0$ and $s$ is a maximizer of $s^\top w$ over $s \in
|P|(F)$. Thus, if the unique solution $w$ has support $C$, and ordered constant sets $(A_i)_{i=1,\dots,m}$ and sign vector $\varepsilon$, then, from Prop.~\ref{prop:optsupporttight-indep}, there exists $v_1>\cdots >v_m$, such that $w = \sum_{i=1}^m v_i 1_{A_i}$, and
\BEAS
&&\frac{1}{n} X^\top X w - \frac{1}{n} X^\top y + s = 0 \\
&&\forall i \in \{1,\dots,m\}, \ (\varepsilon \circ s)(A_1 \cup \cdots \cup A_i) = F(A_1 \cup \cdots \cup A_i) .
\EEAS
Without loss of generality, we assume that $\varepsilon_C \geqslant 0$.
Denoting by $\tilde{X}$ the matrix in $\rb^{n \times m}$ with columns $X 1_{A_i}$, then we have
$ \tilde{X}^\top \tilde{X} v - \tilde{X}^\top y + n \Delta = 0$, where $\Delta_j = F(A_1 \cup \cdots \cup A_j) - F(A_1 \cup \cdots \cup A_{j-1})$, and thus
$v = ( \tilde{X}^\top \tilde{X}  )^{-1} \tilde{X}^\top y - n  ( \tilde{X}^\top \tilde{X}  )^{-1} \Delta$. For any $k \in V \backslash C$, we have
\BEAS
 s_k & = & \frac{1}{n} X_k^\top y - \frac{1}{n} X_k^\top \tilde{X} v \\
 & = & 
\frac{1}{n} X_k^\top y  - \frac{1}{n} X_k^\top \tilde{X} ( \tilde{X}^\top \tilde{X}  )^{-1} \tilde{X}^\top y
+ X_k^\top \tilde{X}   ( \tilde{X}^\top \tilde{X}  )^{-1} \Delta.
\EEAS
For any $k \in V \backslash C$, since $X$ has full column-rank,
the vector $ X_k   -     \tilde{X} ( \tilde{X}^\top \tilde{X}  )^{-1} \tilde{X}^\top X_k$ is not equal to zero 
(as this is the orthogonal projection of the column $X_k$ on the orthogonal of the columns of $\tilde{X}$). If we further assume that $C$ is not stable, following the same reasoning than for the proof
of Prop.~\ref{prop:facesNORM}, there exists $k$ such that $s_k=0$. This implies that there exists a non zero vector $c$ and a real number $d$, chosen from a finite set, such that $c^\top y = d$. Thus, since
$y \in \rb^n$  has an absolutely continuous density with respect to
the Lebesgue measure, the support of $w$ is not stable with probability zero.
\end{proof}
For simplicity, we have assumed invertibility of $X^\top X$, which forbids the high-dimensional situation $p\geqslant n$, but we could extend to the type of assumptions used in~\cite{jenatton2009structured} or to more general smooth losses. The last proposition shows that we get almost surely stable sparsity patterns; in~\cite{bach2010structured}, in context where the data $y$ are generated as sparse linear combinations of columns of $X$ with additional noise, sufficient conditions for the recovery of the support are derived, potentially in high-dimensional settings where $p$ is larger than $n$, extending the classical results for the Lasso~\cite{Zhaoyu,negahban2009unified}, and the group Lasso~\cite{negahban2008joint}.

\paragraph{Optimization for regularized risk minimization.} 
Given the representation of $\Omega_\infty$ as the maximum of linear functions from Prop.~\ref{prop:greedy-indep}, i.e., $\Omega_\infty(w) = \max_{s \in |P|(F)} w^\top s$, we can easily obtain a subgradient of $\Omega_\infty$ as any maximizer $s$, thus allowing the use of subgradient descent techniques (see \mysec{subgradient}). However, these methods typically require many iterations, with a convergence rate of $O(t^{-1/2})$ after $t$ iterations\footnote{By convergence rate, we mean the function values after $t$ iterations minus the optimal value of the minimization problem.}.
Given the structure of our norms, more efficient methods are available: we describe in \mysec{prox-opt} \emph{proximal methods}, which generalizes soft-thresholding algorithms for the $\ell_1$-norm and grouped $\ell_1$-norm, and can use efficiently the combinatorial structure of the norms, with convergence rates of the form $O(t^{-2})$.

%
%
%
%
%
%
%

\section{$\ell_q$-relaxations of submodular penalties$^\ast$}
\label{sec:l2}
\label{sec:lprelax}

As can be seen in \myfig{balls-new} and \myfig{3dballs}, there are some extra effects due to additional extreme points away from sparsity-inducing corners. Another illustration of the same issue may
be obtained for $F(A) = \sum_{i=1}^m \min\{|A_i|,1\}$, which leads to the so-called
$\ell_1/\ell_\infty$-norm, which may have some undesired effects when used as a regularizer~\cite{negahban2008joint}. In this section, our goal is to design a convex relaxation framework so that the corresponding norm for the example above is the $\ell_1/\ell_q$-norm, for $q \in (1,\infty)$.

We thus consider $q \in (1,+\infty)$ and $r \in (1,+\infty)$ such that $1/q+1/r=1$. It is well known that the $\ell_q$-norm and the $\ell_r$-norm are dual to each other. In this section, we assume that $F$ is a non-decreasing  function such that $F(\{k\})>0$ for all $k \in V$ (not always necessarily submodular). 

Following~\cite{submodlp}, we consider a function $g: \rb^p \to \rb$ that penalizes both supports and $\ell_q$-norm on the supports, i.e., 
$$
g(w) = \frac{1}{q} \| w\|_q^q + \frac{1}{r} F( \supp(w)).
$$
Note that when $q$ tends to infinity, this function tends to $F( \supp(w))$ restricted to the $\ell_\infty$-ball (i.e., equal to $+\infty$ outside of the $\ell_\infty$-ball), whose convex envelope is $\Omega_\infty$. The next proposition applies to all $q \in (1,+\infty)$ and  even when $F$ is not submodular.

\begin{proposition} \textbf{(convex relaxation of $\ell_q$-norm based penalty)}
\label{prop:relaxl2}
Let $F$ be a non-decreasing   function such that $F(\{k\})>0$ for all $k \in V$. The tightest convex homogeneous lower-bound of $g: w \mapsto \frac{1}{q} \| w\|_q^q + \frac{1}{p} F( \supp(w))$ is a norm, denoted 
$\Omega_q$,  such that its dual norm is equal to, for $s \in \rb^p$,
\BEQ
\label{eq:l2}
\Omega_q^\ast(s) = \sup_{A \subseteq V, \ A \neq \varnothing} \frac{ \| s_A\|_r } {F(A)^{1/r}}.
\EEQ
\end{proposition}
\begin{proof}
We first show that $\Omega_q^\ast$ defined in \eq{l2} is a norm. It is immediately finite, positively homogeneous and convex; if $\Omega^\ast(s)=0$, then $\|s_V\|_r = 0$, and thus $s=0$; thus $\Omega_q^\ast$ is a norm. We denote by $\Omega_q$ its dual norm.

In order to find the tightest convex homogeneous lower bound of a convex function $g$, it suffices to (a) first compute the homogeneized version of $g$, i.e., $h$ defined as $h(w) = \inf_{ \lambda > 0 } \frac{ g(\lambda w)}{ \lambda}$, and (b) compute its Fenchel bi-dual~\cite{rockafellar97}. In this particular case, we have, for $ w \in \rb^p$:
\BEAS
h(w) & = & \inf_{\lambda > 0 } \frac{1}{q} \| w\|_q^q \lambda^{q-1} + \frac{1}{r} F( \supp(w)) \lambda^{-1}.
\EEAS
Minimizing with respect to $\lambda$ may be done by setting to zero the derivative of this convex function of $\lambda$, leading to $\frac{1}{q} \| w\|_q^q (q-1) \lambda^{q-2} - \frac{1}{r} F( \supp(w)) \lambda^{-2}=0$, i.e.,
$\lambda =\big( \frac{ \frac{1}{r} F( \supp(w)) }{\frac{1}{q} \| w\|_q^q (q-1) } \big)^{1/q}$, and an optimal value of
\BEAS
h(w) & = &  \| w\|_q F( \supp(w)) ^{1/r} .
\EEAS
We then have, for $s \in \rb^p$,
\BEAS
h^\ast(s) & = & \sup_{ w \in \rb^p} s^\top w - \| w\|_q F( \supp(w)) ^{1/r}  \\
& = & \max_{A \subseteq V} \sup_{ w \in \rb^p, \ \supp(w) = A} s^\top w - \| w\|_q F( A) ^{1/r} \\
& = & \max_{A \subseteq V} \bigg\{
\begin{array}{l}
0 \mbox{ if } \|s_A\|_r^r \leqslant F(A), \\
+\infty \mbox{ otherwise.} 
\end{array}
\EEAS
Thus $h^\ast(s) = 0$ if  $\Omega_q^\ast(s) \leqslant 1$ and $+\infty$ otherwise. Thus the Fenchel-conjugate of $h$ is the indicator function of the ball of the norm $\Omega_q^\ast$ (which is a norm as shown above); this implies that $h = \Omega_q$ (and hence it is a norm).
\end{proof}
Note that for a submodular function, as detailed at the end of \mysec{faces-indep}, only stable inseparable sets may be kept in the definition of $\Omega_q^\ast$ in \eq{l2}. Moreover, by taking a limit when
$q $ tends to $+\infty$, we recover the norm $\Omega_\infty$ from \mysec{increasing}.

While the convex relaxation is defined through its dual norm, we can give a variational primal formulation for submodular functions. It corresponds to usual reweighted squared-$\ell_2$ formulations for certain norms referred to as subquadratic~\cite{fot}, a classical example being the $\ell_1$-norm equal to
$\|w \|_1 = \inf_{ \eta \geqslant 0} \frac{1}{2} \sum_{k \in V} \frac{ |w_k|^2}{\eta_k} + \frac{1}{2} \eta^\top 1_V$. These representations may be used within alternative optimization frameworks (for more details, see~\cite{fot} and references therein).

\begin{figure}
\begin{center}

\hspace*{-1.5cm}
\includegraphics[scale=.32]{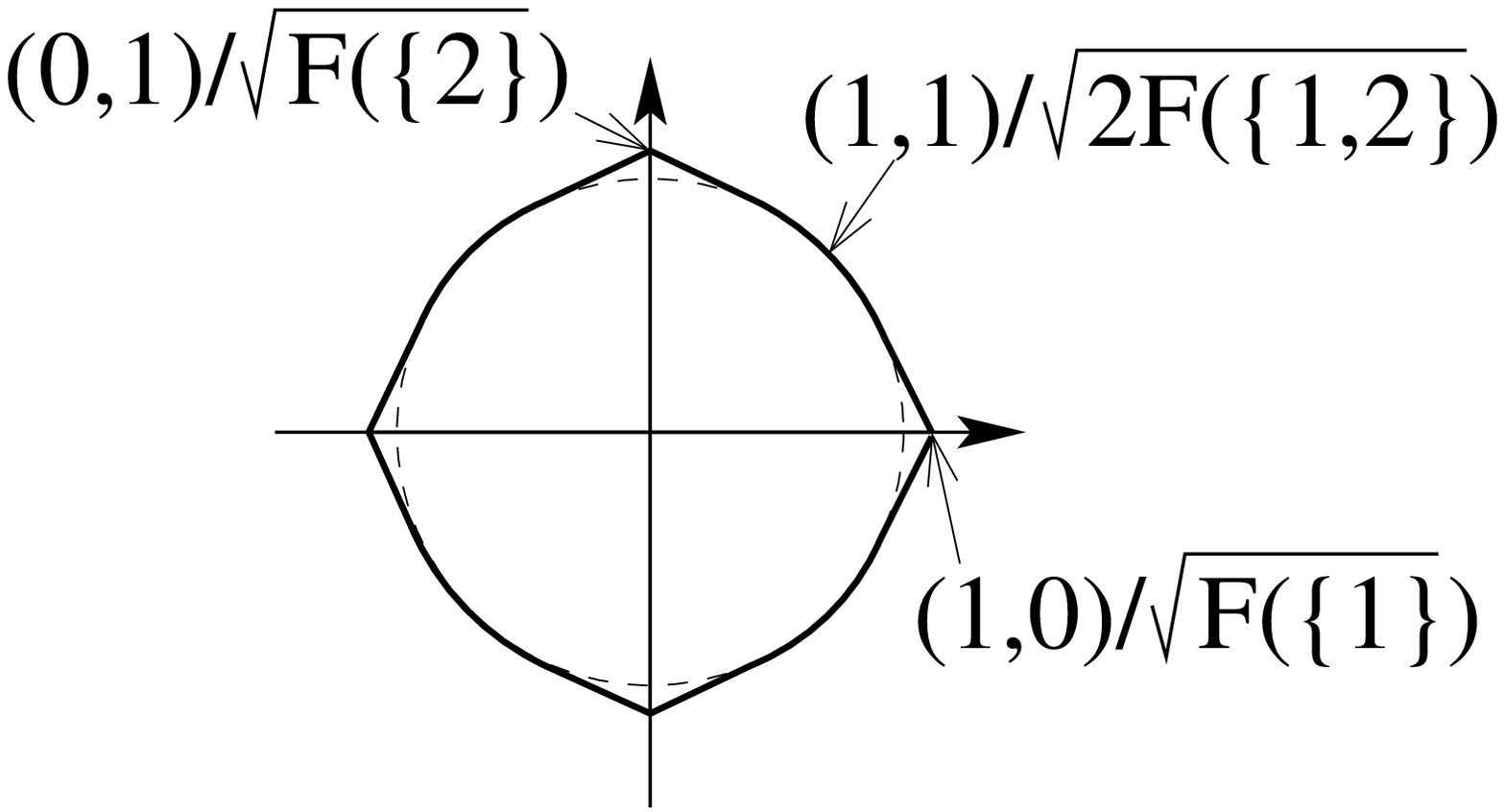} 
\hspace*{-1.15cm}
\includegraphics[scale=.32]{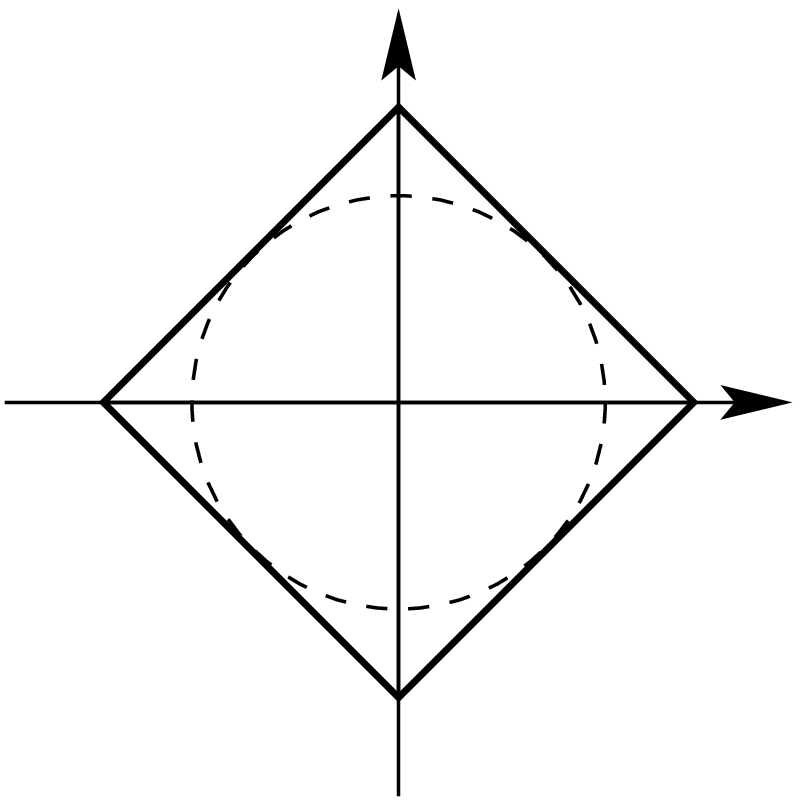}
\hspace*{-.05cm}
\includegraphics[scale=.32]{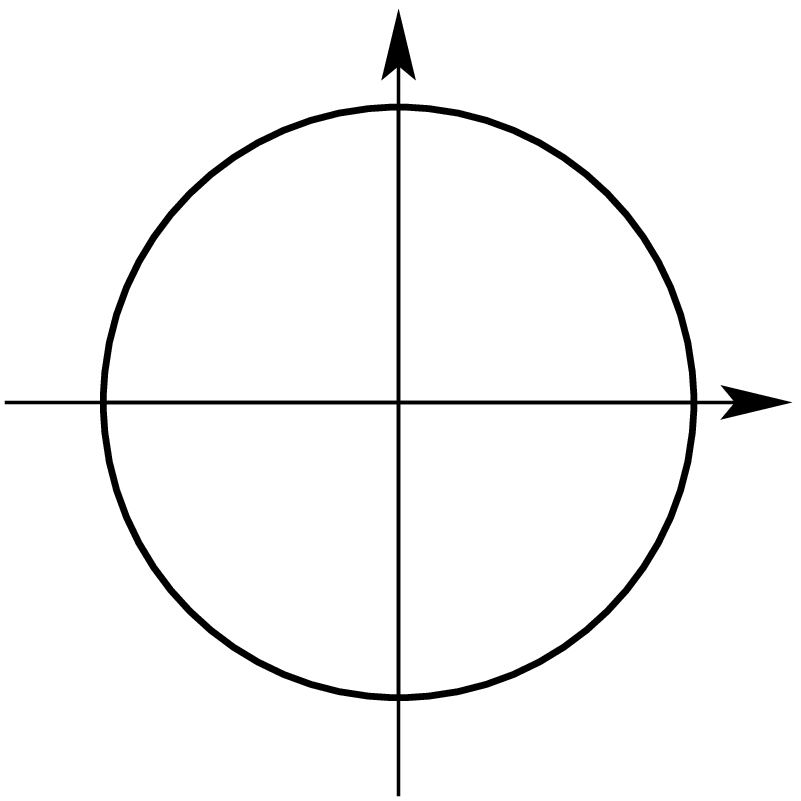}
\hspace*{-.05cm}
\includegraphics[scale=.32]{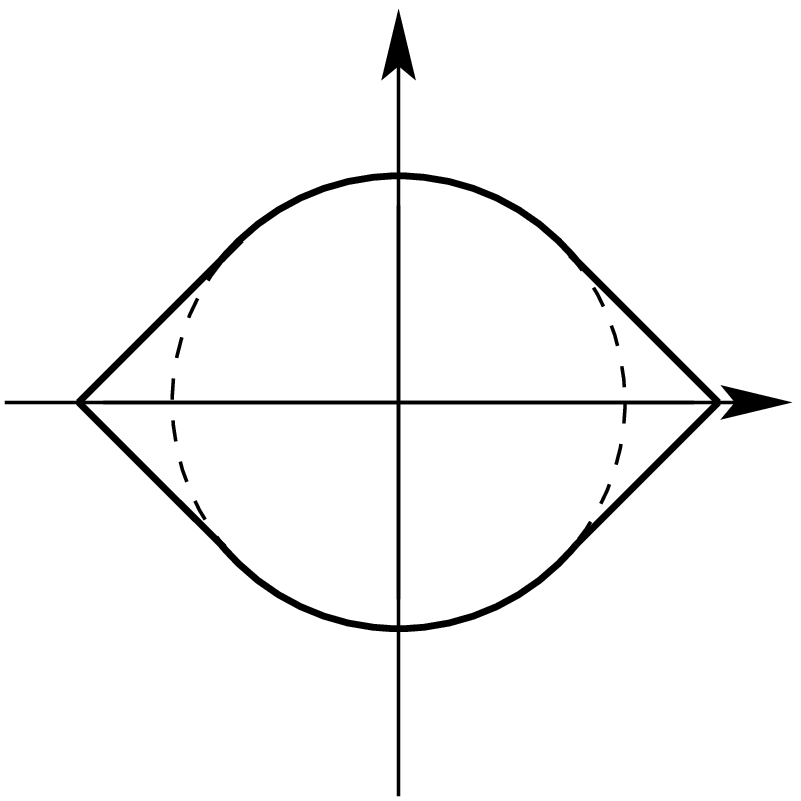} 
\hspace*{-.5cm}

\hspace*{-1.5cm}
\includegraphics[scale=.32]{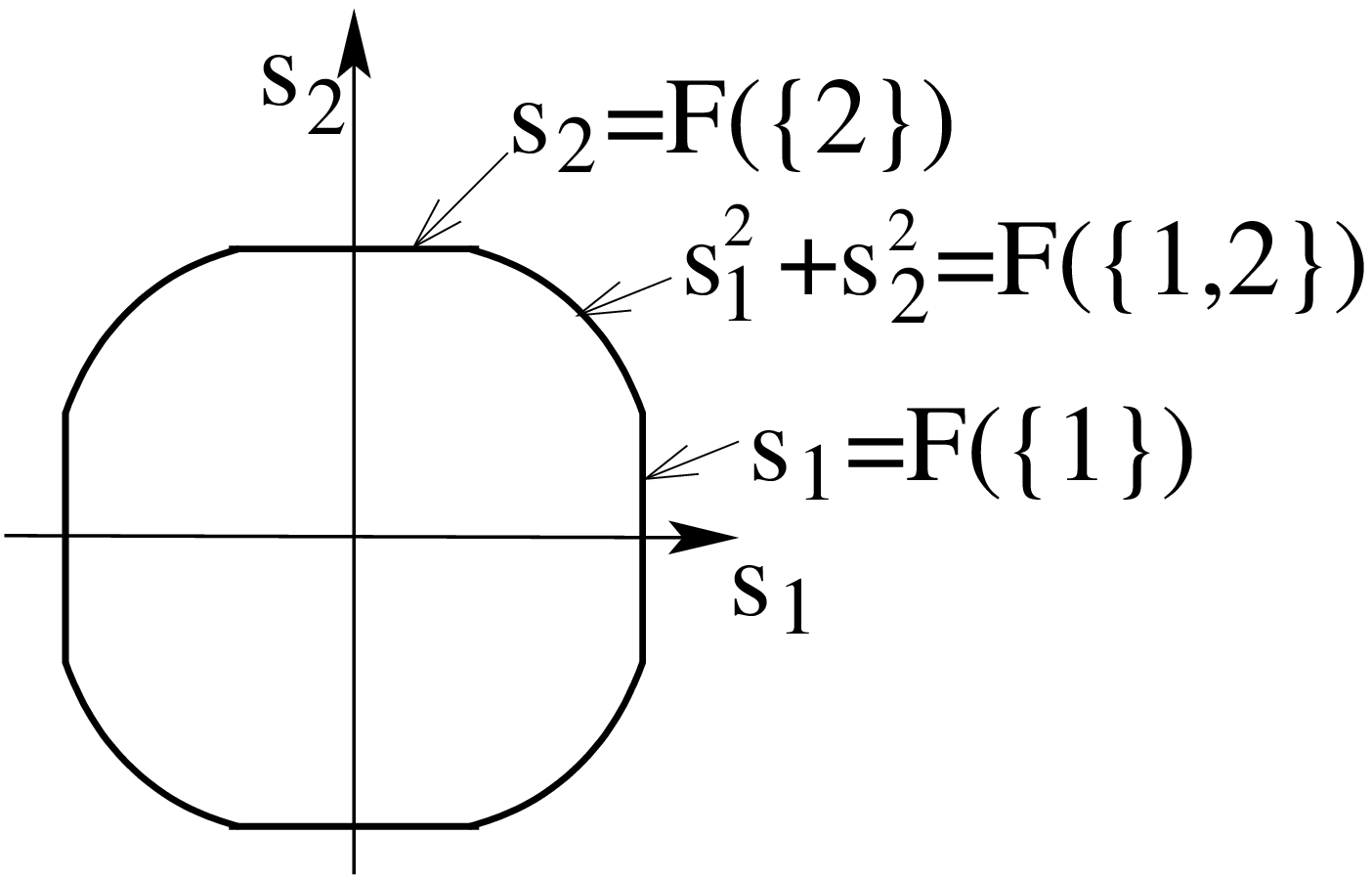} 
\hspace*{-.5cm}
\includegraphics[scale=.3]{ball_4_no_new.eps}
\hspace*{-.05cm}
\includegraphics[scale=.32]{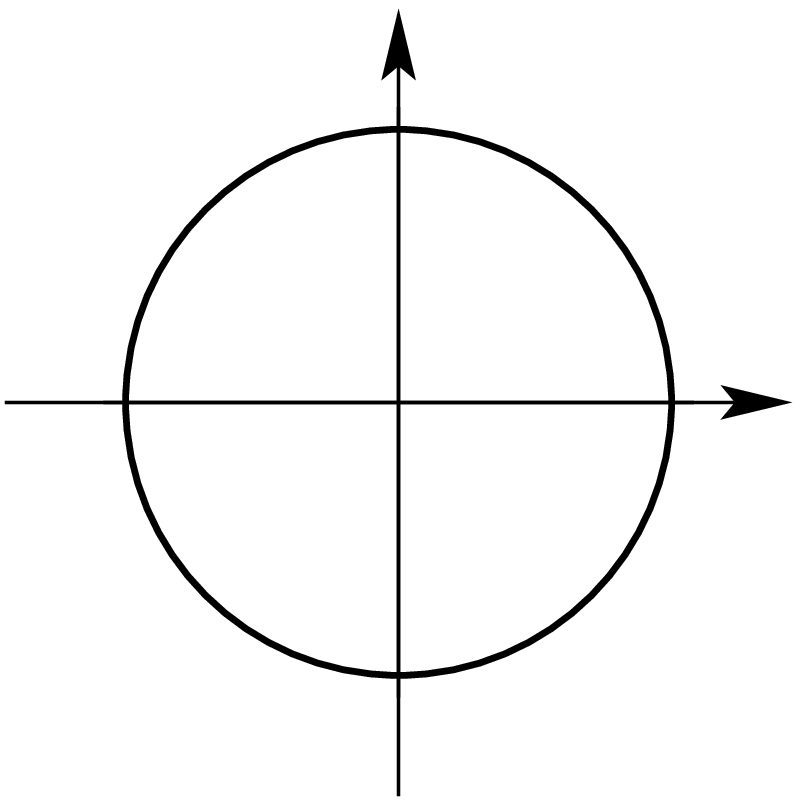}
\hspace*{-.05cm}
\includegraphics[scale=.32]{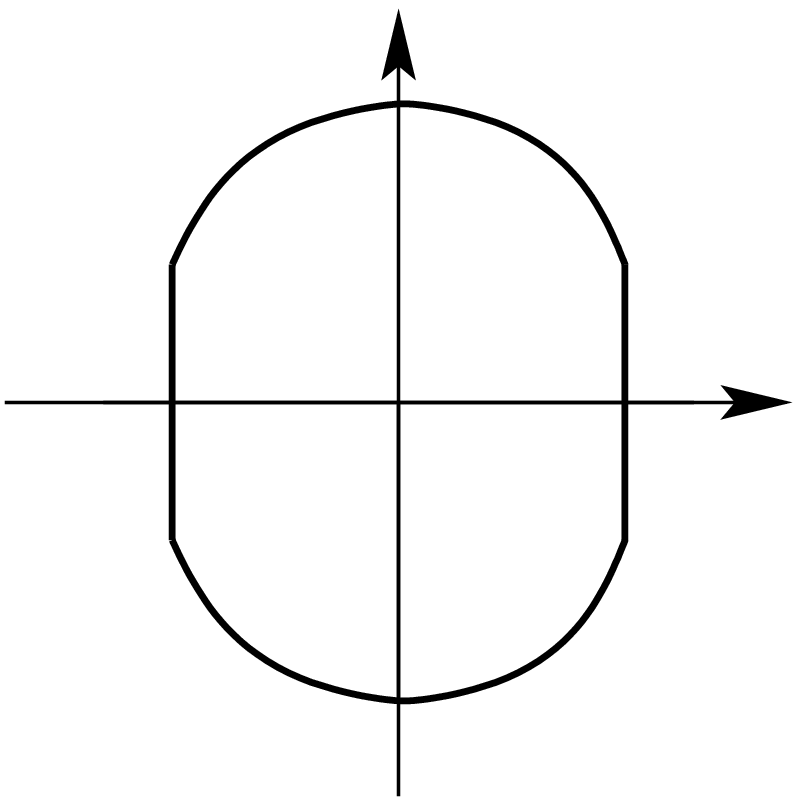} 
\hspace*{-.5cm}

\vspace*{-.6cm}

\end{center}
\caption{Polyhedral unit ball of $\Omega_2$ (top) with the associated dual unit ball (bottom), for 4 different submodular functions (two variables), with different  sets of extreme points; changing values of $F$ may make some of the extreme points disappear (see the notion of stable sets in \mysec{faces-indep}). From left to right: $F(A) = |A|^{1/2}$ (all possible extreme points), $F(A) = |A|$ (leading to the $\ell_1$-norm),
$F(A) = \min\{|A|,1\}$ (leading to the $\ell_2$-norm), $F(A) =  \frac{1}{2} 1_{ \{A \cap \{2\} \neq \varnothing\} } +  1_{ \{A \neq \varnothing\} }  $. 
}
\label{fig:new}
\end{figure}

\begin{proposition} \textbf{(variational formulation of $\Omega_q$)}
Let $F$ be a non-decreasing submodular function such that $F(\{k\})>0$ for all $k \in V$.  The norm defined in Prop.~\ref{prop:relaxl2} satisfies for $w \in \rb^p$,
\BEQ
\Omega_q^\ast(w) = \inf_{ \eta \geqslant 0}  \frac{1}{q}  \sum_{k \in V}
 \frac{|w_k|^q}{  \eta_k^{q-1}} + \frac{1}{r} f(\eta) ,
\EEQ
 using the usual convention that that $\frac{ |w_k|^q }{\eta_k^{q-1}}$ is equal to zero as soon as $w_k=0$, and equal to $+\infty$ if $w_k \neq 0$ and $\eta_k=0$.
\end{proposition}
\begin{proof}
We have, for any $ w \in \rb^p$, by definition of dual norms:
\BEAS
\Omega_q(w) & = & \sup_{s \in \rb^p} w^\top s \mbox{ such that } \forall A\subseteq V, \ \|s_A\|_r^r \leqslant F(A) \\
 & = & \sup_{s \in \rb_+^p} |w|^\top s \mbox{ such that } \forall A\subseteq V, \ \|s_A\|_r^r \leqslant F(A) \\[-.15cm]
 & = & \sup_{t \in \rb_+^p} \sum_{k \in V} |w_k| t_k^{1/r} \mbox{ such that } \forall A\subseteq V, t(A) \leqslant F(A)
 \\[-.15cm]
 & = & \sup_{ t \in P_+(F) } \sum_{k \in V} |w_k| t_k^{1/r} ,
\EEAS
using the change of variable $t_k = s_k^r, \ k \in V$. We can now use the identity
$
|w_k| t_k^{1/r}  = \inf_{ \eta_k \geqslant 0} \frac{ \eta_k t_k}{r} + \frac{|w_k|^q}{q \eta_k^{q-1}}
$ (with solution $\eta_k = |w_k| t_k^{-1/q}$), to get
$\displaystyle 
\Omega_q(w) =  \sup_{t \in P_+(F)}  \inf_{ \eta \geqslant 0} \sum_{k \in V}
\bigg\{
\frac{ \eta_k t_k}{r} + \frac{|w_k|^q}{q \eta_k^{q-1}}
\bigg\}$, which is equal to 
$\inf_{ \eta \geqslant 0}  \frac{1}{q}  \sum_{k \in V}
 \frac{|w_k|^q}{  \eta_k^{q-1}} + \frac{1}{r} f(\eta) 
$ using Prop.~\ref{prop:greedy}.

Note that a similar result holds for any set-function for an appropriately defined function $f$. However,
such $f$ cannot be computed in closed form in general (see~\cite{submodlp} for more details).
\end{proof}
In \mysec{graph}, we provide simulation experiments to show benefits of $\ell_2$-relaxations over $\ell_\infty$-relaxations, when the extra clustering behavior is not desired.

 \paragraph{Lasso and group Lasso as special cases.}
 
 For the cardinality function $F(A) = |A|$, we have $\Omega_q^\ast(s) = \| s\|_\infty$ and thus $\Omega_q(w) = \| w\|_1$, for all values of $q$, and we recover the $\ell_1$-norm. This shows an interesting result that the $\ell_1$-norm is the homogeneous convex envelope of the sum of the $\ell_0$-pseudo-norm $\| w\|_0 = | { \rm Supp}(w)|$ and an $\ell_q$-norm, i.e., it has a joint effect of sparsity-promotion and regularization.
 
 For the function $F(A) = \min\{ |A|,1\}$, then we have $f(w) = \max \{w_1,\dots,w_p\}$ and thus $\Omega_q(w) = \| w\|_q$. For $q>1$, this norm is not sparsity-promoting and this is intuively natural since the set-function it corresponds to is constant for all non-empty sets.

We now consider the set-function counting elements in a partitions, i.e., we assume that 
$V$ is partitioned into $m$ sets $G_1,\dots,G_m$,  the function $F$ that counts for a set $A$ the number of elements in the partition which intersects $A$ may be written as $F(A) = \sum_{j=1}^m \min\{| A \cap G_j|,1\}$ and the norm as $\Omega_q(w) = \sum_{j=1}^m \| w_{G_j} \|_q$, which was our original goal.

\paragraph{Latent formulations and primal unit balls.}
The dual norm $\Omega_q^\ast(s)$ is a maximum of the $2^p-1$ functions
$g_A(s) = F(A)^{-1/r}\| s_A\|_r = \max_{ w \in K_A } w^\top s$, for  $A$ a non-empty subset of $V$, where
$K_A = \{ w \in \rb^p, \ \| w \|_q F(A)^{1/r} \leqslant 1, \ \supp(w) \subseteq A  \big\}$ is an $\ell_q$-ball restricted to the support $A$.

This implies that 
$$
\Omega_q^\ast(s) = \max_{ w \in \bigcup_{A \subseteq V}  K_A } w^\top s.
$$
This in turn implies that the unit ball of $\Omega_q$ is the convex hull of the union of all sets $K_A$. This provides an additional way of constructing the unit primal balls. See illustrations in \myfig{new} for $p=2$ and $q=2$, and in \myfig{3dballsL2} for $p=3$ and $q=2$. 

Moreover, we have, for any $w \in \rb^p$,
  $\Omega_q(w) \leqslant 1$ if and only if $w$ may be decomposed as
$$
w = \sum_{ A \subseteq V} \lambda_A v^A,
$$
where $\supp(v^A) = A$ and $\| v^A \|_q F(A)^{1/r} \leqslant 1$. This implies
that 
$$
\Omega_q(w) = \min_{ w = \sum_{A \subseteq V} v^A, \ \supp(v^A) \subseteq A } \sum_{A \subseteq V} F(A)^{1/r} \| v^A \|_q,
$$
i.e., we have a general instance of a latent group Lasso (for more details, see~\cite{LaurentGuillaumeGroupLasso,submodlp}).

\paragraph{Sparsity-inducing properties.}
As seen above, the unit ball of~$\Omega_q$ may be seen as a convex hull of $\ell_q$-balls restricted to any possible supports. Some of this supports will lead to singularities, some will not. While formal statements like Prop.~\ref{prop:patterns} could be made (see, e.g.,~\cite{jenatton2009structured,submodlp}), we only give here informal arguments justifying why the only stable supports corresponds to stable sets, with no other singularities (in particular, there are no extra clustering effects).

Consider $q \in (1,\infty)$ and $w \in \rb^p$, with support $C$.  We study the set of maximizers of $s^\top w$ such that $\Omega_q^\ast(s) \leqslant 1$: $s$ may be obtained by maximizing 
$\sum_{k \in V} |w_k| t_k^{1/r}$ with respect to $t \in P_+(f)$ and choosing $s_k$ such that $s_k w_k \geqslant 0$ and $|s_k|  = t_k^{1/r}$. Since $\sum_{k \in V} |w_k| t_k^{1/r}
= \sum_{k \in C} |w_k| t_k^{1/r}$, $t_{V \backslash C}$ is not appearing in the objective function, and this objective function is strictly increasing with respect to each $t_k$; thus, solutions $t$ must satisfy $t_C \in B(F_C) \cap \rb_+^C$ and $t_{V \backslash C} \in P_+(F^C)$, where $F_C$ and $F^C$ are the restriction and contraction of $F$ on $C$. Moreover, $t_C$ is the maximizer of a strictly convex function (when $r>1$), and hence is unique. The discussion above may be summarized as follows: the maximizers $s$  are uniquely defined on the support $C$, while they are equal (up to their signs) to component-wise powers of elements of 
$P_+(F^C)$. This implies that (a) $\Omega_q$ is differentiable at any $w$ with no zero components (i.e., lack of singularities away from zeros), and   that (b), with arguments similar than in the proof of Prop.~\ref{prop:patterns}, stable sets are the only patterns that can be attained (requiring however the use of the implicit function theorem, like in~\cite{jenatton2009structured}).

\paragraph{Computation of norm and optimization.}
While $\Omega_\infty$ may be computed in closed form, this is not the case for $q < +\infty$. In \mysec{extensions}, we present a divide-and-conquer algorithms for computing~$\Omega_q$. We also show in that section, how the proximal operator may be computed from a sequence of submodular function minimizations, thus allowing the use of proximal methods presented in \mychap{prox}.

\begin{figure}

\begin{center}
\parbox{6.5cm}{\centering \includegraphics[scale=.45]{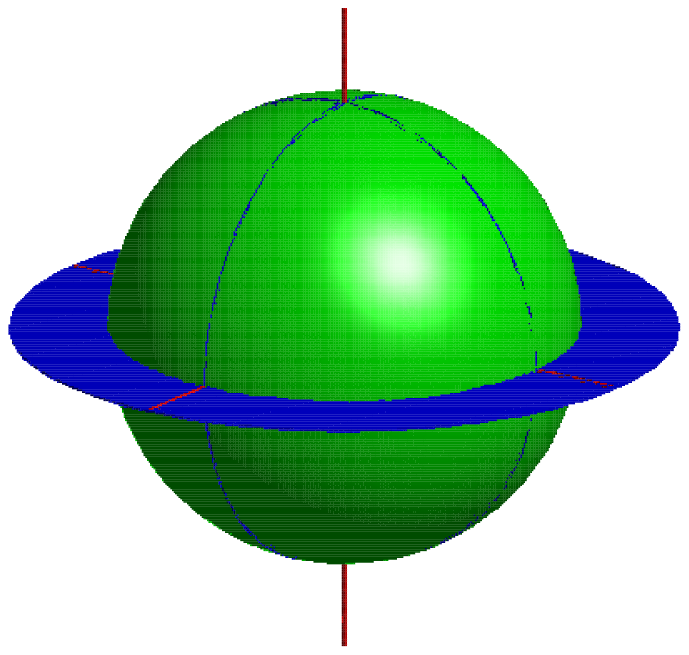}
\includegraphics[scale=.45]{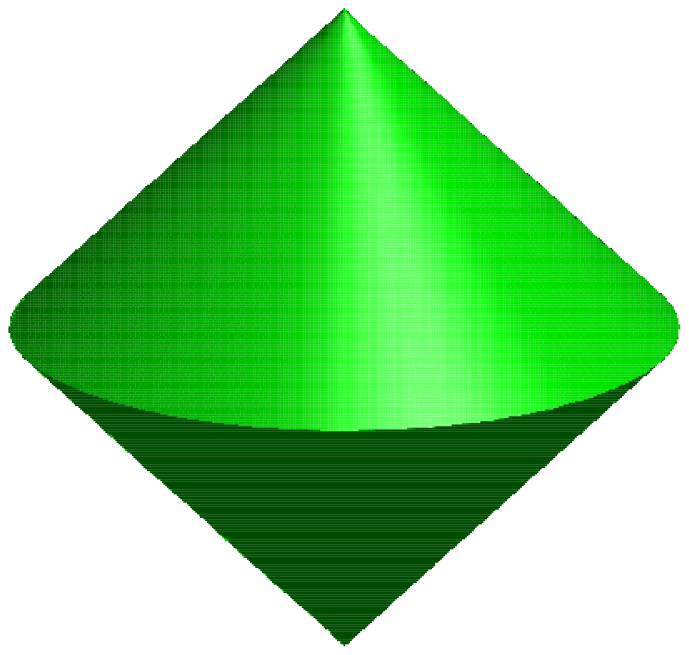} \\
\small
$F(A) =    1_{ \{A \cap \{3\} \neq \varnothing\} } + 1_{ \{A \cap \{1,2\} \neq \varnothing\} }  $ \\  $\Omega_2(w) = |w_3| +   \| w_{\{1,2\}} \|_2$
}
 \end{center}

 \begin{center}

\parbox{6cm}{\centering \includegraphics[scale=.45]{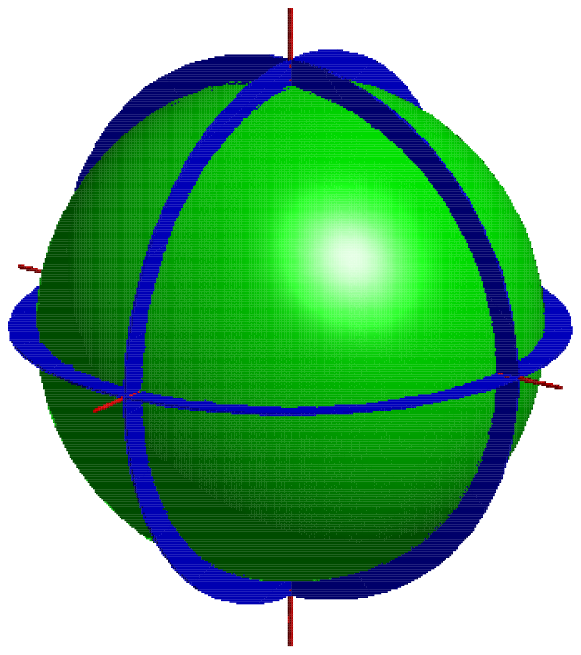} 
\includegraphics[scale=.45]{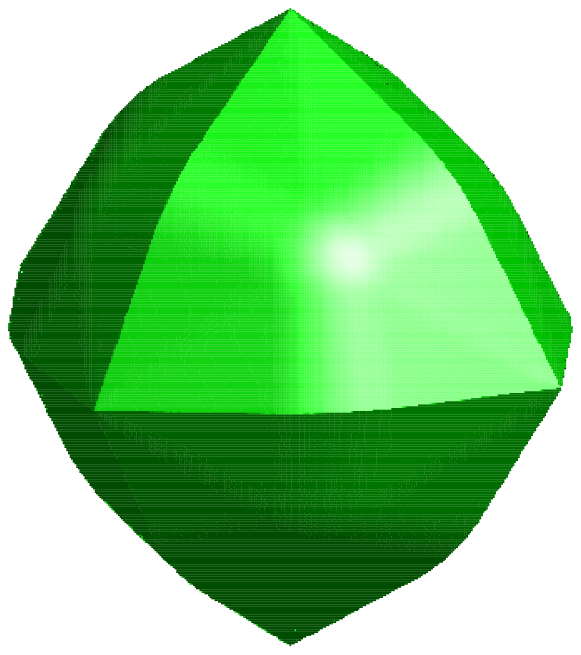} \\
\small
$F(A) =    |A|^{1/2}  $ 
\\
all possible extreme points
}
\parbox{6cm}{\centering \small \includegraphics[scale=.58]{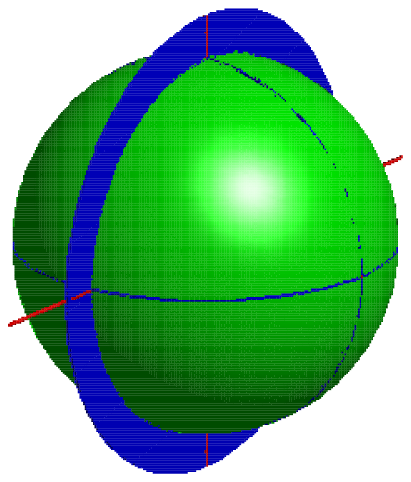}
\includegraphics[scale=.58]{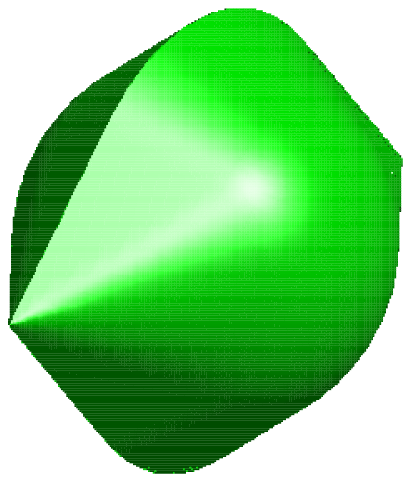} \\
$F(A) =    1_{ \{A \cap \{1,2,3\} \neq \varnothing\} } $ \hspace*{1.5cm} \\
\hspace*{1.5cm} $ + 1_{ \{A \cap \{2,3\} \neq \varnothing\} }  
+ 1_{ \{A \cap \{2\} \neq \varnothing\} }  $  
}

\end{center}

  \caption{Unit balls for structured sparsity-inducing norms, with the corresponding submodular functions and the associated norm, for $\ell_2$-relaxations. For each example, we plot on the left side the sets $K_A$ defined in \mysec{l2} as well as the convex hulls of their unions in green (on the right side).  \label{fig:3dballsL2} }

\end{figure}

\section{Shaping level sets$^\ast$}
\label{sec:shaping}

For a non-decreasing submodular function $F$, we have defined in \mysec{increasing} a norm $\Omega_\infty(w) = f(|w|)$, that essentially allows the definition of a prior knowledge on supports of predictors $w$. Since this norm was creating extra-behavior (i.e., clustering of the components of $|w|$), we designed a new norm in \mysec{l2} which does not have these problems.
In this section, we take the opposite approach and leverage the fact that when using the \lova extension as a regularizer, then some of the components of~$w$ will be equal.

We now consider a submodular function $F$ such that $F(\varnothing) =  F(V) = 0$. This includes (a) cuts in graphs and (b) functions of the cardinality $A \mapsto h(|A|)$ for $h$ concave such that $h(0) = h(p) =0$. We now show that using the \lova extension as a regularizer corresponds to a convex relaxation of a function of all level sets of $w$. Note that from property (d) of Prop.~\ref{prop:lova}, $f(w + \alpha 1_V)  = f(w)$, and it is thus natural to consider in the next proposition a convex set invariant by translation by constants times $1_V$.

\begin{proposition} \textbf{(Convex envelope for level sets)}
\label{prop:envelope}
The \lova extension $f(w)$ is the convex envelope of the function
$w \mapsto \max_{ \alpha \in \rb}  F( \{ w \geqslant \alpha \} )$ on the set $[0,1]^p + \rb 1_V = \{ w \in \rb^p , \ \max_{k \in V} w_k - \min_{k \in V} w_k \leqslant 1 \}$.
\end{proposition}
\begin{proof}
For any $w \in \rb^p$, level sets of $w$ are characterized by a partition $(A_1,\dots,A_m)$ of $V$ so that $w$ is constant on each $A_j$, with value $v_j$, $j=1,\dots,m$, and so that $(v_j)$ is a strictly decreasing sequence. We can now decompose the minimization with respect to $w$ using these  partitions $(A_j)$ and the values $(t_j)$. The level sets of $w$ are then $A_1 \cup \cdots \cup A_j$, $j \in \{1,\dots,m\}$.

In order to compute the convex envelope, as already done in the proofs of Prop.~\ref{prop:relax} and Prop.~\ref{prop:relaxl2}, we simply need to compute twice the Fenchel conjugate of the function we want to find the envelope of.

Let $s \in \rb^p$; we consider the function $g:w \mapsto \max_{ \alpha \in \rb}  F( \{ w \geqslant \alpha \} )$, and we compute its Fenchel conjugate on $ [0,1]^p + \rb 1_V$:
\BEAS
& & g^\ast(s)  = \max_{ w \in [0,1]^p + \rb 1_V  } w^\top s - g(w), \\[-.2cm]
& = & \max_{(A_1,\dots,A_m)  \  { \rm    partition} }  \  \Big\{  \ \max_{ t_1 > \cdots > t_m, \ t_1 - t_m \leqslant 1 }
\sum_{j=1}^{m}t_j s(A_j)  \\[-.2cm]
& & \hspace*{5cm}  - \max_{j \in \{1,\dots,m\} } F(A_1 \cup \cdots \cup A_j)  \Big\}.
\EEAS
By integration by parts, $g^\ast(s)$ is then equal to
\BEAS
& = & \!\!\!\! \max_{(A_1,\dots,A_m)  \  { \rm    partition} }   \!\!\!\!  \  \Big\{  \ \max_{ t_1 > \cdots > t_m, \ t_1 - t_m \leqslant 1 }
\sum_{j=1}^{m-1} (t_j - t_{j+1} ) s(A_1 \cup \cdots \cup A_j)  \\[-.2cm]
 & & \hspace*{3cm} + t_m s(V)  - \max_{j \in \{1,\dots,m\} } F(A_1 \cup \cdots \cup A_j)  \Big\}  
\\[-.2cm]
& = & I_{ s(V) =0 } (s) +  \max_{(A_1,\dots,A_m) \  { \rm     partition} }
\Big\{
\max_{j \in \{1,\dots,m-1\} } s(A_1 \cup \cdots \cup A_j) \\[-.2cm]
& & \hspace*{5cm}
- \max_{j \in \{1,\dots,m\} } F(A_1 \cup \cdots \cup A_j)
\Big\},
\\[-.2cm]
& = & I_{ s(V) =0 } (s) +  \max_{(A_1,\dots,A_m) \  { \rm     partition} }
\Big\{
\max_{j \in \{1,\dots,m-1\} } s(A_1 \cup \cdots \cup A_j)  \\
& & \hspace*{5cm} - \max_{j \in \{1,\dots,m-1\} } F(A_1 \cup \cdots \cup A_j)
\Big\},
\EEAS
where $I_{ s(V) =0 } $ is the indicator function of the set $\{ s(V) = 0\}$ (with values $0$ or $+\infty$). Note that $\max_{j \in \{1,\dots,m\} } F(A_1 \cup \cdots \cup A_j)
  = \max_{j \in \{1,\dots,m-1\} } F(A_1 \cup \cdots \cup A_j) $ because $F(V) = 0$.

Let $h(s) =I_{ s(V) =0 }(s)  +  \max_{A \subseteq V}  \{ s(A) - F(A) \} $. We clearly have $g^\ast(s) \geqslant h(s)$, because we take a maximum over a larger set (consider $m=2$ and the partition $(A,V \backslash A)$). Moreover, for all partitions $(A_1,\dots,A_m)$, if $s(V)=0$,
$\max_{j \in \{1,\dots,m-1\} } s(A_1 \cup \cdots \cup A_j)  \leqslant 
\max_{j \in \{1,\dots,m-1\} } \big\{ h(s) + F(A_1 \cup \cdots \cup A_j)  \big\}
= h(s) + \max_{j \in \{1,\dots,m-1\} } F(A_1 \cup \cdots \cup A_j)$, which implies that $g^\ast(s) \leqslant h(s)$. Thus $g^\ast(s) = h(s)$.

Moreover, we have, since $f$ is invariant by adding constants (property (d) of Prop.~\ref{prop:lova}) and $f$ is submodular,
\BEAS
 \max_{ w \in [0,1]^p + \rb 1_V  } w^\top s - f(w) & = & 
 I_{ s(V) =0 } (s) +  \max_{ w \in [0,1]^p   }  \{w^\top s - f(w) \} \\ 
 & = & 
 I_{ s(V) =0 } (s) +  \max_{ A \subseteq V  } \{s(A) - F(A) \}  = h(s),
\EEAS
where we have used the fact that minimizing a submodular function is equivalent to minimizing its \lova extension on the unit hypercube.
Thus $f$ and $g$ have the same Fenchel conjugates. The result follows from the convexity of $f$, using the fact the convex envelope is the Fenchel bi-conjugate~\cite{boyd,borwein2006caa}.

Alternatively, to end the proof, we could have computed
\BEAS
\max_{s(V) = 0} w^\top s - \max_{ A \subseteq V  } \{s(A) - F(A) \} 
& \!\! =\!\! & \max_{s(V) = 0} w^\top s - \min_{ v \in [0,1]^p} \{s^\top v - f(v) \} \\
& \!\! =\!\! &
\min_{ v \in [0,1]^p, \ v = w + \alpha 1_V} f(v) = f(w).
\EEAS\end{proof}

Thus, when using the \lova extension directly for symmetric submodular functions, then  the effect is on all sub-level sets $\{w \leqslant \alpha\}$ and not only on the support $\{w\neq 0\}$. 


\paragraph{Sparsity-inducing properties.}
While the facial structure of the symmetric submodular polyhedron $|P|(F)$ was key to analyzing the regularization properties for shaping supports, the base polyhedron $B(F)$ is the proper polyhedron to consider.

From now on, we assume that the set-function $F$ is submodular, has strictly non-negative values for all non trivial subsets of $V$. Then, the set $\mathcal{U} = \{ w \in \rb^p, \ f(w) \leqslant 1, \ w^\top 1_p = 0 \}$ is a polytope dual to the base polyhedron. A face of the base polyhedron (and hence, by convex strong duality), a face of the polytope $\mathcal{U}$, is characterized by a partition of $V$ defined by disjoint sets $A_1,\dots,A_m$. These corresponds to faces of $\mathcal{U}$ such that $w$ is constant on each set $A_i$, and the corresponding values are ordered. From Prop.~\ref{prop:faces}, the face has non-empty interior only if
$A_j$ is inseparable for the function $G_j: B \mapsto F( A_1 \cup \cdots  \cup A_{j-1} \cup B) - F( A_1 \cup \cdots  \cup A_{j-1})$ defined on subsets of $A_j$.  This property alone shows that some arrangements of level sets cannot be attained robustly, which leads to interesting behaviors, as we now show for two examples.

\paragraph{Cuts.}
When $F$ is the cut in an undirected graph, then a necessary condition for $A_j$ to be inseparable for the function $G_j: D \mapsto F( A_1 \cup \cdots  \cup A_{j-1} \cup D) - F( A_1 \cup \cdots  \cup A_{j-1})$ defined on subsets of $A_j$, is that $A_j$ is a connected set in the original graph\footnote{Since cuts are second-order polynomial functions of indicator vectors, contractions are also of this form, and the quadratic part is the same than for the cut in the corresponding subgraph.}. Thus, the regularization by the \lova extension (often referred to as the total variation) only allows constant sets which are connected in the graph, which is the traditional reason behind using such penalties.
In \mysec{isotonic}, we also consider the cases of directed graphs, leading to isotonic regression problems.

\paragraph{Cardinality-based functions.}
As will be shown in \mysec{card}, concave functions of the cardinality are submodular. Thus, if $h:[0,p] \to
\rb$ is concave and such that $h(0)=h(p)=0$, then $F:A \mapsto h(|A|)$ has levet-set shaping properties.
Since $F$ only depends on the cardinality, and is thus invariant by reordering of the variables, the only potential constraints is on the size of level sets.

The \lova extension   depends on the order statistics of $w$, i.e., if $w_{j_1} \geqslant \dots \geqslant w_{j_p}$, then $f(w) = \sum_{k=1}^{p-1} h(k) ( w_{j_k} - w_{j_{k+1}})$. While these examples do not provide significantly different behaviors for the non-decreasing submodular functions explored by~\cite{bach2010structured} (i.e., in terms of \emph{support}), they lead to interesting behaviors here in terms of \emph{level sets}, i.e., they will make the components $w$ cluster together in specific ways (by constraining the sizes of the clusters).  Indeed, allowed constant sets $A$ are such that $A$ is inseparable for the function $C \mapsto h(|B \cup C|) - h(|B|)$ (where $B \subseteq V$ is the set of components with higher values than the ones in $A$). As can be shown from a simple convexity argument, this imposes that  the concave function $h$ is not linear on $[|B|,|B|\!+\!|A|]$. We consider the following examples; in \myfig{card}, we show regularization paths, i.e., the set of minimizers of
$w \mapsto \frac{1}{2} \| w- z\|_2^2 + \lambda f(w)$ when $\lambda$ varies.

\begin{list}{\labelitemi}{\leftmargin=1.1em}
   \addtolength{\itemsep}{-.0\baselineskip}

\item[--] $F(A) = |A| \cdot | V \backslash A| $, leading to
$f(w) = \sum_{i,j=1}^p |w_i - w_j|$. This function can thus be also seen as the cut in the fully connected graph. All patterns of level sets are allowed as the function~$h$ is strongly concave (see left plot of \myfig{card}). This function has been extended in~\cite{lindsten2011clustering,toby} by considering situations where each $w_j$ is a vector, instead of a scalar, and replacing the absolute value $| w_i - w_j|$ by any norm $\| w_i - w_j\|$, leading to convex formulations for clustering.

\item[--] $F(A) =  1$  if $A \neq \varnothing$ and $A \neq V$, and $0$ otherwise, leading to
$f(w) = \max_{i,j} | w_i - w_j|$. Here, the function $h$ is linear between $1$ and $p$, and thus between the level sets with smallest and largest values, no constant sets are allowed; hence, there are two large level sets at the top and bottom, all the rest of the variables are in-between and separated (\myfig{card}, middle plot).

\item[--] $F(A) =  \max\{ |A|,  | V \backslash A| \} $. This function is piecewise affine, with only one kink, thus only one level set of cardinality greater than one (in the middle) is possible, which is observed in \myfig{card} (right plot). This may have applications to multivariate outlier detection by considering extensions similar to~\cite{toby}.

\end{list}

\begin{figure}
\begin{center}

 \vspace*{-.25cm}
 
\hspace*{-2.2cm}
\includegraphics[scale=.33]{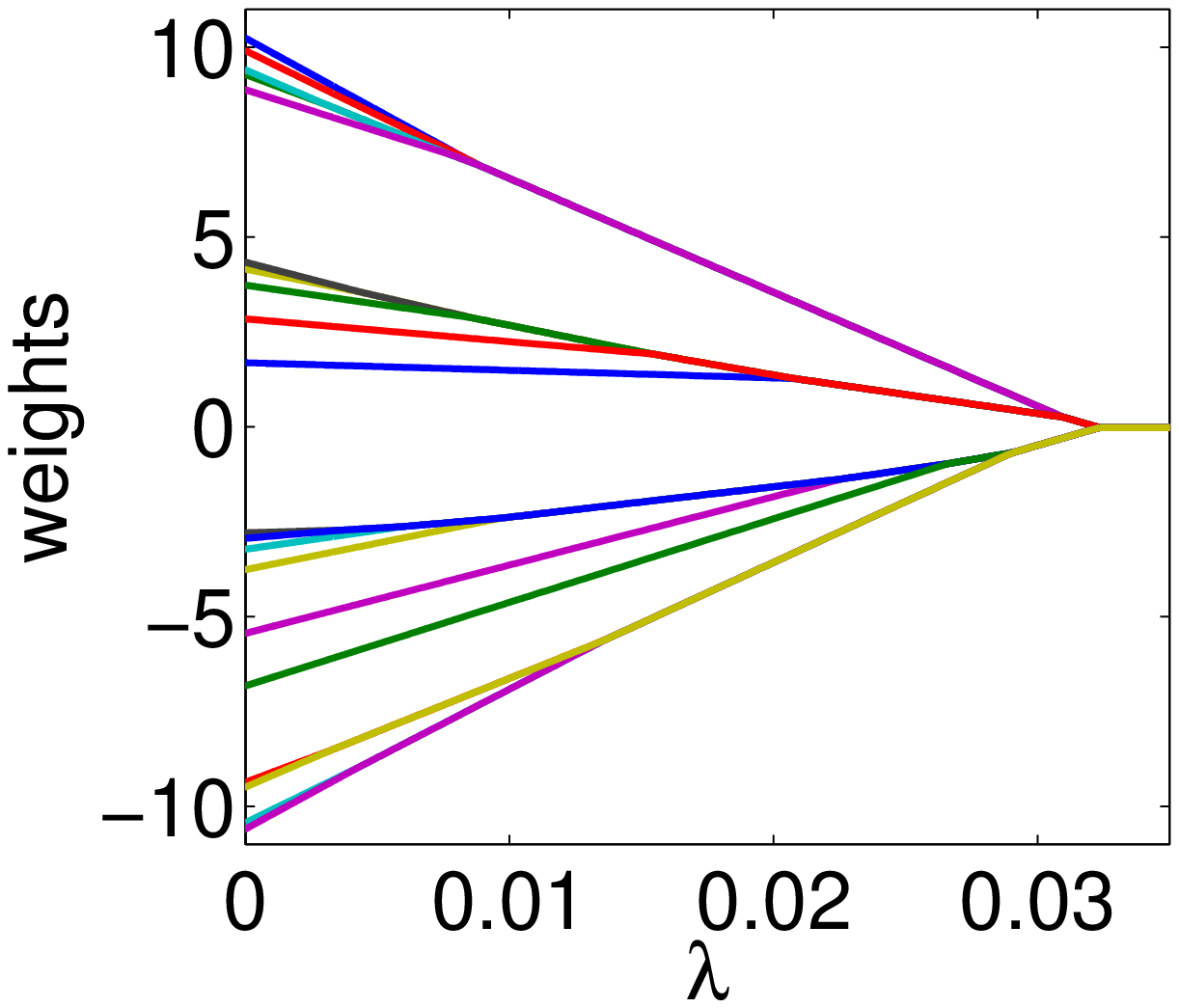}  
\includegraphics[scale=.33]{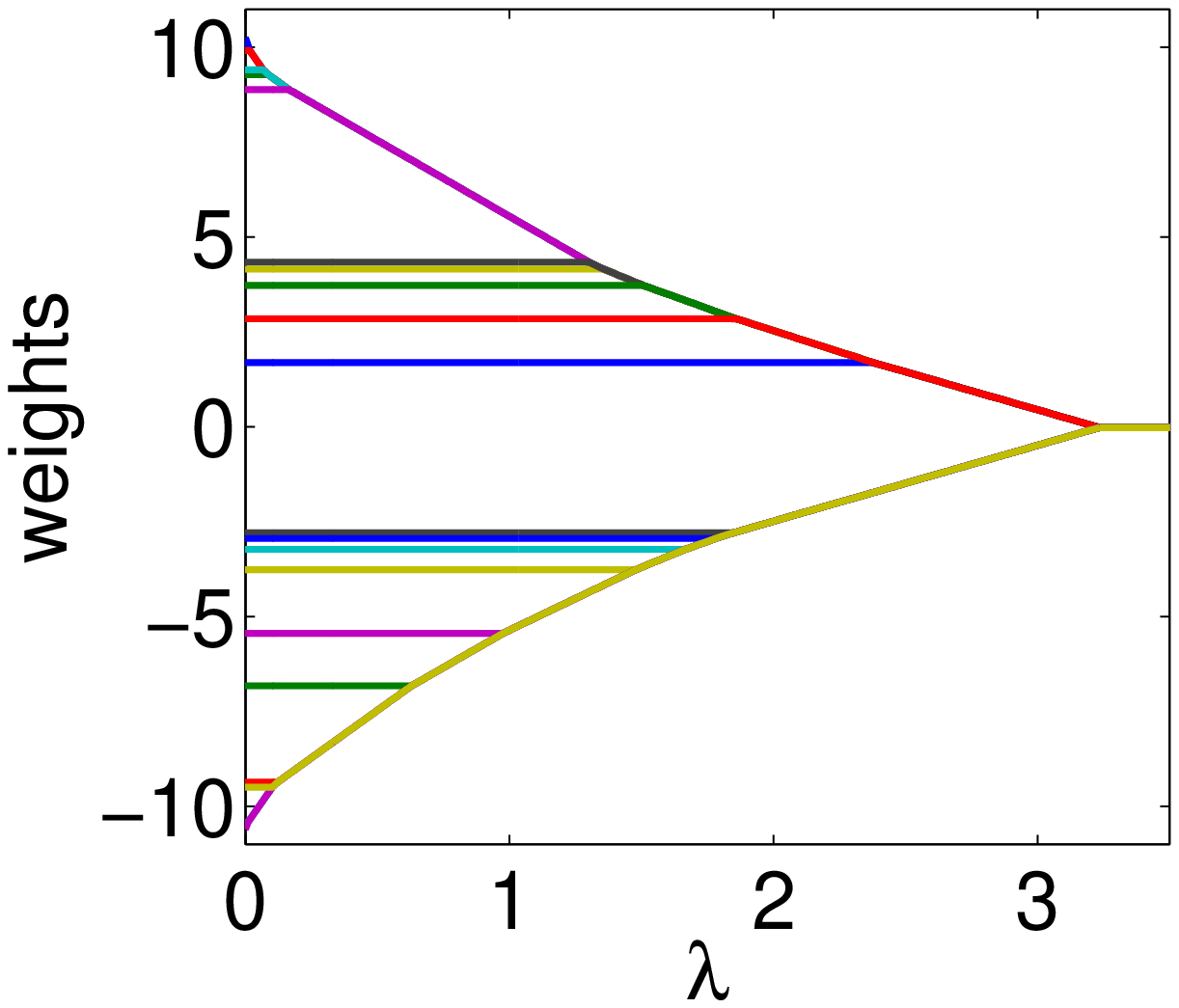}  
\includegraphics[scale=.33]{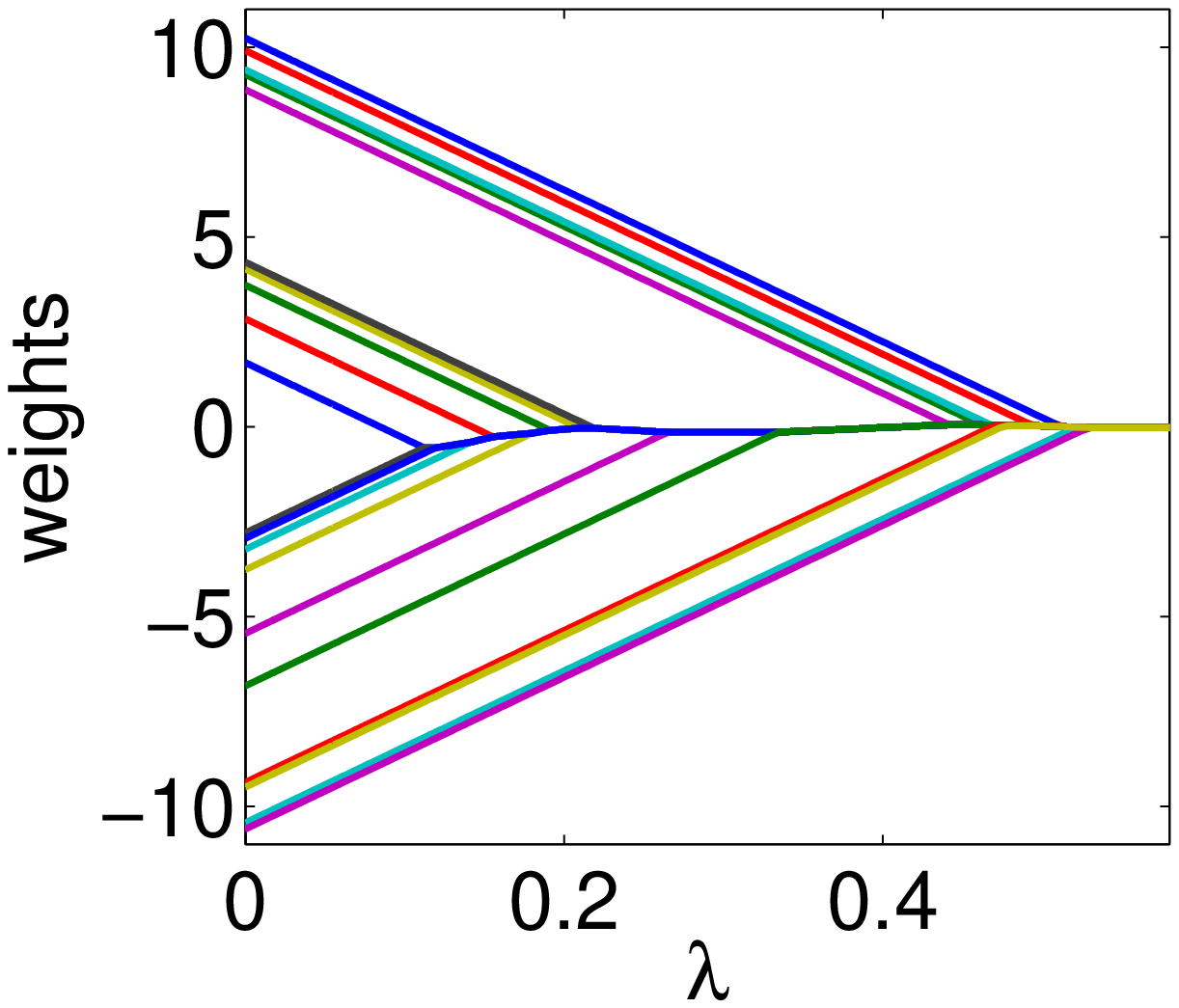} 
\hspace*{-1.5cm}

\end{center}

\vspace*{-.5cm}

\caption{ Piecewise linear regularization paths of  the minimization of 
$w \mapsto \frac{1}{2} \| w- z\|_2^2 + \lambda f(w)$, for different functions of cardinality. From left to right: quadratic function (all level sets allowed),  second example in \mysec{card}  (two large level sets at the top and bottom),   piecewise linear with two pieces (a single large level set in the middle).  Note that in all these particular cases the regularization paths for orthogonal designs are \emph{agglomerative}, while for general designs, they would still be piecewise affine but not agglomerative. For more details, see~\cite{shapinglevelsets}.}
\label{fig:card}

\end{figure}

\chapter{Examples and Applications of Submodularity}

\label{chap:examples}

We now present classical examples of submodular functions. For each of these, we also describe the corresponding \lova extensions, and, when appropriate, the associated submodular polyhedra. We also present applications to machine learning, either through formulations as combinatorial optimization problems of through the regularization properties of the \lova extension---in \mychap{relax}, we have defined several sparsity-inducing norms based on the \lova extension, namely $\Omega_\infty$ and $\Omega_q$, for $q \in (1,+\infty)$. We are by no means exhaustive and other applications may be found in facility location~\cite{cornuejols1977location,cornuejols1977uncapacitated,ahmed2011maximizing}, game theory~\cite{feige2006maximizing}, document summarization~\cite{lin2011-class-submod-sum}, social networks~\cite{kempe2003maximizing}, or clustering~\cite{narasimhan2007local}.

Note that in Appendix~\ref{app:ope}, we present several operations that preserve submodularity (such as symmetrization and partial minimization), which can be applied to any of the functions presented in this chapter, thus defining new functions.

\section{Cardinality-based functions}
\label{sec:card}
\label{sec:cardinality}
We consider functions that depend only on $s(A)$ for a certain $s \in \rb_+^p$. If $s = 1_V$, these are functions of the cardinality. The next proposition shows that only concave functions lead to submodular functions, which is consistent with the diminishing return property from \mychap{definitions}~(Prop.~\ref{prop:firstorder}).

\begin{proposition} \textbf{(Submodularity of cardinality-based set-functions)}
\label{prop:card}
If $s \in \rb^p_+$ and $g:\rb_+ \to \rb$ is a concave function, then $F:A \mapsto g( s(A) )$  is submodular. If $F:A \mapsto g( s(A) )$  is submodular for all $s \in \rb_+^p$, then $g$ is concave.
\end{proposition}
\begin{proof} The function
$F:A \mapsto g( s(A) )$ is submodular if and only if for all $A \subseteq V$ and $j,k \in V \backslash A $: 
$g(s(A) + s_k ) - g(s(A)) \geqslant g(s(A) + s_k +s _j) -  g(s(A) + s_j)$. If $g$ is concave and $a \geqslant 0$, $ t \mapsto g(a+t)-g(t) $ is non-increasing, hence the first result. Moreover, if $ t \mapsto g(a+t)-g(t) $ is  non-increasing for all $a \geqslant 0$, then $g$ is concave, hence the second result.
\end{proof}

\begin{proposition} \textbf{(\lova extension of cardinality-based set-functions)}
Let $s \in \rb^p_+$ and $g:\rb_+ \to \rb$ be a concave function such that $g(0)=0$, the \lova extension of the submodular function $F:A \mapsto g( s(A) )$  is equal to (with the same notation than Prop.~\ref{prop:lova} that $j_k$ is the index of the $k$-the largest component of $w$):
$$
f(w) =  \sum_{k=1}^p w_{j_k} [ g( s_{j_1}+ \cdots + s_{j_k}) -  g( s_{j_1}+ \cdots + s_{j_{k-1}}) ].
$$
If $s = 1_V$, i.e., $F(A) = g(|A|)$, then 
$f(w) =  \sum_{k=1}^p w_{j_k} [ g(k) -g(k-1) ]$.
\end{proposition}
 Thus, for functions of the cardinality (for which $s=1_V$), the \lova extension is thus a linear combination of order statistics (i.e., $r$-th largest component of $w$, for $r \in \{1,\dots,p\}$).
 
 \paragraph{Application to machine learning.}
When minimizing set-functions, considering $g(s(A))$ instead of $s(A)$ does not make a significant difference. However, it does in terms of the \lova extension as outlined at the end of \mysec{shaping}: using the \lova extension for regularization encourages components of $w$ to be equal, and hence provides a convex prior for clustering or outlier detection, depending on the choice of the concave function $g$ (see more details in~\cite{shapinglevelsets,toby}).

 Some special cases   of non-decreasing functions are of interest, such as $F(A) = |A|$, for which $f(w) = w^\top 1_V$ and $\Omega_q$ is the $\ell_1$-norm for all $q \in (1,+\infty]$, and $F(A) = 1_{|A|>0}
 = \min \{|A|,1\}$ for which $f(w) = \max_{k \in V} w_k$ and $\Omega_q$ is the $\ell_q$-norm. When restricted to subsets of $V$ and then linearly combined, we obtain set covers defined in \mysec{cover}. Other interesting examples of combinations of functions of restricted weighted cardinality functions may be found in~\cite{stobbe,kolmogorov2010minimizing}.

\section{Cut functions}
\label{sec:cuts}
Given a set of (non necessarily symmetric) weights $d:V \times V \to \rb_+$, we define   the cut as
$$F(A) = d(A,V \backslash A) = \sum_{k \in A, \ j \in V \backslash A} d(k,j),$$
where we denote   $d(B,C) = \sum_{k \in B, \ j \in C} d(k,j)$ for any two sets $B,C$. We give several proofs of submodularity for cut functions.

\paragraph{Direct proof of submodularity.}
For a cut function and disjoint subsets $A,B,C$, we always have (see~\cite{cunningham1985minimum} for more details):
\BEA
\nonumber
F(A \cup B \cup C) & \!\!\!= \!\!\!&  F(A \cup B )  + F(A   \cup C)  + F( B \cup C)  \\
\label{eq:AAA} & & \hspace*{1.5cm}  - F(A )  - F(  B  )  - F(  C)  + F(\varnothing), \\
\nonumber
F( A \cup B) & \!\!\!=\!\!\! & d(A \cup B,  (A \cup B)^{\mathsf{c}} ) = d(A, A^{\mathsf{c}} \cap B^{\mathsf{c}}) + d(B, A^{\mathsf{c}} \cap B^{\mathsf{c}})
\\
\nonumber
& \!\!\!\leqslant \!\!\! &  d(A,A^{\mathsf{c}})+ d(B,B^{\mathsf{c}}) = F(A) + F(B),
\EEA
where we denote $A^{\mathsf{c}} = V \backslash A$. This implies that $F$ is sub-additive.
We then have, for any sets $A,B \subseteq V$:
\BEAS 
 & & F(A \cup B) 
 =  F([A \cap B] \cup [A \backslash B ]\cup [B \backslash A]) \\
&  \!\!\!=\!\!\!  & F([A \cap B] \cup [A \backslash B ] ) + F([A \cap B] \cup    [B \backslash A]) + F(  [A \backslash B ]\cup [B \backslash A]) \\
& &  - 
F(A \cap B ) - F(  A \backslash B ) - F(  B \backslash A) + F(\varnothing) 
\mbox{ using \eq{AAA}.}
\EEAS
By expanding all terms, we obtain that $F( A \cup B)$ is equal to
\BEAS
&  \!\!\!=\!\!\!  & F(A) + F(B) + F( A \Delta B )  - 
F(A \cap B ) - F(  A \backslash B ) - F(  B \backslash A)  \\
&  \!\!\!=\!\!\!  & F(A) + F(B)   - 
F(A \cap B ) + [ F( A \Delta B )- F(  A \backslash B ) - F(  B \backslash A) ] \\
&  \!\!\!\leqslant \!\!\!  & F(A) + F(B)    - 
F(A \cap B ), \mbox{ by sub-additivity}, \EEAS
which shows submodularity.

\paragraph{\lova extension.}
The cut function is equal to $F(A) = \sum_{k \in V, \ j \in V } d(k,j) (1_A)_k \big[ 1 - (1_A)_j \big]$ and it is thus the positive linear combination of the functions  $G_{kj}: A \mapsto (1_A)_k \big[ 1 - (1_A)_j \big]$. The function $G_{kj}$ is the extension to $V$ of a function $\widetilde{G}_{kj}$ defined only on the power set of $\{j,k\}$, where $\widetilde{G}_{k j}(\{k\}) = 1$ and all other values are equal to zero. Thus from \eq{lova2d} in \mychap{lova}, $\widetilde{G}_{kj}(w_k,w_j) = w_k - \min \{w_k,w_j\} = (w_k - w_j)_+$. Thus, the \lova extension of $F$ is equal to 
$$f(w) = \sum_{k,j \in V} d(k,j) ( w_k - w_j )_+,$$ 
(which is convex and thus provides an alternative proof of submodularity owing to Prop.~\ref{prop:convexity}). 

If  the weight function $d$ is symmetric, then the submodular function is also symmetric, i.e.,
for all $A \subseteq V$, $F(A)=F(V \backslash A)$, and the \lova extension is even (from Prop.~\ref{prop:lova}). When $d$ takes values in $\{0,1\}$ then we obtain the cut function in an undirected graph and the \lova extension is often referred to as the total variation (see below).

\begin{figure}

\begin{center}
\parbox[b]{6cm}{
 \includegraphics[scale=.45]{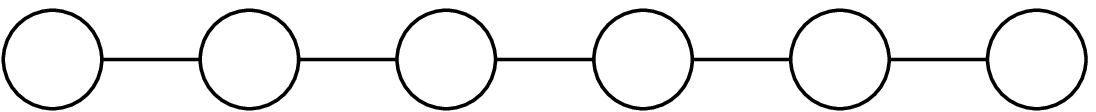} 
 \vspace*{1.5cm}}\hspace*{1cm}
\includegraphics[scale=.45]{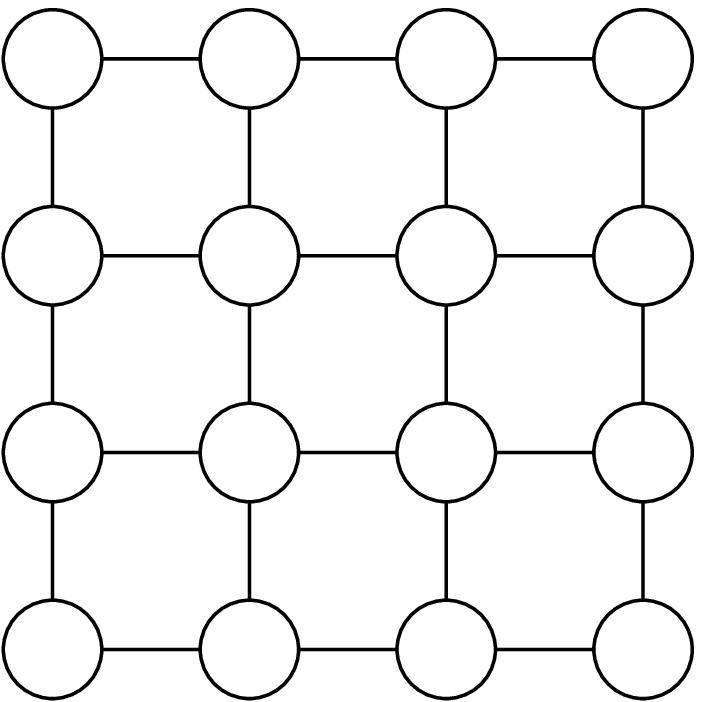} 
 \end{center}

\vspace*{-.5cm}

\caption{Left: chain graphs. Right: two-dimensional grid with $4$-connectivity. The cut in these undirected graphs lead to  \lova extensions which are certain versions of   total variations, which enforce level sets of $w$ to be connected with respect to the graph.}
\label{fig:2dgrid}
\end{figure}

%
%
%

\paragraph{Total variation and piecewise constant signals.}
Given an undirected graph $G=(V,E)$, the total variation is the \lova extension associated to the cut-function corresponding to this graph. For example for the chain graph (left plot in \myfig{2dgrid}), we have
$f(w) = \sum_{i=1}^{p-1} | w_{i+1} - w_i|$.

As shown in \mysec{shaping}, used as a regularizer, it leads to vectors~$w$ which are piecewise constant with respect to the graph, i.e., the constant sets are almost surely connected subsets of the graph. This property is the main reason for its wide-spread use in signal processing (see, e.g.,~\cite{osher2003level,chambolle2004algorithm,chambolle2009total}), machine learning and statistics~\cite{tibshirani2005sparsity}. For example, for a chain graph, the total variation is commonly used to perform change-point detection, i.e., to approximate a one-dimensional signal by a piecewise constant one~(see, e.g.,~\cite{harchaoui2008catching}). We perform experiments with this example in \mysec{graph}, where we relate it to related concepts. In \myfig{tv2d}, we show an example of application of total variation denoising in two dimensions.

\paragraph{Isotonic regression.}
\label{sec:isotonic}
Consider $p$ real numbers $z_1,\dots,z_p$. The goal of isotonic regression is to find $p$ other real number $w_1,\dots,w_p$, so that (a) $w$ is close to $z$ (typically in squared $\ell_2$-norm), and (b) the components of $w$ satisfy some pre-defined order constraints, i.e., given a subset $E \subseteq V \times V$, we want to enforce that for all $(i,j) \in E$, then $w_i \geqslant w_j$. 

Isotonic regression has several applications in machine learning and statistics, where these monotonic constraints are relevant, for example in genetics~\cite{luss2012efficient}, biology~\cite{obozinski2008consistent}, medicine~\cite{schell1997reduced}, statistics~\cite{barlow1972statistical} and multidimensional scaling for psychology applications~\cite{kruskal1964multidimensional}. See an example for the linear ordering in \myfig{pava}.

The set of constraints may be put into a directed graph. For general sets of constraints, several algorithms that run in $O(p^2)$ have been designed. In this section, we show how it can be reformulated through the regularization by the \lova extension of a submodular function, thus bringing to bear the submodular machinery (in particular algorithms from \mychap{prox-algo}).

Let $F(A)$ be the cut function in the graph $G=(V,E)$. Its \lova extension is equal to
$f(w) = \sum_{(i,j) \in E} (w_i - w_j)_+$, and thus $w \in \rb^p$ satisfies the order constraints if and only if $f(w) = 0$. Given that $f(w)\geqslant 0$ for all $w \in \rb^p$,  then the problem is equivalent to minimizing $\frac{1}{2} \| w - z\|_2^2 + \lambda f(w)$ for $\lambda$ large enough\footnote{More precisely, if the directed graph $G=(V,E)$ is strongly connected (see, e.g.,~\cite{cormen89introduction}), then $F(A)>0$ for all non-trivial subset $A$ of $V$. Let $w$ be the solution of the isotonic regression problem, which satisfies $|w_j - z_j| \leqslant \max_{k \in V} z_k
- \min_{k \in V} z_k$. It is optimal for the $\lambda$-regularized problem as soon as $ \lambda^{-1}(z - w) \in B(F)$. A sufficient condition is that $\lambda \geqslant
( \max_{k \in V} z_k
- \min_{k \in V} z_k ) \max_{A \subseteq V, \ A \neq V, \ A \neq 0} \frac{ |A| }{ F(A) }$.
}.
See \mysec{decomposition} for an efficient algorithm based on a general divide-and-conquer strategy for separable optimization on the base polyhedron (see also~\cite{lussdecomposing} for an alternative description  and empirical comparisons to alternative approaches). This algorithm will be obtained as a sequence of $p$ min-cut/max-flow problems and is valid for all set of ordering contraints. When the constraints form a chain, i.e., we are trying to minimize $\frac{1}{2} \| w - x\|_2^2 $ such that $w_t \geqslant w_{t+1}$ for $t \in \{1,\dots,p-1\}$, a simple algorithm called the ``pool adjacent violators'' algorithm may be used with a cost of $O(p)$.
Following~\cite{best1990active}, we present it in Appendix~\ref{app:pava} and show how it relates to a dual active-set algorithm. Note that the chain ordering is intimately related to submodular function since the \lova extension is linear for vectors with a fixed ordering, and that the pool adjacent violators  algorithm will allow us in \mychap{prox-algo} to improve on existing convex optimization problems involving the \lova extension.

\begin{figure}
\begin{center}
 \includegraphics[scale=.6]{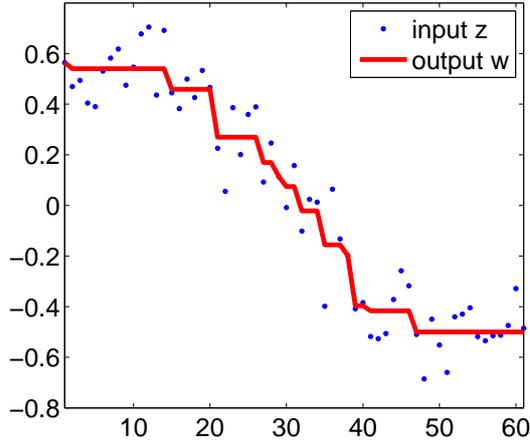}
 \end{center}
 
 \vspace*{-.6cm}

\caption{Isotonic regression with chain constraints. The output of isotonic regression is always monotonic, and potentially has long plateaux.}
\label{fig:pava}
\end{figure}

\paragraph{Extensions.}
Note that cut functions can be extended to cuts in hypergraphs, which may have interesting applications in computer vision~\cite{boykov2001fast}. Moreover, directed cuts (i.e., when  $d(k,j)$ and $d(j,k)$ may be different)
may be interesting to favor increasing or decreasing jumps along the edges of the graph (such as for isotonic regression).

 \paragraph{Interpretation in terms of quadratic functions of indicator variables.}
 For undirected graphs (i.e., for which the function $d$ is symmetric), we may rewrite the cut as follows:
 $$\!\!\!
 F(A) \! = \! \frac{1}{2} \sum_{k=1}^p \sum_{j = 1}^p d(k,j) | (1_A)_k - (1_A)_j| \\
\!=\!  \frac{1}{2} \sum_{k=1}^p \sum_{j = 1}^p d(k,j)  | (1_A)_k - (1_A)_j|^2,
 $$
 because $ | (1_A)_k - (1_A)_j|  \in \{0,1\}$. This leads to
 $$
 F(A)
  \!= \! \frac{1}{2} \sum_{k=1}^p \sum_{j = 1}^p  (1_A)_k   (1_A)_j 
  \big[
1_{j=k} \sum_{i=1}^p d(i,k)  - d(j,k)
  \big] \!= \!\frac{1}{2} 1_A^\top Q 1_A,
 $$
 with $Q$ the square matrix of size $p$ defined as
 $Q_{ij} = \delta_{i=j} \sum_{k=1}^p d_{ik}  - d_{ij}$  ($Q$ is the Laplacian of the graph~\cite{chung1997spectral}); see \mysec{graph} for experiments relating total variation to the quadratic function defined from the graph Laplacian. It turns out that a sum of linear and quadratic functions of~$1_A$ is submodular only in this situation.
 
 \begin{proposition} \textbf{(Submodularity of quadratic functions)}
 Let $Q \in \rb^{ p \times p}$ and $q \in \rb^p$. Then the function $F:A \mapsto q^\top 1_A + \frac{1}{2} 1_A^\top Q 1_A$ is submodular if and only if all off-diagonal elements of $Q$ are non-positive.
 \end{proposition}
 \begin{proof}
 Since cuts are submodular, the previous developments show that the condition is sufficient. It is necessary by simply considering the inequality $0 \leqslant F(\{i\})+F(\{j\}) - F(\{i,j\}) = q_i + \frac{1}{2}Q_{ii} + q_j + \frac{1}{2}Q_{jj} - [ q_i + q_j + \frac{1}{2}Q_{ii} + \frac{1}{2}Q_{jj} + Q_{ij} ] =  - Q_{ij}$.  \end{proof}

\paragraph{Regular functions and robust total variation.}  
\label{sec:regular}
By partial minimization, we obtain so-called \emph{regular functions}~\cite{boykov2001fast, chambolle2009total}. Given our base set $V$, some extra vertices (in a set $W$ disjoint from $V$) are added and a (potentially weighted) graph is defined on the vertex set
$V \cup W$, and the cut in this graph is denoted by $G$. We then define a set-function $F$   on $V$ as
$F(A) = \min_{ B \subseteq W} G( A \cup B)$, which is submodular because partial minimization preserves submodularity (Prop.~\ref{prop:partial}). Such regular functions are useful for two reasons: (a) they define new set-functions, as done below, and (b) they lead to efficient reformulations of existing functions through cuts, with efficient dedicated algorithms for minimization.

One new class of set-functions are ``noisy cut functions'':
  for a given  weight function $d: W \times W \to \rb_+$,  where each node in $W$ is uniquely associated to a node in $V$, we consider the submodular function obtained as the minimum cut adapted to $A$ in the augmented graph (see top-right plot of \myfig{cuts}):
$
F(A) =  \min_{B \subseteq W} \  \sum_{k \in B, \ j \in W \backslash B} d(k,j) + \lambda | A \Delta B|$, where
$A \Delta B = ( A \backslash B) \cup ( B \backslash A)$ is the symmetric difference between the sets $A$ and $B$.
 This allows for robust versions of cuts, where some gaps may be tolerated; indeed, compared to having directly a small cut for $A$, $B$ needs to have a small cut and to be close to $A$, thus allowing some elements to be removed or added to $A$ in order to lower the cut (see more details in~\cite{shapinglevelsets}).  Note that this extension from cuts to noisy cuts is similar to the extension from Markov chains to hidden Markov models~\cite{wainwright2008graphical}. 
For a detailed study of the expressive power of functions expressible in terms of graph cuts, see, e.g.,~\cite{zivnõ2009expressive,charpiatexhaustive}.

 \paragraph{Efficient algorithms.}
The class of cut functions, and more generally regular functions, is particularly interesting, because it leads to a family of  submodular functions for which dedicated fast algorithms exist. Indeed, minimizing the cut functions or the partially minimized cut, plus a modular function defined by $z \in \rb^p$, may be done with a min-cut/max-flow algorithm (see, e.g.,~\cite{cormen89introduction} and Appendix~\ref{app:maxflow} for the proof of equivalence based on linear programming duality). Indeed, following~\cite{boykov2001fast,chambolle2009total}, we add two nodes to the graph, a source $s$ and a sink $t$. All original edges have non-negative capacities $d(k,j)$, while, the edge that links the source $s$ to the node $k \in V$ has capacity $(z_k)_+$ and the edge that links the node $k \in V$ to the sink $t$ has weight $-(z_k)_-$ (see bottom line of \myfig{cuts}). Finding a minimum cut or maximum flow in this graph leads to a minimizer of $F-z$. 

In terms of running-time complexity, several algorithmic frameworks lead to polynomial-time algorithm: for example, with $p$ vertices and $m$ edges, ``push-relabel'' algorithms~\cite{goldberg1988new} may reach a worst-case complexity of $O(p^2 m^{1/2})$. See also~\cite{cherkassky1997implementing}.

 For proximal methods (i.e., the total variation denoising problem), such as defined in \eq{proxalpha} (\mychap{prox}), we have $z = \psi(\alpha)$ and we need to solve an instance of a \emph{parametric max-flow} problem, which may be done using efficient dedicated algorithms with worst-case complexity which is only a constant factor greater than a single max-flow problem~\cite{gallo1989fast,babenko2007experimental,hochbaum2001efficient,chambolle2009total}. See also \mysec{proxcomb} for generic algorithms based on a sequence of submodular function minimizations.

\begin{figure}

\begin{center}

\vspace*{-.1cm}

\hspace*{.8cm}
\includegraphics[scale=.6]{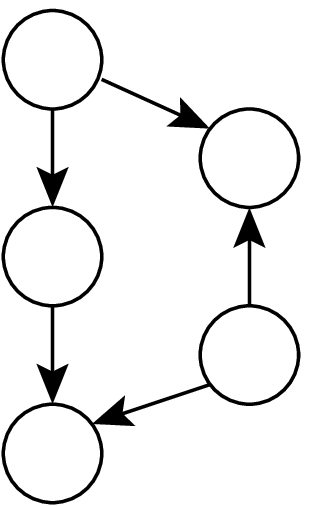} \hspace*{1.5cm}
\parbox[b]{5cm}{  \includegraphics[scale=.5]{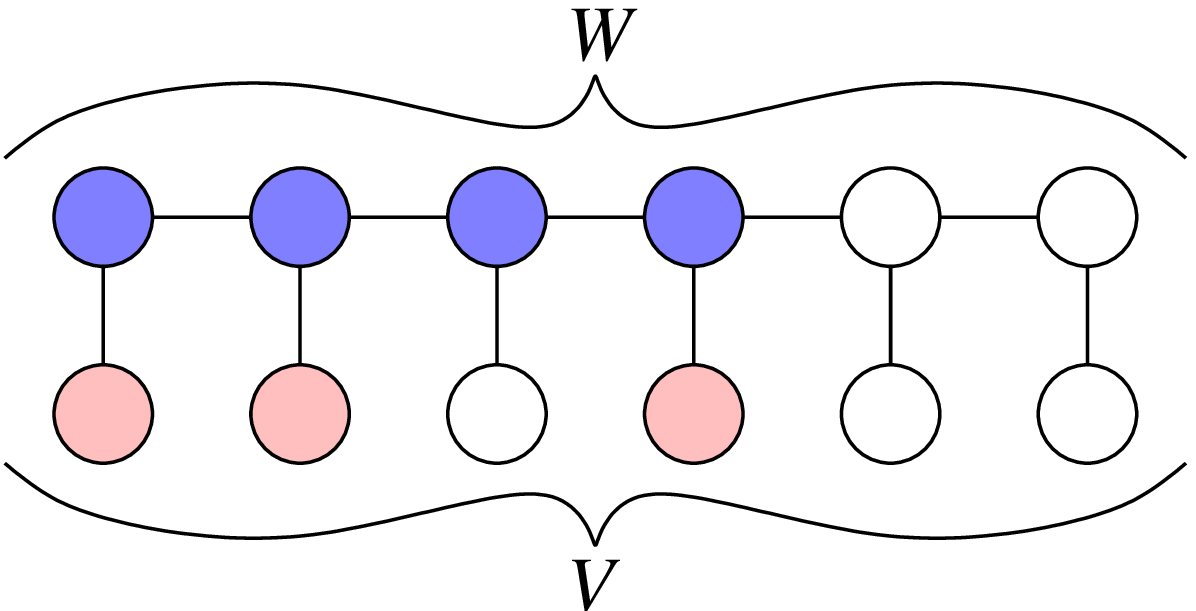} \vspace*{-.2cm}}

\vspace*{.3cm}

\includegraphics[scale=.6]{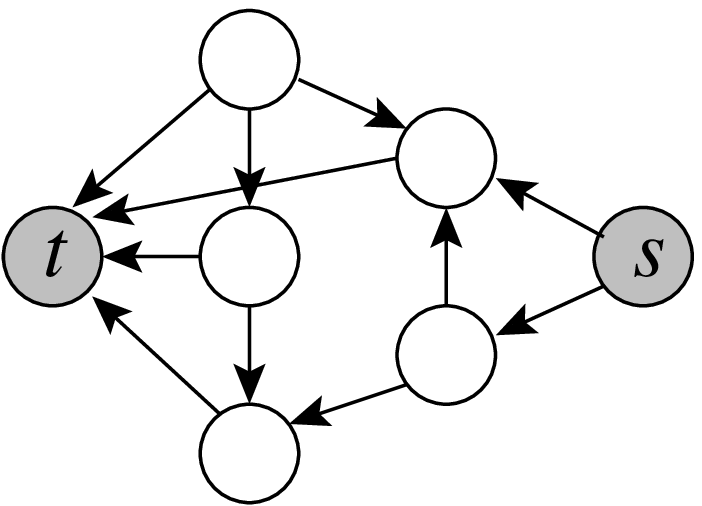} \hspace*{.32cm}
\includegraphics[scale=.5]{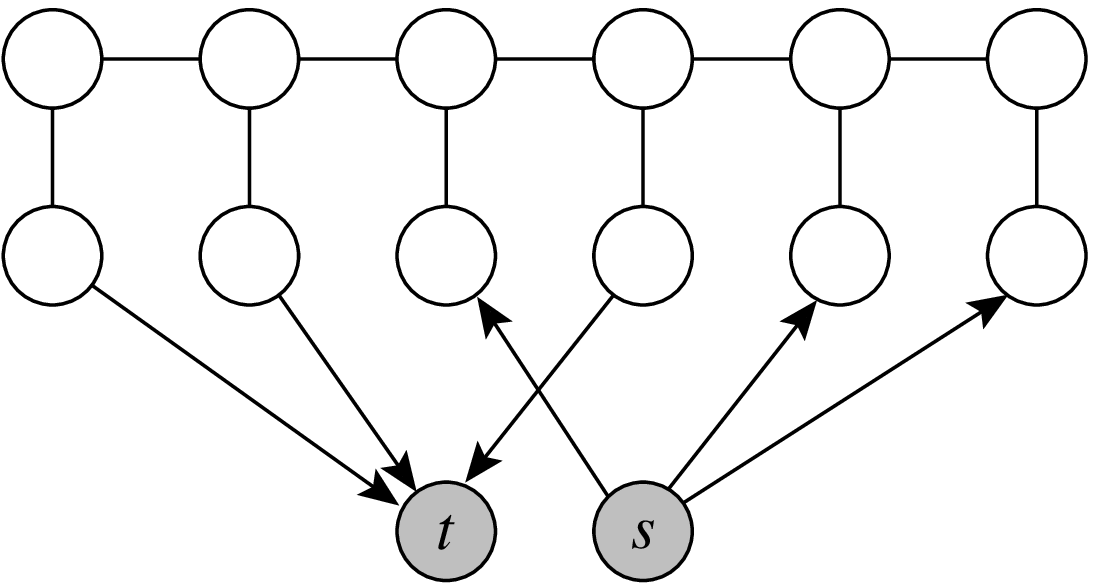}
 \end{center}

\vspace*{-.4cm}

\caption{Top: directed graph (left) and undirected corresponding to regular functions (which can be obtained from cuts by partial minimization; a set $A \subseteq V$ is displayed in red, with a set $B \subseteq W$ with small cut but one more element than $A$, see text in \mysec{cuts} for details). Bottom:  graphs corresponding to the min-cut formulation for minimizing the submodular function above plus a modular function (see text for details).
}
\label{fig:cuts}
\end{figure}

\begin{figure}

\begin{center}
 \includegraphics[scale=.66]{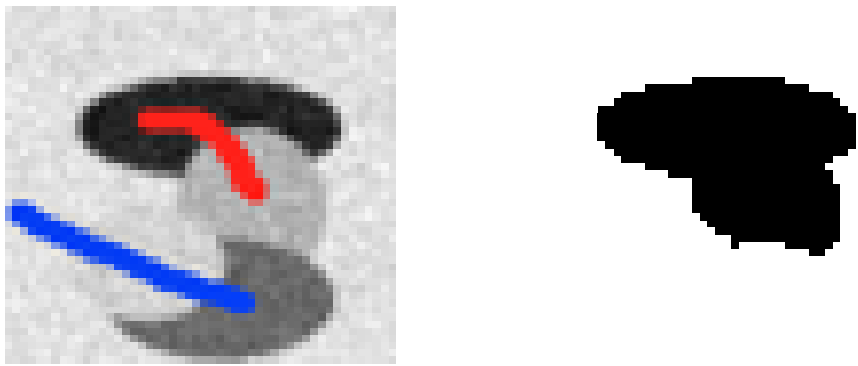}
\end{center}
 
\caption{Semi-supervised image segmentation with weighted graph cuts.
 (left) noisy image with supervision (red and blue strips), (right) segmentation obtained by minimizing a cut subject to the labelling constraints, with a weighted graph whose weight between two pixels is a decreasing function of their distance in the image and the magnitude of the difference of pixel intensities.}
\label{fig:imageseg}
\end{figure}

\begin{figure}

\begin{center}
 \includegraphics[scale=.66]{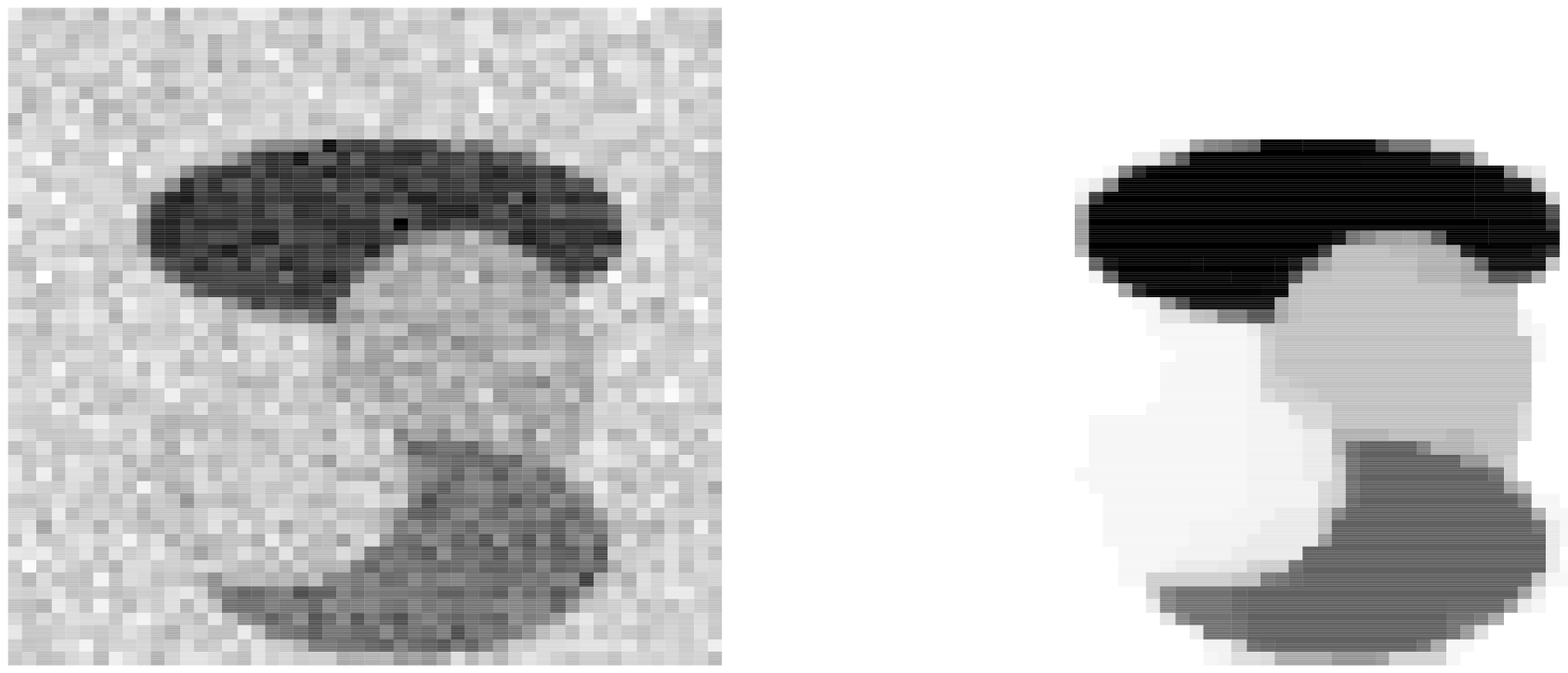}
 \end{center}
 
\caption{Image denoising with total variation in two dimensions, i.e., the \lova extension of the cut-function in the two-dimensional grid (right plot of \myfig{2dgrid}): (left) noisy image, (right) denoised image with piecewise constant level sets, obtained by minimization of $\frac{1}{2} \| w - z\|_2^2 + f(w)$; see corresponding algorithms in \mychap{prox-algo}.}
\label{fig:tv2d}
\end{figure}

  \paragraph{Applications to machine learning.}
 Finding minimum cuts in undirected graphs such as   two-dimensional grids or extensions thereof in more than two dimensions has become an important tool in computer vision for image segmentation, where it is commonly referred to as \emph{graph cut} techniques (see an example in \myfig{imageseg} and, e.g.,~\cite{kolmogorov2004energy} and references therein). In this context, several extensions have been considered, such as multi-way cuts, where exact optimization is in general not possible anymore, and a sequence of binary graph cuts is used to find an approximate minimum (note that in certain cases where labels are ordered, an exact formulation is possible~\cite{ishikawa2003exact,darbon2008global,hochbaum2013}). See also~\cite{narasimhan2006q} for a specific multi-way extension based on different submodular functions.
 
 The \lova extension of cuts in an undirected graph, often referred to as the total variation, has now become a classical regularizer in signal processing and machine learning: given a graph, it will encourages solutions to be piecewise-constant according to the graph~\cite{hoefling910path,toby}. 
 See \mysec{shaping} for a formal description of the sparsity-inducing properties of the \lova extension; for chain graphs, we obtain usual piecewise constant vectors, and the have many applications in sequential problems (see, e.g.,~\cite{harchaoui2008catching,tibshirani2005sparsity,mairal2010online,chambolle2009total} and references therein). Note that in this context, separable optimization problems considered in \mychap{prox} are heavily used and that algorithms presented in \mychap{optim-polyhedra} provide unified and efficient algorithms for all these situations.
 
 The sparsity-inducing behavior is to be contrasted with a penalty of the form $\sum_{i,j=1}^p d_{ij} ( w_i -w _j)^2$, a quantity often referred to as the graph Laplacian~\cite{chung1997spectral}, which enforces that the weight vector  is \emph{smooth} with respect to the graph (as opposed to piecewise constant). See \mysec{graph} for empirical comparisons.

 \section{Set covers}
\label{sec:cover}
\label{sec:covers}
Given a \emph{non-negative} set-function $D: 2^V \to \rb_+$, then we can define a set-function $F$ through
\BEQ
\label{eq:cover}
F(A)   =   \sum_{G \subseteq V, \ G  \cap A \neq \varnothing} D(G)
= \sum_{G \subseteq V} D(G) \min \{1, |A \cap G| \},
\EEQ
 with \lova extension$f(w) = \sum_{G \subseteq V}D(G) \max_{k \in G} w_k$.

 The submodularity and the \lova extension can be obtained using linearity and the fact that the \lova extension of $A \mapsto 1_{G \cap A \neq \varnothing} = \min \{|A|,1\}$ is $w \mapsto \max_{ k \in G } w_k$. In the context of structured sparsity-inducing norms (see \mysec{sparse}), these correspond to penalties of the form $\Omega_\infty: w \mapsto f(|w|) = \sum_{G \subseteq V}D(G) \|w_G\|_\infty$, thus leading to (potentially overlapping) group Lasso formulations (see, e.g.,~\cite{cap,jenatton2009structured,huang2009learning,LaurentGuillaumeGroupLasso,kim,jenattonmairal,Mairal10aNIPS}).   For example, when $D(G)=1$ for elements of a given partition, and zero otherwise, then $F(A)$ counts the number of elements of the partition with non-empty intersection with~$A$, a function which we have used as a running example throughout this monograph. This leads to the classical non-overlapping grouped $\ell_1/\ell_\infty$-norm.

 However, for $q \in (1,\infty)$, then, as discussed in~\cite{submodlp}, the norm $\Omega_q$ is not equal to $\sum_{G \subseteq V}D(G) \|w_G\|_q$, unless the groups such that $D(G)>0$ form a partition. As shown in~\cite{submodlp} where the two norms are  compared, the norm $\Omega_q$ avoids the overcounting effect of the overlapping group Lasso formulations, which tends to penalize too much the amplitude of variables present in multiple groups.

\paragraph{M\"obius inversion.}
Note that any set-function $F$ may be written as 
\BEAS
 F(A)&  \!\!\! = \!\!\! &  \!\!\! \sum_{G \subseteq V, \ G  \cap A \neq \varnothing}D(G)
=\sum_{ G \subseteq V  }D(G)   -  \sum_{ G \subseteq V \backslash A }D(G),
\\[-.05cm]
\mbox{i.e., } F(V) - F(V \backslash A) & \!\!\! = \!\!\! &  \sum_{ G \subseteq  A }D(G) ,
\EEAS
 for a certain set-function~$D$, \emph{which is not usually non-negative}. Indeed, by the M\"obius inversion formula\footnote{If $F$ and $G$ are any set functions such that $\forall A \subseteq V$, $F(A) = \sum_{B \subseteq A} G(B)$, then $ \forall A \subseteq V$, $G(A) = \sum_{B \subseteq A} (-1)^{ |A \backslash B|} F(B)$~\cite{stanley2011enumerative}.} (see, e.g.,~\cite{stanley2011enumerative,mobius}), we have:
$$D(G) = \sum_{ A \subseteq G} (-1)^{|G| - |A|} \big[ 
F(V) - F(V \backslash A) \big] 
.$$
Thus, functions for which $D$ is non-negative form a specific subset of submodular functions (note that for all submodular functions, the function $D(G)$ is non-negative for all pairs $G =\{i,j\}$, for $j\neq i$, as a consequence of Prop.~\ref{prop:second}). Moreover, these functions are always non-decreasing. For further links, see~\cite{fujishige2005submodular}, where it is notably shown that $D(G)=0$ for all sets $G$ of cardinality greater or equal to three for cut functions (which are second-order polynomials in the indicator vector).

\paragraph{Reinterpretation in terms of set-covers.}
Let $W$ be any ``base'' measurable set, and $\mu$ an additive measure on $W$. We assume that
 for each $k \in V$, a measurable set $S_k \subseteq W$ is given; we define the cover associated with $A$, as the set-function equal to the measure of the union of sets~$S_k$, $k \in A$, i.e., 
 $F(A) = \mu \Big( \bigcup_{k \in A} S_k \Big)$.  See \myfig{covers} for an illustration. Then, $F$ is submodular (as a consequence of the equivalence with the previously defined functions, which we now prove).
 
 These two types of functions  (set covers and the ones defined in \eq{cover}) are in fact equivalent. Indeed, for a weight function $D: 2^V \to \rb_+$, we consider the base set $W$ to be the power-set of $V$, i.e., $W = 2^V$, with the finite measure associating mass $D(G)$ to $G \in 2^V=W$, and with $S_{k} = \{ G \subseteq V, G \ni k \}$. We then get, for all $A \subseteq V$:
 \BEAS
 F(A) &  = & \sum_{G \subseteq V} D(G) 1_{ A \cap G \neq \varnothing } = \sum_{G \subseteq V} D(G) 1_{ \exists k \in A, k \in G } \\
 &= &  \sum_{G \subseteq V} D(G) 1_{ \exists k \in A, G \in S_k} = \mu \Big(  \bigcup_{k \in A} S_k \Big).
 \EEAS
 This implies that $F$ is a set-cover.
 
Moreover, for a certain set cover defined by a mesurable set $W$ (with measure $\mu$), and sets
$S_k \subseteq W$, $k \in V$, we may define for any $ x\in W$, the set $G_x$ of elements of $V$ such that $x \in S_k$, i.e.,  $G_x = \{ k \in V, \ S_k \ni x\}$. We then have:
\BEAS
F(A) & = & \mu \Big( \bigcup_{k \in A} S_k \Big) = \int 1_{x \in \cup_{k \in A} S_k} d\mu(x) = 
\int 1_{A \cap G_x \neq \varnothing} d\mu(x), \\
& = & \sum_{G \subseteq V} 1_{A \cap G \neq \varnothing} \mu \big( \{ x \in W, \ G_x = G \} \big).
\EEAS
Thus, with $D(G) = \mu \big( \{ x \in W, \ G_x = G \} \big)$, we obtain a set-function expressed in terms of groups and non-negative weight functions.

\begin{figure}

\begin{center}
  \includegraphics[scale=1]{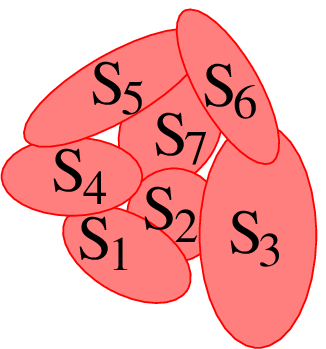}  
 \hspace*{.95cm} \includegraphics[scale=1]{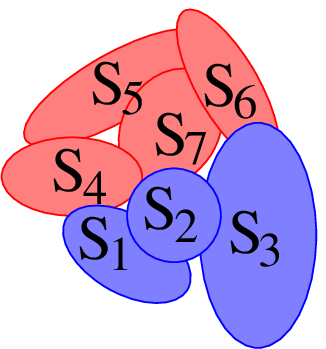}   
  \end{center}

\vspace*{-.6cm}

\caption{Set covers: (left) $p=7$ sets in the two-dimensional planes; (right) in blue, the union of all elements obtained from $A = \{1,2,3\}$.}

\label{fig:covers}

\end{figure}

%
%

\paragraph{Applications to machine learning.}
\label{sec:sensor}
Submodular set-functions which can be expressed as set covers (or equivalently as a sum of maximum of certain components) have several applications, mostly as regular set-covers or through their use in sparsity-inducing norms.

When considered directly as set-functions, submodular functions are traditionally used because algorithms for maximization with theoretical guarantees may be used (see \mychap{max-ds}). See~\cite{krause11submodularity} for several applications, in particular to sensor placement, where the goal is to maximize coverage while bounding the number of sensors.

When considered through their \lova extensions, we obtain structured sparsity-inducing norms which can be used to impose specific prior knowledge into learning problems: indeed, as shown in \mysec{increasing}, they correspond to a convex relaxation to the set-function applied to the support of the predictor. Morever, as shown in~\cite{jenatton2009structured,bach2010structured} and Prop.~\ref{prop:patterns}, they lead to specific sparsity patterns (i.e., supports), which are stable for the submodular function, i.e., such that they cannot be increased without increasing the set-function. For this particular example, stable sets are exactly intersections of complements of groups~$G$ such that $D(G)>0$ (see more details in~\cite{jenatton2009structured}), that is, some of the groups $G$ with non-zero weights $D(G)$ carve out the set $V$ to obtain the support of the predictor. Note that following~\cite{Mairal10aNIPS}, all of these may be interpreted in terms of network flows (see \mysec{flows}) in order to obtain fast algorithms to solve the proximal problems.

\begin{figure}

\begin{center}
 \includegraphics[scale=.65]{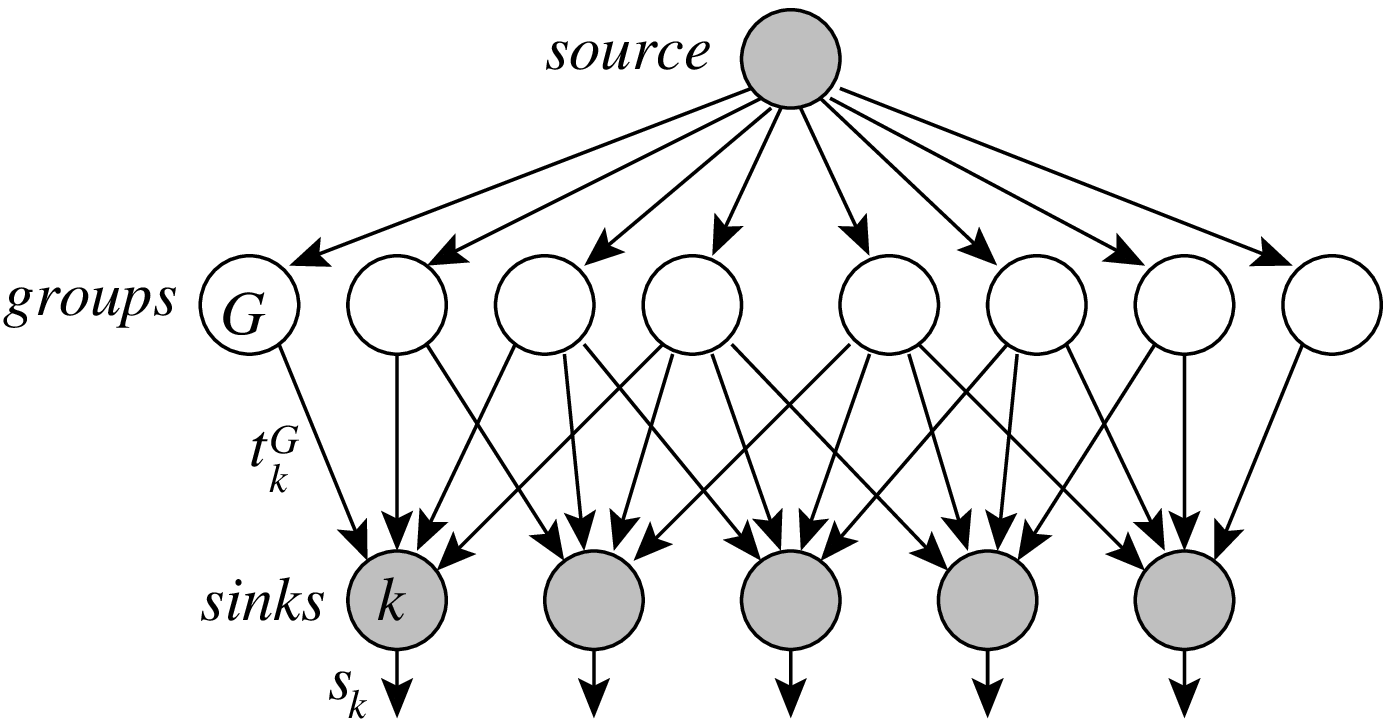} 
 
  \vspace*{.4cm}

 \hspace{1.35cm}
\includegraphics[scale=.65]{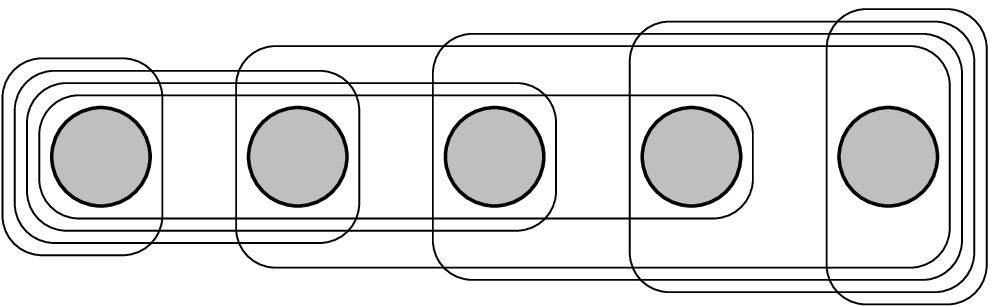}
 \end{center}

\vspace*{-.5cm}

\caption{Flow (top) and set of groups (bottom) for sequences. 
When these groups have unit weights (i.e., $D(G)=1$ for these groups and zero for all others), then the submodular function $F(A)$ is equal (up to constants) to the length of the range of $A$ (i.e., the distance beween the rightmost element of $A$ and the leftmost element of $A$). When applied to sparsity-inducing norms, this leads to supports which are contiguous segments (see applications in~\cite{SparseStructuredPCA}).
}

\label{fig:1dbis}

\end{figure}

 By choosing certain set of groups $G$ such that $D(G)>0$, we can model several interesting behaviors (see more details in \cite{statscience,fot}):

\begin{list}{\labelitemi}{\leftmargin=1.1em}
   \addtolength{\itemsep}{-.0\baselineskip}

\item[--] \textbf{Line segments}: Given $p$ variables organized in a sequence, using the set of groups of Figure~\ref{fig:1dbis}, it is only possible to select \emph{contiguous nonzero patterns}. In this case, we have $p$ groups with non-zero weights, and the submodular function is equal to $p-2$ plus
the length of the range of $A$ (i.e., the distance beween the rightmost element of $A$ and the leftmost element of $A$), if $A \neq \varnothing$ (and zero otherwise). This function is often used together with the cardinality function $|A|$ to avoid selecting long sequences (see an example in \mysec{exp-graph}).

\begin{figure}

\begin{center}
  \includegraphics[scale=.65]{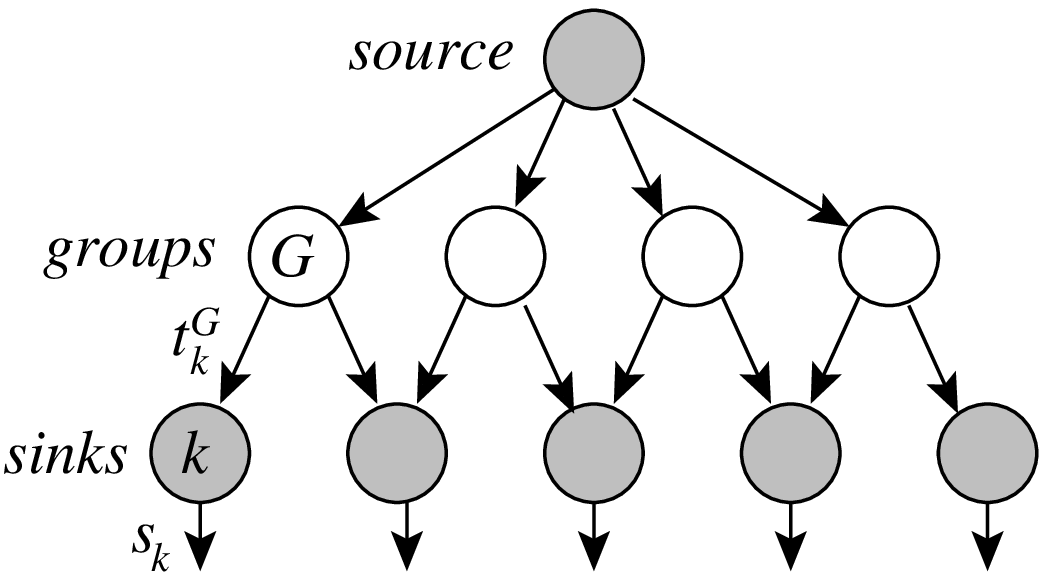}  
  
  \vspace*{.5cm}
  
 \hspace*{.95cm} \includegraphics[scale=.65]{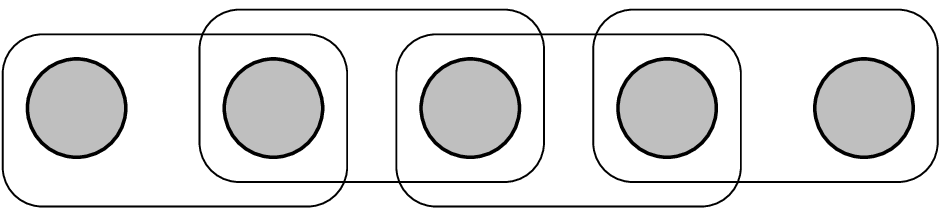}   
  \end{center}

\vspace*{-.5cm}

\caption{Flow (top) and set of groups (bottom) for sequences.
When these groups have unit weights (i.e., $D(G)=1$ for these groups and zero for all others), then the submodular function $F(A)$ is equal to the number of sequential pairs with at least one present element. When applied to sparsity-inducing norms, this leads to supports that have no isolated points (see applications in~\cite{Mairal10aNIPS}).
}

\label{fig:1d}

\end{figure}

\item[--] \textbf{Two-dimensional convex supports}:
Similarly, assume now that the $p$ variables are organized on a two-dimensional grid.
To constrain the allowed supports to be the set of all rectangles on this grid, 
a possible set of groups  to consider may be composed of half planes with specific orientations: if only vertical and horizontal orientations are used, the set of allowed patterns is the set of rectangles, while with more general orientations, more general convex patterns may be obtained. These can be applied for images, and in particular in structured sparse component analysis where the dictionary elements can be assumed to be localized in space~\cite{SparseStructuredPCA}.

\item[--] \textbf{Two-dimensional block structures on a grid}: Using sparsity-inducing regularizations built upon groups which are composed of variables together with their spatial neighbors (see \myfig{1dbis} in one-dimension) leads to 
 good performances for background subtraction~\cite{Cevher2008,Baraniuk2008,huang2009learning,Mairal10aNIPS},
 topographic dictionary learning~\cite{Kavukcuoglu2009, mairal2011b}, wavelet-based denoising~\cite{Rao2011}.  This norm typically prevents isolated variables from being selected.
 
\item[--] \textbf{Hierarchical structures}:
here we assume that the variables are organized in a hierarchy. Precisely, 
we assume that the $p$ variables can be assigned to the nodes of a tree  (or a forest of trees), and that a given variable may be selected 
 only if all its ancestors in the tree have already been selected. This corresponds to a set-function which counts the number of ancestors of a given set $A$ (note that the stable sets of this set-function are exactly the ones described above).
 
This hierarchical rule is exactly respected when using the family of groups displayed on Figure~\ref{fig:tree}.
The corresponding penalty was first used in~\cite{cap}; one of it simplest instance in the context of regression is the sparse group Lasso~\cite{sprechmann2010collaborative,friedman2010note}; 
it has found numerous applications, for instance, 
 wavelet-based denoising~\cite{cap, Baraniuk2008,huang2009learning,Jenatton2010b},
 hierarchical dictionary learning for both topic modelling and image restoration~\cite{jenattonmairal,Jenatton2010b},
 log-linear models for the selection of potential orders~\cite{Schmidt2010}, 
 bioinformatics, to exploit the tree structure of gene networks for multi-task regression~\cite{kim},
 and multi-scale mining of fMRI data for the prediction of simple cognitive tasks~\cite{Jenatton2011}. See also \mysec{exp-wavelet} for an application to non-parametric estimation with a wavelet basis.

\item[--] \textbf{Extensions}:
Possible choices for the sets of groups (and thus the set functions) are not limited to the aforementioned examples; more complicated topologies can be considered, for example three-dimensional spaces discretized in cubes or spherical volumes discretized in slices (see an application to neuroimaging by \cite{Varoquaux2010a}), and more complicated hierarchical structures based on directed acyclic graphs can be encoded as further developed in~\cite{hkl} to perform non-linear variable selection.

\end{list}

\paragraph{Covers vs. covers.} 
Set covers also classically occur in the context of submodular function maximization, where the goal is, given certain subsets of $V$, to find the least number of these that completely cover~$V$. Note that the main difference is that in the context of set covers considered here, the cover is considered on a potentially different set $W$ than $V$, and each element of $V$ indexes a subset of $W$.

\begin{figure}

\begin{center}

\includegraphics[scale=.55]{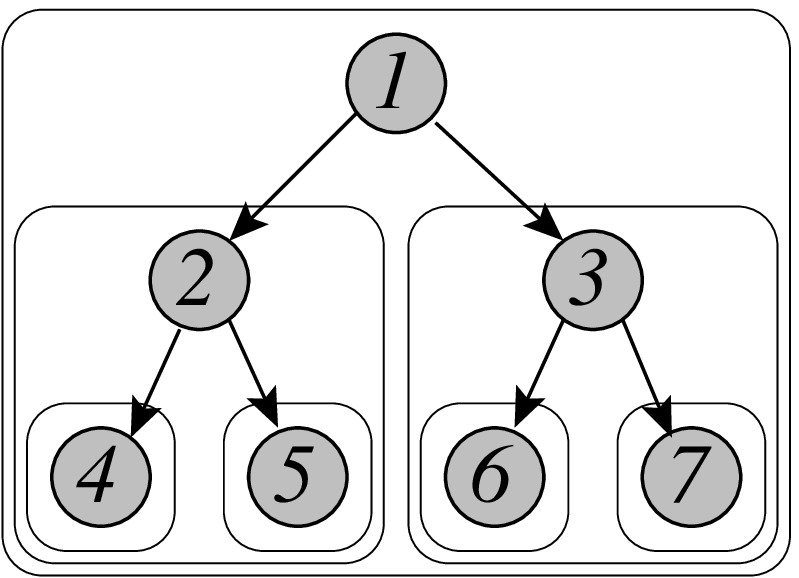} \hspace*{1cm}
\includegraphics[scale=.55]{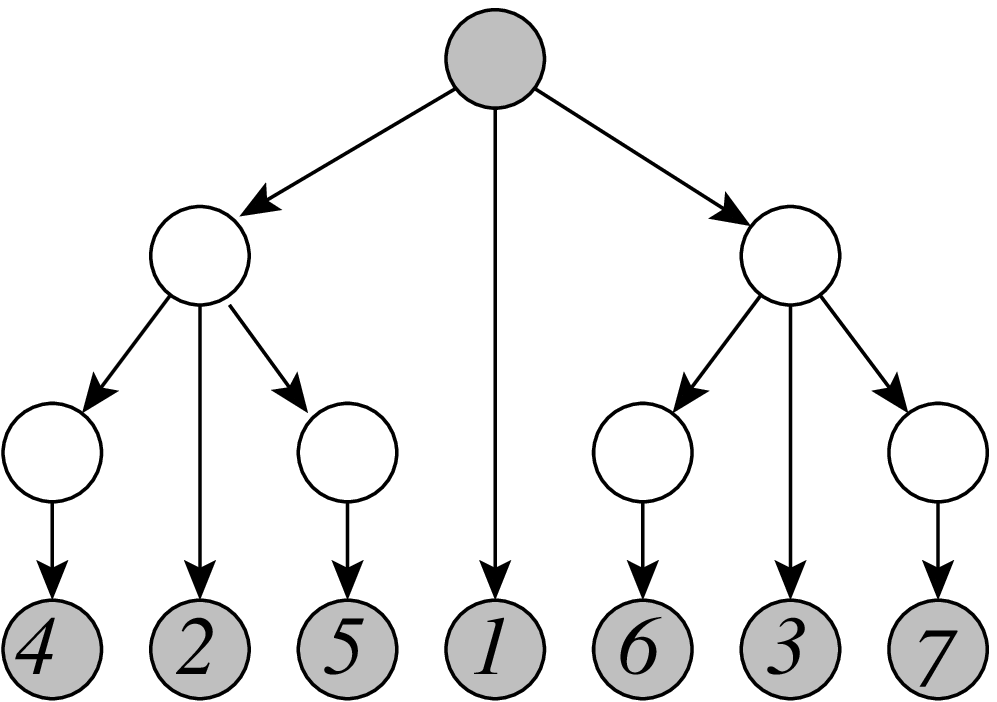}
 \end{center}

\vspace*{-.5cm}

\caption{Left: Groups corresponding  to a hierarchy. Right: (reduced) network flow interpretation of same submodular function
(see \mysec{flows}). When these groups have unit weights (i.e., $D(G)=1$ for these groups and zero for all others), then the submodular function $F(A)$ is equal to the cardinality of the union of all ancestors of $A$. When applied to sparsity-inducing norms, this leads to supports that select a variable only after all of its ancestors have been selected (see applications in~\cite{jenattonmairal}).
}
\label{fig:tree}
\end{figure}

\section{Flows}
\label{sec:flows}
Following~\cite{megiddo1974optimal}, we can obtain a family of non-decreasing submodular set-functions (which include set covers from \mysec{covers}) from multi-sink multi-source networks. We consider  a set $W$ of vertices, which includes a set $S$ of sources and a set $V$ of sinks (which will be the set on which the submodular function will be defined). We assume that we are given capacities, i.e., a function $c$ from $W \times W$ to $\rb_+$. For all functions $\varphi: W \times W \to \rb$, we use the notation $\varphi(A,B) = \sum_{k \in A, \ j \in B} \varphi(k,j)$.

A flow is a function $\varphi: W \times W \to \rb_+$ such that:

\begin{list}{\labelitemi}{\leftmargin=2.1em}
   \addtolength{\itemsep}{-.0\baselineskip}

\item[(a)]  capacity constaints: $\varphi \leqslant c$ for all arcs,
\item[(b)]  flow conservation: for all $w \in W \backslash (S \cup V )$, the net-flow at $w$, i.e., $\varphi(W,\{w\}) - \varphi( \{w\},W)$, is zero,
\item[(c)] positive incoming flow: for all sources $s \in S$,  the net-flow at $s$ is non-positive, i.e., $\varphi(W,\{s\}) - \varphi( \{s\},W) \leqslant 0$,
\item[(d)] positive outcoming flow:   for all sinks $t \in V$, 
the net-flow at $t$ is non-negative, i.e., $\varphi(W,\{t\}) - \varphi( \{t\},W) \geqslant 0$.
\end{list}
We denote by $\mathcal{F}$ the set of flows, and $\varphi(w_1,w_2)$ is the flow going from $w_1$ to $w_2$. This set $\mathcal{F}$ is a polyhedron in $\rb^{W \times W}$ as it is defined by a set of linear equality and inequality constraints

\begin{figure}
\begin{center}
 \includegraphics[scale=.7]{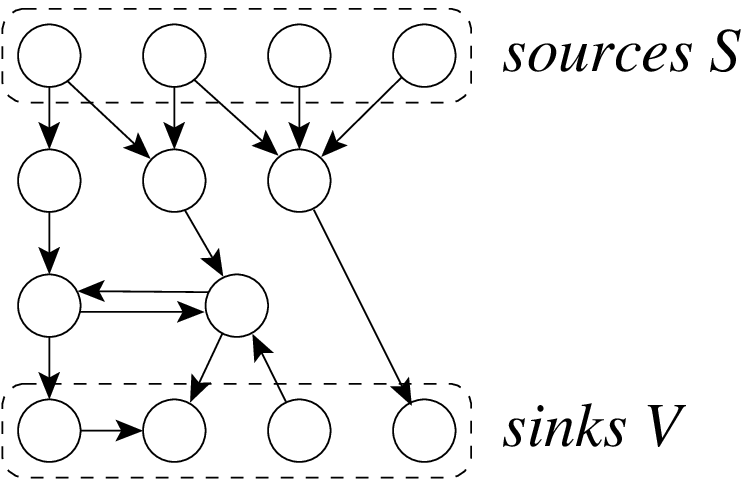}
 \end{center}
 
 \vspace*{-.4cm}
 
\caption{Flows. Only arcs with strictly positive capacity are typically displayed. Flow comes in by the sources and gets out from the sinks.}
\label{fig:flow}
\end{figure}

For $A \subseteq V$ (the set of sinks), we define 
$$ \displaystyle F(A) = \max_{ \varphi \in \mathcal{F}} \ \varphi(W,A) - \varphi(A,W),$$
which is the maximal net-flow getting out of $A$. From the max-flow/min-cut theorem (see, e.g.,~\cite{cormen89introduction} and Appendix~\ref{app:maxflow}), we have immediately that 
$$
F(A) = \min_{X \subseteq W, \ S \subseteq X, \ A \subseteq W \backslash X } c(X, W \backslash X).
$$

One then obtains that $F$ is submodular (as the partial minimization of a cut function, see Prop.~\ref{prop:partial}) and non-decreasing by construction. One particularity is that for this type of submodular  non-decreasing functions, we have an explicit description of the intersection of the positive orthant and the submodular polyhedron, i.e., of the positive submodular polyhedron $P_+(F)$ (potentially simpler than through the supporting hyperplanes $\{ s(A) = F(A)\}$). Indeed, $s \in \rb_+^p$ belongs to $P(F)$ if and only if, there exists a flow $\varphi \in \mathcal{F}$ such that for all $k \in V$, $s_k = \varphi(W,\{k\}) - \varphi( \{k\},W)$ is the net-flow getting out of $k$.

Similarly to other cut-derived functions from \mysec{cuts}, there are dedicated algorithms for proximal methods and submodular minimization~\cite{hochbaum1995strongly}.   See also \mysec{decomp} for a general divide-and-conquer strategy for solving separable optimization problems based on a sequence of submodular function minimization problems (here, min cut/max flow problems).

\paragraph{Flow interpretation of set-covers.} Following~\cite{Mairal10aNIPS}, we now show that the submodular functions defined in this section includes
the set covers defined in \mysec{cover}. Indeed, consider a  {non-negative} function $D: 2^V \to \rb_+$, and define $F(A) = \sum_{G \subseteq V,\ G  \cap A \neq \varnothing}D(G)$. The \lova extension may be written as, for all $w \in \rb_+^p$ (introducing variables $t^G$ in a scaled simplex reduced to variables indexed by $G$):
\BEAS
f(w) & = &  \sum_{G \subseteq V}D(G) \max_{k \in G} w_k
\\
& = &  \sum_{G \subseteq V} \ \ \max_{ t^G \in \rb_+^p, \ t^G_{V \backslash G}=0, \ t^G(G) = D(G)}   w^\top t^G \\
& = &   \max_{ t^G \in \rb_+^p, \ t^G_{V \backslash G}=0, \ t^G(G) = D(G), \ G \subseteq V}   \sum_{G \subseteq V}  w^\top t^G \\
& = &   \max_{ t^G \in \rb_+^p, \ t^G_{V \backslash G}=0, \ t^G(G) = D(G), \ G \subseteq V}   \sum_{ k \in V} \bigg(  \sum_{G \subseteq V, \ G \ni k}  t^G_k \bigg) w_k .
\EEAS
Because of the representation of $f$ as a maximum of linear functions shown in Prop.~\ref{prop:greedy}, $s \in P(F)
\cap \rb_+^p = P_+(F)
$, if and only there exists $t^G \in \rb_+^p, \ t^G_{V \backslash G}=0, \ t^G(G) = D(G)$ for all $G \subseteq V$, such that for all $k \in $V, $ s_k =  \sum_{G \subseteq V, \ G \ni k}  t^G_k  $. This can be given a network flow interpretation on the graph composed of a single source, one node per subset $G \subseteq V $ such that $ D (G)>0$, and the sink set $V$.  
 The source is connected to all subsets $G$, with capacity $ D (G)$, and each subset is connected to the variables it contains, with infinite capacity. In this representation, $t^G_k$ is the flow from the node corresponding to $G$, to the node corresponding to the sink node $k$; and $s_k = \sum_{G \subseteq V} t_k^G$ is the net-flow getting out of the sink vertex $k$. Thus, $s \in P(F) \cap \rb_+^p$ if and only if, there exists a flow in this graph so that the net-flow getting out of $k$ is $s_k$, which corresponds exactly to a network flow submodular function.

 We give examples of such networks in \myfig{1d} and \myfig{1dbis}. This reinterpretation allows the use of fast algorithms for proximal problems (as there exists fast algorithms for maximum flow problems). The number of nodes in the network flow is the number of groups $G$ such that $D(G)>0$, but this number may be reduced in some situations (for example, when a group is included in another, see an example of a reduced graph in \myfig{tree}). See~\cite{Mairal10aNIPS,mairal2011b} for more details on such graph constructions (in particular in how to reduce the number of edges in many situations).

\paragraph{Application to machine learning.}
Applications to sparsity-inducing norms (as decribed in \mysec{cover}) lead to applications to hierarchical dictionary learning and topic models~\cite{jenattonmairal}, structured priors for image denoising~\cite{jenattonmairal,Jenatton2010b}, background subtraction~\cite{Mairal10aNIPS}, and bioinformatics~\cite{LaurentGuillaumeGroupLasso,kim}. Moreover, many submodular functions may be interpreted in terms of flows, allowing the use of fast algorithms (see, e.g.,~\cite{hochbaum1995strongly,ahuja1993network} for more details).

\section{Entropies}
\label{sec:entropies}
Given $p$ random variables $X_1,\dots,X_p$ which all take a finite number of values, we define $F(A)$ as the joint entropy of the variables $(X_k)_{k \in A}$ (see, e.g.,~\cite{cover91elements}). This function is submodular because, if $A \subseteq B$ and $k \notin B$,
$F(A \cup \{ k\}) - F(A) = H(X_A,X_k)-H(X_A) = H(X_k|X_A)  \geqslant  H(X_k|X_B) =   F(B \cup \{ k\}) - F(B)
$ (because conditioning reduces the entropy~\cite{cover91elements}). Moreover, its symmetrization\footnote{For any submodular function $F$, one may defined its symmetrized version as $G(A) = F(A) + F(V \backslash A) - F(V)$, which is submodular and symmetric. See further details in \mysec{posi} and Appendix~\ref{app:ope}.} leads to the mutual information between variables indexed by $A$ and variables indexed by~$V \backslash A$.

This can be extended to any distribution by considering differential entropies. For example, if $X \in \rb^p$ is a multivariate  random vector having a Gaussian distribution with invertible covariance matrix $Q$, then $H(X_A) = \frac{1}{2} \log \det ( 2 \pi e  Q_{AA})$~\cite{cover91elements}. This leads to the submodularity of the function defined through $F(A) = \log \det Q_{AA}$, for some positive definite matrix $Q \in \rb^{p \times p}$ (see further related examples in 
\mysec{spectral}).

\paragraph{Entropies are less general than submodular functions.}
Entropies of discrete variables are non-decreasing, non-negative submodular set-functions. However, they are more restricted than this, i.e., they satisfy other properties which are not satisfied by all submodular functions~\cite{zhang1998characterization}. Note also that it is not known if their special structure can be fruitfully exploited to speed up certain of the algorithms presented in~\mychap{sfm}.

\paragraph{Applications to graphical model structure learning.}
In the context of probabilistic graphical models, entropies occur in particular in algorithms for structure learning: indeed, for directed graphical models, given the directed acyclic graph, the minimum Kullback-Leibler divergence between a given distribution and a distribution that factorizes into the graphical model may be expressed in closed form through entropies~\cite{lauritzen96graphical,heckerman1995learning}. 
Applications of submodular function optimization may be found in this context, with both minimization~\cite{narasimhan2004pac,chechetka2007efficient} for learning bounded-treewidth graphical model and maximization for learning naive Bayes models~\cite{krause2005near}, or both (i.e., minimizing differences of submodular functions, as shown in \mychap{max-ds}) for discriminative learning of structure~\cite{narasimhan2006submodular}.
In particular, for undirected graphical models, finding which subsets of vertices are well-separated by a given subset $S$ corresponds to minimizing over all non-trivial subsets of $V\backslash S$ the symmetric submodular function $B \mapsto I(X_B,X_{V\backslash (S \cup B)} | X_S)$~\cite{narasimhan2004pac,chechetka2007efficient}, which may be done in polynomial time (see \mysec{posi}).

\paragraph{Applications to experimental design.}
\label{sec:design}
Entropies also occur in \emph{experimental design} in Gaussian linear models~\cite{seeger-submod}. Given a design matrix $X \in \rb^{ n \times p}$, assume that the vector $y \in \rb^n$ is  distributed as $X v + \sigma \varepsilon$, where  $v$ has a normal prior distribution with mean zero and covariance matrix  $\sigma^2 \lambda^{-1} \idm$, and $\varepsilon \in \rb^n$ is a standard normal vector. 

The joint distribution of $(v,y)$ is normally distributed with mean zero and covariance matrix
$\sigma^2 \lambda^{-1}  \left( \begin{array}{cc} \idm & X^\top \\ X & XX^\top + \lambda \idm \end{array} \right)$. The posterior distribution of $v$ given $y$ is thus normal with mean
$ 
\cov(v,y)  \cov(y,y)^{-1} y  =  X^\top (XX^\top + \lambda \idm)^{-1} y =  (X^\top X + \lambda \idm)^{-1} X^\top y $
and covariance matrix 
\BEAS & & \cov(v,v) - \cov(v,y)  \cov(y,y)^{-1} \cov(y,v) \\
& =  & 
\lambda^{-1} \sigma^2  \Big[
\idm - X^\top  (XX^\top + \lambda \idm)^{-1}  X
\Big] \\
& = &  \lambda^{-1} \sigma^2  \Big[
\idm -  ( X^\top X + \lambda \idm)^{-1}  X^\top  X
\Big] =  \sigma^2 ( X^\top X + \lambda \idm)^{-1} 
,\EEAS
where we have used the matrix inversion lemma~\cite{horn1990matrix}.
 The posterior entropy of $v$ given $y$ is thus equal (up to constants)
to  $p \log \sigma^2 - \log \det ( X^\top X + \lambda   \idm)$.  If only the observations in $A$ are observed, then the posterior entropy of $v$ given $y_A$ is equal to
$F(A) = p \log \sigma^2 - \log \det ( X_A^\top X_A + \lambda  \idm)$,
 where $X_A$ is the submatrix of $X$ composed of rows of $X$ indexed by $A$. We have moreover
 $F(A) = p \log \lambda^{-1} \sigma^2 - \log \det ( \lambda^{-1} X_A^\top X_A +   \idm)
 =p \log \lambda^{-1} \sigma^2 - \log \det ( \lambda^{-1}  X_A X_A^\top +   \idm)$, and thus
 $F(A)$ is supermodular because the entropy of a Gaussian random variable is the logarithm of its determinant. In experimental design, the goal is to select the set $A$ of observations so that the posterior entropy of $v$ given $y_A$ is minimal (see, e.g.,~\cite{fedorov1972theory}), and is thus equivalent to maximizing a submodular function (for which forward selection has theoretical guarantees, see \mysec{max-card}). Note the difference with subset selection (\mysec{subset}) where the goal is to select columns of the design matrix instead of rows.

In this particular example, the submodular function we aim to maximize may be written as $F(A) = g(1_A)$, where $g(w) = p \log ( \sigma^2 \lambda^{-1}) - \log \det
\big(  \idm + \lambda^{-1} \sum_{i = 1}^p w_i X_i X_i^\top
\big)
$, where $X_i \in \rb^{n}$ is the $i$-th column of $X$. The function $g$ is concave and  should not be confused with the \lova extension which is convex and piecewise affine. It is not also the concave closure of $F$ (see definition in \mysec{closure}); therefore maximizing $g(w)$ with respect to $w \in [0,1]^p$ is not equivalent to maximizing $F(A)$ with respect to $A \subseteq V$. However, it readily leads to a convex relaxation of the problem of maximizing $G$, which is common in experimental design (see, e.g.,~\cite{pukelsheim2006optimal,bouhtou2010submodularity}).

\paragraph{Application to semi-supervised clustering.}
Given $p$ data points $x_1,\dots,x_p$ in a certain set $\mathcal{X}$, we assume that we are given a Gaussian process $(f_x)_{x \in \mathcal{X}}$. For any subset $A \subseteq V$, then $f_{x_A}$ is normally distributed with mean zero and covariance matrix $K_{AA}$ where $K$ is the $p\times p$ kernel matrix of the $p$ data points, i.e., $K_{ij} = k(x_i,x_j)$ where $k$ is the kernel function associated with the Gaussian process (see, e.g.,~\cite{GP}). We assume an independent prior distribution on subsets of the form $p(A) \propto \prod_{k \in A} \eta_k
\prod_{k \notin A} (1-\eta_k) $ (i.e., each element $k$ has a certain prior probability $\eta_k$  of being present, with all decisions being statistically independent).

 Once a set $A$ is selected, we only assume that we want to model the two parts, $A$ and $V \backslash A$ as  two \emph{independent} Gaussian processes with covariance matrices $\Sigma_A$ and $\Sigma_{V \backslash A}$. In order to maximize the likelihood under the joint Gaussian process, the best estimates are $\Sigma_A = K_{AA}$ and
 $\Sigma_{V \backslash A} = K_{V \backslash A,V \backslash A}$. This leads to the following negative log-likelihood
 $$I(f_A, f_{V \backslash A}) - \sum_{k \in A} \log \eta_k -  \sum_{k \in V \backslash A} \log ( 1 - \eta_k ),$$
where $I(f_A, f_{V \backslash A})$ is the mutual information between two Gaussian processes (see similar reasoning in the context of independent component analysis~\cite{cardoso2003dependence}).

In order to estimate $A$, we thus need to minimize a modular function plus a mutual information between the variables indexed by $A$ and the ones indexed by $V \backslash A$, which is submodular and symmetric. Thus in this Gaussian process interpretation, clustering may be cast as submodular function minimization. This probabilistic interpretation extends the minimum description length interpretation of~\cite{narasimhan2006q} to semi-supervised clustering.

Note here that similarly to the unsupervised clustering framework of~\cite{narasimhan2006q}, the mutual information may be replaced by any symmetric submodular function, such as a cut function obtained from appropriately defined weigths. In Figure~\ref{fig:twomoon}, we consider $\mathcal{X} = \rb^2$ and sample points from a traditional distribution in semi-supervised clustering, i.e., twe ``two-moons'' dataset. We consider 100 points and 8 randomly chosen labelled points, for which we impose $\eta_k \in \{0,1\}$, the rest of the $\eta_k$ being equal to 1/2 (i.e, we impose a hard constraint on  the labelled points to be on the correct clusters). We consider a Gaussian kernel $k(x,y) = \exp(- \alpha \| x - y\|_2^2)$, and we compare two symmetric submodular functions: mutual information and the weighted cuts obtained from the same matrix $K$ (note that the two functions use different assumptions regarding the kernel matrix, positive definiteness for the mutual information, and pointwise positivity for the cut). As shown in Figure~\ref{fig:twomoon}, by using more than second-order interactions, the mutual information is better able to capture the structure of the two clusters.  This example is used as an illustration and more experiments and analysis would be needed to obtain sharper statements. In \mychap{experiments}, we use this example for comparing different submodular function minimization procedures.  Note that even in the case of symmetric submodular functions $F$, where more efficient algorithms  in $O(p^3)$ for submodular function minimization (SFM) exist~\cite{queyranne1998minimizing} (see also \mysec{posi}), the minimization of functions of the form $  F(A) - z(A)$, for $z \in \rb^p$ is provably as hard as general SFM~\cite{queyranne1998minimizing}.

\begin{figure}
\begin{center}
\hspace*{-.5cm}
\includegraphics[scale=.4]{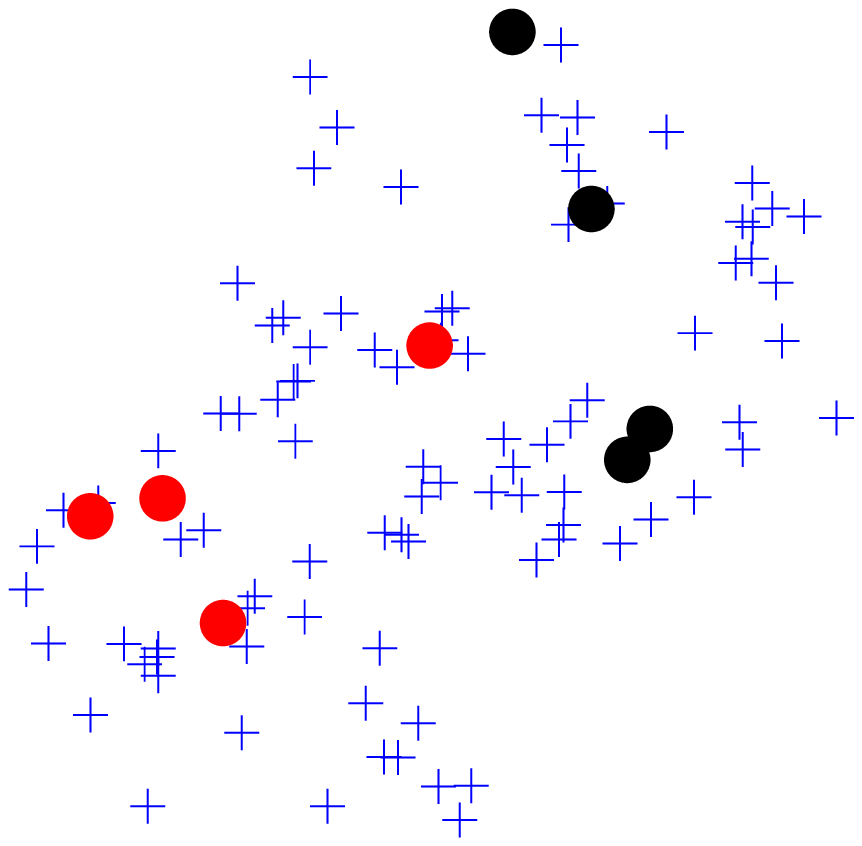} \hspace*{.2cm}
\includegraphics[scale=.4]{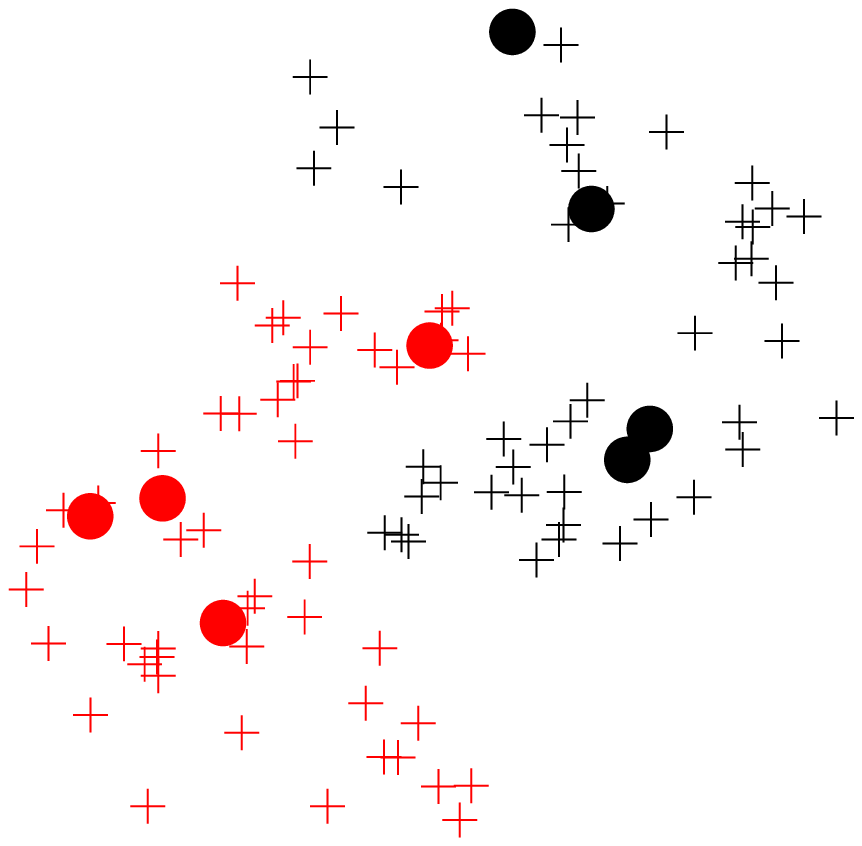} \hspace*{.2cm}
\includegraphics[scale=.4]{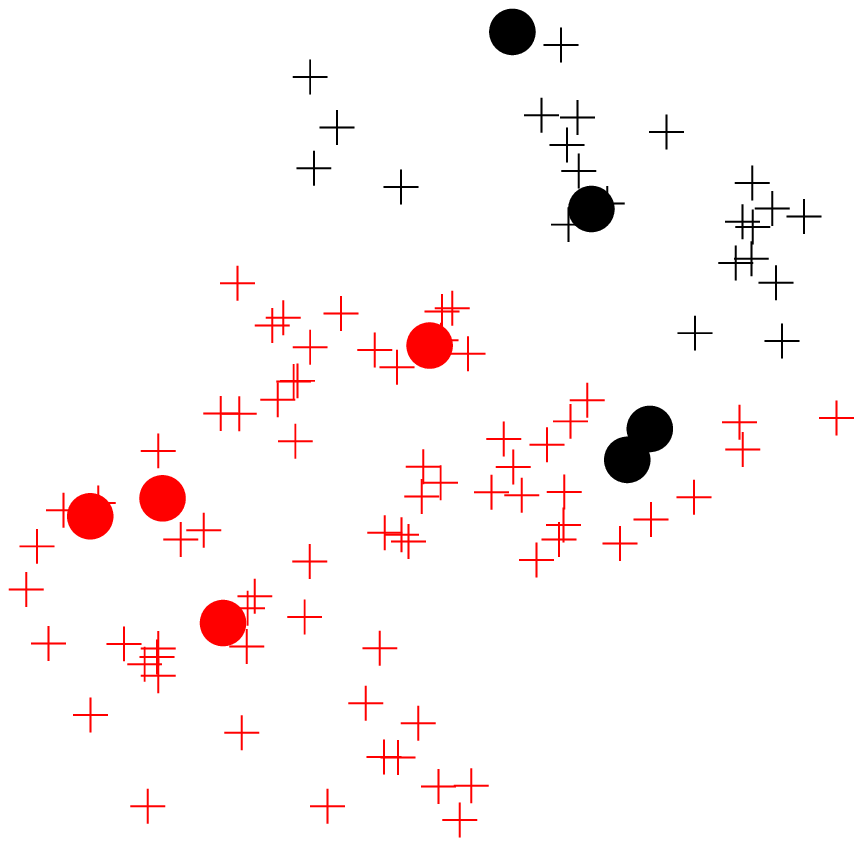}
\hspace*{-.5cm}
\end{center}

\vspace*{-.25cm}

\caption{Examples of semi-supervised clustering : (left) observations, (middle) results of the semi-supervised clustering algorithm based on submodular function minimization, with eight labelled data points, with the mutual information, (right) same procedure with a cut function.}
\label{fig:twomoon}
\end{figure}

\paragraph{Applying graphical model concepts to submodular functions.}

In a graphical model, the entropy of the joint distribution decomposes as a sum of marginal entropies of subsets of variables; moreover, for any distribution, the entropy of the closest distribution factorizing in the graphical model provides an  bound on the entropy. For directed graphical models, this last property turns out to be a direct consequence of the submodularity of the entropy function, and allows the generalization of graphical-model-based upper bounds to any submodular functions. In~\cite{kumar13GM}, these bounds are defined and used within a variational inference framework for the maximization of submodular functions.

\section{Spectral functions of submatrices}
\label{sec:spectral}

Given a positive semidefinite matrix $Q \in \rb^{ p \times p }$ and a real-valued  function $h$ from $\rb_+$ to $\rb$, one may define the matrix function $Q \mapsto h(Q)$ defined on positive semi-definite matrices by leaving unchanged the eigenvectors of $Q$ and applying $h$ to each of the eigenvalues~\cite{golub83matrix}. This leads to the expression of $\tr [h(Q)]$ as $\sum_{i=1}^p h(\lambda_i)$ where $\lambda_1,\dots,\lambda_p$ are the (nonnegative) eigenvalues of~$Q$~\cite{horn1990matrix}. We can thus define the function $F(A) = \tr h(Q_{AA})$ for $A \subseteq V$.  Note that for $Q$ diagonal (i.e., $Q = \Diag(s)$), we exactly recover functions of modular functions considered in \mysec{card}.

The concavity of $h$ is not sufficient  however in general to ensure the submodularity of $F$, as can be seen by generating random examples with $h(\lambda) = \lambda/( \lambda+1)$.

Nevertheless, we know  that the functions $h(\lambda) = \log( \lambda + t)$ for $t\geqslant 0$ lead to submodular functions since they correspond to the entropy of a Gaussian random variable with joint covariance matrix $Q + \lambda \idm$. Thus, since for $\rho \in (0,1)$,
$\lambda^\rho = \frac{ \rho \sin \rho \pi }{\pi} \int_0^\infty  \log (1+ \lambda/t) t^{\rho-1} dt$~(see, e.g.,~\cite{ando1979concavity}), $h(\lambda) = \lambda^\rho$ for $\rho \in (0,1]$ is a positive linear combination  of functions that lead to non-decreasing submodular set-functions.  We thus obtain a
non-decreasing submodular function. 

This can be generalized to functions of the singular values of submatrices of $X$ where $X$ is a rectangular matrix, by considering the fact that singular values of a matrix $X$ are related to the non-zero eigenvalues of
$\left( \begin{array}{cc} 0 & X \\ X^\top & 0 \end{array} \right)$ (see, e.g.,~\cite{golub83matrix}).

\paragraph{Application to machine learning (Bayesian variable selection).}
As shown in~\cite{bach2010structured}, such spectral functions naturally appear in the context of variable selection using the Bayesian marginal likelihood~(see, e.g.,~\cite{gelman2004bayesian}). Indeed, given a subset $A$, assume that the vector $y \in \rb^n$ is  distributed as $X_A w_A + \sigma \varepsilon$, where $X$ is a design matrix in $\rb^{n \times p}$ and $w_A$ a vector with support in $A$, and $\varepsilon \in \rb^n$ is a standard normal vector. Note that there, as opposed to the experimental design situation of \mysec{entropies}, $X_A$ denotes the submatrix of $X$ with \emph{columns} indexed by~$A$. If a normal prior with covariance matrix $\sigma^2 \lambda^{-1} \idm$ is imposed on $w_A$, then the negative log-marginal likelihood of $y$ given $A$ (i.e., obtained by marginalizing out $w_A$), is equal to (up to constants):
\BEQ
\label{eq:wa}
\min_{w_A \in \rb^{|A|}} \frac{1}{2\sigma^2} \| y - X_A w_A \|_2^2 + \frac{\lambda}{2\sigma^2}\| w_A\|_2^2 + \frac{1}{2}\log \det [\sigma^2 \lambda^{-1} X_A X_A^\top + \sigma^2 \idm].
\EEQ
Indeed, the marginal likelihood is obtained by the best log-likelihood when \emph{maximizing} with respect to $w_A$ plus the entropy of the covariance matrix~\cite{seeger2008bayesian}.

Thus, in a Bayesian model selection setting, in order to find the best subset $A$, it is necessary to minimize with respect to $w$:
$$
\!\! \min_{w \in \rb^p}  \frac{1}{2\sigma^2} \| y - X w \|_2^2 + \frac{\lambda}{2 \sigma^2}\| w\|_2^2 + \frac{1}{2}\log \det [\lambda^{-1} \sigma^2 X_{\supp(w)} X_{\supp(w)}^\top + \sigma^2 \idm],
$$
which, in the framework outlined in \mysec{l2}, leads to the submodular function 
$$F(A) = \frac{1}{2}\log \det [\lambda^{-1} \sigma^2 X_{A} X_{A}^\top + \sigma^2 \idm] = 
\frac{1}{2}\log \det [X_{A} X_{A}^\top + \lambda \idm] + \frac{n}{2} \log ( \lambda^{-1} \sigma^2  ).$$
 Note also that, since we use a penalty which is the sum of a squared $\ell_2$-norm and a submodular function applied to the support, then a direct convex relaxation may be obtained through reweighted least-squares formulations using the $\ell_2$-relaxation of combinatorial penalties presented in \mysec{l2} (see also~\cite{submodlp}). See also related simulation experiments for random designs from the Gaussian ensemble in~\cite{bach2010structured}.

Note that a traditional frequentist criterion is to penalize larger subsets $A$ by the Mallow's $C_L$ criterion~\cite{Mal:1973}, which is equal to $A \mapsto \tr( X_A X_A^\top + \lambda \idm)^{-1} X_A X_A^\top$, which is \emph{not} a submodular function.

\section{Best subset selection}
\label{sec:subset}
Following~\cite{das2008algorithms}, we consider $p$ random variables (covariates) $X_1,\dots,X_p$, and a random response $Y$ with unit variance, i.e., $\var (Y)=1$. We consider predicting $Y$ linearly from $X$. We consider
$F(A) = \var(Y) - \var(Y|X_A)$. The function $F$ is a non-decreasing function (the conditional variance of $Y$ decreases as we observed more variables).  In order to show the submodularity of $F$ using Prop.~\ref{prop:second}, we compute, for all $A \subseteq V$, and $i,j $ distinct elements in $ V \backslash A$, the following quantity:
\BEAS
& & F(A \cup \{j,k\}) - F (A \cup \{j\}) - F(A \cup\{k\}) + F(A) \\
& \!\!\! \!\!\!  =  \!\!\!  & [ \var(Y|X_A,X_k) - \var(Y|X_A)  ]  - [ \var(Y|X_A,X_j,X_k) - \var(Y|X_A,X_j) ] \\
& \!\!\!  \!\!\!  =  \!\!\! & - {\rm Corr}( Y, X_k | X_A)^2 + {\rm Corr}( Y, X_k | X_A,X_j)^2,
\EEAS
using standard arguments for conditioning variances (see more details in~\cite{das2008algorithms}). Thus, the function is submodular if and only if the last quantity is always non-positive, i.e., $ |{\rm Corr}( Y, X_k | X_A,X_j)| \leqslant |{\rm Corr}( Y, X_k | X_A)|$, which is often referred to as the fact that the variables $X_j$ is not a suppressor for the variable $X_k$ given $A$.

Thus greedy algorithms for maximization have theoretical guarantees (see \mychap{max-ds}) \emph{if} the assumption is met. Note however that the condition on suppressors is rather strong, although it can be appropriately relaxed in order to obtain more widely applicable guarantees for subset selection~\cite{das2011submodular}.

\paragraph{Subset selection as the difference of two submodular functions.}
We may also consider the linear model from the end of \mysec{spectral}, where a Bayesian approach is taken for model selection, where the parameters $w$ are marginalized out. We can now also maximize the marginal likelihood with respect to the noise variance $\sigma^2$, instead of considering it as a fixed hyperparameter. This corresponds to minimizing \eq{wa} with respect to $\sigma$ and $w_A$, with optimal values
$w_A = (X_A^\top X_A + \lambda \idm)^{-1} X_A^\top y$ and
$
\sigma^2 =  \frac{1}{n} \| y - X_A w_A \|_2^2 + \frac{\lambda}{n} \| w_A\|^2
$,
leading to the following cost function in $A$ (up to constant additive terms):
\BEAS
 & &  \frac{n}{2} \log   y^\top ( \idm - X_A  (X_A^\top X_A + \lambda \idm)^{-1} X_A^\top ) y   
+  \frac{1}{2}\log \det [  X_A^\top  X_A +  \lambda \idm] \\
 & = \!\!\!&   \frac{n}{2} \log \det \bigg(
\begin{array}{cc}
X_A^\top X_A + \lambda \idm  & X_A^\top y \\
y^\top X_A & y^\top y
\end{array}
\bigg)
 -  \frac{n-1}{2} \log \det ( X_A^\top X_A + \lambda \idm ) , 
 \EEAS
which is a difference of two submodular functions (see \mysec{ds} for related optimization schemes). 
Note the difference between this formulation (aiming at minimizing a set-function directly by marginalizing out or maximizing out $w$) and the one from \mysec{spectral} which provides a convex relaxation of the maximum likelihood problem by maximizing the likelihood with respect to $w$.

\section{Matroids}
\label{sec:matroids}

Matroids have emerged as combinatorial structures that generalize the notion of linear independence betweens columns of a matrix. 
Given a set~$V$, we consider a family $\mathcal{I}$ of subsets of $V$ with the following properties:
\begin{list}{\labelitemi}{\leftmargin=2.1em}
   \addtolength{\itemsep}{-.0\baselineskip}

\item[(a)] $\varnothing \in \mathcal{I}$,
\item[(b)]  ``hereditary property'': $I_1 \subseteq I_2 \in \mathcal{I} \Rightarrow I_1 \in \mathcal{I}$,
\item[(c)]  ``exchange property'': for all $I_1,I_2 \in \mathcal{I}$, $|I_1| < |I_2| \Rightarrow \exists k \in I_2 \backslash I_1, \ I_1 \cup \{k \} \in \mathcal{I}$.
\end{list}

The pair $(V, \mathcal{I})$ is then referred to as a \emph{matroid}, with $\mathcal{I}$ its family of \emph{independent sets}. Given any set $A \subseteq V$, then a \emph{base} of $A$ is any independent subset of $A$ which is maximal for the inclusion order (i.e., no other independent set contained in $A$ contains it). An immediate consequence of property (c) is that all bases of $A$ have the same cardinalities, which is defined as the \emph{rank} of $A$. The following proposition shows that the set-function thus defined is a submodular function.

\begin{proposition} \textbf{(Matroid rank function)}
The rank function of a matroid, defined as
$F(A) = \max_{ I \subseteq A, \ A \in \mathcal{I} } |I| $, is submodular. Moreover, for any set $A \subseteq V$ and $k \in V \backslash A$, $F(A \cup \{k\}) - F(A) \in \{0,1\}$.
\end{proposition}
\begin{proof}
We first prove the second assertion. Since $F$ has integer values and is non-decreasing (because of the hereditary property (b)), we only need to show that  $F(A \cup \{k\}) - F(A)  \leqslant 1$. Let $B_1$ be a base of $A$ and~$B_2$ be a base of $A \cup \{k\}$. If $|B_2|>|B_1|+1$, then, by applying the exchange property (c) twice, there exists two distincts elements $i,j$ of $B_2 \backslash B_1$ such that $B_1 \cup \{i,j \}$ is a base. One of these elements cannot be $k$ and thus has to belong to $A$ which contradicts the maximality of $B_1$ as an independent subset of $A$; this proves by contradiction  that $|B_2|\leqslant |B_1|+1$, and thus  $F(A \cup \{k\}) - F(A)  \in \{0,1\}$.

To show the submodularity of $F$, we consider Prop.~\ref{prop:second} and a set  $A \subseteq  V$ and $j,k \in V \backslash A$. Given the property shown above, we only need to show that if
 $ F(A \cup \{k\}) =F(A)$, then   $F(A \cup \{j,k\}) = F(A \cup \{ j\} ) $. This will immediately imply that
 $ F(A \cup \{k\}) - F(A) \geqslant  F(A \cup \{j,k\}) - F(A \cup \{ j\} ) $ (and thus $F$ is submodular). Assume by contradiction that
 $F(A \cup \{j,k\}) = F(A \cup \{ j\} ) + 1$. If $F(A \cup \{j\}) = F(A)$, then we have
$ F(A \cup \{j,k\}) = F(A) + 1$, and thus
  by the exchange property, we must have $F(A \cup \{k\}) = F(A) + 1$, which is a contradiction. If $F(A \cup \{j\}) = F(A)+1$, then
  $F(A \cup \{j,k\}) = F(A \cup \{ j\} ) + 2$, and we must have $F(A \cup \{k\}) = F(A) + 1$ (because the increments of $F$ are in $\{0,1\}$), which is also a contradiction.
  \end{proof}
Note that  matroid rank functions are exactly the ones for which all extreme points are in $\{0,1\}^p$. They are also exactly the submodular functions for which $F(A)$ is integer and $F(A) \leqslant |A|$~\cite{schrijver2004combinatorial}.

A classical example is the \emph{graphic matroid}; it corresponds to $V$ being the edge set of a certain graph, and $\mathcal{I}$ being the set of subsets of edges leading to a subgraph that does not contain any cycle. The rank function $\rho(A)$ is then equal to $p$ minus the number of connected components of the subgraph induced by $A$. Beyond the historical importance of this matroid (since, as shown later, this leads to a nice proof of exactness for Kruskal's greedy algorithm for maximum weight spanning tree problems), the base polyhedron, often referred to as the spanning tree polytope, has interesting applications in machine learning, in particular for variational inference in probabilistic graphical models~\cite{wainwright2008graphical}.

The other classical example is the \emph{linear matroid}. Given a matrix $M$ with $p$ columns, then a set $I$ is independent if and only if the  columns indexed by $I$ are linearly independent. The rank function $\rho(A)$ is then the rank of the set of columns indexed by $A$ (this is also an instance of functions from \mysec{spectral} because the rank is the number of non-zero eigenvalues, and when $\rho \to 0^+$, then $\lambda^\rho \to 1_{ \lambda > 0 }$). For more details on matroids, see, e.g.,~\cite{schrijver2004combinatorial}.

\paragraph{Greedy algorithm.}

For matroid rank functions, extreme points of the base polyhedron have components equal to zero or one (because $F(A \cup \{k\}) - F(A) \in \{0,1\}$ for any $A \subseteq V$ and $k \in V$), and are incidence vectors of the maximal  independent sets. Indeed, the extreme points are such that $s \in \{0,1\}^p$ and $\supp(s)$ is an independent set because, when running the greedy algorithm, the set of non-zero elements of the already determined elements of $s$ is always independent. Moreover, it is maximal, because $s \in B(F)$ and thus $s(V) = F(V)$.

The greedy algorithm for maximizing linear functions on the base polyhedron may be used to find maximum weight maximal independent sets, where a certain weight $w_k$ is given to all elements of
$k$~$V$, that is, it finds a maximal independent set $I$, such that $\sum_{k \in I} w_k$ is maximum. In this situation, the greedy algorithm is actually greedy, i.e.,  it first orders the weights of each element of $V$ in decreasing order and select elements of $V$ following this order and skipping the elements which lead to non-independent sets.

For the graphic matroid, the base polyhedron is thus the convex hull of the incidence vectors of sets of edges which form a spanning tree, and is often referred to as the spanning tree polytope\footnote{Note that algorithms presented in \mychap{optim-polyhedra} lead to algorithms for several operations on this spanning tree polytopes, such as line searches and orthogonal projections.}~\cite{chopra1989spanning}. The greedy algorithm is then exactly Kruskal's algorithm to find maximum weight spanning trees~\cite{cormen89introduction}.

\paragraph{Minimizing matroid rank function minus a modular function.}
General submodular functions may be minimized in polynomial time (see \mychap{sfm}). For functions which are equal to the rank function of a matroid minus a modular function, then dedicated algorithms have better running-time complexities, i.e., $O(p^3)$~\cite{cunningham1984testing,narayanan1995rounding}.

\chapter{Non-smooth Convex Optimization}
\label{chap:nonsmooth}

In this chapter, we consider optimization problems of the form
\BEQ
\label{eq:prob}
\min_{w \in \rb^p} \Psi(w) + h(w),
\EEQ
where both functions $\Psi$ and $h$ are convex. In this section, we   always assume that
$h$ is non-smooth and positively homogeneous; hence we  consider only algorithms adapted to non-smooth optimization problems. Problems of this type appear many times when dealing with submodular functions (submodular function minimization in \mychap{sfm}, separable convex optimization in Chapters~\ref{chap:prox} and \ref{chap:prox-algo}, sparsity-based problems in \mysec{increasing}); however, they are typically applicable much more widely, in particular to all polytopes where maximizing linear functions may be done efficiently, which is the case for the various polytopes defined from submodular functions.

Our first three algorithms  deal with generic problems where few assumptions are made beyond convexity, namely the subgradient method in \mysec{subgrad}, the ellipsoid method in \mysec{ellipsoid}, Kelley's
 method (an instance of cutting planes) in \mysec{cutting}, and analytic center cutting planes in \mysec{accpm}.

The next algorithms we present rely on the strong convexity of the function $\Psi$ and have natural dual intepretations: in \mysec{subcondgrad}, we consider mirror descent techniques whose dual interpretations are conditional gradient algorithms, which are both iterative methods with cheap iterations. In \mysec{simplicial}, we consider bundle methods, whose dual corresponding algorithms are simplicial methods. They share the same principle than the previous iterative techniques, but the memory of all past information is explicitly kept and used at each iteration.

The next two algorithms  require additional efficient operations related to  the function $h$ (beyong being able to compute function values and subgradients).
In \mysec{dualsimplicial}, we present dual simplicial methods, which use explicitly the fact that $h$ is a gauge function (i.e., convex homogeneous and non-negative), which leads to iterative methods with no memory and algorithms that keep and use explicitly all past information. This requires to be able to maximize $w^\top s$ with respect to $w$ under the constraint that $h(w) \leqslant 1$.

We finally present in \mysec{proximal} proximal methods, which are adapted to situations where $\Psi$ is differentiable, under the condition that problems with $\Psi(w)$ being an isotropic quadratic function, i.e., $\Psi(w) = \frac{1}{2} \| w - z\|_2^2$, are easy to solve. These methods are empirically the most useful for problems with sparsity-inducing norms and are one of the motivations behind the focus on solving separable problems in Chapters~\ref{chap:prox} and \ref{chap:prox-algo}.

\section{Assumptions}
\label{sec:assumptions}

\paragraph{Positively homogeneous convex functions.}

Throughout this chapter on convex optimization, the function $h$ is always assumed positively homogeneous and non-negative. Such functions are often referred to as \emph{gauge functions} (see Appendix~\ref{app:gauge}). Since we assume that $h$ is finite (i.e., it has full domain), this implies that there exists a compact convex set $K$ that contains the origin such that for all $w \in \rb^p$, 
\BEQ
\label{eq:defh} h(w) = \max_{ s \in K} \ s^\top w.
 \EEQ
 This is equivalent to $h^\ast(s) = I_K(s)$, where  $I_K(s)$ is the indicator function of set $K$, equal to zero on $K$, and to $+\infty$ otherwise. In this monograph, $K$ will typically be:
  \begin{list}{\labelitemi}{\leftmargin=1.1em}
   \addtolength{\itemsep}{-.0\baselineskip}

\item[--] the base polyhedron $B(F)$ with $h$ being the \lova extension of $F$,
\item[--]  the symmetric submodular polyhedron $|P|(F)$ with $h$ being the norm $\Omega_\infty$ defined in \mysec{increasing},
\item[--]  the dual unit ball of the norm $\Omega_q$ with $h$ being the norm $\Omega_q$ defined in \mysec{l2}, for $ q \in (1,\infty)$. 
\end{list}

The most important assumption which we are using is that the maximization defining $h$ in \eq{defh} may be performed efficiently, i.e., a maximizer $s \in K$ of the linear function $s^\top w$ may be efficiently obtained. For the first two examples above, it may be done using the greedy algorithm.
 Another property that will be important is the polyhedral nature of~$K$. This is true for $B(F)$ and $|P|(F)$. Since $K$ is bounded, this   implies that there exists a finite number of extreme points $(s_i)_{ i \in H}$, and thus that $K$ is the convex hull of these $|H|$ points. Typically, the cardinality $|H|$ of $H$ may be exponential in $p$, but any solution may be expressed with at most $p$ such points (by Carath\'eodory's theorem for cones~\cite{rockafellar97}).

 \paragraph{Smooth, strongly convex or separable convex functions.}
 
 The function $\Psi$ in \eq{prob} may exhibit different properties that   make the optimization problem potentially easier to solve. All these assumptions may also be seen in the Fenchel-conjugate $\Psi^\ast$ defined as $\Psi^\ast(s) = \sup_{w \in \rb^p} w^\top s - \Psi(w)$ (see~Appendix~\ref{app:convex}).
 
  \begin{list}{\labelitemi}{\leftmargin=1.1em}
   \addtolength{\itemsep}{-.2\baselineskip}

\item[--] \textbf{Lipschitz-continuity}: $\Psi$ is Lipschitz-continuous on a closed convex set~$C$ with  Lipschitz-constant $B$ if and only if
$$ \forall (w_1,w_2) \in C \times C, \   | \Psi(w_1) - \Psi(w_2) | \leqslant B \| w_1 - w_2 \|_2.$$
This is equivalent to all subgradients of $\Psi$ on $C$ being bounded in $\ell_2$-norm by $B$.
This is the typical assumption in non-smooth optimization. 

Our motivating examples are
$\Psi(w) = \frac{1}{n} \sum_{i=1}^n \ell( y_i, w^\top x_i) = \Phi(X w)$ where the loss function is convex but non-smooth (such as for support vector machines). Also, $\Psi(w) = I_{[0,1]^p}(w)$ for minimizing submodular functions.

\item[--] \textbf{Smoothness}: In this monograph, $\Psi$  is said smooth if its domain is equal to $\rb^p$, and it has Lipchitz-continuous gradients, that is, there exists $L>0$ such that:
$$ \forall (w_1,w_2) \in \rb^p \times \rb^p, \   \| \Psi'(w_1) - \Psi'(w_2) \|_2 \leqslant L \| w_1 - w_2 \|_2.$$
If $\Psi$  is twice differentiable, this is equivalent to $\Psi''(w) \preccurlyeq L \idm$ for all $w \in \rb^p$ (where $ A \preccurlyeq B$ means that $B-A$ is positive semi-definite).
Our motivating examples are
$\Psi(w) = \frac{1}{n} \sum_{i=1}^n \ell( y_i, w^\top x_i) = \Phi(X w)$ where the loss function is convex and smooth (such as for least-squares regression and logistic regression). This includes separable optimization problems with $\Psi(w) = \frac{1}{2} \| w- z\|_2^2$ for some $z \in \rb^p$.

\item[--] \textbf{Strong convexity}: $\Psi$ is said strongly convex if and only if the function
$w \mapsto \Psi(w) - \frac{\mu}{2} \| w\|_2^2$ is convex for some $\mu>0$. This is equivalent to:
$$ \forall (w_1,w_2) \in C \times C, \   \Psi(w_1) \geqslant \Psi(w_2) + \Psi'(w_2)^\top ( w_1 - w_2) + \frac{\mu}{2} \| w_1 - w_2 \|_2^2,$$
i.e., $\Psi$ is lower-bounded by tangent quadratic functions. If $\Psi$  is twice differentiable, this is equivalent to $\Psi''(w) \succcurlyeq \mu \idm$ for all $w \in \rb^p$. Note however that $\Psi$ may be strongly convex but not differentiable (and vice-versa).

Our motivating examples are of the form
$\Psi(w) = \frac{1}{n} \sum_{i=1}^n \ell( y_i, w^\top x_i) = \Phi(X w)$ where the loss function is convex and smooth (such as for least-squares regression and logistic regression), \emph{and} the design matrix has full column rank.  
This includes separable optimization problems with $\Psi(w) = \frac{1}{2} \| w- z\|_2^2$ for some $z \in \rb^p$.

\item[--] \textbf{Separability}: $\Psi$ is said separable if it may be written as $\Psi(w) = \sum_{k=1}^p \Psi_k(w_k)$ for functions $\Psi_k : \rb \to \rb$. The motivating example is $\Psi(w) = \frac{1}{2} \| w - z\|_2^2$. Chapters~\ref{chap:prox} and \ref{chap:prox-algo} are dedicated to the analysis and design of efficient algorithms for such functions (when $h$ is obtained the \lova extension).

\item[--] \textbf{Composition by a linear map}: Many objective functions used in signal processing and machine learning as often composed with a linear map, i.e., we consider functions of the form $w \mapsto \Phi( X w)$, where $\Phi :\rb^n \to \rb$ and $ X \in \rb^{n \times p}$. This explicit representation is particularly useful to derive dual problem as $\Phi$ may have a simple Fenchel conjugate while  $w \mapsto \Phi( X w)$ may not, because $X$ does not have full rank.

\item[--] \textbf{Representations as linear  programs}: A function $\Psi$ is said polyhedral if its epigraph $\{ (w,t) \in \rb^{p+1}, \ \Psi(w) \leqslant t \}$ is a polyhedron. This is equivalent to $\Psi$ having a polyhedral domain and being expressed as the maximum of finitely many affine functions. Thus, the problem of minimizing $\Psi(w)$ may expressed as a linear program
$ \min_{ A w + ct  \leqslant b} t$ for a matrix $A \in \rb^{ k \times p} $ and vectors $c \in \rb^k$ and $b \in \rb^k $. Such linear programs may be solved efficiently by a number of methods, such as the simplex method (which uses heavily the polyhedral aspect of the problem, see \mysec{simplex}) and interior-point methods (see, e.g.,~\cite{bertsimas1997introduction} and \mysec{accpm}). 

\item[--] \textbf{Representations as convex quadratic  programs}: The function $\Psi(w)$ is then of the form
$
\max_{ A w + ct  \leqslant b } t + \frac{1}{2} w^\top Q w$, for a positive semi-definite matrix. Such programs may be efficiently solved by active-set methods~\cite{Nocedal:1999:NO1} or interior point methods~\cite{nesterov1994interior}. Active-set methods  will be reviewed in \mysec{QP}.

\end{list}

\paragraph{Dual problem and optimality conditions.}
Using Fenchel duality, we have
\BEAS
\min_{w \in \rb^p} \Psi(w) + h(w) & = & \min_{w \in \rb^p} \Psi(w) + \max_{s \in K}  w^\top s \\
& = & \max_{s \in K} - \Psi^\ast(-s).
\EEAS
The duality gap is, for $(w,s) \in \rb^p \times K$:
$$
{\rm gap}(w,s) = \big[ h(w) - w^\top s \big] + \big[ \Psi(w) + \Psi^\ast(-s) - w^\top (-s) \big],
$$
and is equal to zero if and only if (a) $s \in K$ is a maximizer of $w^\top s$, and (b) the pair $(w,-s)$ is dual for $\Psi$. The primal minimizer is  always unique only when $\Psi$ is strictly convex (and thus $\Psi^\ast$ is smooth), and we then have $w = (\Psi^\ast)'(s)$, i.e., we may obtain a primal solution directly from any  dual solution $s$. When both $\Psi$ and $\Psi^\ast$ are differentiable, then we may go from $w$ to $s$ as $s = - \Psi'(w)$ and $w = (\Psi^\ast)'(s)$. However, in general it is not possible to  naturally go from a primal candidate to a dual candidate in closed form. In this chapter, we only consider optimization algorithms which  exhibit primal-dual guarantees, i.e., generate both primal candidates $w$ and dual candidates~$s$.

\section{Projected subgradient descent}
\label{sec:subgradient}
\label{sec:subgrad}

When no smoothness assumptions are added to $\Psi$, we may consider without loss of generality that $h=0$, which we do in this section (like in the next three sections). Thus, we only assume that $\Psi$ is Lipschitz-continuous on a compact set $C$, with Lipschitz-constant $B$. Starting from any point in~$C$,  the subgradient method is an iterative algorithm that goes down the direction of negative subgradient. More precisely:

\begin{list}{\labelitemi}{\leftmargin=1.1em}
   \addtolength{\itemsep}{-.0\baselineskip}

\item[(1)] \textbf{Initialization}: $w_0 \in C$.
\item[(2)] \textbf{Iteration}: for $t \geqslant 1$, compute a subgradient $\Psi'(w_{t-1})$ of $\Psi$ at $w_{t-1}$ and compute
$$
w_{t} = \Pi_C \big(
w_{t-1} - \gamma_t \Psi'(w_{t-1})
\big),
$$
where $\Pi_C$ is the orthogonal projection onto $C$.
\end{list}

This algorithm is not a descent algorithm, i.e., it is possible that $\Psi(w_t) > \Psi(w_{t-1})$. There are several strategies to select the constants~$\gamma_t$. If the diameter $D$ of $C$ is  known, then by selecting $\gamma_t = \frac{D}{B \sqrt{t}}$, 
if we denote $\Psi^{\rm opt} = \min_{ x \in C} \Psi(x)$, then   we have for all $t>0$,  the following convergence rate (see proof in~\cite{Nesterov2004,condgrad}):
$$0 \leqslant \min_{u \in \{0,\dots,t\}} \Psi(x_u) - \Psi^{\rm opt} \leqslant 
\frac{4DB}{\sqrt{t}}
.$$
Note that this convergence rate is independent of the dimension $p$ (at least not explicitly, as constants $D$ and $B$ would typically grow with~$p$), and that it is optimal for methods that look only at subgradients at certain points and linearly combine them~\cite[Section 3.2.1]{Nesterov2004}. Moreover, the iteration cost is limited, i.e., $O(p)$, beyond the computation of a subgradient. Other strategies exist for the choice of the step size, in particular Polyak's rule: $\gamma_t = \frac{\Psi(x_{t-1}) - \Psi^\ast}{\|\Psi'(x_{t-1}) \|_2^2}$, where $\Psi^\ast$ is any lower bound on the optimal value (which may usually obtained from any dual candidate).

\paragraph{Certificate of optimality.}
If one can compute $\Psi^\ast$ efficiently, the average $\bar{s}_t$ of all subgradients, i.e.,
$\bar{s}_t = \frac{1}{t} \sum_{u=0}^{t-1} \Psi'(w_u)$, provides a certificate of suboptimality with offline guarantees, i.e., 
if $\bar{w}_t = \frac{1}{t} \sum_{u=0}^{t-1} w_u$,
$$
{\rm gap}(\bar{w}_t,\bar{s}_t) = 
\Psi(\bar{w}_t) + \Psi^\ast(\bar{s}_t) + \max_{w \in C} \big\{  -\bar{s}_t^\top w \big\}  \leqslant \frac{4DB}{\sqrt{t}}.
$$
See~\cite{nedic2009} and a detailed proof in~\cite{condgrad}. In the context of this monograph, we will apply the projected subgradient method to the problem of minimizing $f(w)$ on $[0,1]^p$, leading to algorithms with small iteration complexity but slow convergence (though a decent solution is found quite rapidly in applications).

Note that when $\Psi$ is obtained by composition by a linear map~$X$, then similar certificates of optimality may be obtained~\cite{condgrad}. 

\section{Ellipsoid method}
\label{sec:ellipsoid}

Like in the previous section, we  assume that $h=0$ and that $\Psi$ is Lipschitz-continuous on a compact set $C$. Moreover, we assume that
$C$ is contained in the ellipsoid $\mathcal{E}_0 =  \{ V_0 u + w_0, \ \| u\|_2 \leqslant 1\}$, with $V_0 \in \rb^{p \times p}$ an invertible matrix and $w_0 \in \rb^p$ is the center of the ellipsoid. We denote by $\mathcal{E}_C$ the minimum volume ellipsoid containing $C$.

The ellipsoid method builds a sequence of ellipsoids that contain all minimizers of $\Psi$ on $K$. At every iteration, the volume of the ellipsoid is cut by a fixed multiplicative constant. Starting from an ellipsoid containing $K$, we consider its center. If it is in $C$, then a subgradient of $\Psi$ will divide the space in two, and the global minima have to be in a known half. Similarly, if the center not in $C$, then a separating hyperplane between the center and $C$ plays the same role. We can then iterate the process. The precise algorithm is as follows:


\begin{list}{\labelitemi}{\leftmargin=1.1em}
   \addtolength{\itemsep}{-.0\baselineskip}

\item[(1)] \textbf{Initialization}: ellipsoid $\mathcal{E}_0 =  \{ V_0 u + w_0, \ \| u\|_2 \leqslant 1\}$ that contains  the optimization set $C$.

\item[(2)] \textbf{Iteration}: for $t > 0$, 
\BNUM
\item[(a)] \textbf{Select half-plane} \\
$-$ If $w_{t-1}$ is feasible (i.e., $w_{t-1} \in C$), then $z_{t-1} = \Psi'(w_{t-1})$, i.e., any subgradient of $\Psi$ at $w_{t-1}$. \\
$-$ Otherwise $z_{t-1}$ is a normal vector to any hyperplane passing by $w_{t-1}$ and not intersecting $C$, i.e., $z_{t-1}$ is such that $C \subseteq \{ w \in \rb^p, \  (w - w_{t-1})^\top z_{t-1} \leqslant 0 \}$.
\item[(b)]  \textbf{Compute new ellipsoid} $\mathcal{E}_t =  \{ V_t u + w_t, \ \| u\|_2 \leqslant 1\}$ with
\BEAS
\tilde{z}_{t-1} & = &   (z_{t-1}^\top V_{t-1} V_{t-1}^\top z_{t-1})^{-1/2}  z_{t-1} \\
 w_{t}  & = &  w_{t-1} - \frac{1}{p+1}  V_{t-1} V_{t-1}^\top \tilde{z}_{t-1} \\
 V_{t} &  =  &  \frac{p^2}{p^2-1}
\bigg( V_{{t-1}} -  \frac{2}{p+1}    V_{t-1}V_{t-1}^\top \tilde{z}_{t-1}\tilde{z}_{t-1}^\top V_{t-1}V_{t-1}^\top \bigg).
\EEAS
\ENUM
\end{list}
See an illustration in \myfig{ellipsoid}.
As shown in~\cite{nemirovski} by using the invariance by linear rescaling that allows to consider the unit Euclidean ball,
(a) $\mathcal{E}_t $ is the minimum volume ellipsoid containing the intersection of 
$\mathcal{E}_{t-1}$ and the half-plane $\{ w \in \rb^p, \  (w - w_{t-1})^\top z_{t-1}  \leqslant 0 \}$,
 (b) the volume of $\mathcal{E}_t =  \{ V_t u + w_t, \ \| u\|_2 \leqslant 1\}$ is less than the volume of $\mathcal{E}_{t-1}$ times $\exp \big( \frac{-1}{2p} \big)$, and   (c) $\mathcal{E}_t$ contains any global minimizer of $\Psi$ on $C$.  
 
 Thus, the ellipsoid $\mathcal{E}_t$ has a volume decreasing at an exponential rate. This allows to obtain an exponential rate of convergence for the minimization problem. Indeed, following~\cite{nemirovski}, let $1 > \varepsilon > \min\{ 1, \big( \frac{ {\rm vol}(\mathcal{E}_t) }{ {\rm vol}(\mathcal{E}_K)  } \big)^{1/p} \}$
and $w^{ \rm opt}$ a minimizer of $\Psi$ on $K$. We define
$K^\varepsilon = w^{ \rm opt} + \varepsilon ( K  - w^{ \rm opt} )$. We have
$ {\rm vol} (K^\varepsilon) = \varepsilon^p   {\rm vol}(K) > 
 {\rm vol}( \mathcal{E}_t) $. The two sets $K^\varepsilon$ and
$ \mathcal{E}_t$ have at least the point $w^{\rm opt}$ in common; given the volume inequality, there must be at least one element $v \in K^\varepsilon \backslash  \mathcal{E}_t$. Since
$\varepsilon \leqslant 1$, $K^\varepsilon  \subseteq K$, and hence $v \in K$. Since it is not in $\mathcal{E}_t$, it must have been removed in one of the steps of the ellipsoid method, hence its value $\Psi(v)$ is greater than $\min_{ i \in \{0,\dots,t\}} \Psi(w_i)$.
Moreover, by convexity,
$\Psi(v) \leqslant (1 - \varepsilon) \Psi(w^{\rm opt}) + \varepsilon \max_{ w \in K} \Psi(w)$, which implies
$$
\min_{ i \in \{0,\dots,t\}} \Psi(w_i) - 
\Psi(w^{\rm opt})
\leqslant \varepsilon \big[
\max_{ w \in K} \Psi(w) - \min_{ w \in K} \Psi(w)
\big].
$$
This implies that there exists $i \leqslant t$, such that $w_i \in K$ and
$$ \Psi(w_i) - \min_{w \in K} \Psi(w)
  \leqslant      \big(\max_{w \in K} \Psi(w) - \min_{w \in K} \Psi(w)
  \big)
  \times \min \Big\{ 1, \Big( \frac{ { \rm vol}(\mathcal{E}_t)}{ { \rm vol}(\mathcal{E}_K)} \Big)^{1/p}
  \Big\}
. $$
See \cite{nemirovsky1983problem,nemirovski} for more details. 

The convergence rate is exponential, but there is a direct and strong dependence on the dimension of the problem $p$. Note that dual certificates may be obtained at limited additional computational cost, with no extra information~\cite{nemirovski2010accuracy}. This algorithm is typically slow in practice since it has a running-time of $O(p^3)$ per iteration. Moreover, it makes slow progress and cannot take advantage of additional properties of the function $\Psi$ as the approximation of the reduction in volume has a tight dependence on  $p$: that is, it cannot really converge faster than the bound.

This algorithm has an important historical importance, as it implies that most convex optimization problems may be solved in polynomial-time, which implies polynomial-time algorithms for many combinatorial problems that may be expressed as convex programs; this includes the problem of minimizing submodular functions~\cite{grotschel1981ellipsoid}. See \mysec{sfm_ellipsoid} for a detailed convergence rate when applied to this problem.

\begin{figure}
 
 \begin{center}
\includegraphics[scale=.65]{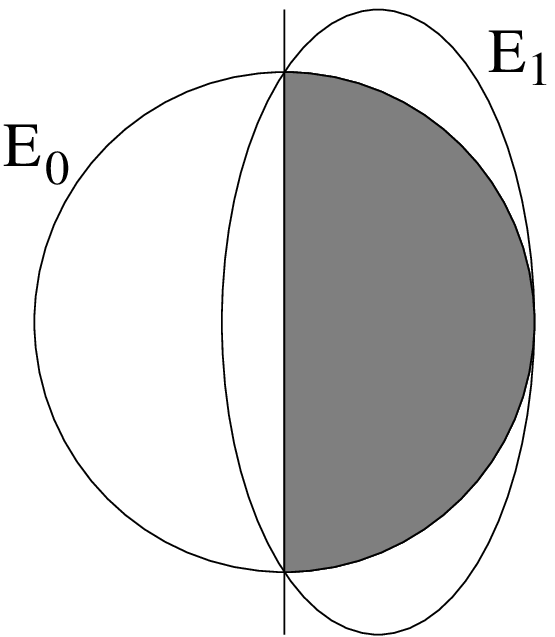} \hspace*{1cm}
\includegraphics[scale=.65]{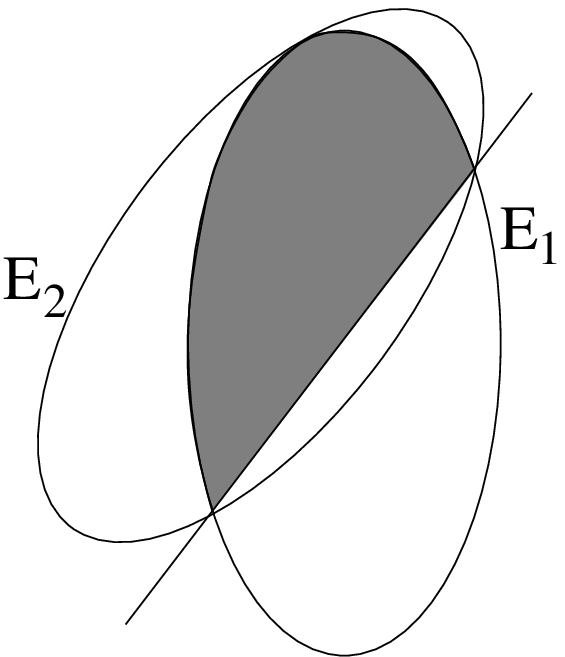}
 \end{center}
 
 \vspace*{-.25cm}
 
 \caption{Two iterations of the ellipsoid method, showing how the intersection of an half-plane and an ellipsoid may be inscribed in another ellipsoid, which happens to have a smaller volume.}
\label{fig:ellipsoid}
\end{figure}

\section{Kelley's method}
\label{sec:cutting}
Like in the previous section, we  assume that $h=0$ and that $\Psi$ is Lipschitz-continuous on a compact set $C$. In the subgradient and ellipsoid methods, only a vector (for the subgradient method) or a pair of a vector and a matrix (for the ellipsoid method) are kept at each iteration, and the values of the function
$\Psi(w_i)$ and of one of its subgradients $\Psi'(w_i)$, for $i < t$, are discarded.

Bundle methods aim at keeping and using exactly the bundle of information obtained from past iterations. This is done by noticing that for each $t$, the function $\Psi$ is lower bounded by
$$\widetilde{\Psi}_t(w) = \max_{ i \in \{0,\dots,t\} }  \big\{ \Psi(w_i) + \Psi'(w_i)^\top ( w - w_i) \big\}.$$
The function $\widetilde{\Psi}_t$ is a piecewise-affine function and we present an illustration in \myfig{kelley}. 

\begin{figure}
 \begin{center}
\includegraphics[scale=.8]{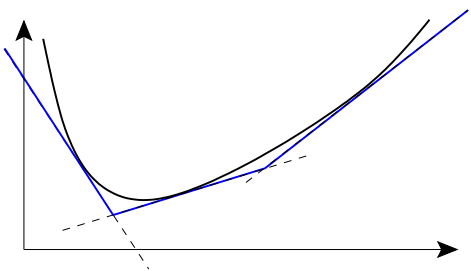}
 \end{center}
 
 \vspace*{-.5cm}
 
 \caption{Lower-bounding a convex function (in black) by the maximum of affine functions (in blue).}
\label{fig:kelley}
\end{figure}

Kelley's method (see, e.g.,~\cite{Nesterov2004}) simply minimizes this lower bound  $\widetilde{\Psi}_t$ at every iteration, leading to the following algorithm:
\begin{list}{\labelitemi}{\leftmargin=1.1em}
   \addtolength{\itemsep}{-.0\baselineskip}

\item[(1)] \textbf{Initialization}: $w_0 \in C$.
\item[(2)] \textbf{Iteration}: for $t \geqslant 1$, compute a subgradient $\Psi'(w_{t-1})$ of $\Psi$ at $w_{t-1}$ and compute any minimizer
$$
w_{t} \in \arg \min_{ w \in C } \max_{ i \in \{0,\dots,t-1\} }  \big\{ \Psi(w_i) + \Psi'(w_i)^\top ( w - w_i) \big\}.$$
\end{list}
The main iteration of Kelley's method (which can be seen in particular as an instance of a cutting-plane method~\cite{bonnans}) thus requires to be able to solve a subproblem which may  be complicated. When $C$ is a polytope, then it may be cast a linear programming problem and then solved by interior-point methods or the simplex algorithm. The number of iterations to reach a given accuracy may be very large (see lower bounds in~\cite{Nesterov2004}), and the method is typically quite unstable, in particular when they are multiple minimizers in the local optimization problems.

However, the method may take advantage of certain properties of $\Psi$ and $C$, in particular the representability of $C$ and $\Psi$ through linear programs. In this situation, the algorithm terminates after a finite number of iterations with an exact minimizer~\cite{bertsekas2011unifying}. Note that although Kelley's method is more complex than subgradient descent, the best known convergence rate is still of the order $O(1/\sqrt{t})$ after $t$ iterations~\cite{teo2010bundle}.

In the context of this monograph, we will apply Kelley's method to the problem of minimizing the \lova extenstion $f(w)$ on $[0,1]^p$, and, when the simplex algorithm is used to minimize $\widetilde{\Psi}_t(w)$  this will be strongly related  to the simplex algorithm applied directly to a linear program with exponentially many constraints (see \mysec{simplexsfm}).  

In our simulations in \mysec{exp-sfm}, we have observed that when an interior point method is used to minimize $\widetilde{\Psi}_t(w)$ (this is essentially what the weighted analytic center cutting plane method from \mysec{accpm} and \mysec{accpm-sfm} does), then the minimizer $w_t$ leads to a better new subgradient than with an extreme point (see \mysec{accpm-sfm}).

\section{Analytic center cutting planes}
\label{sec:accpm}

We consider a similar situation than the previous sections, i.e., we  assume that $h=0$ and that $\Psi$ is Lipschitz-continuous on a compact set~$C$. The ellipsoid method is  iteratively reducing a candidate set which has  to contain all optimal solutions. This is done with a provable (but small) constant reduction in volume. If the center of the smallest ellipsoid containing the new candidate set is replaced by its center of gravity, then an improved bound holds~\cite{nemirovski} (with $p^2$ replaced by $p$ in the complexity bound); however, this algorithm is not implemented in practice as there is no known efficient algorithm to find the center of gravity of a polytope. An alternative strategy is to replace the center of gravity of the polytope by its \emph{analytic center}~\cite{goffin1993computation}.

The analytic center of a polytope with non-empty interior defined as the intersection of half-planes $a_i^\top w \leqslant b_i$, $i \in I$, is the unique minimizer of
$$
\min_{w \in \rb^p} - \sum_{i \in I} \log ( b_i - a_i^\top w).
$$
The analytic center may be found with arbitrary precision using Newton's method~\cite{boyd}.
 For the original  problem of minimizing $\Psi$, there is a non-exponential complexity bound that decay as $O(1/\sqrt{t})$~\cite{nesterov1995complexity} but no bound similar to the ellipsoid method; however, its empirical behavior is often much improved, and this was confirmed in our simulations in \mysec{exp-sfm}. 
 
 In this tutorial, we consider the epigraph version of the problem, where we minimize $u$ with respect to $(w,u)$ such that $w \in C$ and $\Psi(w) \leqslant u$. For simplicity, we assume that $C$ is a polytope with non-empty interior, which is defined through the set of half-planes $a_i^\top w \leqslant b_i$, $i \in I_0$.
 The   algorithm is as follows:
\begin{list}{\labelitemi}{\leftmargin=1.1em}
   \addtolength{\itemsep}{-.2\baselineskip}

\item[(1)] \textbf{Initialization}: set of  half-planes  $a_i^\top w \leqslant b_i$, $i \in I_0$ with analytic center $w_0$, $u_0 = +\infty$.
\item[(2)] \textbf{Iteration}: for $t > 0$, 
\BNUM
\item[(a)] \textbf{Compute function   $\Psi(w_{t-1})$  and gradient and $\Psi'(w_{t-1})$}: 
 \\
$-$ add hyperplane $ u \geqslant \Psi(w_{t-1}) + \Psi'(w_{t-1}) ^\top(w -w_{t-1})$, \\
$-$ it $t=1$, add the plane $u \leqslant \Psi(w_{0})$, \\
$-$ if $\Psi(w_{t-1}) \leqslant u_{t-1}$, replace $u \leqslant u_{t-1}$ by $u \leqslant u_t = \Psi(w_{t-1})$,
\item[(b)]  \textbf{Compute analytic center $(w_t,u_t)$ of the new polytope}.
\ENUM
\end{list}
 Note that when computing the analytic center,  if we put a  large weight for the logarithm of the constraint $u \leqslant u_t$, then we recover an instance of Kelley's method, since the value of $u$ will be close to the piecewise affine lower bound  $\widetilde{\Psi}_t(w) $ defined in \mysec{cutting}. Note the difference   with the simplex method: here, the candidate $w$ is a center of the set of minimizers, rather than an extreme point, which makes considerable difference in practice~(see experiments in \mysec{exp-sfm}).

\section{Mirror descent/conditional gradient}
\label{sec:subcondgrad}
\label{sec:condgrad}
We now assume that the function $\Psi$ is $\mu$-strongly convex, and for simplicity, we assume that its domain is $\rb^p$ (this may be relaxed in general, see~\cite{condgrad}). We are now going to use the special properties of our problem, namely that $h$ may be written as $h(w) = \max_{ s \in K} w^\top s$.  Since $\Psi$ is $\mu$-strongly convex, $\Psi^\ast$ is $(1/\mu)$-smooth. The dual problem we aim to solve is then
$$
\max_{ s \in K } - \Psi^\ast(-s).
$$
We are thus faced with the optimization of a smooth function on a compact convex set on which linear functions may be maximized efficiently. This is exactly the situation where conditional gradient algorithms are useful (they are also often referred to as ``Frank-Wolfe'' algorithms~\cite{frank2006algorithm}). The algorithm is as follows:

\begin{list}{\labelitemi}{\leftmargin=1.1em}
   \addtolength{\itemsep}{-.2\baselineskip}

\item[(1)] \textbf{Initialization}: $s_0 \in K$ (typically an extreme point, obtained by maximizing $w_0^\top s$ for a certain $w_0 \in \rb^p$).
\item[(2)] \textbf{Iteration}: for $t \geqslant 1$, find a maximizer $\bar{s}_{t-1}$ of $(\Psi^\ast)'(s_{t-1})^\top s $ w.r.t. $s \in K$, and set $s_t = ( 1 - \rho_t) s_{t-1} + \rho_t \bar{s}_{t-1}$, for some $\rho_t \in [0,1]$.
\end{list}

There are two typical choices for $\rho_t \in [0,1]$:
\begin{list}{\labelitemi}{\leftmargin=1.1em}
   \addtolength{\itemsep}{-.2\baselineskip}

\item[--] \textbf{Line search} (adaptive schedule): we either maximize $-\Phi^\ast(-s)$ on the segment $[s_{t-1},\bar{s}_{t-1}]$, or a quadratic lower bound (traditionally obtained from the smoothness of $\Psi^\ast$), which is tight at $s_{t-1}$, i.e., 
\BEAS
\!\! \rho_{t-1} &  \!\!= \!\! &  \arg \max_{ \rho_t \in [0,1] } \rho_t ( \bar{s}_{t-1} - s_{t-1} )^\top 
(\Psi^\ast)'(s_{t-1}) - \frac{1}{2 \mu}  \rho_t ^2 \| \bar{s}_{t-1} - s_{t-1} \|_2^2 \\
 &  \!\!=\!\! &  \min \bigg\{ 1 , \frac{ \mu  ( \bar{s}_{t-1} - s_{t-1} )^\top 
(\Psi^\ast)'(s_{t-1}) }{ \| \bar{s}_{t-1} - s_{t-1} \|_2^2}
\bigg\}.
\EEAS

\item[--] \textbf{Fixed schedule}:  $\rho_t = \frac{2}{t+1}$.

\end{list}

\paragraph{Convergence rates.}
It may be shown~\cite{dunn1978conditional,jaggi,condgrad} that, if we denote $\Psi^{\rm opt} = \min_{w \in K} \Psi(w) + h(w)
= \max_{s \in K} - \Psi^\ast(-s)$, then, for $\rho_t$ obtained by line search,   we have for all $t>0$,  the following convergence rate:
$$0 \leqslant \Psi^{\rm opt} + \Psi^\ast(-s_{t}) \leqslant 
\frac{2 D^2}{\mu t},
$$
where $D $ is the diameter of $K$.
Moreover, the natural choice of primal variable $w_t = (\Psi^\ast)'(-s_t)$ leads to a duality gap of the same order. See an illustration in \myfig{FW2} for $\Psi^\ast(-s) = - \frac{1}{2} \| z - s\|_2^2$ (i.e., the dual problem is equivalent to an orthogonal projection of $z$ onto $K$).
For the fixed-schedule, a similar bound holds; moreover, the relationship with a known primal algorithm will lead to a further interpretation.

\paragraph{Primal interpretation.} For simplicity, we assume that $\Psi$ is essentially smooth so that $(\Psi^\ast)'$ is a bijection from $\rb^p$ to $K$.
For a fixed schedule $\rho_t  = \frac{2}{t+1}$, then by considering
$w_{t} = (\Psi^\ast)'(-s_{t})$---so that $s_t = - \Psi'(w_t)$, and seeing $\bar{s}_{t}$ as one of the subgradient of $h$ at $w_{t}$, i.e., denoting $\bar{s}_{t} = h'(w_{t})$, the iteration is as follows:
$$ 
\Psi'(w_{t}) = \big( 1 -  \rho_t \big)  \Psi'(w_{t-1}) - \rho_t  h'(w_{t-1}).
$$

\begin{figure}
 \begin{center}
\includegraphics[scale=.8]{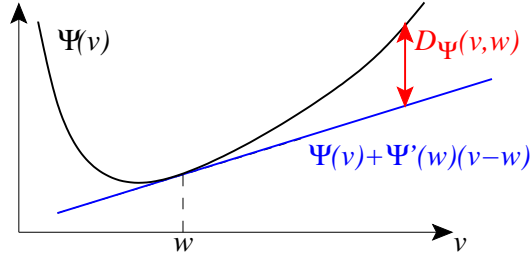}
 \end{center}
 
 \vspace*{-.5cm}
 
 \caption{Bregman divergence: $D_\Psi(v,w)$ is equal to the difference between the convex function $\Psi(v)$ and the affine function tangent at $w$. When $\Psi$ is $\mu$-strongly convex, then $D_\Psi(v,w) \geqslant 0$ with equality if and only if $v=w$. Moreover, we have $D_\Psi(v,w) \geqslant \frac{\mu}{2} \| v - w\|^2$. Note that for a quadratic function, the Bregman divergence is a squared Euclidean norm.}
\label{fig:bregman}
\end{figure}

We denote by $D_\Psi$ the Bregman divergence associated with the strongly convex function $\Psi$, i.e.,
$D_\Psi(v,w) = \Psi(v)-\Psi(w)-(v-w)^\top \Psi'(w)$. See \myfig{bregman} for an illustration and further properties in~\cite{banerjee2005clustering}. 
The previous iteration may be seen as the one of minimizing with respect to $w$ a certain function, i.e., 
$$
w_{t} = \arg \min_{w \in \rb^p} \ D_\Psi(w,w_{t-1}) -  \rho_t \big[
\Psi'(w_{t-1}) +  h'(w_{t-1})
\big]^\top ( w - w_{t-1}),
$$
which is an instance of mirror-descent~\cite{nemirovski}. Indeed, the solution of the previous optimization problem
is characterized by $\Psi'(w_t) - \Psi'(w_{t-1}) - \rho_t \big[
\Psi'(w_{t-1}) +  h'(w_{t-1}) 
\big] = 0 $.

For example, when $\Psi(w) = \frac{\mu}{2} \| w\|_2^2$, we obtain regular subgradient descent with step-size $\rho_t / \mu $. In~\cite{condgrad}, a convergence rates of order $O( 1 / \mu t)$ is provided
for the averaged primal iterate $\bar{w}_t = \frac{2}{t(t+1)} \sum_{k=0}^t k w_k$ when $\rho_t = 2/(t+1)$, using the traditional proof technique from mirror descent, but also a convergence rate of $O(1/\mu t)$ for the dual variable $s_t$, and for one of the primal iterates. 
Moreover, when $\Psi$ is obtained by composition by a linear map $X$, then similar certificates of optimality may be obtained. See more details in~\cite{condgrad}.

Note finally, that if $\Psi^\ast$ is also strongly convex (i.e., when $\Psi$ is smooth) and the global optimum is in the interior of $K$, then the convergence rate is exponential~\cite{guelat1986some,beck2004conditional}.

\section{Bundle and simplicial methods}
\label{sec:simplicial}
In \mysec{cutting}, we have considered the minimization of a function $\Psi$ over a compact set $K$ and kept the entire information regarding the function values and subgradients encountered so far. In this section, we extend the same framework to the type of problems considered in the previous section. Again, primal and dual interpretations will emerge.

We consider the minimization of the function $\Psi(w) + h(w)$, where $h$ is non-smooth, but $\Psi$ may have in general any additional assumptions such as strong-convexity, representability as quadratic or linear programs. The algorithm is similar to Kelley's method in that we keep all information regarding the subgradients of~$h$ (i.e., elements of~$K)$, but each step performs optimization where $\Psi$ is not approximated
(see \myfig{kelley} for an illustration of the piecewise linear approximation of~$h$):

\begin{list}{\labelitemi}{\leftmargin=1.1em}
   \addtolength{\itemsep}{-.2\baselineskip}

\item[(1)] \textbf{Initialization}: $w_0 \in K$.
\item[(2)] \textbf{Iteration}: for $t \geqslant 1$, compute a subgradient $s_{t-1}  = h'(w_{t-1}) \in K $ of $h$ at $w_{t-1}$ and compute
$$
w_{t} \in \arg \min_{ w \in \rb^p } \Psi(w) +  \max_{ i \in \{0,\dots,t-1\} } s_i^\top  w .
$$
\end{list}

Like Kelley's method, the practicality of the algorithm depends on how the minimization problem at each iteration is performed. In the common situation where each of the subproblems is solved with high accuracy, the algorithm is only practical for functions $\Psi$ which can be represented as linear programs or quadratic programs. 
Moreover, the method may take advantage of certain properties of $\Psi$ and $K$, in particular the representability of $K$ and $\Psi$ through linear programs. In this situation, the algorithm terminates after a finite number of iterations with an exact minimizer~\cite{bertsekas2011unifying}. In practice, like most methods considered in this monograph, using the dual interpretation described below, one may monitor convergence using primal-dual pairs.

\paragraph{Dual interpretation.} We first may see
$s_{t-1}$ as the maximizer of $s^\top w_{t-1}$ over $s \in K$. Moreover, we have, by Fenchel duality:
\BEAS
& & \min_{ w \in \rb^p } \Psi(w) +  \max_{ i \in \{0,\dots,t-1\} } s_i^\top  w \\[-.2cm]
& = &  \min_{ w \in \rb^p } \Psi(w) +  \max_{  \eta \geqslant 0, \ \sum_{i=0}^{t-1} \eta_i = 1 } w^\top  \bigg(\sum_{i=0}^{t-1} \eta_i s_i \bigg)\\[-.1cm]
& = &   \max_{  \eta \geqslant 0, \ \sum_{i=0}^{t-1} \eta_i = 1 } \min_{ w \in \rb^p } \Psi(w) +  w^\top \bigg( \sum_{i=0}^{t-1} \eta_i s_i \bigg) \\[-.2cm]
& = &   \max_{  \eta \geqslant 0, \ \sum_{i=0}^{t-1} \eta_i = 1 } - \Psi^\ast \bigg( - \sum_{i=0}^{t-1} \eta_i s_i \bigg).
\EEAS
This means that when $\Psi^\ast$ is differentiable (i.e., $\Psi$ strictly convex) we may interpret the algorithm as iteratively building inner approximations of the compact convex set $K$ as the convex hull of the point $s_0,\dots, s_{t-1}$ (see illustration in \myfig{inner}). The function $-\Psi^\ast(-s)$ is then maximized over this convex-hull. Given the optimum $\bar{s} = \sum_{i=0}^{t-1} \eta_i s_i $, then it is globally optimum if and only if
$\max_{ s \in K} (\Psi^\ast)'(- \bar{s})^\top ( s - \bar{s})  = 0$, i.e., denoting
$\bar{w} = (\Psi^\ast)'(- \bar{s})$, $h(w) = w^\top s$.

\begin{figure}
 \begin{center}
\includegraphics[scale=.8]{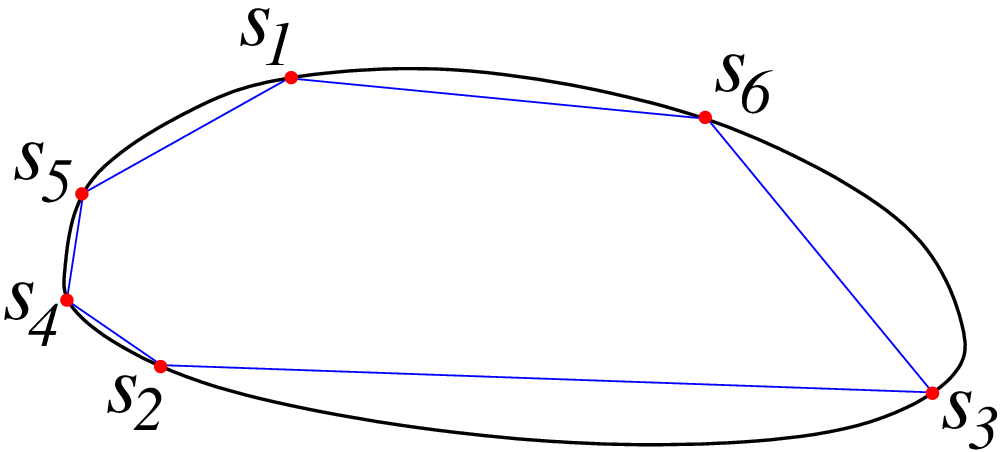}

\vspace*{.5cm}

\includegraphics[scale=.8]{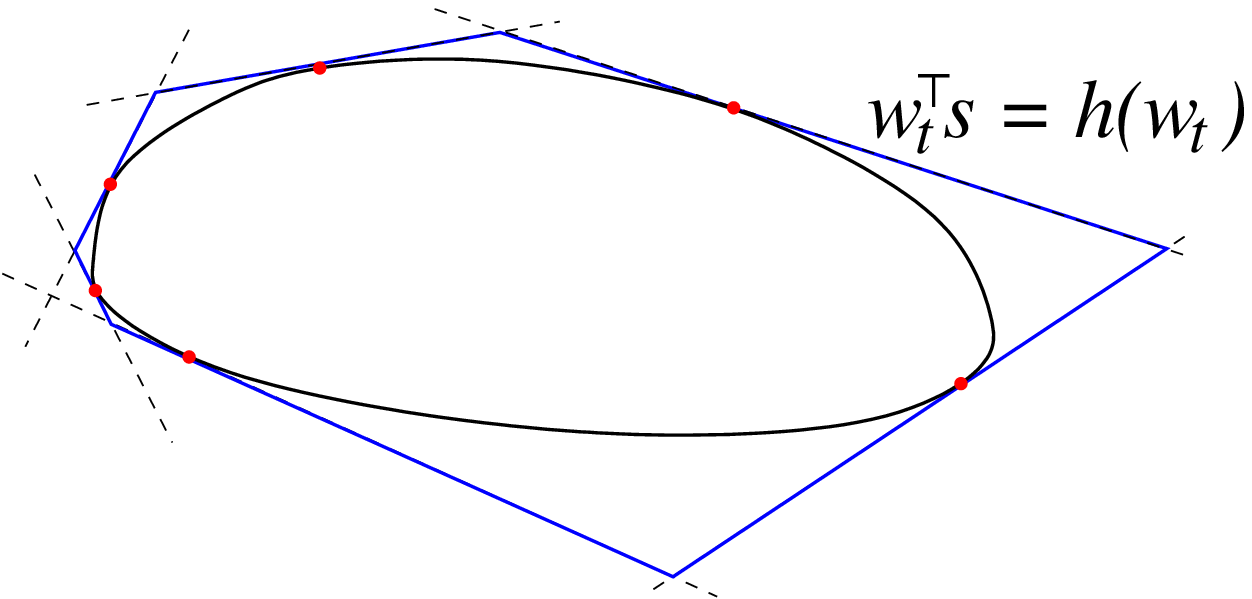}

 \end{center}
 
 \vspace*{-.5cm}
 
 \caption{Approximation of compact convex set $K$: (top) inner approximation as a convex hull, used by simplicial methods
 in \mysec{simplicial}, (bottom) outer approximation as an intersection of half hyperplanes, used by dual simplicial methods in \mysec{dualsimplicial}.}
\label{fig:inner}
\end{figure}

Note the difference with the conditional gradient algorithm from \mysec{condgrad}. Both algorithms are considering extreme points of $K$; however, conditional gradient algorithms only make a step towards  the newly found extreme point, while simplicial methods defined in this section will optimize over the entire convex hull of all extreme points generated so far, leading to better function values at the expense of extra computation.

In the context of this monograph, we will apply this   method to the problem of minimizing $\frac{1}{2} \| w - z \|_2^2 + f(w)$ on $\rb^p$, and, when an active set is used to minimize the subproblem,  this will correspond almost exactly   to an active set  algorithm applied directly to the quadratic program with exponentially many constraints (the only difference between the active-set and the simplicial method is that in the active-set methods, the set of extreme points of $K$ which are used is not only growing, but may also be reduced during line-search iterations, see \mysec{QP} for more details). See the illustration in  \myfig{FW} in \mysec{mnp}, and contrast it with \myfig{FW2}.

%
%
%
%
%

\section{Dual simplicial method}
\label{sec:dualsimplicial}
We now assume that $0 \in K$, which implies that $h(w) = \max_{s \in K} w^\top s$ is non-negative. The set $K$ may be seen as the intersection of the (potentially uncountably infinitely many) hyperplane
$\{ s \in \rb^p, \ s^\top w \leqslant h(w) \}$ for $w \in \rb^p  $. In this section, we assume that given any $s$, we may test efficiently whether  $s \in K$. If $s \notin K$, we assume that we can also provide a certificate $w \in \rb^p$ such that $s^\top w > h(w)$. One such possibility is to consider situations where $\max_{ h(w) \leqslant 1 } w^\top s$ may be computed efficiently. In our submodular context, when $K=B(F)$ for a non-negative submodular function, this amounts to computing
$\max_{ f(w) \leqslant 1 } w^\top s = \max_{ A \subseteq V, \ A \neq \varnothing} \frac{s(A)}{F(A)}$, which can be done efficiently if
one can minimize $F(A) - t(A)$ with respect to $A \subseteq V$, where $F$ is  our submodular function and $t \in \rb^p$ (see \mysec{extensions}).  Similarly, when $K$ is the unit ball of the norms $\Omega_q$, we would need to compute similar quantities. See more details in~\mysec{extensions}.

The dual simplicial method works by iteratively building \emph{outer} approximations
$\bar{K}_{t-1} = \{ s \in \rb^p, \ \forall k \in \{0,\dots,t-1\}, \bar{w}_k^\top s \leqslant h(\bar{w}_k)\}$ of $K$ as the intersection of half hyperplanes (see an illustration in \myfig{inner}); the algorithms is as follows:

\begin{list}{\labelitemi}{\leftmargin=1.1em}
   \addtolength{\itemsep}{-.2\baselineskip}

\item[(1)] \textbf{Initialization}: $w_0 \in \rb^p$
\item[(2)] \textbf{Iteration}: for $t \geqslant 1$, compute
$$
s_t \in \arg\max_{s \in \rb^p } -\Psi^\ast(-s)
\mbox{ such that } \forall i \in \{0,\dots,t-1\}, \ \bar{w}_i^\top s \leqslant h(\bar{w}_i) ,
$$
and let $\bar{w}_t \in \arg \max_{ h(\bar{w}) \leqslant 1 } \bar{w}^\top s_t$. If $\bar{w}_t^\top s_t \leqslant 1$, $\bar{w}_t$ is the optimal solution.
\end{list}
Like Kelley's method or bundle methods, the practicality of the algorithm depends on how the minimization problem at each iteration is performed. In the common situation where each of the subproblems is solved with high accuracy, the algorithm is only practical for functions $\Psi$ which can be represented as linear programs or quadratic programs. 

Moreover, the method may also take advantage of certain properties of $\Psi^\ast$ and $K$, in particular the representability of $K$ and $\Psi^\ast$ through linear programs. In this situation, the algorithm terminates after a finite number of iterations with an exact minimizer. This can be checked by testing if $s_t \in K$, i.e, 
$\max_{ h(w) \leqslant 1 } w^\top s_t \leqslant 1$ (in which case, the outer approximation is tight enough).

\paragraph{Interpretation using gauge functions.}
We define the gauge function $\gamma_K$ of the convex set $K$ (that we have assumed to contain $0$) as 
$
\gamma_K(s) = \min \{ \lambda \in \rb_+, \ s \in \lambda K \}
$.
If  $h$ is the support function of $K$, i.e., for all $w \in \rb^p$, $h(w) = \max_{ s \in K } w^\top s$, then, $h$ is the gauge function associated to the polar set $K^\circ$ (see more details in Appendix~\ref{app:convex}).

\paragraph{Primal interpretation.}
The iteration may be given a primal interpretation. Indeed, we have:
\BEAS
&&\max_{s \in \rb^p } \Psi^\ast(-s)
\mbox{ such that } \forall i \in \{0,\dots,t-1\}, \ \bar{w}_i^\top s \leqslant h(\bar{w}_i) 
\\
& =  & \max_{s \in \rb^p} \min_{ \lambda \in \rb_+^t}
-\Psi^\ast(-s) - \sum_{i=0}^{t-1} \lambda_i \big[ \bar{w}_i^\top s - h(\bar{w}_i) \big] \\
 & =  & \min_{ \lambda \in \rb_+^t} \max_{s \in \rb^p} 
-\Psi^\ast(-s) - \sum_{i=0}^{t-1} \lambda_i \big[ \bar{w}_i^\top s - h(\bar{w}_i) \big] \\
& =  & \min_{ \lambda \in \rb_+^t} \Psi \Big(
\sum_{i=0}^{t-1} \lambda_i \bar{w}_i
\Big) +  \sum_{i=0}^{t-1} \lambda_i   h(\bar{w}_i).
\EEAS
It may then be reformulated as
$$ \min_{  w \in \rb^p } \Psi(w)
+ \min_{ \mu \in \rb_+^t, \ w =  \sum_{i=0}^{t-1} \lambda_i \bar{w}_i  / h(\bar{w}_i) }   \sum_{i=0}^{t-1} \lambda_i   .
$$
That is, the iteration first consists in replacing $h(w)$ by a certain (convex) upper-bound. This upper-bound may be given a special interpretation using gauge functions. Indeed, if we consider the polar set $K^\circ$ and the (potentially infinite set) $C$ of its extreme points, then $h$ is the gauge function of the set $K^\circ$, and also of the set $C = \{ v_i, \ i \in I\}$, that is:
$$
h(w) = \inf_{ w = \sum_{i \in I} \lambda_i v_i, \ \lambda \geqslant 0 } \sum_{i=1}^m \lambda_i.
$$
This means that we may reformulate the original problem as minimizing with respect to $\lambda\in \rb^I$ the function $\Psi(\sum_{i \in I} \lambda_i v_i) + \sum_{i \in I} \lambda_i$. We are thus using an active set method with respect to all elements of $K^\circ$. Note that we may represent $K^\circ$ by its set of extreme points, which is finite when $K$ is a polytope. When $I$ is not finite, some care has to be taken but the algorithm also applies (see~\cite{dudik2012lifted,siammatrix} for details).

Moreover, the second part of the iteration is $\bar{w}_t \in \arg \min_{ h(\bar{w}) \leqslant 1 } \bar{w}^\top \Psi'(w_t)$, which is exactly equivalent to
testing    $\min_{ i \in I} v_i^\top \Psi'(\bar{w}_t) \geqslant -1$, which happens to be the optimality condition for the problem with respect to $\lambda$.

\paragraph{Convergence rates.}
Like Kelley's method or bundle methods, the dual simplicial method is finitely convergent when $K$ is a polytope. However, no bound is known regarding the number of iterations. Like the simplicial method, a simpler method which does  not require to fully optimize the subproblem comes with a convergence rate in $O(1/t)$. It replaces the full minimization with respect to $w$ in the conic hull of all~$w_i$ by a simple line-search over one or two parameters~\cite{zhang2012accelerated,harchaoui2013conditional,siammatrix}.

\section{Proximal methods}
\label{sec:proximal}
\label{sec:prox-opt}
When $\Psi$ is smooth, then the particular form on non-smoothness of the objective function may be taken advantage of. Proximal methods essentially allow to solve the problem regularized with a new regularizer at low implementation and computational costs. For a more complete presentation of optimization techniques adapted to sparsity-inducing norms, see, e.g.,~\cite{fot} and references therein.
Proximal-gradient methods constitute a class of first-order techniques typically designed to solve problems of the following form~\cite{Nesterov2007,Beck2009,Combettes2010}:
\BEQ
\label{eq:formulation}
\min_{w \in \rb^p} \Psi(w) + h(w),
\EEQ
where $\Psi$ is smooth. They take advantage of the structure of Eq.~(\ref{eq:formulation}) as the sum of two convex terms, only one of which is assumed smooth.
Thus, we will typically assume that $\Psi$ is differentiable (and in our situation in \eq{erm}, where
$\Psi$ corresponds to the data-fitting term, that the loss function $\ell$ is convex and differentiable), with Lipschitz-continuous gradients (such as the logistic or square loss), while $h$ will only be assumed convex.

Proximal methods have become increasingly popular over the past few years, 
both in the signal processing~(see, e.g.,~\cite{Becker2009, Wright2009,
Combettes2010} and numerous references therein) and in the machine learning communities~(see, e.g.,~\cite{fot} and
references therein).
In a broad sense, these methods can be described as providing a natural extension of gradient-based techniques
when the objective function to minimize has a non-smooth part.
Proximal methods are iterative procedures. 
Their basic principle is to linearize, at each iteration, the function $g$ around the current estimate
$\hat{w}$, and to update this estimate as the (unique, by strong convexity) solution of the following \textit{proximal problem}:
\begin{equation}\label{intro:eq:proxupdate}
\min_{w \in \rb^p}\bigg[ \Psi(\hat{w}) + (w - \hat{w})^\top \Psi'(\hat{w}) +   h(w) + \frac{L}{2}\|w - \hat{w}\|_2^2 \bigg].
\end{equation}
The role of the added quadratic term is to keep the update in a neighborhood of $\hat{w}$ where $\Psi$ stays close to its current linear approximation;
 $L\!>\!0$ is a parameter which is an upper bound on the Lipschitz constant of the gradient $\Psi'$.

Provided that we can solve efficiently the proximal problem in Eq.~(\ref{intro:eq:proxupdate}),
this first iterative scheme constitutes a simple way of solving problem in Eq.~(\ref{eq:formulation}).
It appears under various names in the literature:
proximal-gradient techniques~\cite{Nesterov2007},
forward-backward splitting methods~\cite{Combettes2010}, and
iterative shrinkage-thresholding algorithm~\cite{Beck2009}.
Furthermore, it is possible to guarantee convergence rates for the function values \cite{Nesterov2007,Beck2009},
and after $t$ iterations, the precision be shown to be of order $O(1/t)$, which should contrasted with rates for the subgradient case, that are rather $O(1/\sqrt{t})$.

This first iterative scheme can actually be extended to ``accelerated'' versions~\cite{Nesterov2007,Beck2009}. 
In that case, the update is not taken to be exactly the result from Eq.~(\ref{intro:eq:proxupdate}); instead, it is obtained as the solution of the proximal problem applied to
a well-chosen linear combination of the previous estimates.
In that case, the function values converge to the optimum with a rate of $O(1/t^2)$, where $t$ is the iteration number.
From~\cite{Nesterov2004}, we know that this rate is optimal within the class of first-order techniques; in other words, accelerated proximal-gradient methods 
can be as fast as without non-smooth component.

We have so far given an overview of proximal methods, without specifying how we precisely handle its core part, 
namely the computation of the proximal problem, as defined in Eq.~(\ref{intro:eq:proxupdate}).

\paragraph{Proximal problem.}
We first rewrite problem in Eq.~(\ref{intro:eq:proxupdate}) as
$$
\min_{w \in \rb^p}~ \frac{1}{2}\big\|w - \big(\hat{w} - \frac{1}{L}\Psi'(\hat{w}) \big) \big\|_2^2 + \frac{1}{L} h(w).
$$
Under this form, we can readily observe that when $h=0$, the solution of the proximal problem is identical to
the standard gradient update rule.
The problem above can be more generally viewed as an instance of the \textit{proximal operator}~\cite{Moreau1962} associated with $  h$: 
\begin{equation*}\label{intro:eq:proxoperator}
\text{Prox}_{  h}: u\in\rb^p \mapsto \argmin_{v \in \rb^p} \frac{1}{2} \|{u-v}\|_2^2 +    h(v).
\end{equation*}

 For many choices of regularizers $h$, the proximal problem has a closed-form solution, which makes proximal methods particularly efficient.
If $h$ is chosen to be the $\ell_1$-norm,  the proximal operator is simply the soft-thresholding operator applied elementwise~\cite{donoho}.
 In this monograph the function $h$ will be either  the \lova extension $f$ of the submodular function $F$, or, for non-decreasing submodular functions, the norm $\Omega_q$ defined in \mysec{increasing}
 and \mysec{l2}. In both
 cases, the proximal operator can be cast a exactly one of the separable optimization problems we consider \mychap{prox}.

 \section{Simplex algorithm for linear programming}
 \label{sec:simplex}
 
 We follow the exposition of~\cite{bertsimas1997introduction}.
We consider a vector $c \in \rb^n$ and a matrix $A \in \rb^{ m \times n}$, and the following linear program:
\BEQ
\min_{ Ax = b , \ x \geqslant 0 } c^\top x,
\EEQ
with dual problem
\BEQ
\max_{A^\top y \leqslant c} b^\top y ,
\EEQ
and optimality conditions: (a) $x \geqslant 0$, (b) $A^\top y \leqslant c$, (c) $y^\top ( A x - b)  = 0$.
We assume that $m \leqslant n$ and that the $m$ rows of $A$ are linearly independent.

The simplex method is an iterative algorithm that will explore vertices of the polyhedron of $\rb^n$ defined by $Ax = b$ and $ x \geqslant 0$. 
Since we are maximizing a linear function over this polyhedron, if the problem is bounded, the solution may be found within these vertices (the set of solution is typically a face of the polyhedron and the simplex outputs one of its vertices).

A \emph{basic feasible solution} is defined by a subset $J$ of $m$ linearly independent columns of $A$ (among the $n$ possible ones), and such that
$x_J = A_J^{-1} b$ has non-negative components, where $A_J$ denotes the submatrix of $A$ composed of the columns indexed by $J$ (note that since $|J|=m$, $A_J$ is a square matrix which is invertible because we have assumed $A$ has full rank). This defines a feasible solution $x \in \rb^n$. It is said non-degenerate if all components of $x_J$ are strictly positive. For simplicity, we assume that all basic feasible solutions are non-degenerate.

Given $J$ and $x_J = A_J^{-1} b$ a non-degenerate feasible solution, we consider a descent direction $d$ such that $d_{J^{\sf c}} = 0 $ except for a single component $j \in J^{\sf c}$ (which is equal to one). This will allow the variable $x_j$ to enter the active set by considering $x + u d$ for a sufficiently small $u>0$. In order to satisfy the constraint $A x = b$, we must have $A d = 0$. Since only the $m$ components of $d$ in $J$ are undetermined and $A$ has rank $m$, the vector $d$ is fully specified.  Indeed, we have $ A_J d_J + A_j = 0$, leading to $d_J = - A_J^{-1} A_j$ (with $A_j \in \rb^n$ being the $j$-th column of $A$). Along this direction, we have
$c^\top ( x + u d) = c_J^\top x_J + u  \big[ c_j - c_J^\top A_J^{-1} A_j ]$. If we define  $ y = A_J^{-\top} c_J \in \rb^{m} $ and   
$\bar{c} = c - A^\top y$, then the feasible direction $d$ including the new variable $j \in J^{\sf c}$ will lead to a rate of increase (or decrease) of $\bar{c}_j$.

Intuitively, if $\bar{c} \geqslant 0 $, there is no possible descent direction and $x$ should be optimal. Indeed,
  $(x,y)$ is then a primal-dual optimal pair (since the optimality conditions are then satisfied), otherwise, since $x$ is assumed non-degenerate, the direction $d$ for a $j \in J^{\sf c}$  such that $\bar{c}_j<0$ is a strict descent direction. This direction may be followed as long as $(x+u d)_J \geqslant 0$. 
  If $d_J $ has only nonnegative components, then the problem is unbounded. Otherwise, 
  the largest positive $u$ is  $u = \min_{ i, \ d_i< 0} \frac{x_{J(i)}}{-d_i} = \frac{x_{J(k)}}{-d_k} $.
  We then replace $J(k)$ by $j$ in $J$ and obtain a new basic feasible solution.

For non-degenerate problems, the iteration described above leads to a strict decrease of the primal objective, and since the number of basic feasible solution is finite, the algorithm terminates in finitely many steps; note however that there exists problem instances for which exponentially many basic feasible solutions
are visited, although the average-case complexity is polynomial~(see, e.g., \cite{spielman2004smoothed} and references therein). When the parameters $A$, $b$ and $c$ come from data with absolutely continuous densities, the problem is non-degenerate with probability one. However, linear programs coming from combinatorial optimization (like the ones we consider in this monograph) do exhibit degenerate solutions. Several strategies for the choice of basic feasible solutions may be used in order to avoid cycling of the iterations. See~\cite{bertsimas1997introduction,Nocedal:1999:NO1} for further details, in particular in terms of efficient associated numerical linear algebra.
 
In this monograph, the simplex method will be used for submodular function minimization, which will be cast a linear program with exponentially many variables (i.e., $n$ is large), but for which every step has a polynomial-time complexity owing to the greedy algorithm (see \mysec{simplexsfm} for details).

 \section{Active-set methods for quadratic programming}
 \label{sec:QP}
 We consider a vector $c \in \rb^n$, a positive semi-definite matrix $Q \in \rb^{n \times n}$ and a matrix $A \in \rb^{ m \times n}$, and the following quadratic program:
\BEQ
\min_{ Ax = b , \ x \geqslant 0 } \frac{1}{2} x^\top Q x + c^\top x.
\EEQ
 For simplicity, we assume that $Q$ is invertible,    $m < n$ and  $A \in \rb^{m \times n}$  has full column rank. The dual optimization problem is obtained as follows:
 \BEAS
& & \min_{ Ax = b , \ x \geqslant 0 } \frac{1}{2} x^\top Q x + c^\top x \\
& = & \min_{ x  \in \rb^n } \max_{ \lambda \in \rb^m, \ \mu \in \rb_+^n} \frac{1}{2} x^\top Q x + c^\top x - \lambda^\top ( Ax - b) - \mu^\top x \\
& = & \max_{ \lambda \in \rb^m, \ \mu \in \rb_+^n}  \min_{ x  \in \rb^n }  \frac{1}{2} x^\top Q x + c^\top x - \lambda^\top ( Ax - b) - \mu^\top x,
 \EEAS
 and the optimality conditions are (a) stationarity: $Qx + c - A^\top \lambda - \mu = 0$, (b) feasibility: $Ax=b$ and $\mu \geqslant 0$ and (c) complementary slackness: $\mu^\top x = 0$.
 
 Active-set methods rely on the following fact:  if the indices $J$ of the non-zero components of $x$ are known, then the optimal $x_J$ may be obtained as
   $\min_{A_J x_J = b } \frac{1}{2} x_J^\top Q_{JJ} x_J + c_J^\top x_J$. This is a problem with linear equality constraints but no inequality constraints. Its minimum may be found through a primal-dual formulation:
    \BEA
\nonumber & & \min_{ A_J x_J = b   } \frac{1}{2} x_J^\top Q_{JJ} x_J + c_J^\top x_J \\
\label{eq:QPJ}
& = & \min_{ x_J  \in \rb^{|J|} } \max_{ \lambda \in \rb^m } \frac{1}{2} x_J^\top Q_{JJ} x_J + c_J^\top x_J - \lambda^\top ( A_Jx_J - b),
 \EEA
 with optimality conditions: (a) $Q_{JJ} x_J + c_J - A_J^\top \lambda  = 0$ and (b) $ A_J x_J = b$. 
 Primal-dual pairs for \eq{QPJ} may thus be obtained as the solution of the following linear system:
 \BEQ
 \label{eq:linear}
 \bigg(
 \begin{array}{cc}
 Q_{JJ} & - A_J^\top \\
 -A_j & 0 
 \end{array}
 \bigg) \bigg(
 \begin{array}{c}
x_J \\ \lambda
 \end{array}
 \bigg) = \bigg(
 \begin{array}{c}
-c_J \\ 0 
 \end{array}
 \bigg).
 \EEQ
 The solution is globally optimal if
 and only if  $x_J \geqslant 0$ and $\mu_{J^{\sf c}} = Q_{J^{\sf c} J } x_J + c_{J^{\sf c}} - A_{J^{\sf c}}^\top \lambda \geqslant 0$.
 
 The iteration of an active set method is as follows, starting from a feasible point $x \in \rb^n$ (such that $x \geqslant 0$ and $Ax=b$) and an active set~$J$. From the set $J \subset \{1,\dots,n\}$, the potential solution $y \in \rb^n$ may be obtained as the solution of \eq{QPJ}, with dual variable $\lambda \in \rb^m$: 
 \begin{list}{\labelitemi}{\leftmargin=1.1em}
   \addtolength{\itemsep}{-.2\baselineskip}

\item[(a)] If $ y_J \geqslant 0$ and $ Q_{J^{\sf c} J } y_J + c_{J^{\sf c}} - A_{J^{\sf c}}^\top \lambda \geqslant 0$, then $y$ is globally optimal
\item[(b)] If  $y_J \geqslant 0$ and  there exists $j \in J^{\sf c} $ such that $ Q_{ jJ } y_J + c_{j} - A_{j}^\top \lambda < 0$, then $j$ is added to $J$, and $x$ replaced by $y$. 
\item[(c)] If $ \exists j \in J$ such that $y_j < 0$. Then let $u$ be the largest positive scalar so that $x + u(y-x) \geqslant 0$ and $k$ be an index so that 
$x_k + u(y_k - x_k) = 0$. The set $J$ is replaced by $ ( J \cup \{ j\} ) \backslash \{ k\}$ and $x$ by $x + u(y-x)$.
\end{list}

We can make the following observations:
\begin{list}{\labelitemi}{\leftmargin=1.1em}
   \addtolength{\itemsep}{-.2\baselineskip}

\item[--] The unique solution (since we have assumed that $Q$ is invertible) of the quadratic problem may have more than $m$ non-zero components for $x$ (as opposed to the simplex method). 

 \item[--] All iterations are primal-feasible.
\item[--] It is possible to deal with exponentially many components of $x$, i.e., $n$ very large, as long as it is possible to compute
$
\max_{j \in J^{\sf c}} Q_{ jJ } y_J + c_{j} - A_{j}^\top \lambda 
$ efficiently.

\item[--] In terms of numerical stability, care has to be taken to deal with approximation solutions of the linear system in \eq{linear}, which may be ill-conditioned. See more practical details in~\cite{Nocedal:1999:NO1}.
\item[--] Active sets methods may also be used when the matrix $Q$ is not positive definite~\cite{Nocedal:1999:NO1}. In this monograph, we will always consider adding an extra ridge penalty proportional to $\| x\|_2^2$ for a small $\varepsilon>0$. It in this section, we assume for simplicity that $I$ is finite, but it can be extended easily to infinite uncountable sets using gauge functions.

\item[--] Classical examples that will be covered in this monograph are $\min_{ \eta \geqslant 0, \ \eta^\top 1  =1} \frac{1}{2} \| S^\top \eta \|^2$ (then obtaining the minimum-norm-point algorithm described in \mysec{mnp}),
or least-squares problems $\frac{1}{2n} \| y - X w\|_2^2 + h(w)$, for $h(w)$ a polyhedral function, which may be represented either as
$h(w) = \max_{s \in K} s^\top w$ with $K$ being a polytope (or only its extreme points), or as $h(w) = \inf_{ w = \sum_{i \in I} \eta_i w_i}
\sum_{i \in I} \eta_i$, for a certain family $(w_i)_{i \in I}$. See next section
 for more details.

 \end{list}

   \section{Active set algorithms for least-squares problems$^\ast$}
   \label{sec:activeQPLS}
   We consider a design matrix $X \in \rb^{n \times p}$ and the following optimization problem
   \BEQ
\min_{w \in \rb^p}   \frac{1}{2n} \| y - X w\|_2^2 + \lambda h(w),
   \EEQ
   for a certain non-negative polyhedral convex function $h$, which may be represented either as
$h(w) = \max_{s \in K} s^\top w$ with $K$ being a polytope (or only its extreme points), or as $h(w) = \inf_{ w = \sum_{i \in I} \eta_i w_i}
\sum_{i \in I} \eta_i$, for a certain family $(w_i)_{i \in I}$. 
    We will assume for simplicity that $X^\top X$ is invertible. In practice, one may add a ridge penalty $\frac{\varepsilon}{2} \| w\|_2^2$.
  
  \paragraph{Primal active-set algorithm.}
  We consider the first representation $h(w) = \inf_{ w = \sum_{i \in I} \eta_i w_i}
\sum_{i \in I} \eta_i$, for a certain family $(w_i)_{i \in I}$, which leads to the following optimization problem in $\eta \in \rb^I$:
\BEQ
\min_{ \eta \in \rb_+^I} \frac{1}{2n} \bigg\| y -  \sum_{i \in I} \eta_i X w_i \bigg\|_2^2 + \lambda \sum_{i \in I} \eta_i.
\EEQ
We consider the matrix $Z \in \rb^{n \times |I|}$ with columns $ X w_i \in \rb^n$, for $i \in I$. The problem is then equivalent to a least-square problem with non-negative constraints (with algorithms similar to algorithms for $\ell_1$-regularized problems~\cite{fot}).

  The active-set algorithm starts from $J = \varnothing$ and $\eta = 0$, and perform the following iteration:
  \begin{list}{\labelitemi}{\leftmargin=1.1em}
   \addtolength{\itemsep}{-.2\baselineskip}

\item[--] Compute $\zeta$ such that $\zeta_{ J^{\sf c}}=0$ and $\zeta_J \in \arg\min_{\zeta_J \in \rb^J}
\frac{1}{2n} \big\| y - Z_{J} \zeta_J \big\|_2^2 + \lambda \sum_{i \in J} \zeta_i$, which is equal to
$\zeta_J = (Z_J^\top Z_J)^{-1} (Z_J^\top y - n \lambda 1_J)$.

\item[--] If $\zeta_J \geqslant 0$ and $ Z_{J^ { \sf c} }^\top ( Z_J \zeta_J - y ) + n \lambda 1_{J^ { \sf c} } \geqslant 0$, then $\zeta$ is globally optimal.
\item[--] If $\zeta_J \geqslant 0$ and $\exists j \in J^{\sf c}$,  $ Z_{ j } ^\top ( Z_J \zeta_J - y ) + n \lambda   <  0$, then replace $J$ by $J \cup \{j \}$ and $\eta $ by $\zeta$.
\item[--] If $\exists j \in J$, $\zeta_j < 0$, then let $u$ be the largest positive scalar so that $\eta + u(\zeta-\eta) \geqslant 0$ and $k$ be an index so that 
$\eta_j + u(\zeta_k-\eta_k)= 0$, i.e., $k \in \argmin_{ k  \in J, \ \zeta_k < \eta_k } \frac{\eta_k}{ \eta_k - \zeta_k}$. The set $J$ is replaced by $ ( J \cup \{ j\} ) \backslash \{ k\}$ and $\eta$ by $\eta + u(\zeta-\eta)$.
\end{list}

The algorithm terminates after finitely many iterations (but the number of these, typically of order $O(p)$, may be exponential in $p$ in general).
Note that the algorithm needs to access the potentially large number of columns of $Z$ through the maximization of $Z_j^\top t$ with respect to $j \in I$ for a certain vector $t \in \rb^p$. This corresponds to the support function of the convex hull of columns $Z_j$, $j \in I$. Moreover, only linear systems with size $|J|$ need to be solved, and these are usually small.

Note that this algorithm is close to a specific instantiation of the dual simplicial method of \mysec{dualsimplicial}. Indeed, every time we are in the situation where we add a new index to $J$ (i.e., $\zeta_J \geqslant 0$ and $\exists j \in J^{\sf c}$,  $ Z_{ j } ^\top ( Z_J \zeta_J - y ) + n \lambda   <  0$), then we have the solution of the original problem (with positivity constraints) on the reduced set of variables~$J$. Note that when a variable is removed in the last step, it may  re-enter the active set later on (this appears very unfrequently in practice, see a counter-example in \myfig{counterexample} for the  minimum-norm-point algorithm, which is a dual active set algorithm for a  least-square problem with no design), and thus we only have a partial instantiation of the dual simplicial method.

\paragraph{Primal regularization paths}
 A related algorithm computes the entire \emph{regularization path}, i.e., the set of solutions for all $\lambda \in \rb_+$. These algorithms hinge on the fact that the set $J\subset I$ is globally optimal as long as 
$\zeta_J = (Z_J^\top Z_J)^{-1} (Z_J^\top y - n \lambda 1_J) \geqslant 0$ and  $ Z_{J^ { \sf c} }^\top ( Z_J \zeta_J - y ) + n \lambda 1_{J^ { \sf c} }\geqslant 0$, which defines an \emph{interval} of validity in $\lambda$, leading to a solution path which is  piecewise affine in $\lambda$~\cite{markowitz,osborne}.

Starting from the first break-point, $\lambda_0 = \max_{ j \in J} \frac{1}{n} Z_j^\top y$, the solution $\eta^0=0$ and the set $J_0$ composed of the index $j$ maximizing  $ Z_j^\top y$ (so that $w = 0$ for all $\lambda \geqslant \lambda_0$), the following iterations are performed:
 \begin{list}{\labelitemi}{\leftmargin=1.1em}
   \addtolength{\itemsep}{-.2\baselineskip}

\item[--] For $J = J_k$, compute the smallest $\lambda_k \geqslant 0$ such that (a)  $(Z_J^\top Z_J)^{-1}  Z_J^\top y - n \lambda  (Z_J^\top Z_J)^{-1} 1_J  \geqslant 0$ and 
(b) $ Z_{J^ { \sf c} }^\top \big(    Z_J (Z_J^\top Z_J)^{-1}  Z_J^\top - \idm \big)  y   + n \lambda  \big( 1_{J^ { \sf c} }  - Z_{J^ { \sf c} }^\top (Z_J^\top Z_J)^{-1}  \big) \geqslant 0$.
\item[--] On the interval $ [ \lambda_{k} , \lambda_{k-1} ]$, the optimal set of $J$ and $\eta_J =  (Z_J^\top Z_J)^{-1} (Z_J^\top y - n \lambda 1_J) $, set $\eta_J^k = (Z_J^\top Z_J)^{-1} (Z_J^\top y - n \lambda_k 1_J)$.
\item[--] If $\lambda_k=0$, the algorithm terminates.
\item[--] If the constraint (a) is the limiting one, with corresponding index $j \in J$, then set $J_k = J_k \backslash \{ j\}$.
\item[--] If the constraint (b) is the limiting one, with corresponding index $j \in J^c$, then set $J_k = J_k \cup \{ j\}$.
\item[--] Replace $k$ by $k+1$.
\end{list}
The algorithm stops with a sequence of break-points $(\lambda_k)$, and corresponding vectors $(\eta^k)$. The number of break-points is typically of order $O(|I|)$ but it may be exponential in the worst-case~\cite{mairal2012complexity}. Note that typically, the algorithm may be stopped after a certain maximal size of active set $|J|$ is attained. Then, beyond the linear system with size less than $|J|$ that need to be solved, the columns of $Z$ are accessed to satisfy constraint (b) above, which requires more than simply maximizing $Z_j^\top u$ for some $u$ (but can be solved by binary search using such tests).

\begin{figure}
\begin{center}
 
  \includegraphics[scale=.37]{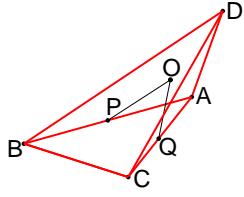}  
   
  \end{center}

 \vspace*{-.6cm}

\caption{Counter-example for re-entering of a constraint in quadratic programming. We consider the problem of finding the projection of $O$ onto the convex hull of four points in three dimensions $A$, $B$, $C$ and $D$. Since $O$ is in this convex hull, it is equal to its projection. $P$ is the projection of $O$ on the segment $AB$. Starting from the active set $\{A,B\}$, the best new point to add is $C$. The points $Q$  is the projection of $O$ on the triangle $ABC$ and happens to be on the segment $AC$. This implies that $B$ has exited the active set; however, it needs to re-enter in the next iteration because the optimal solution includes it.
}
\label{fig:counterexample}
\end{figure}

  \paragraph{Dual active-set algorithm.}
    We consider the representation $h(w) =  \max_{ i \in I} s_i^\top w$, for a family $(s_i)_{i \in I}$ of vectors in $\rb^p$. For simplicity, we assume that $0$ is in the convex hull of points $s_i$ with known linear combination coefficients (in practice, one may simply add $0$ to the set of $s_i$, $i\in I$, so that one can start the algorithm with $\alpha=0$ and $w$ the ordinary least-square solution) and that $S \in \rb^{|I| \times p}$ is the matrix whose columns are the vectors $s_i$, $i \in I$. This leads to the optimization problem:
      \BEQ
\min_{w \in \rb^p}   \frac{1}{2n} \| y - X w\|_2^2 + \lambda \max_{ i \in I} s_i^\top w .
   \EEQ
   We may derive a dual optimization problem by introducing a new variable $u = X w \in \rb^n$ and its associated Lagrange multipler $\alpha \in \rb^n$:
\BEAS
&&\min_{w \in \rb^p}   \frac{1}{2n} \| y - X w\|_2^2 + \lambda \max_{ i \in I} s_i^\top w \\
& \!\!\!\!\!\!=  \!\!\!\!\!\! & \!\! \min_{ (w,u) \in \rb^{p+n}} \frac{1}{2n} \| y  - u \|_2^2 + \!\! \!\!  \max_{ (\eta_i)_i \in \rb_+^I,  \ \eta^\top 1_I = \lambda }  \!\! \!\! w^\top  \sum_{i \in I } \eta_i s_i   + \max_{ \alpha \in \rb^n} \alpha^\top ( u - X w)\\
&\!\!\!\!\!\! =  \!\!\!\!\!\!& \min_{ (w,u) \in \rb^{p+n}} \max_{\alpha \in \rb^n, \  (\eta_i)_i \in \rb_+^I} 
\frac{1}{2n} \| y  - u \|_2^2 + w^\top  \sum_{i \in I } \eta_i s_i   +  \alpha^\top ( u - X w) \\
&  &  \hspace*{6cm}  \mbox{ such that }   \eta^\top 1_I = \lambda  \\
&\!\!\!\!\!\! =  \!\!\!\!\!\!& \max_{\alpha \in \rb^n, \  (\eta_i)_i \in \rb_+^I}  \min_{ (w,u) \in \rb^{p+n}} 
\frac{1}{2n} \| y  - u \|_2^2 + w^\top  \sum_{i \in I } \eta_i s_i   +  \alpha^\top ( u - X w) \\
&  & \hspace*{6cm} \mbox{ such that }   \eta^\top 1_I = \lambda  \\
&\!\!\!\!\!\! =  \!\!\!\!\!\!& \max_{\alpha \in \rb^n, \  (\eta_i)_i \in \rb_+^I, \ \eta^\top 1_I = \lambda }   
- \frac{n}{2} \| \alpha\|_2^2 + y^\top \alpha \mbox{ such that }  S^\top \eta = X^\top \alpha,
\EEAS   
where the optimal $u$ is obtained from $\alpha$ as $u = y - n \alpha$.
The problem above is a quadratic program  in the variables $\eta$ and $\alpha$. The active set algorithm described in \mysec{QP} may thus be applied, and starting from feasible dual variables $(\alpha,\eta)$ (which are easy to find with $\alpha=0$, since $0$ is in the convex hull of all $s_i$), and a subset $J \subseteq I$, the following iteration is performed:

\begin{list}{\labelitemi}{\leftmargin=1.1em}
   \addtolength{\itemsep}{-.2\baselineskip}
\item[--] Compute a maximizer $(\beta,\zeta_J)$ of $- \frac{n}{2} \| \beta \|_2^2 + y^\top \beta$ subjet to $ S_J^\top \zeta_J = X^\top \beta$ and $\zeta_J^\top 1_J = \lambda$. This problem is may be put in variational form as follows:
\BEAS
& & \max_{\beta \in \rb^n} \max_{ \zeta_J^\top 1_J = \lambda}  \min_{ w \in \rb^p} - \frac{n}{2} \| \beta\|_2^2 + y^\top \beta + w^\top ( S_J^\top \zeta_J - X^\top \beta)  \\
& =  &  \min_{ w \in \rb^p,  \ S_J w  = c 1_J, \ c \in \rb}   \frac{1}{2n}  \| y - X w\|_2^2 + \lambda c,
\EEAS
with the following optimality conditions (a) stationarity: $y - Xw - n \beta = 0$, (b) feasibility: $S_J^\top \zeta_J = X^\top \beta$, $\zeta_J^\top 1_J = \lambda$ and $S_J w  = c 1_J$. These may be put in a single symmetric linear system:
$$
\Bigg(
\begin{array}{ccc}
\frac{1}{n}X^\top X & S_J^\top &  0 \\
S_J & 0 & - 1_J \\
0 & -1_J^\top & 
 \end{array}
\Bigg)
\Bigg(
\begin{array}{c}
w \\
\zeta_J \\
c
\end{array}
\Bigg)
= 
\Bigg(
\begin{array}{c}
\frac{1}{n} X^\top y  \\
0_J \\
-\lambda
\end{array}
\Bigg).
$$
\item[--] If $\zeta_J \geqslant 0$ and $\max_{j \in J^{\sf c}} s_j^\top w \leqslant c$, then the pair $(\zeta, \beta)$ is globally optimal.
\item[--] If $\zeta_J \geqslant 0$ and $\max_{j \in J^{\sf c}} s_j^\top w >  c$, then the set $J$ is replaced by $J \cup \{j\}$ with $j$ the corresponding maximizer in $J^{ \sf c}$, and $(\eta,\alpha)$ by $(\zeta, \beta)$.  
\item[--] If $\exists j \in J$ such that $\zeta_j < 0$, then  then let $u$ be the largest positive scalar so that $\eta + u(\zeta-\eta) \geqslant 0$ and $k$ be an index so that 
$\eta_j + u(\zeta_k-\eta_k)= 0$, i.e., $k \in \argmin_{ k  \in J, \ \zeta_k < \eta_k } \frac{\eta_k}{ \eta_k - \zeta_k}$. The set $J$ is replaced by $ ( J \cup \{ j\} ) \backslash \{ k\}$ and $(\eta,\alpha)$ by $(\eta + u(\zeta-\eta), \alpha + u ( \beta - \alpha) )$.
\end{list}

Note that the full family of vectors $s_i$ is only accessed through the maximization of a linear function $s_j^\top w$ for a certain $w$. This is thus well adapted to our situation where $s_i$ are the extreme points of the the base polytope (or of a polyhedral dual ball). Moreover, this algorithm is close to a  particular instantiation of the simplicial algorithm from \mysec{simplicial}, and, like in the primal active-set method, once a variable is removed, it may  re-enter the active set (this is not frequent in practice, see a counter-example in \myfig{counterexample}).
  
  In our context, where $h(w)$ may be a sparsity-inducing norm, then  the potential sparsity in $w$ is not used (as opposed to the primal active-set method). This leads in practice to large active sets and potential instability problems (see \mychap{experiments}).
 Finally, regularization paths may be derived using the same principles as before, since the local solution with a known active set has an affine dependence in $\lambda$.

\chapter{Separable Optimization Problems: Analysis}
\label{chap:prox}

In this chapter, we consider separable convex functions and   the minimization of such functions penalized by the \lova extension of a submodular function. When the separable functions are all quadratic functions, those problems are often referred to as \emph{proximal problems} and are often used as inner loops in convex optimization problems regularized by the \lova extension (see a brief introduction in \mysec{prox-opt} and, e.g.,~\cite{Combettes2010,fot} and references therein).  Beyond their use for convex optimization problems, we show in this chapter  that they are also intimately related to submodular function minimization, and can thus be also useful to solve discrete optimization problems.

We first study the separable optimization problem and derive its dual---which corresponds to maximizing a separable function on the base polyhedron $B(F)$---and associated optimality conditions in \mysec{optcond}. We then  consider in \mysec{relation} the equivalence between separable optimization problems and a sequence of submodular minimization problems. 
In \mysec{quadprox}, we focus on quadratic functions, with intimate links with submodular function minimization and orthogonal projections on $B(F)$. Finally, in \mysec{other}, we consider optimization problems on the other polyhedra we have defined, i.e., $P(F)$, $P_+(F)$ and $|P|(F)$ and show how solutions may be obtained from solutions of the separable problems on $B(F)$. For related algorithm see \mychap{prox-algo}.

\section{Optimality conditions for base polyhedra}
\label{sec:optcond}
Throughout this chapter, we make the simplifying assumption that the problem is strictly convex and differentiable (but not necessarily quadratic) and such that the derivatives are unbounded, but sharp statements could also be made in the general case. The next proposition shows that by convex strong duality (see Appendix~\ref{app:convex}), it is equivalent to the maximization of a separable concave function over the base polyhedron.

\begin{proposition}\textbf{(Dual of proximal optimization problem)}
\label{prop:prox} 
Let $F$ be a submodular function and $f$ its \lova extension.
Let $\psi_1,\dots,\psi_p$ be  $p$ continuously differentiable strictly convex functions on $\rb$ such
that for all $j \in V$,  $\sup_{\alpha \in \rb} \psi_j'(\alpha) = +\infty$
 and $\inf_{\alpha \in \rb} \psi_j'(\alpha) = -\infty$. Denote
$\psi_1^\ast, \dots,\psi^\ast_p$ their Fenchel-conjugates (which then have full domain). The two following optimization problems are dual of  each other:
 \BEA
\label{eq:prox}
 \min_{w \in \rb^p} f(w) + \sum_{j=1}^p \psi_j(w_j), \\
 \label{eq:proxdual} 
   \max_{s \in B(F)}   - \sum_{j=1}^p \psi_j^\ast(-s_j) .
 \EEA
 The pair $(w,s)$ is optimal if and only if (a)  $s_k = - \psi_k'(w_k)$ for all $k \in \{1,\dots,p\}$, and (b) $s \in B(F)$ is optimal for the maximization of $w^\top s$ over $s \in B(F)$ (see Prop.~\ref{prop:optsupporttight} for optimality conditions).
\end{proposition}
\begin{proof}
We have assumed that for all $j \in V$,  $\sup_{\alpha \in \rb} \psi_j'(\alpha) = +\infty$
 and $\inf_{\alpha \in \rb} \psi_j'(\alpha) = -\infty$.  This implies that the Fenchel-conjugates $\psi_j^\ast$ (which are already differentiable because of the strict convexity of $\psi_j$~\cite{borwein2006caa}) are defined and finite on $\rb$; moreover, since each $\psi_k$ is continuously differentiable, each $\psi_k^\ast$ is strictly convex. This implies that both $w$ and $s$ are unique.
 
We have (since strong duality applies because of Fenchel duality, see Appendix~\ref{app:convexana} and~\cite{borwein2006caa}):
\BEAS
\nonumber \min_{w \in \rb^p} f(w) + \sum_{j=1}^p \psi_j(w_j)
\nonumber &  = &  \min_{w \in \rb^p}  \max_{s \in B(F)} w^\top s  + \sum_{j=1}^p \psi_j(w_j), \\[-.1cm]
\nonumber &  = &  \max_{s \in B(F)}   \min_{w \in \rb^p}  w^\top s  + \sum_{j=1}^p \psi_j(w_j), \\[-.2cm]
 &  = &  \max_{s \in B(F)}   - \sum_{j=1}^p \psi_j^\ast(-s_j) ,
\EEAS
where $\psi_j^\ast$ is the Fenchel-conjugate of $\psi_j$.
Thus the separably penalized problem defined in \eq{prox} is equivalent to a separable maximization over the base polyhedron (i.e., \eq{proxdual}). Moreover, the unique optimal~$s$ for \eq{proxdual} and the unique optimal $w$ for \eq{prox} are related through $s_j = - \psi_j'(w_j)$ for all $j \in V$.
\end{proof}

 \paragraph{Duality gap.}
 Given a pair of candidate $(w,s)$ such that $w \in \rb^p$ and $s \in B(F)$, then the difference between the primal objective function in \eq{prox} and the dual objective in \eq{proxdual} provides a certificate of suboptimality for both $w$ and $s$. It is equal to:
\BEQ
\label{eq:gap} {\rm gap} (w,s) = f(w) - w^\top s + \sum_{j \in V} \big\{
 \psi_j(w_j) + \psi^\ast(-s_j) - w_j(-s_j)
 \big\}.
 \EEQ
 Note that ${\rm gap} (w,s) $ is always non-negative, is the sum of the non-negative terms (by Fenchel-Young inequality, see Appendix~\ref{app:convex}):
 $ f(w) - w^\top s$ and  $\psi_j(w_j) + \psi^\ast(-s_j) - w_j(-s_j)$, for $j \in \{1,\dots,p\}$; this gap is thus equal to zero if and only  these two terms are equal to zero.

 \section{Equivalence with submodular function minimization}
\label{sec:relation}

 Following~\cite{chambolle2009total}, we also consider a sequence of set optimization problems, parameterized by $\alpha \in \rb$:
\BEQ
\label{eq:proxalpha}
 \min_{A \subseteq V} F(A) + \sum_{j \in A}  \psi_j'(\alpha).
 \EEQ
We denote by $A^\alpha$ any minimizer of \eq{proxalpha}; typically, there may be several minimizers. Note that $A^\alpha$ is a minimizer of a submodular function $F + \psi'(\alpha)$, where $\psi'(\alpha) \in \rb^p$ is the vector of components $\psi_k'(\alpha)$, $k \in \{1,\dots,p\}$.

The key property we highlight in this section is that, as shown in~\cite{chambolle2009total}, solving \eq{prox}, which is a convex optimization problem, is equivalent to solving \eq{proxalpha} for all possible $\alpha \in \rb$, which are submodular optimization problems. We first show a monotonicity property of solutions of \eq{proxalpha}. Note that in the sequence of arguments showing equivalence between the separable convex problems, submodularity is only used here.

\begin{proposition}\textbf{(Monotonicity of solutions)}
\label{prop:monotonicity}
Under the same assumptions than in Prop.~\ref{prop:prox},
if $\alpha < \beta$, then  any solutions $A^\alpha$ and $A^\beta$ of \eq{proxalpha} for $\alpha$ and $\beta$ satisfy $A^\beta \subseteq A^\alpha$.
\end{proposition}
\begin{proof}
We have, by optimality of $A^\alpha$ and $A^\beta$:
\BEAS
F(A^\alpha) + \sum_{j \in A^\alpha}  \psi_j'(\alpha) 
 & \leqslant &   F(A^\alpha \cup A^\beta) + \sum_{j \in A^\alpha \cup A^\beta}  \psi_j'(\alpha) \\[-.05cm]
F(A^\beta) + \sum_{j \in A^\beta}  \psi_j'(\beta) 
 & \leqslant &   F(A^\alpha \cap A^\beta) + \sum_{j \in A^\alpha \cap A^\beta}  \psi_j'(\beta) , 
\EEAS
and by summing the two inequalities and using the submodularity of~$F$, 
$$
 \sum_{j \in A^\alpha}  \psi_j'(\alpha) + \sum_{j \in A^\beta}  \psi_j'(\beta) 
 \leqslant  \sum_{j \in A^\alpha \cup A^\beta}  \psi_j'(\alpha)
 + \sum_{j \in A^\alpha \cap A^\beta}  \psi_j'(\beta),
$$
which is equivalent to
$\sum_{j \in  A^\beta \backslash A^\alpha}  \big[ \psi_j'(\beta) - \psi_j'(\alpha)  \big] \leqslant 0$, which implies, since for all $j \in V$, $\psi_j'(\beta)>\psi_j'(\alpha)$ (because of strict convexity), that $A^\beta \backslash A^\alpha = \varnothing$.
\end{proof}

The next proposition shows that we can obtain the unique solution of \eq{prox} from all solutions of \eq{proxalpha}.

\begin{proposition}\textbf{(Proximal problem from submodular function minimizations)}
\label{prop:proxmin}
Under the same assumptions than in Prop.~\ref{prop:prox},
given any solutions $A^\alpha$ of problems in \eq{proxalpha}, for all $\alpha \in \rb$, we define the vector $u \in \rb^p$ as
$$
u_j = \sup( \{
\alpha \in \rb, \ j \in A^\alpha
\}).
$$
Then $u$ is the unique solution of the convex optimization problem in \eq{prox}.
\end{proposition}
\begin{proof}
Because $\inf_{\alpha \in \rb} \psi_j'(\alpha) = -\infty$, for $\alpha$ small enough, we must have $A^\alpha=V$, and thus $u_j$ is well-defined and finite for all $j \in V$.

If $ \alpha >  u_j $, then, by definition of $u_j$,  $j \notin A^\alpha$. This implies
that $A^\alpha \subseteq \{ j \in V, u_j \geqslant \alpha \}  = \{ u \geqslant \alpha \}$.
Moreover, if $u_j > \alpha$, there exists $\beta \in ( \alpha, u_j) $ such that $j \in A^\beta$. By the monotonicity property of Prop.~\ref{prop:monotonicity}, $A^\beta$ is included in $A^\alpha$. This implies 
$\{ u >\alpha \} \subseteq A^\alpha$.

We have for all $w \in \rb^p$, and $ \beta $ less than the smallest of $(w_j)_-$ and the smallest of $(u_j)_-$, $j \in V$, using \eq{lova4} from Prop.~\ref{prop:lova}:
\BEAS
& & f(u)  + \sum_{j=1}^p \psi_j(u_j) \\[-.15cm]
& \!\!\! = \!\!\!  & 
\int_{0}^\infty F( \{u \geqslant \alpha\} ) d\alpha
+ \int_{\beta }^0 ( F( \{u \geqslant \alpha\} )  - F(V) ) d\alpha
\\[-.25cm]
& & \hspace*{4cm} + \sum_{j=1}^p \bigg\{ \int_{\beta}^{u_j}  \psi_j'(\alpha) d \alpha + \psi_j(\beta) \bigg\} \\[-.05cm]
&  \!\!\!=  \!\!\!&   C + 
\int_{\beta}^\infty
\bigg[
 F( \{u \geqslant \alpha\} ) + \sum_{j=1}^p  (1_{u_j \geqslant \alpha})_j  \psi_j'(\alpha)
\bigg]
 d\alpha  \\[-.25cm]
 & & \hspace*{3cm}  \mbox{ with } C = \int_0^\beta F(V) d\alpha + \sum_{j=1}^p \psi_j(\beta), \\[-.2cm]
& \!\!\! \leqslant  \!\!\! &   C + 
\int_{\beta}^\infty
\bigg[
 F( \{w \geqslant \alpha\} ) + \sum_{j=1}^p  (1_{w_j \geqslant \alpha})_j \psi_j'(\alpha)
\bigg] d\alpha ,
\EEAS
because  $A^\alpha$ is optimal for $F + \psi'(\alpha)$ and
$\{u > \alpha\} \subseteq A^\alpha \subseteq \{ u \geqslant \alpha\}$ (and what happens  when $\alpha$ is equal to one of the components of $u$ is irrelevant for integration).
By performing the same sequence of steps on the last equation, we get:
$$
  f(u)  + \sum_{j=1}^p \psi_j(u_j)  \leqslant  f(w)  + \sum_{j=1}^p \psi_j(w_j).
$$
This shows that $u$ is indeed the unique optimum of the problem in \eq{prox}.
\end{proof}

From the previous proposition, we also get the following corollary, i.e., all solutions of the submodular function minimization problems in \eq{proxalpha} may be obtained from the  unique solution  of the convex optimization problem in \eq{prox}. Note that we immediately get the maximal and minimal minimizers, but that there is no general characterization of the set of minimizers (see more details in \mysec{mini}).

\begin{proposition}\textbf{(Submodular function minimizations from proximal problem)}
\label{prop:subfromprox}
Under the same assumptions than in Prop.~\ref{prop:prox}, if $u$ is the unique minimizer of \eq{prox}, then for all $\alpha \in \rb$, the minimal minimizer of \eq{proxalpha} is $\{ u > \alpha \}$ and the maximal minimizer is $\{ u \geqslant \alpha\}$. Moreover, we have $\{ u > \alpha\} \subseteq A^\alpha \subseteq \{ u \geqslant \alpha \}$ for any minimizer $A^\alpha$.
\end{proposition}
\begin{proof}
From the definition of the supremum in Prop.~\ref{prop:proxmin}, then we immediately obtain that $\{ u > \alpha\} \subseteq A^\alpha \subseteq \{ u \geqslant \alpha \}$ for any minimizer $A^\alpha$. Moreover, if $\alpha$ is not a value taken by some $u_j$, $j \in V$, then this defines uniquely $A^\alpha$. If not, then we simply need to show that  $\{ u \geqslant \alpha \}$ and $ \{ u > \alpha \}$ are indeed maximizers, which can be obtained by taking limits of $A^\beta$ when $\beta $ tends to $\alpha$ from below and above.
\end{proof}

\paragraph{Duality gap.} 
The previous proposition relates the optimal solutions of different optimization problems. The next proposition shows that approximate solutions also have a link, as the duality gap for \eq{prox} is the integral over $\alpha$ of the duality gaps for \eq{proxalpha}.

\begin{proposition}\textbf{(Decomposition of duality gap)}
With  the same assumptions than Prop.~\ref{prop:prox}, let $s \in B(F) $ and $w \in \rb^p$. The 
gap ${\rm gap}(w,s)$ defined in \eq{gap} decomposes as follows:
\BEA
\nonumber\!\!\! \!\!\!\!\!\!\!\!\!{\rm gap}(w,s) & \!\!\!\!= \!\!\! & f(w) - w^\top s + \sum_{j=1}^p \Big\{
\psi_j(w_j) + \psi_j^\ast(-s_j) + w_j s_j
\Big\} \\[-.2cm]
\label{eq:dualitygapprox}
  &\!\!\!\! = \!\!\! &  \!\!\!
\int_{-\infty}^{+\infty} \!\!
\Big\{
(F + \psi'(\alpha) )( \{ w \geqslant \alpha \}) - ( s + \psi'(\alpha) )_-(V)
\Big\} d\alpha.
\EEA
\end{proposition}
\begin{proof}
From \eq{lova4} in Prop.~\ref{prop:lova}, for $M \geqslant 0 $ large enough,
$$f(w) = \int_{-M}^{+\infty} F( \{ w \geqslant \alpha\}) d \alpha - M F(V).$$
 Moreover, for any 
$j \in \{1,\dots,p\}$,
$$
\psi_j(w_j) = \int_{-M}^{w_j} \psi_j'(\alpha) d\alpha + \psi_j(-M)
= \int_{-M}^{+\infty} \psi_j'(\alpha) 1_{ w_j \geqslant \alpha} d\alpha + \psi_j(-M).
$$
Finally, since $s_j + \psi_j'(\alpha) \Leftrightarrow \alpha \leqslant (\psi^\ast)'(-s_j)$, we have:
\BEAS
\int_{-M}^{+\infty} ( s_j + \psi_j'(\alpha))_- d\alpha
&\!\!\! = \!\!\! & \int_{-M}^{(\psi^\ast)'(-s_j)} ( s_j + \psi_j'(\alpha)) d\alpha \\
& \!\!\! = \!\!\! &  s_j \big[ (\psi^\ast)'(-s_j) + M
\big] +  \psi_j((\psi^\ast)'(-s_j) )  - \psi_j(-M)  \\
& \!\!\!  =  \!\!\!  & 
 s_j   M  - \psi_j(-M) - \psi_j^\ast(-s_j),  \EEAS
 the last equality stemming from equality in Fenchel-Young inequality (see Appendix~\ref{app:convexana}).
 By combining the last three equations, we obtain (using $s(V) =F(V)$):
 $$
 {\rm gap}(w,s)  = 
 \int_{-M}^{+\infty} \!\!
\Big\{
(F + \psi'(\alpha) )( \{ w \geqslant \alpha \}) - ( s + \psi'(\alpha) )_-(V)
\Big\} d\alpha.
 $$
 Since the integrand is equal to zero for $\alpha \leqslant -M$, the result follows.
\end{proof}
Thus, the duality gap of the separable optimization problem in Prop.~\ref{prop:prox}, may be written as the integral of a function of~$\alpha$. It turns out that, as a consequence of Prop.~\ref{prop:dualmin} (\mychap{sfm}), this function of $\alpha$ is the duality gap for the minimization of the submodular function $F + \psi'(\alpha)$. Thus, we obtain another direct proof
of the previous propositions. \eq{dualitygapprox} will be particularly useful when relating an approximate solution of the convex optimization problem to an approximate solution of the combinatorial optimization problem of minimizing a submodular function (see \mysec{approxsfm}).

\section{Quadratic optimization problems}
\label{sec:quadprox}
When specializing Prop.~\ref{prop:prox} and Prop.~\ref{prop:subfromprox} to quadratic functions, we obtain the following corollary, which shows how to obtain minimizers of $F(A) + \lambda |A|$ for all possible $\lambda \in \rb$ from a single convex optimization problem:

\begin{proposition}\textbf{(Quadratic optimization problem)}
\label{prop:quadprox}
Let $F$ be a submodular function and $w \in \rb^p$ the unique minimizer of $w \mapsto f(w) + \frac{1}{2}\|w\|_2^2$. Then:
\\
(a) $s=-w$ is the point in $B(F)$ with minimum $\ell_2$-norm,
\\
(b) For all $\lambda \in \rb$, the maximal minimizer of $A \mapsto F(A) + \lambda |A|$ is $\{ w \geqslant -\lambda \}$ and the minimal minimizer of $F$ is $\{w > -\lambda \}$.
\end{proposition}
One of the consequences of the last proposition is that some of the solutions to the problem of minimizing a submodular function subject to cardinality constraints may be obtained directly from the solution of the quadratic separable optimization problems (see more details in~\cite{ICML2011Nagano_506}).

Another crucial consequence is obtained for $\lambda=0$: a minimizer of the submodular function $F$ may be obtained by thresholding the orthogonal projection of $0$ onto the base polyhedroon $B(F)$~\cite{fujishige2006minimum}. See more details in \mychap{sfm}.

\paragraph{Primal candidates from dual candidates.} From Prop.~\ref{prop:quadprox}, given the \emph{optimal} solution $s$ of $\max_{s \in B(F)} -\frac{1}{2}\| s\|_2^2$, we obtain the \emph{optimal} solution $w=-s$ of $\min_{w \in \rb^p} f(w) + \frac{1}{2} \| w\|_2^2$. However, when using approximate algorithms such as the ones presented in \mychap{prox-algo}, one may actually get only an \emph{approximate} dual solution $s$, and in this case, one can improve over the natural candidate primal solution $w=-s$. Indeed, assume that the components of $s$ are sorted in increasing order $s_{j_1} \leqslant \cdots \leqslant s_{j_p}$, and denote $t \in B(F)$ the vector defined by 
$t_{j_k} = F( \{ j_1,\dots, j_{k} \}) -F( \{ j_1,\dots, j_{k-1} \})$ . Then we have $f(-s) = t^\top (-s)$, and for any $w$ such that $w_{j_1} \geqslant \cdots \geqslant w_{j_p}$, we have $f(w) = w^\top t$. Thus, by minimizing $w^\top t + \frac{1}{2} \|w\|_2^2$ subject to this constraint, we improve on the choice $w=-s$. Note that this is exactly an isotonic regression problem with total order, which can be solved simply and efficiently in $O(p)$ by the ``pool adjacent violators'' algorithm~(see, e.g.,~\cite{best1990active} and Appendix~\ref{app:pava}). In \mysec{exp-prox}, we show that this leads to much improved approximate duality gaps. In \myfig{pava-interpretation}, we illustrate the following geometric interpretation: each permutation of the $p$ elements of $V$ defines a (typically non-unique) extreme point $t$ of $B(F)$, and the dual optimization problem of the isotonic regression problem (see Appendix~\ref{app:pava}) corresponds to the orthogonal projection of $0$ onto the set of $u\in \rb^b$ such thats $u(V) = F(V)$ and for all $k$,
$s( \{ j_1,\dots, j_{k} \}) \leqslant  F( \{ j_1,\dots, j_{k} \})$. This set is an outer approximation of $K$ that contains the tangent cone of $B(F)$ at $u$ (it may not be equal to it when several orderings lead to the same base $t$).

\begin{figure}
\begin{center}
 
  \includegraphics[scale=.37]{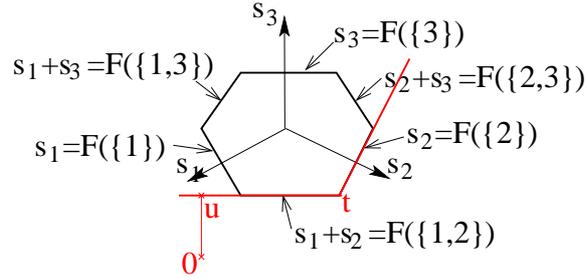}  
   
  \end{center}

 \vspace*{-.6cm}

\caption{Representation of $B(F)$ for a submodular function with $p=3$ (projected onto the set $s(V) = F(V)$), with the projection $u$ of $0$ onto the tangent cone at an extreme point $t$  corresponding to the ordering  (2,1,3).
}
\label{fig:pava-interpretation}
\end{figure}

\paragraph{Additional properties.}
Proximal problems with the square loss exhibit further interesting properties. For example, when considering problems of the form $\min_{w \in \rb^p} \lambda f(w) + \frac{1}{2} \| w - z\|_2^2$, for varying $\lambda$, some set-functions (such as the cut in the chain graph) leads to an agglomerative path, i.e., as $\lambda$ increases, components of the unique optimal solutions cluster together and never get separated~\cite{shapinglevelsets}. This is illustrated in \myfig{card} for cardinality-based functions.

Also, one may add an additional $\ell_1$-norm penalty to the regularized quadratic separable problem defined above, and it is shown in~\cite{shapinglevelsets} that, for any submodular function,  the solution of the optimization problem may be obtained by soft-thresholding the result of the original proximal problem; note that this is not true for other separable optimization problems (see also~\cite[Appendix B]{Mairal2010}).

\section{Separable problems on other polyhedra$^\ast$}
\label{sec:other}

We now show how to minimize a separable convex function on the submodular polyhedron $P(F)$,
the positive submodular polyhedron $P_+(F)$ and the symmetric submodular   polyhedron $|P|(F)$, given the minimizer on the base polyhedron $B(F)$.
We first show the following proposition for the submodular polyhedron of any submodular function (non necessarily non-decreasing), which relates the unrestricted proximal problem with the proximal problem restricted to $\rb_+^p$. Note that we state results with the same regularity assumptions than for Prop.~\ref{prop:prox}, but that these could be relaxed.

\begin{proposition}\textbf{(Separable optimization on the submodular polyhedron)}
\label{prop:sepsub}
With the same conditions than for Prop.~\ref{prop:prox}, let $(v,t)$ be a primal-dual optimal pair for the problem
\BEQ
 \label{eq:dd}
 \min_{v \in \rb^p} f(v) + \sum_{k \in V} \psi_k(v_k)
=  \max_{ t \in B(F)} - \sum_{k \in V} \psi_k^\ast(-t_k). 
\EEQ

For $k \in V$, let $s_k$ be a maximizer of $-\psi_k^\ast(-s_k)$ on $(-\infty,t_k]$. Define $w = v_+$.  Then $(w,s)$ is a primal-dual optimal pair for
the problem
\BEQ
\label{eq:proxdualP} \min_{w \in \rb^p_+} f(w) + \sum_{k \in V} \psi_k(w_k)
=  \max_{ s \in P(F)} - \sum_{k \in V} \psi_k^\ast(-s_k).
\EEQ
\end{proposition}
\begin{proof}
 The pair $(w,s)$ is optimal for \eq{proxdualP} if and only if (a) $ w_k s_k + \psi_k(w_k) + \psi_k^\ast(-s_k) = 0$, i.e., $(w_k,-s_k)$ is a Fenchel-dual pair for $\psi_k$, and (b) $f(w) = s^\top w$. 
 
For each $k \in V$, there are two possibilities, $s_k = t_k$ or  $s_k < t_k$.
The equality $s_k=t_k$  occurs when the function $s_k \mapsto -\psi_k^\ast(-s_k)$ has positive derivative at $t_k$, i.e., 
$(\psi_k^\ast)'(-t_k) \geqslant 0$. Since $v_k = (\psi_k^\ast)'(-t_k) $ by optimality for \eq{dd}, this occurs when $v_k \geqslant 0$, and thus $w_k =s_k$ and the pair $(w_k, -s_k)$ is Fenchel-dual.  The inequality $s_k < t_k$ occurs when $(\psi_k^\ast)'(-t_k) < 0$, i.e., $v_k<0$. In this situation, by optimality of $s_k$, $(\psi_k^\ast)'(-s_k)=0$, and thus the pair $(w_k,-s_k)$ is optimal. This shows that condition (a) is met.

 For the second condition (b),   notice that $s$ is obtained from $t$ by keeping the components of $t$ corresponding to strictly positive values of $v$ (let $K$ denote that subset), and lowering the ones for $V \backslash K$. For $\alpha > 0$, the level sets $\{ w \geqslant \alpha \}$ are equal to $\{ v \geqslant \alpha\} \subseteq K$. Thus, by Prop.~\ref{prop:optsupporttight},  all of these are tight  for $t$ (i.e., for these sets $A$, $t(A)=F(A)$)
 and hence for $s$ because these sets are included in $K$, and $s_K = t_K$. This shows, by  Prop.~\ref{prop:optsupporttightSUB}, that $s \in P(F)$ is optimal for $\max_{ s \in P(F)} w^\top s$.
  \end{proof}

We can apply Prop.~\ref{prop:sepsub} to perform the orthogonal projection onto $P(F)$:
for $z \in \rb^p$, and $\psi_k(w_k) = \frac{1}{2} ( w_k - z_k)^2$, then $\psi_k^\ast(s_k) = \frac{1}{2} s_k^2 + s_k z_k$, and the problem in \eq{proxdualP} is  indeed the orthogonal projection of $z$ onto $P(F)$. Given the optimal primal-dual pairs $(v,t)$ for \eq{dd}---i.e., $t$ is the orthogonal projection of $z$ onto $B(F)$, we get $w = v_+$ and $s = z - ( z - t)_+$, i.e., 
 $s_k =t_k$ if $z_k \geqslant t_k$, and $s_k = z_k$ otherwise.
See illustration in \myfig{projections}.

Note that  we can go from the solutions of separable problems on $B(F)$ to the ones on $P(F)$, but not vice-versa.
Moreover,  Prop.~\ref{prop:sepsub} involves primal-dual pairs $(w,s)$ and $(v,t)$, but that we can define $w$ from $v$ only, and define $s$ from $t$ only; thus,  primal-only views and dual-only views are possible.
 This also applies to Prop.~\ref{prop:sepsubsymm} and Prop.~\ref{prop:sepsubpos}, which extends  Prop.~\ref{prop:sepsub} to the symmetric and positive submodular polyhedra (we denote by $a \circ b$ the pointwise product between two vectors of same dimension).

\begin{figure}
\begin{center}
\includegraphics[scale=.5]{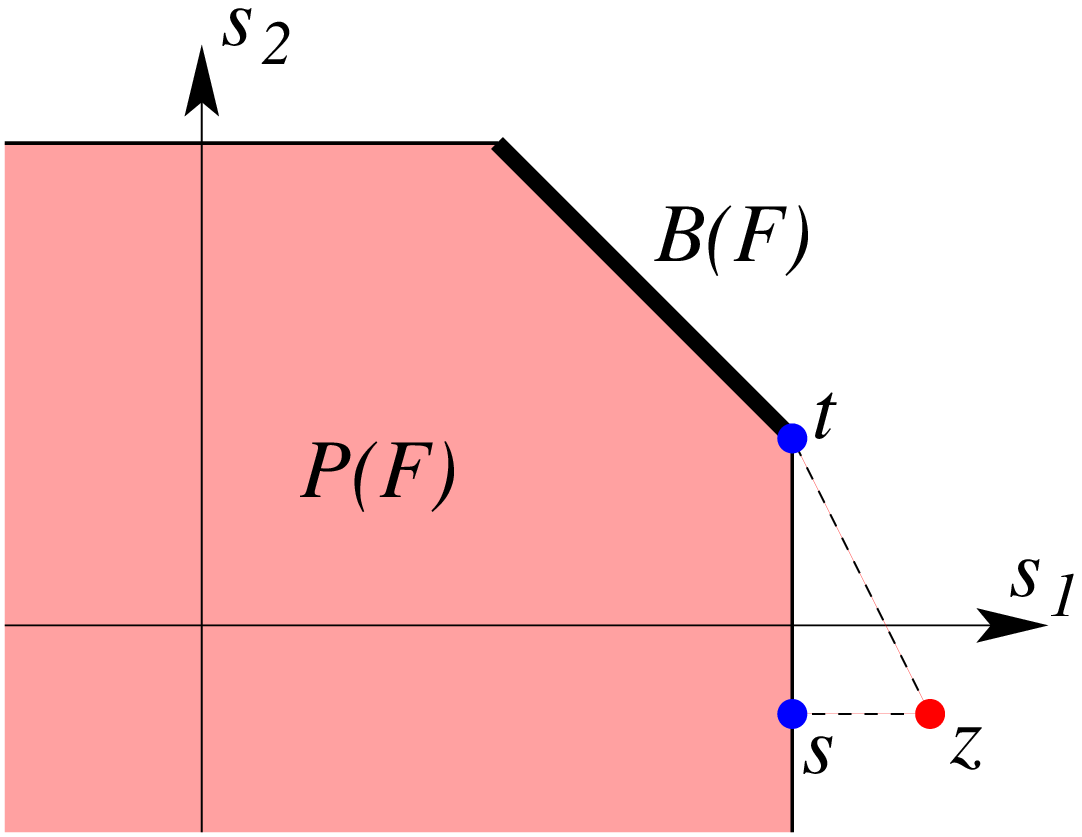} \hspace*{.2cm}
\includegraphics[scale=.5]{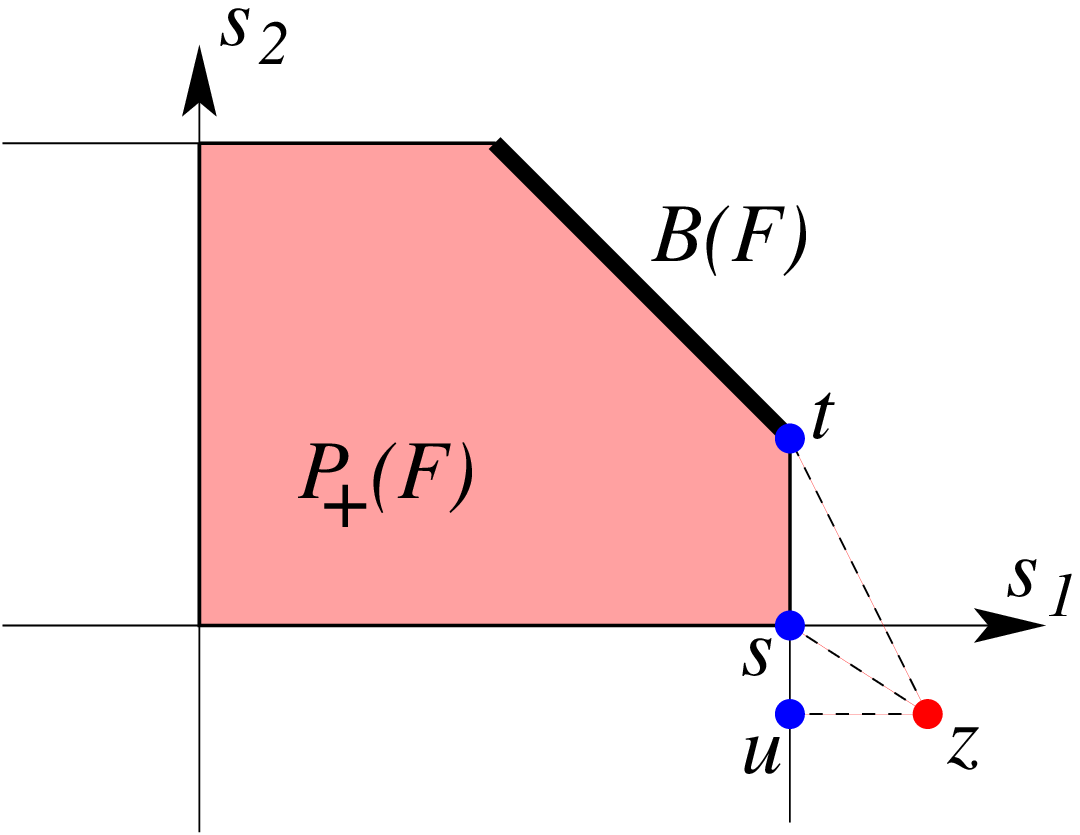}
\end{center}

\caption{From orthogonal projections onto $B(F)$ to orthogonal projections onto $P(F)$ and $P_+(F)$.
Left: $t$ is the projection of $z$ onto $B(F)$, from which we obtain the projection $s$ onto $P(F)$.
Right: $t$ is the projection of $z$ onto $B(F)$, from which we obtain the projection $u$ onto $P(F)$,
then the projection   $s$ onto $P_+(F)$.
}
\label{fig:projections}
\end{figure}

 \begin{proposition}\textbf{(Separable optimization on the positive submodular polyhedron)}
\label{prop:sepsubpos} 
Assume $F$ is submodular and non-decreasing.
With the same conditions than for Prop.~\ref{prop:prox}, let $(v,t)$ be a primal-dual optimal pair for the problem
\BEAS
\min_{v \in \rb^p} f(v) + \sum_{k \in V}  {\psi}_k(v_k)
=  \max_{ t \in B(F)} - \sum_{k \in V}  {\psi}_k^\ast(- t_k). \EEAS
Let $w_k $ be the minimizer of $\psi_k(w_k)$ on $(-\infty,(v_k)_+]$ and $s_k$ be  the positive part of the  maximizer of $-\psi_k^\ast(-  s_k)$ on $(-\infty,t_k]$.  Then $(w,s)$ is a primal-dual optimal pair for
the problem
\BEAS
\min_{w \in \rb^p} f(w_+) + \sum_{k \in V} \psi_k(w_k)
=  \max_{ s \in P_+(F)} - \sum_{k \in V} \psi_k^\ast(-s_k). \EEAS
\end{proposition}
\begin{proof}
Let $(a,u)$ be the primal-dual pair obtained from Prop.~\ref{prop:sepsub}, itself obtained from $(v,t)$. The pair $(w,s)$ is obtained from $(a,u)$, as $s_k = ( u_k)_+$ and $w_k$  be the minimizer of $\psi_k(w_k)$ on $(-\infty, a_k]$ for all $k \in V$. We use a similar argument than in the proof of Prop.~\ref{prop:sepsub}, to show that (a) the pair $(w_k,-s_k)$ is a Fenchel dual pair for $\psi_k$ and (b) $s \in P_+(F)$ maximizes $w^\top s$.

If $u_k \geqslant 0$, then we have $s_k = (u_k)_+ = u_k$, moreover this means that $ \psi_k'(a_k) = - u_k \leqslant 0$ and thus $w_k = a_k$ (as the minimizer defining $w_k$ is attained on the boundary of the interval). This implies from Prop.~\ref{prop:sepsub} that $(w,-s) = (a,-u)$ is a Fenchel-dual pair.
If $u_k < 0$, then $s_k = (u_k)_+ =   0$, moreover, since $\psi_k'(a_k) = - u_k >0$, then the optimization problem defining $w_k$ has a solution away from the boundary, which implies that $\psi'_k(w_k)=0$, and thus the pair $(w,-s)$ is optimal for $\psi_k$. This implies condition (a). 

In order to show condition (b), from Prop.~\ref{prop:optsupporttight-positive}, we simply need to show that all strictly positive suplevel-sets of $w$ are tight for $F$ and that $w_k<0$ implies $s_k=0$. Since $w_k<0$ can only occur when $u_k<0$ (and then $s_k=0$), we only need to check the first property.
We now need to consider two other cases: (1) if $v_k<0$, then $u_k = (v_k)_+ = 0$ and $w_k \leqslant u_k = 0$, and thus this does not contribute to the positive level sets of $w$. (2) If $v_k \geqslant 0$,  then $a_k = v_k$ and $u_k = t_k$; moreover, since $B(F) \subset \rb^p_+$, we must have $t_k \geqslant 0$,  which implies $s_k = u_k = t_k$ and $w_k = a_k  = v_k$. Thus the positive suplevel sets of $w$ are the ones of $v$ and for these indices $k$, $s_k = t_k$. Thus the positive suplevel sets of $w$ are tight for $s$ from the optimality of $t \in B(F)$.
\end{proof}
For quadratic functions $\psi_k(w_k) = \frac{1}{2} ( w_k - z_k)^2$, the previous proposition can be seen as projecting on $P(F)$, then projecting on the positive orthant $\rb_+^p$. Doing the projections in the other order would here lead to the same result. See illustration in \myfig{projections}.

 \begin{proposition}\textbf{(Separable optimization on the symmetric submodular polyhedron)}
\label{prop:sepsubsymm}
Assume $F$ is submodular and non-decreasing.
With the same conditions than for Prop.~\ref{prop:prox}, let $\varepsilon_k \in \{-1,1\}$ denote the sign of $(\psi_k^\ast)'(0) = \arg\max_{s_k \in \rb} \psi_k^\ast(s_k)$ (if it is equal to zero, then the sign can be either  $-1$ or $1$). Let $(v,t)$ be a primal-dual optimal pair for the problem
\BEAS
\min_{v \in \rb^p} f(v) + \sum_{k \in V}  {\psi}_k(\varepsilon_k v_k)
=  \max_{ t \in B(F)} - \sum_{k \in V}  {\psi}_k^\ast(- \varepsilon_k t_k). \EEAS
Let $w = \varepsilon \circ ( v_+) $ and $s_k$ be $\varepsilon_k$ times a maximizer of $-\psi_k^\ast(-\varepsilon_k s_k)$ on $(-\infty,t_k]$.  Then $(w,s)$ is a primal-dual optimal pair for
the problem
\BEAS
\min_{w \in \rb^p} f(|w|) + \sum_{k \in V} \psi_k(w_k)
=  \max_{ s \in |P|(F)} - \sum_{k \in V} \psi_k^\ast(-s_k). \EEAS
\end{proposition}
\begin{proof} 
Without loss of generality we may assume that $\varepsilon \geqslant 0$, by the change of variables $w_k \to \varepsilon_k w_k$; note that since $\varepsilon_k \in \{-1,1\}$, the Fenchel conjugate of $w_k \mapsto \psi_k(\varepsilon_k w_k)$ is $s_k \mapsto \psi_k^\ast (\varepsilon_k s_k)$.

Because $f$ is non-decreasing with respect to each of its component, the global minimizer of $f(|w|) + \sum_{k \in V} \psi_k(w_k)$ must have non-negative components (indeed, if one them has a strictly negative component $w_k$, then with respect to the $k$-th variable, around $w_k$, $f(|w|)$ is non-increasing and $\psi_k(w_k)$ is strictly decreasing, which implies that $w$ cannot be optimal, which leads to a contradiction).

We may then   apply Prop.~\ref{prop:sepsub} to $w_k \mapsto \psi_k(\varepsilon_k w_k)$, which has Fenchel conjugate $s_k \mapsto \psi_k^\ast(\varepsilon_k s_k)$ (because $\varepsilon_k^2=1$), to get the desired result.
 \end{proof}

\paragraph{Applications to sparsity-inducing norms.}
Prop.~\ref{prop:sepsubpos}  is particularly adapted to sparsity-inducing norms defined in \mysec{sparse}, as it describes how to solve the proximal problem for the norm $\Omega_\infty(w) = f(|w|)$. For a quadratic function, i.e., $\psi_k(w_k) = \frac{1}{2} (w_k - z_k)^2$ and
$\psi_k^\ast(s_k) = \frac{1}{2} s_k^2 + s_k z_k$. Then $\varepsilon_k$ is the sign of $z_k$, and we thus have to minimize
$$
\min_{v \in \rb^p} f(v) + \frac{1}{2}\sum_{k \in V} ( v_k - |z_k|)^2,
$$
which is the classical quadratic separable problem on the base polyhedron,
and select $w = \varepsilon \circ v_+$. Thus, proximal operators for the norm $\Omega_\infty$ may be obtained from the proximal operator for the \lova extension. See \mysec{extensions} for the proximal operator for the norms $\Omega_q$, $q \in (1,+\infty)$.

\chapter{Separable Optimization Problems: Algorithms}
\label{chap:optim-polyhedra}
\label{chap:prox-algo}

 In the previous chapter, we have analyzed a series of optimization problems which may be defined as the minimization of a separable function on the base polyhedron. In this chapter, we consider two main types of algorithms to solve these problems.
 The algorithm we present in \mysec{decomp} is a divide-and-conquer exact method that will recursively solve the separable optimization problems by defining smaller problems. This algorithm requires to be able to solve submodular function minimization problems of the form $\min_{A
\subseteq V} F(A) - t(A)$, where $t \in \rb^p$, and is thus applicable only when such algorithms are available (such as in the case of cuts, flows or cardinality-based functions). 

The next two sets of algorithms are iterative methods for convex optimization on convex sets for which the support function can be computed, and are often referred to as ``Frank-Wolfe'' algorithms.    This only assumes the availability of an efficient algorithm for maximizing linear  functions on the base polyhedron (greedy algorithm from Prop.~\ref{prop:greedy}).  
The min-norm-point algorithm that we present in \mysec{minnorm-prox} is an active-set algorithm dedicated to quadratic functions and converges after finitely many operations (but with no complexity bounds), while the conditional gradient algorithms that we consider in \mysec{approx-prox} do not exhibit finite convergence but have known convergence rates.
Finally, in \mysec{extensions}. we consider extensions of proximal problems, normally line-search in the base polyhedron and the proximal problem for the norms $\Omega_q$, $q \in (1,+\infty)$ defined in \mysec{lprelax}.

\section{Divide-and-conquer algorithm for proximal problems}
\label{sec:decomp}
\label{sec:decomposition}
 
 We now consider an algorithm for proximal problems, which is based on a sequence of submodular function minimizations. It is based on a divide-and-conquer strategy. We adapt the algorithm of~\cite{groenevelt1991two} and the algorithm presented here is the dual version of the one presented in~\cite[Sec.~8.2]{fujishige2005submodular}. Also, as shown at the end of the section, it can be slightly modified for problems with non-decreasing submodular functions~\cite{groenevelt1991two} (otherwise, Prop.~\ref{prop:sepsubpos} and
 Prop.~\ref{prop:sepsubsymm} may be used).

For simplicity, we consider \emph{strictly convex differentiable} functions $\psi_j^\ast$, $j=1,\dots,p$,  defined on $\rb$, (so that the minimum in $s$ is unique) and the following recursive algorithm:

\begin{list}{\labelitemi}{\leftmargin=1.1em}
   \addtolength{\itemsep}{-.2\baselineskip}

\item[(1)]  Find the unique minimizer $t \in \rb^p$ of  $\sum_{j\in V}  \psi_j^\ast(-t_j)$ such that $t(V) = F(V)$.
\item[(2)] Minimize the submodular function $F-t$, i.e., find  a set  $A \subseteq V$  that minimizes $F(A)-t(A)$.
\item[(3)] If $F(A) = t(A)$, then $t$ is optimal. Exit.
\item[(4)] Find a minimizer $s_A$ of $\sum_{j \in A}  \psi_j^\ast(-s_j)$ over $s$ in the base polyhedron associated to $F_A$,
the restriction of $F$ to  $A$.
\item[(5)] Find the unique minimizer $s_{V \backslash A }$ of $\sum_{j \in V \backslash A}  \psi_j^\ast(-s_j)$ over $s$ in the base polyhedron associated to  the contraction $F^A$ of $F$ on A, defined as $F^A(B) = F(A \cup B) - F(A)$, for $B \subseteq V \backslash A$.
\item[(6)] Concatenate $s_A$ and  $s_{V \backslash A }$. Exit.
\end{list}

The algorithm must stop after \emph{at most} $p$ iterations. Indeed, if $F(A) \neq t(A)$ in step (3), then we must have $A \neq \varnothing$  and $A \neq V$ since by construction $t(V) = F(V)$. 
 Thus we actually split $V$ into two non-trivial parts $A$ and $V \backslash A$. Step (1) is a separable optimization problem with one linear constraint. When $\psi^\ast_j$ is a quadratic polynomial, it may be obtained in closed form; more precisely, one may minimize $\frac{1}{2}\|t - z \|_2^2$ subject to $t(V) = F(V)$ by taking $t = \frac{F(V)}{p} 1_V + z - \frac{1_V 1_V^\top}{p} z$.

\paragraph{Geometric interpretation.}
The divide-and-conquer algorithm has a simple geometric interpretation. In step (1), we minimize our function on a larger subset than $B(F)$, i.e., the affine hyperplane $\{s \in \rb^p, s(V) = F(V) \}$. In step (2), we check if the obtained projection $t \in \rb^p$ is in $B(F)$ by minimizing the set-function $F - t$. To decide if $t \in B(F)$ or not, only the minimal value is necessary in step (3). However, the minimizer $A$ turns out to provide additional information by reducing the problems to two decoupled subproblems---steps (4) and (5).

The algorithm may also be interpreted in the primal, i.e., for minimizing $f(w) + \sum_{j=1}^p \psi_j(w_j)$. The information given by $A$ simply allows to reduce the search space to all $w$ such that 
$\min_{k \in A } w_k \geqslant \min_{k \in V \backslash A} w_k$, so that $f(w)$ decouples into the sum of a \lova extension of the restriction to $A$ and the one of the contraction to $A$.

\paragraph{Proof of correctness.} Let $s$ be the output of the  recursive algorithm. If the algorithm stops at step (3), then we indeed have an optimal solution. Otherwise, we first show that $s \in B(F)$. We have for any $B \subseteq V$:
\BEAS
s(B) & = & s(B \cap A) + s( B \cap ( V \backslash A) ) \\
&  \leqslant &   F(B \cap A) + F( A \cup B) - F(A) \mbox{ by definition of } s_A \mbox{ and } s_{V \backslash A }\\
& \leqslant & F(B) \mbox{ by submodularity}.
\EEAS
Thus $s$ is indeed in the submodular polyhedron $P(F)$. Moreover, we have
$s(V) = s_A(A) + s_{V \backslash A }(V \backslash A ) = F(A) + F(V) - F(A) = F(V)$, i.e., $s$ is in the base polyhedron $B(F)$. 
 
 Our proof technique now relies on using the equivalence between separable problems and a sequence of submodular function minimizations, shown in Prop.~\ref{prop:subfromprox}. Let $w_j$ be the Fenchel-dual to $-t_j$ for the convex function $\psi^\ast_j$, we have $w_j = (\psi_j^\ast)'(-t_j)$. Since $t$ is obtained by minimizing $\sum_{j\in V}  \psi_j^\ast(-t_j)$ such that $t(V) = F(V)$, by introducing a Lagrange multiplier for the constraint $s(V) = F(V)$, we obtain that $w$ is proportional to $1_V$ (i.e., has uniform components). Let $\alpha \in \rb$ be the common value of the components of $w_j$. We have $t_j = - \psi_j'(\alpha)$ for all $j \in V$. And thus $A$ is a minimizer of $F + \psi'(\alpha)$. This implies from Prop.~\ref{prop:subfromprox} that the minimizer $w$ of the proximal problem is such that 
 $\{ w> \alpha \} \subseteq A \subseteq    \{ w \geqslant\alpha \}$. This implies that we may look for $w$ such that
 $\min_{k \in A } w_k \geqslant \min_{k \in V \backslash A} w_k$. For such a $w$,  $f(w)$ decouples into the sum of a \lova extension of the restriction to $A$ and the one of the contraction to $A$. Since the rest of the cost function is separable, the problem becomes separable and the recursion is indeed correct.

 Note finally that similar algorithms may be applied when we restrict~$s$ to have integer values (see, e.g.,~\cite{groenevelt1991two, hochbaum2001efficient}).

 \paragraph{Minimization separable problems in other polyhedra.}
 In this chapter, we have considered the minimization of separable functions on the base polyhedron $B(F)$. In order to minimize over the submodular polyhedron $P(F)$, we may use Prop.~\ref{prop:sepsub} that shows how to obtain the solution in $P(F)$ from the solution on $B(F)$. Similarly, for a non-decreasing submodular function, when minimizing with respect to the symmetric submodular polyhedron $|P|(F)$ or $P_+(F)$, we may use
 Prop.~\ref{prop:sepsubpos} or  Prop.~\ref{prop:sepsubsymm}. Alternatively, we may use a slightly different algorithm that is dedicated
 to these situations. For $P_+(F)$, this is exactly the algorithm of~\cite{groenevelt1991two}. 
 
 The only changes are in the first step. For minimizing with respect to $s \in P_+(F)$, the vector $t$ is defined as the minimizer of $\sum_{j\in V}  \psi_j^\ast(-t_j)$ such that $t(V) \leqslant F(V)$ and $t \geqslant 0$, while for minimizing with respect to $s \in |P|(F)$, the vector $t$ is defined as the minimizer of $\sum_{j\in V}  \psi_j^\ast(-t_j)$ such that $|t|(V) \leqslant F(V)$. These first steps (projection onto the simplex or the $\ell_1$-ball) may be done in $O(p)$ (see, e.g.,~\cite{brucker1984,maculan1989linear}).
 In our experiments in \mysec{exp-prox}, the decomposition algorithm directly on $|P|(F)$ is slightly faster.

   \paragraph{Making splits more balanced.}
   In the divide-and-conquer algorithm described above, at every step, the problem in dimension $p$ is divided into two problems with dimensions $p_1$ and $p_2$ summing to $p$.  Unfortunately, in practice, the splitting may be rather unbalanced, and the total complexity of the algorithm may then be $O(p)$ times the complexity of a single submodular function minimization (instead of $O(\log p)$ for binary splits). Following the algorithm of~\cite{tarjan2006balancing} which applies to cut problems, an algorithm is described in~\cite{treesubmod} which reaches a $\varepsilon$-approximate solution by using a slightly different splitting strategy, with an overall complexity which is only $\log(1/\varepsilon)$ times the complexity of a single submodular function minimization problem.

\section{Iterative algorithms - Exact minimization}
\label{sec:minnorm-prox}
\label{sec:mnp}
In this section, we focus on quadratic separable problems. Note that modifying the submodular function by adding a modular term\footnote{Indeed, we have $\frac{1}{2} \| w - z\|_2^2 + f(w) = \frac{1}{2} \|w\|_2^2
+ (f(w) - w^\top z) + \frac{1}{2} \|z\|^2$, which corresponds (up to the irrelevant constant term $\frac{1}{2} \| z\|_2^2$) to the proximal problem for the \lova extension of $A \mapsto F(A) - z(A)$.}, we can consider $\psi_k = \frac{1}{2} w_k^2$.
As shown in Prop.~\ref{prop:prox},  minimizing $f(w) + \frac{1}{2} \| w\|_2^2$ is equivalent to minimizing $\frac{1}{2} \| s\|_2^2$ such that $s \in B(F)$.

Thus, we can minimize $f(w) + \frac{1}{2} \| w\|_2^2$ by computing the minimum $\ell_2$-norm element of the polytope $B(F)$, or equivalently the orthogonal projection of $0$ onto $B(F)$. Although $B(F)$ may have exponentially many extreme points, the greedy algorithm of Prop.~\ref{prop:greedy} allows to maximize a linear function over $B(F)$ at the cost of $p$ function evaluations. The minimum-norm point algorithm of~\cite{wolfe1976finding} is dedicated to such a situation, as  outlined by~\cite{fujishige2006minimum}. It turns out that the minimum-norm point algorithm can be interpreted as a standard active-set algorithm for quadratic programming (\mysec{QP}), which we now describe.

\paragraph{Frank Wolfe algorithm as an active-set algorithm.} We consider $m$ points $x_1,\dots,x_m$ in $\rb^p$ and the following optimization problem:
$$
\min_{ \eta \in \rb_+} \frac{1}{2} \Big\| 
\sum_{i=1}^m \eta_i x_i
\Big\|_2^2 \mbox{ such that } \eta \geqslant 0, \ \eta^\top 1_V = 1.
$$
In our situation, the vectors $x_i$ will be the extreme points of $B(F)$, i.e., outputs of the greedy algorithm, but they will always be used \emph{implicitly} through the maximization of linear functions over $B(F)$. We will apply the primal active set strategy outlined in Section 16.4 of~\cite{Nocedal:1999:NO1} and in \mysec{QP}, which is exactly the algorithm of \cite{wolfe1976finding}. The active set strategy hinges on the fact that if the set of indices $j \in J$ for which $\eta_j > 0$ is known, the solution $\eta_J$ may be obtained in closed form by computing the affine projection on the set of points indexed by $I$ (which can be implemented by solving a positive definite linear system, see step 2 in the algorithm below). Two cases occur: (a) If the affine projection happens to have non-negative components,
i.e., $\eta_J \geqslant 0$ (step (3)), then we obtain in fact the projection onto the convex hull of the points indexed by $J$, and we simply need to check optimality conditions and make sure that no other point needs to enter the hull (step 5), and potentially add it to go back to step (2). (b) If the projection is not in the convex hull, then we make a move towards this point until we exit the convex hull (step (4)) and start again at step (2). We describe in Figure~\ref{fig:FW} an example of several iterations.

\begin{list}{\labelitemi}{\leftmargin=1.1em}
   \addtolength{\itemsep}{-.3\baselineskip}

\item[(1)] \textbf{Initialization}: We start from a feasible point $\eta \in \rb_+^p$ such that $\eta^\top 1_V = 1$, and denote $J$ the set of indices such that $\eta_j>0$ (more precisely a subset of $J$ such that the set of vectors indexed by the subset is linearly independent). Typically, we select one of the original points, and $J$ is a singleton.

\item[(2)] \textbf{Projection onto affine hull}: Compute $\zeta_{J}$ the unique minimizer $\frac{1}{2} \big\| \sum_{ j \in J } \eta_j x_j \big\|_2^2$
such that $1_J^\top \eta_{J} = 1$, i.e., the orthogonal projection of $0$ onto the \emph{affine} hull of the points $(x_i)_{i \in J}$.
\item[(3)] \textbf{Test membership in convex hull}: If $\zeta_{J} \geqslant 0$ (we in fact have an element of the convex hull), go to step (5).
\item[(4)]  \textbf{Line search}: Let $\alpha \in [0,1)$ be the largest $\alpha$ such that $\eta_J + \alpha ( \zeta_J - \eta_J) \geqslant 0$. Let $K$ the sets of $j$  such that $\eta_j + \alpha ( \zeta_j - \eta_j) = 0$. Replace $J$ by $J \backslash K$ and $\eta$ by $\eta + \alpha ( \zeta - \eta)$, and go to step (2).
\item[(5)]  \textbf{Check optimality}: Let $y  = \sum_{j \in J} \eta_j x_j$. Compute a minimizer $i$ of $y^\top x_i$. If
$y^\top x_i = y^\top \eta$, then $\eta$ is optimal. Otherwise, replace $J$ by $J \cup \{i \}$, and go to step (2).

\end{list}

\begin{figure}

\begin{center}
\includegraphics[scale=.64]{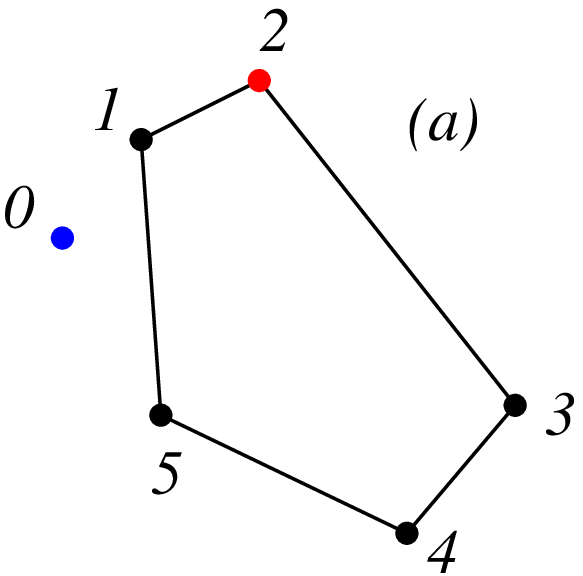}
\includegraphics[scale=.64]{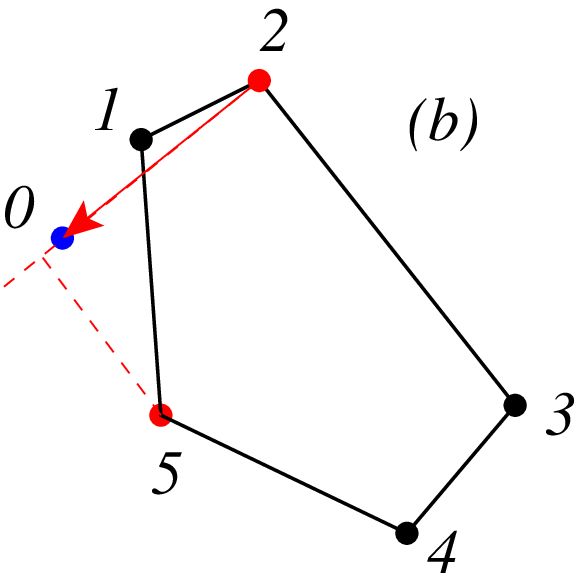}
\includegraphics[scale=.64]{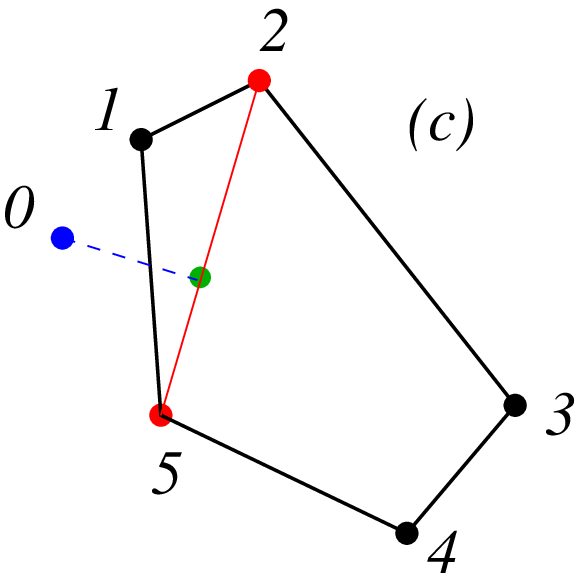}

\includegraphics[scale=.64]{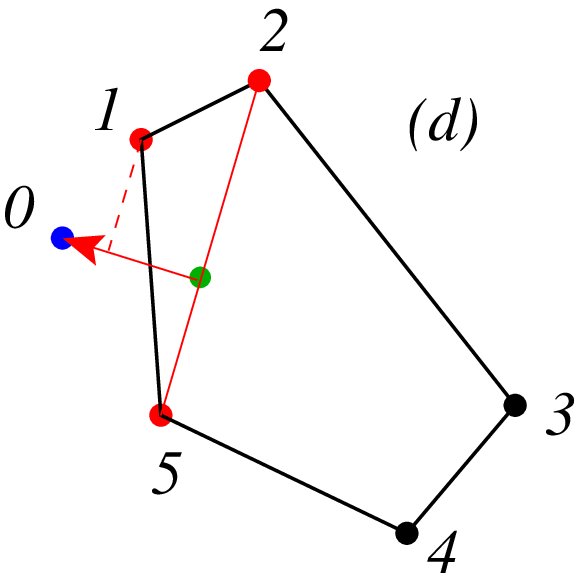}
\includegraphics[scale=.64]{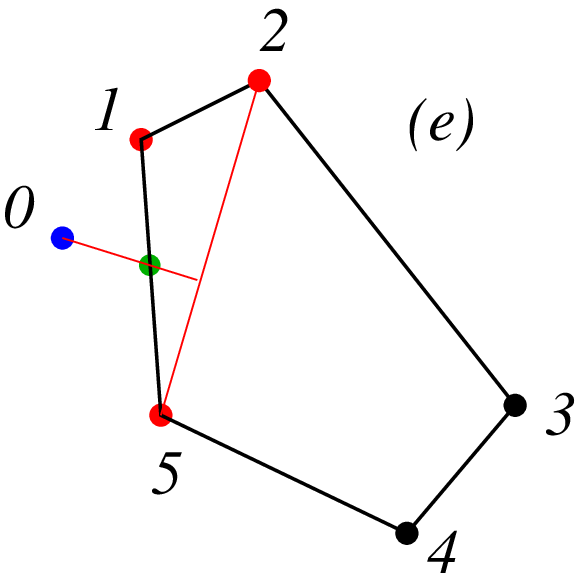}
\includegraphics[scale=.64]{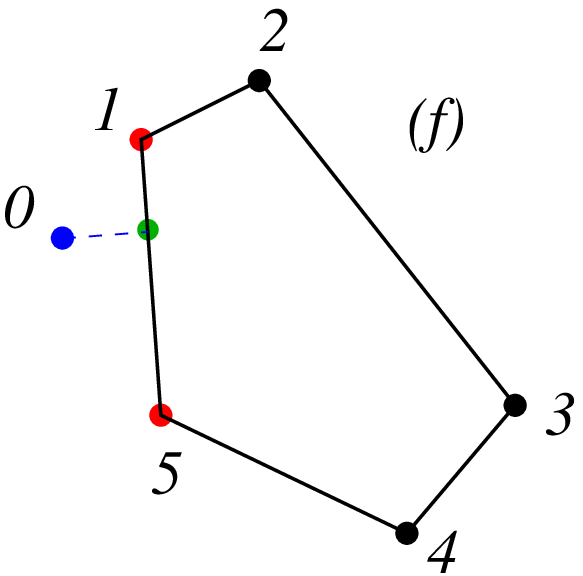}
\end{center}

\caption{Illustration of Frank-Wolfe minimum norm point algorithm: (a) initialization with $J=\{2\}$ (step (1)), (b) check optimality (step (5)) and take $J=\{2,5\}$, (c) compute affine projection (step (2)), (d) check optimality and take $J =\{1,2,5\}$, (e) perform line search (step (3)) and take $J = \{1,5\}$, (f) compute affine projection (step (2)) and obtain optimal solution.}
\label{fig:FW}
\end{figure}

The previous algorithm terminates in a finite number of iterations because it strictly decreases the quadratic cost function at each iteration; 
however, there is no known bounds regarding the number of iterations (see more details in~\cite{Nocedal:1999:NO1}). Note that in pratice, the algorithm is stopped after either (a) a certain duality gap has been achieved---given the candidate $\eta$, the duality gap for $\eta$ is equal to $ \|\bar{x}  \|_2^2
+ \max_{ i \in \{1,\dots,m\}} \bar{x}_i$, where  $ \bar{x} = \sum_{i=1}^m \eta_i x_i$ (in the context of application to orthogonal projection on $B(F)$, following \mysec{quadprox}, one may get an improved duality gap by solving an isotonic regression problem); or (b), the affine projection cannot be performed reliably because of bad condition number (for more details regarding stopping criteria, see~\cite{wolfe1976finding}).

 \paragraph{Application to $|P|(F)$ or $P_+(F)$.}
 When projecting onto the symmetric submodular polyhedron, one may either use the algorithm defined above which projects onto $B(F)$ and use Prop.~\ref{prop:sepsubsymm}. It is also possible to
 apply the min-norm-point algorithm directly to this problem, since we can also maximize linear functions on $|P|(F)$ or $P_+(F)$ efficiently, by the greedy algorithm presented in Prop.~\ref{prop:greedy-indep}. In our experiments in \mysec{exp-wavelet}, we show that the number of iterations required for the minimum-norm-point algorithm applied directly to $|P|(F)$ is lower; however, in practice, warm restart strategies, that start the min-norm-point algorithm from a set of already computed extreme points, are not as effective.

\section{Iterative algorithms - Approximate minimization}
\label{sec:approx-prox}
In this section, we describe an algorithm strongly related to the minimum-norm point algorithm presented in \mysec{minnorm-prox}. As shown in \mysec{condgrad}, this ``conditional gradient'' algorithm is dedicated to minimization of any convex smooth functions on the base polyhedron. Following the same argument than for the proof of Prop.~\ref{prop:prox}, this is equivalent to  the minimization of any   strictly convex separable function regularized by the \lova extension. As opposed to the mininum-norm point algorithm, it is not convergent in finitely many iterations; however, as explained in~\mysec{condgrad},  it comes with approximation guarantees.

\paragraph{Algorithm.}
If $g$ is a smooth convex function defined on $\rb^p$ with Lipschitz-continuous gradient (with constant $L$), then the conditional gradient algorithm is an iterative algorithm that will (a) start from a certain $s_0 \in B(F)$, and (b) iterate the following procedure for $t \geqslant 1$: find a minimizer $\bar{s}_{t-1}$ over the (compact) polytope $B(F)$ of the Taylor expansion of $g$ around $s_{t-1}$, i.e, $s \mapsto g(s_{t-1}) + g'(s_{t-1})^\top ( s- s_{t-1})$, and perform a step towards $\bar{s}_{t-1}$, i.e., compute
$s_{t} = \rho_{t-1} \bar{s}_{t-1} + (1- \rho_{t-1}) s_{t-1}$. 

There are several strategies for computing $\rho_{t-1}$. The first is to take $\rho_{t-1}=2/(t+1)$~\cite{dunn1978conditional,jaggi}, while the second one is to perform a line search on the quadratic upper-bound on $g$ obtained from the $L$-Lipschitz continuity of $g$ (see \mysec{condgrad} for details).
They both exhibit the same upper bound on the sub-optimality of the iterate $s_t$, together with $g'(w_{t})$ playing the role of a certificate of optimality. More precisely,  the base polyhedron is included in the hyper-rectangle $\prod_{k \in V} [ F(V) - F( V \backslash \{k\}), F(\{k\})]$ (as a consequence of the greedy algorithm applied to $1_{\{k\} }$ and $-1_{\{k\}}$). We denote by $\alpha_k$ the length of the interval for variable $k$, i.e., $\alpha_k = F(\{k\}) +  F( V \backslash \{k\}) - F(V)$.
 Using results from~\mysec{condgrad}, we have for the two methods:
 $$
g(s_t) -
 \min_{s \in B(F)} g(s) \leqslant \frac{ L\sum_{k=1}^p \alpha_k^2 }{t+1},$$
and the \emph{computable} quantity $\max_{s \in B(F)} g'(s_t)^\top ( s - s_t)$ provides a certificate of optimality, that is, we always have that $g(s_t) -
 \min_{s \in B(F)} g(s) \leqslant \max_{s \in B(F)} g'(s_t)^\top ( s - s_t)$, and the latter quantity has (up to constants) the same convergence rate. Note that while this certificate comes with an offline approximation guarantee, it can be significantly improved, following \mysec{quadprox}, by  solving an appropriate isotonic regression problem (see simulations in \mychap{experiments}).
 
 In Figure~\ref{fig:FW2}, we consider the conditional gradient algorithm (with line search) for the quadratic problem considered in \mysec{minnorm-prox}. These two algorithms are very similar as they both consider a sequence of extreme points of $B(F)$ obtained from the greedy algorithm, but they differ in the following way: the min-norm-point algorithm is finitely convergent but with no convergence rate, while the conditional gradient algorithm is not finitely convergent, but with a convergence rate. Moreover, the cost per iteration for the min-norm-point algorithm is much higher as it requires linear system inversions. In context where the function $F$ is cheap to evaluate, this may become a  computational bottleneck; however, in our simulations in \mychap{experiments}, we have focused on situations where the bottleneck is evaluation of functions (i.e., we compare algorithms using number of function calls or number of applications of the greedy algorithm).

 \begin{figure}

\begin{center}
\includegraphics[scale=.64]{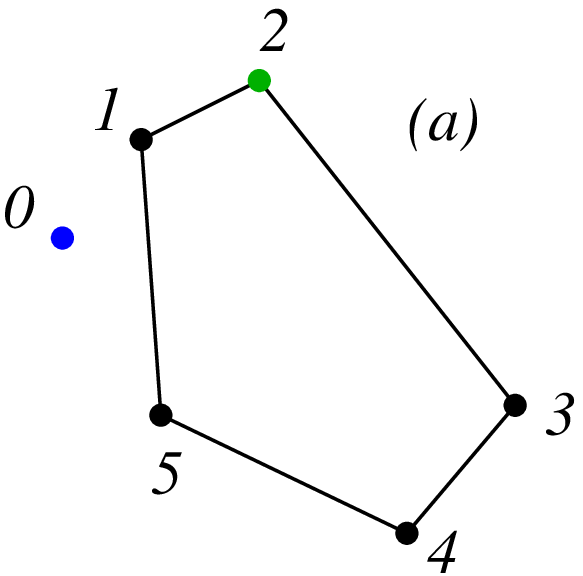}
\includegraphics[scale=.64]{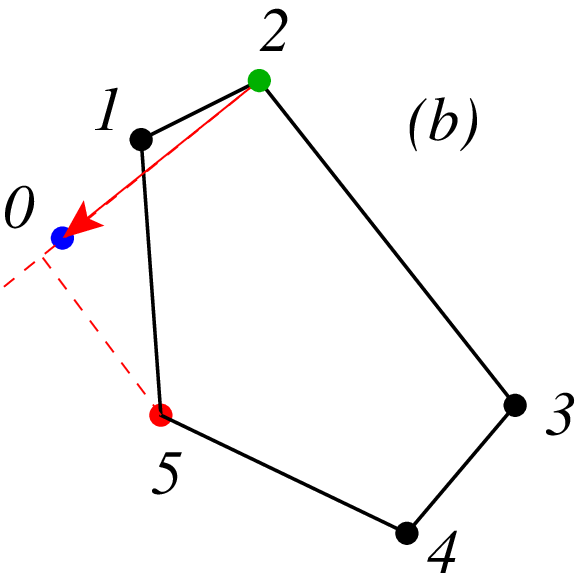}
\includegraphics[scale=.64]{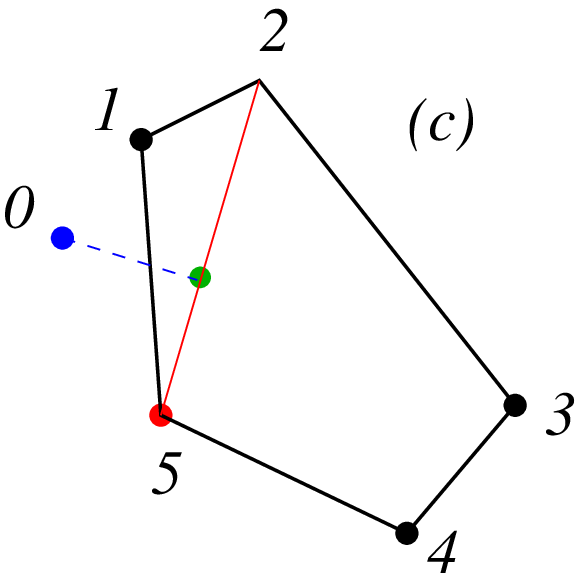}

\includegraphics[scale=.64]{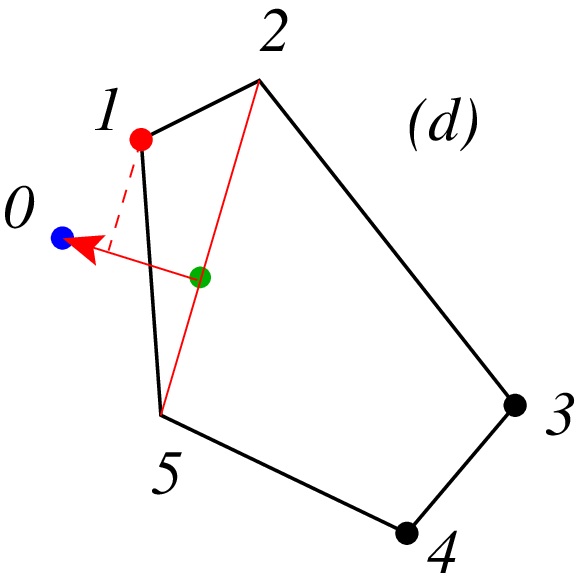}
\includegraphics[scale=.64]{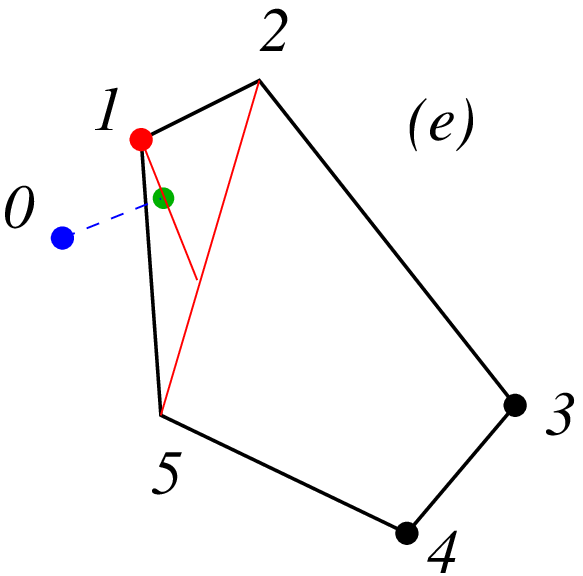}
\includegraphics[scale=.64]{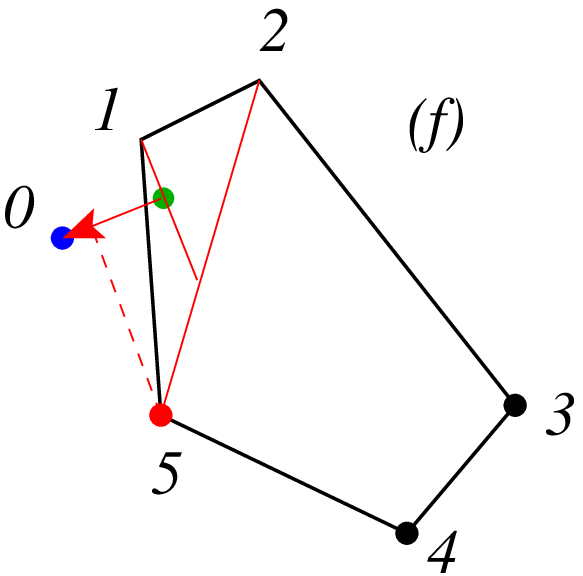}

\includegraphics[scale=.64]{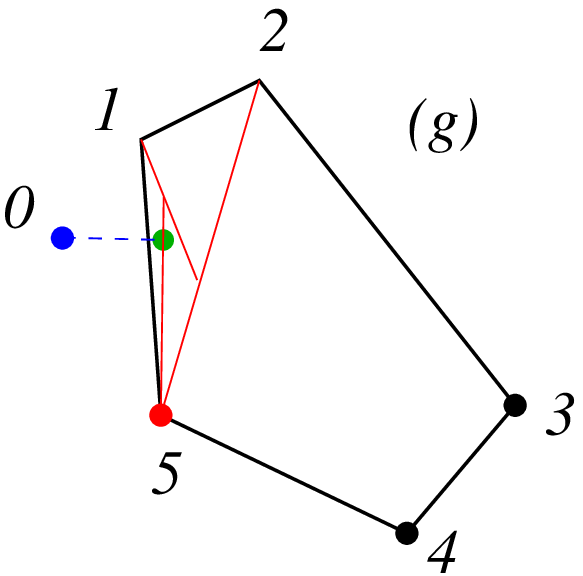}
\includegraphics[scale=.64]{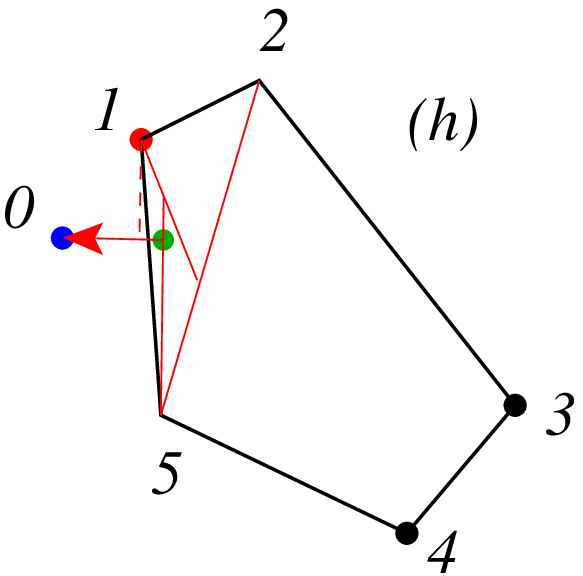}
\includegraphics[scale=.64]{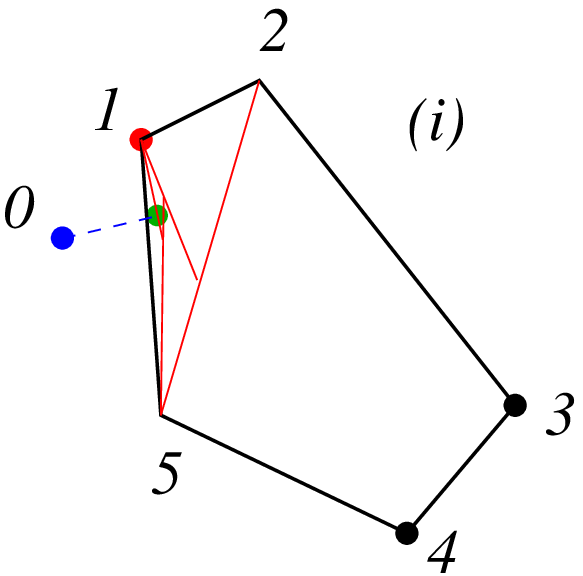}

\end{center}

\caption{Illustration of Frank-Wolfe conditional gradient algorithm: starting from the initialization (a), in steps (b),(d),(f),(h), an extreme point on the polytope is found an in steps (c),(e),(g),(i), a line search is performed. Note the oscillations to converge to the optimal point (especially compared to Figure~\ref{fig:FW}).
}
\label{fig:FW2}
\end{figure}

 \section{Extensions}
 \label{sec:extensions}
 
 In this section, we consider extensions of the algorithms presented above, with application to the computation of dual norms and proximal operators for the norms $\Omega_q$ presented in \mysec{l2}. These are key to providing either efficient algorithms or efficient ways of providing approximate optimality certificates (see more details in \cite{fot}).
 
  \paragraph{Line search in submodular polyhedron and computation of dual norms.}
 Given an element $s \in P(F)$ and a non-zero vector $t \in \rb^p$, the line search problem is the one of finding the largest $\alpha \in \rb_+$ such that
 $s + \alpha t \in P(F)$. 
 We are thus looking for the maximum $\alpha \in \rb_+$ such that for all $A \subseteq V$, $s(A) + \alpha t(A) \leqslant F(A)$, i.e., if we define the  set-function $G: A \mapsto F(A) - s(A)$,
 $
 \alpha t(A) \leqslant G(A) .
 $
 The function $G$ is submodular and non-negative. In order to test if $\alpha$ is allowed, we need to minimize the submodular function $G - \alpha t$. If we denote $g(\alpha) = \min_{A \subseteq V} G(A) - \alpha t(A)$, then the function $g$ is concave, non-negative, such that $g(0)=0$, piecewise affine and our goal is to find the largest $\alpha \geqslant 0 $ such that $g(\alpha)=0$.
 
 Since $g$ is piecewise affine with at most $2^p$ different pieces, a natural algorithm is to use Newton method to find $\alpha$. Indeed, as outlined by~\cite{nagano2007strongly}, from any $\beta$ such that $g(\beta)>0$, we consider any minimizer $A$ of $G(A) - \beta t(A)$. Since $g(\beta)<0$ and $G \geqslant 0$, we must have $t(A)>0$ and
 we may define   $ \gamma = G(A) / t(A) \in \rb_+ $. Moreover, $\alpha \leqslant \gamma < \beta$ and $\gamma$ is on the same affine part of $g$ if and only if $\gamma = \alpha$; otherwise, it belongs to another affine part. Thus, after at most $2^p$ of such steps, we must obtain the optimal $\alpha$.
 
 The number of iterations is typically much smaller than $2^p$ (see sufficient conditions in~\cite{nagano2007strongly}). Moreover, if all components of $t$ are strictly positive, then the number of iterations is in fact less than $p$. In this situation, $\alpha = \max_{ A \subseteq V} \frac{t(A)}{G(A)}$ and thus the line search problem allows computation of the dual norms defined in \mychap{norms}, as already done by~\cite{mairal2011b} for the special case of the flow-based norms described in \mysec{flows}. We now present a certain proximal problem which provides interesting new insights into the algorithm above. 
 
 We   
 consider the continuous minimization of $g(w) + \frac{1}{2} \sum_{k \in V} t_k w_k^2$,
 where $g$ is the \lova extension of $G$. From results in \mychap{prox}, finding the unique minimizer $w$ is equivalent to minimizing a sequence of
 submodular functions $G + \alpha t(A)$ for $\alpha \in \rb$, and any minimizer  satisfies a monotonicity property. This implies that the minimizers in the Newton algorithms must be included in one another, and thus there can only be at most $p$ iterations. The algorithm is then as follows:
\begin{list}{\labelitemi}{\leftmargin=1.1em}
   \addtolength{\itemsep}{-.2\baselineskip}

\item[--] Initialization: $\beta = G(V) / t(V)$, and $B = V$.
\item[--] Perform the following iterations until termination (which must happens after at most $p$ iterations): let $A$ be any minimizer of $G(A) - \beta t(A)$ on $B$. If $G(A) - \beta t(A)  = 0$, then output $\alpha = \beta $ and stop. Otherwise,  Let $\beta = G(A) / t(A)$ and $B =A$.
 \end{list}
 
While we have shown the validity of the previous algorithm for $t$ having strictly positive components, the same result also holds for $t$ having potentially zero values (because it corresponds to a reduced problem with strictly positive values, defined on a restriction of $F$). Moreover,  it turns out that the divide-and-conquer algorithm of \mysec{decomposition} applied to the minimization of $f(w) + \frac{1}{2} \sum_{k \in V} t_k w_k^2$, i.e., the maximization of $-\frac{1}{2} \sum_{k \in V} \frac{ s_k^2}{t_k}$ over $t \in B(F)$, can be shown  to lead to the exact same algorithm.

\paragraph{Proximal problem for $\ell_2$-relaxation.}
For the norms $\Omega_q$ we have defined in \mysec{lprelax}, we can use the results from this section to compute the proximal operator, i.e., the unique minimizer of $\frac{1}{2} \| w - z\|_2^2 + \Omega_q(w)$ (or equivalently maximizing $\frac{1}{2} \| z\|_2^2 - \frac{1}{2} \| s - z\|^2$ such that $\Omega_q^\ast(s) \leqslant 1$). For $q = \infty$, i.e., $\Omega_\infty(w) = f(|w|)$, then the divide-and-conquer algorithm immediately finds the solution, by applying it to the problem $\min_{w \in \rb^p} \frac{1}{2} \| w - |z| \|_2^2 + f(w)$ and then thresholding, or by using the modification described at the end of \mysec{decomposition}.

For $q = 2$, we may compute the proximal operator as follows.
We first notice that for the problem of minimizing $\frac{1}{2} \| z\|_2^2 - \frac{1}{2} \| s - z\|^2$ such that $\Omega_2^\ast(s) \leqslant 1$, then the signs of the solutions are known, i.e., $s_k z_k \geqslant 0$ for all $k \in V$. Thus, if $\varepsilon$ is the vector of signs of $z$, then $s = \varepsilon \circ t$ with
$t$ a maximizer of  $
  t^\top |z| - \frac{1}{2} \| t\|_2^2
$ for $t$ in the intersection of the positive orthant $\rb_+^p$ and the unit dual ball of $\Omega_2$.
By a change of variable $ u_k = t_k^2$, for $k \in V$, then we need to maximize
$$
\sum_{k \in V} |z_k| u_k^{1/2} - \frac{1}{2} \sum_{k \in V} u_k
$$
over the positive submodular polyhedron $P_+(F)$. We can apply the divide-and-conquer algorithm
(note that we have only shown its optimality for smooth functions, but it holds more generally, and in particular here). The only different element is the minimization of the previous cost function subject to $ u \geqslant 0$ and $\sum_{k \in V} u_k \leqslant F(V)$. This can be obtained in closed form as $u_k = |z_k|^2 \min \{ 1, F(V) / \| z\|_2^2 \}$. The divide-and-conquer algorithm may also be used to compute the norm  $\Omega_2(w)$, by maximizing $\sum_{k \in V} |z_k| u_k^{1/2}$ over $u \in P_+(F)$, the first step now becoming $u_k = |z_k|^2   F(V) / \| z\|_2^2$.
See additional details in~\cite{submodlp}.

\paragraph{Special cases.}
In all the extensions that were presented in this section, faster dedicated algorithms exist for special cases, namely for cardinality-based functions~\cite{submodlp} and cuts in chain graphs~\cite{barbero2011fast}.

\chapter{Submodular Function Minimization}
\label{chap:sfm}

Several generic algorithms may be used for the minimization of a submodular function. In this chapter, we present algorithms that are all based on a sequence of evaluations of $F(A)$ for certain subsets $ A \subseteq V$. For specific functions, such as the ones defined from cuts or matroids, faster algorithms exist (see, e.g.,~\cite{gallo1989fast,hochbaum2001efficient}, \mysec{cuts}
and \mysec{matroids}). For other special cases, such as functions obtained as the sum of simple functions, faster algorithms also exist and are reviewed in \mysec{specialstructure}. 

Submodular function minimization algorithms may be divided in two main categories: exact algorithms aim at obtaining a global minimizer, while approximate algorithms only aim at obtaining an approximate solution, that is, a set $A$ such that $F(A) - \min_{B \subseteq V} F(B) \leqslant \varepsilon$, where $\varepsilon$ is as small as possible. Note that if $\varepsilon$ is less than the minimal absolute difference $\delta$ between non-equal values of $F$, then this leads to an exact solution, but that in many cases, this difference $\delta$ may be arbitrarily small. 

An important practical aspect of submodular function minimization is that most algorithms come with online approximation guarantees; indeed, because of a duality relationship detailed in  \mysec{mini}, in a very similar way to convex optimization, a base $s \in B(F)$ may serve as a certificate for optimality. Note that many algorithms (the simplex algorithm is notably not one of them)  come with offline approximation guarantees.

In \mysec{comb}, we review ``combinatorial algorithms'' for submodular function minimization that come with complexity bounds and are not explicitly based on convex optimization. Those are however not used in practice in particular due to their high theoretical complexity (i.e., $O(p^5)$), except for the particular class of posimodular functions, where algorithms scale as $O(p^3)$ (see \mysec{posi}). In \mysec{ellipsoid_SFM}, we show how the ellipsoid algorithm may be applied with a well-defined complexity bounds. While this provided the first polynomial-time algorithms for the submodular function minimization problem, it is too slow in practice.

In \mysec{simplexsfm}, we show how a certain ``column-generating'' version of the simplex algorithm may be considered for this problem, while in
\mysec{accpm-sfm}, the analytic center cutting-plane method center is considered. We show in particular that these methods are closely related to Kelley's method from \mysec{cutting}, which sequentially optimizes piecewise affine lower-bounds to the \lova extension.
 
 In \mysec{minnorm} a formulation based on quadratic separable problem on the base polyhedron, but using the minimum-norm-point algorithm described in \mysec{minnorm-prox}. These last two algorithms  come with no complexity bounds. 
 
 All the  algorithms mentioned above have the potential to find the global minimum of the submodular function if enough iterations are used. They come however with a cost of typically $O(p^3)$ per iteration. The following algorithms have $O(p)$ cost per iteration, but have slow convergence rate, that makes them useful to obtain quickly approximate results (this is confirmed in simulations in \mysec{exp-sfm}): in \mysec{approxsfm}, we describe  optimization algorithms based on separable optimization problems regularized by the
\lova extension. 
Using directly the equivalence presented in Prop.~\ref{prop:minsub}, we can minimize the \lova extension $f$ on the hypercube $[0,1]^p$ using subgradient descent with approximate optimality for submodular function minimization of $O(1/\sqrt{t})$ after $t$ iterations. Using quadratic separable problems, we can use the algorithms of \mysec{approx-prox} to obtain new submodular function minimization algorithms with convergence of the convex optimization problem at rate $O(1/t)$, which translates through the analysis of \mychap{prox} to the same convergence rate of $O(1/\sqrt{t})$ for submodular function minimization, although with improved behavior and better empirical performance (see \mysec{approxsfm} and \mysec{exp-sfm}).

Most algorithms presented in this chapter are \emph{generic}, i.e., they apply to any submodular functions and only access them through the greedy algorithm; in \mysec{specialstructure}, we consider submodular function minimization problems with additional structure, that may lead to more efficient algorithms.

 Note that \emph{maximizing} submodular functions is a hard combinatorial problem in general, with many applications and many recent developments with approximation guarantees. For example, when maximizing a non-decreasing submodular function under a cardinality constraint, the simple greedy method allows to obtain a $(1-1/e)$-approximation~\cite{nemhauser1978analysis} while recent local search methods lead to $1/2$-approximation guarantees (see more details in \mychap{max-ds}).

 \section{Minimizers of submodular functions}
\label{sec:mini}

In this section, we review some relevant results for submodular function minimization (for which algorithms are presented in next sections).

\begin{proposition}\textbf{(Lattice of minimizers of submodular functions)}
\label{prop:minimizerlattice}
Let $F$ be a submodular function such that $F(\varnothing)=0$. The set of minimizers of $F$ is a lattice, i.e., if $A$ and $B$ are minimizers, so are $A \cup B$ and $A \cap B$.
 \end{proposition}
\begin{proof}
Given minimizers $A$ and $B$ of $F$, then, by submodularity, we have
$2 \min_{C \subseteq V} F(C)\leqslant F(A\cup B) + F(A \cap B) \leqslant F(A) + F(B) = 2 \min_{C \subseteq V} F(C)$, hence equality in the first inequality, which leads to the desired result.
\end{proof} 
 
 The following proposition shows that some form of local optimality implies global optimality.
 
\begin{proposition}\textbf{(Property of minimizers of submodular functions)}
\label{prop:minimizer}
\label{prop:mincond}
Let $F$ be a submodular function such that $F(\varnothing)=0$. The set $A \subseteq V$ is a minimizer of $F$ on $2^V$ if and only if $A$ is a minimizer of the function from $2^A$ to $\rb$ defined as $B \subseteq A \mapsto F(B)$, and if $\varnothing$ is a minimizer of the function from $2^{V\backslash A}$ to $\rb$ defined as $B \subseteq V\backslash A \mapsto F(B \cup A) - F(A)$.
 \end{proposition}
\begin{proof}
 The set of two conditions is clearly necessary. To show that it is sufficient, we let $B \subseteq V$,
 we have:
 $F(A) + F(B) \geqslant F(A \cup B) + F(A \cap B) \geqslant F(A) + F(A)$, by using the submodularity of $F$ and then the set of two conditions. This implies that $F(A) \leqslant F(B)$, for all $B \subseteq V$, hence the desired result.
\end{proof}

The following proposition provides a useful step towards submodular function minimization. In fact, it is the starting point of most polynomial-time algorithms presented in \mysec{comb}.
Note that submodular function minimization may also be obtained from minimizing $\| s\|_2^2$ over $s$ in the base polyhedron (see \mychap{prox} and \mysec{quadprox}).

\begin{proposition}\textbf{(Dual of minimization of submodular functions)}
\label{prop:dualmin}
Let $F$ be a submodular function such that $F(\varnothing)=0$.
We have: 
\BEQ \min_{A\subseteq V} F(A) = \max_{ s \in B(F) } s_-(V) = F(V) - \min_{s \in B(F)} \| s\|_1, \EEQ
where $(s_-)_k = \min\{s_k,0\}$ for $k \in V$. Moreover, given $A \subseteq V$ and $s\in B(F)$, we always have $F(A) \geqslant s_-(V)$ with equality if and only if
$\{ s < 0 \} \subseteq A \subseteq \{ s \leqslant 0\}$ and $A$ is \emph{tight} for $s$, i.e., $s(A) = F(A)$.

We also have
\BEQ
 \min_{A\subseteq V} F(A) =  \max_{ s \in P(F), \ s \leqslant 0 } s(V). 
 \EEQ
Moreover,  given $A \subseteq V$ and $s\in P(F) $ such that $s \leqslant 0$, we always have $F(A) \geqslant s(V)$ with equality if and only if
$\{ s < 0 \} \subseteq A $ and $A$ is tight for $s$, i.e., $s(A) = F(A)$.
\end{proposition}
\begin{proof}
We have, by strong convex duality, and Props.~\ref{prop:minlova} and \ref{prop:support}:
\BEAS
\min_{A\subseteq V} F(A) &   = &   \min_{ w \in [0,1]^p } f(w) \\
&= &  \min_{ w \in [0,1]^p } \max_{ s\in B(F)} w^\top s = 
\max_{ s\in B(F)} \min_{ w \in [0,1]^p }  w^\top s =  
\max_{ s\in B(F)}  s_-(V). \EEAS
Strong duality indeed holds because of Slater's condition ($[0,1]^p$ has non-empty interior). Since $s(V)=F(V)$ for all $s \in B(F)$, we have $s_-(V) = F(V) - \| s\|_1$, hence the second equality.

Moreover, we have, for all $A \subseteq V$ and $s\in B(F)$:
$$F(A) \geqslant s(A) = s(A \cap \{ s < 0 \} ) + s(A \cap \{ s > 0 \} )   \geqslant  s(A \cap \{ s < 0 \} )  \geqslant s_-(V),$$
with equality if there is equality in the three inequalities. The first one leads to $s(A) =F(A)$. The second one leads to $A \cap \{ s > 0 \} = \varnothing$, and the last one leads to  $\{ s < 0 \} \subseteq A $.
Moreover,
\BEAS
\!\! \max_{ s \in P(F), \ s \leqslant 0 } \! s(V) \!\!\!
 & = &  \!\!\! \max_{ s \in P(F)} \min_{ w \geqslant 0 } s^\top 1_V - w^\top s
 =  \min_{ w \geqslant 0 }  \max_{ s \in P(F)} s^\top 1_V - w^\top s \\
&  =  & \!\!\!  \min_{  1 \geqslant w \geqslant 0 } f(1_V\!-\!w) \mbox{ because of property (c) in Prop.~\ref{prop:support}},\\
&  = & \!\!\!  \min_{A\subseteq V} F(A) \mbox{ because of Prop.~\ref{prop:min}} .
 \EEAS
Finally, given $s \in P(F)$ such that $s \leqslant 0$ and $A \subseteq V$, we have:
$$
F(A) \geqslant s(A) =  s(A \cap \{ s < 0 \} )  \geqslant s(V),
$$
with equality if and only if $A$ is tight and $ \{ s< 0 \} \subseteq A$.
\end{proof}

\section{Combinatorial algorithms}
\label{sec:comb}
\label{sec:proxcomb}
Most algorithms are based on Prop.~\ref{prop:dualmin}, i.e., on the identity $\min_{A\subseteq V} F(A) = \max_{ s \in B(F) } s_-(V)$.  Combinatorial algorithms will usually output the subset $A$ and a base $s \in B(F)$ such that $A$ is tight for $s$ and $\{ s < 0 \} \subseteq A \subseteq \{ s \leqslant 0\}$, as a certificate of optimality.

Most algorithms, will also output the largest   minimizer $A$ of $F$, or sometimes describe the entire lattice of minimizers.
Best algorithms have polynomial complexity~\cite{schrijver2000combinatorial,iwata2001combinatorial,orlin2009faster}, but still have high complexity (typically $O(p^5)$ or more). Most algorithms  update a sequence of convex combination of vertices of $B(F)$ obtained from the greedy algorithm using a specific order (see a survey of existing approaches in~\cite{mccormick2005submodular}). Recent algorithms~\cite{jegelka2011-fast-approx-sfm} consider  reformulations in terms of generalized graph cuts, which can be approximately solved efficiently.

Note here the difference between the combinatorial algorithm which maximizes $s_-(V)$ and the ones based on the minimum-norm point algorithm which maximizes  $-\frac{1}{2} \|s\|_2^2$ over the base polyhedron $B(F)$. In both cases, the submodular function minimizer $A$ is obtained by taking the negative values of $s$. In fact, the unique minimizer of $\frac{1}{2}\|s\|_2^2$ is also a maximizer of $s_-(V)$, but not vice-versa.

\section{Minimizing symmetric posimodular functions}
\label{sec:posi}
A submodular function $F$ is said symmetric if for all $B \subseteq V$, $F(V \backslash B)  = F(B)$. By applying submodularity, we get that $2 F(B) = F(V \backslash B)  +  F(B) \geqslant F(V) + F( \varnothing) = 2 F(\varnothing) = 0$, which implies that $F$ is non-negative. Hence its global minimum is attained at $V$ and $\varnothing$. Undirected cuts (see \mysec{cuts}) are the main classical examples of such functions.

Such functions can be minimized in time $O(p^3)$ over all \emph{non-trivial} (i.e., different from $\varnothing$ and $V$) subsets of $V$ through a simple algorithm of Queyranne~\cite{queyranne1998minimizing}. Moreover, the algorithm is valid for the regular minimization of \emph{posimodular} functions~\cite{nagamochi1998note}, i.e., of functions that satisfies
$$
\forall A,B \subseteq V, \ F(A) + F(B) \geqslant F( A \backslash B) +  F( B \backslash A).
$$
These include symmetric submodular functions as well as non-decreasing modular functions, and hence the sum of any of those (in particular, cuts with sinks and sources, as presented  in \mysec{cuts}). Note however that this does not include general modular functions (i.e., with potentially negative values); worse, minimization of functions of the form $  F(A) - z(A)$ is provably as hard as general submodular function minimization~\cite{queyranne1998minimizing}. Therefore this $O(p^3)$ algorithm is quite specific and may not be used for solving proximal problems with symmetric functions.

\section{Ellipsoid method}
\label{sec:ellipsoid_SFM}
\label{sec:sfm_ellipsoid}
Following~\cite{grotschel1981ellipsoid}, we may apply the ellipsoid method described in \mysec{ellipsoid} to the problem $\min_{w \in [0,1]^p} f(w)$.
The minimum volume ellipsoid $\mathcal{E}_0$ that contains $[0,1]^p$ is the ball of center $1_V/2$ and radius $\sqrt{p} / 2$. Starting from this ellipsoid, the complexity bound from \mysec{ellipsoid} leads to, after $t$ iterations 
$$
f(w_t) - \min_{A \subseteq V} F(A) \leqslant \exp( - t / 2p^2 ) \big[ \max_{A \subseteq V} F(A) - \min_{A \subseteq V} F(A) \big].
$$
This implies that in order to reach a precision of $\varepsilon \big[ \max_{A \subseteq V} F(A) - \min_{A \subseteq V} F(A) \big]$, at most $2 p^2 \log ( 1/\varepsilon)$ iterations are needed. Given that every iteration has complexity $O(p^3)$, we obtain an algorithm with complexity $O(p^5 \log ( 1/\varepsilon))$, with similar complexity than the currently best-known combinatorial algorithms from \mysec{comb}. Note here the difference between \emph{weakly} polynomial algorithms such as the ellipsoid with a polynomial dependence on $\log ( 1/\varepsilon)$, and \emph{strongly} polynomial algorithms which have a bounded complexity when $\varepsilon$ tends to zero.

\section{Simplex method for submodular function minimization}
 \label{sec:simplexsfm}

 In this section, following~\cite{mccormick2005submodular,simplex_sfm}, we consider the dual optimization problem derived in Prop.~\ref{prop:dualmin}.
We thus assume that we are given $d$ points in $\rb^p$, $s_1,\dots,s_d$, put in a matrix $ S \in \rb^{d \times p}$. We want to maximize
$s_-(V)$ over the convex hull of $s_1,\dots,s_d$. That is, we are looking for $\eta \in \rb^d$ such that $s = S^\top  \eta$ and $\eta \geqslant 0$, and $\eta^\top 1_d =1$. In our situation, $d$ is the number of extreme points of the base polytope $B(F)$ (up to $p!$). We classically represent $S^\top \eta \in \rb^p$
as $S^\top \eta = \alpha - \beta$, where $\alpha$ and $\beta$ are non-negative vectors in $\rb^p$. The problem then becomes:
$$
\min_{ \eta \geqslant 0,  \ \alpha \geqslant 0, \ \beta \geqslant 0} \beta^\top 1_p \mbox{ such that } S^\top \eta - \alpha + \beta = 0, \ \eta^\top 1_d = 1.
$$
It is thus exactly a linear program in standard form (as considered in \mysec{simplex}) with:
$$
x = \Bigg(
\begin{array}{c}
\eta \\
\alpha \\
\beta 
\end{array}
\Bigg), \ c = \Bigg(
\begin{array}{c}
0_d \\
0_p \\
1_p 
\end{array}
\Bigg), \ b = \Bigg(
\begin{array}{c}
1 \\
0_p  
\end{array}
\Bigg),  \ A = \bigg( \begin{array}{ccc}
1_d^\top &  0_p^\top & 0_p^\top \\
S^\top & -\idm_p & \idm_p
\end{array}
\bigg).
$$
This linear program have many variables and a certain version of  simplex method may be seen as a column decomposition approach~\cite{lubbecke2005selected}, as we now describe.

A basic feasible solution is defined by a subset $J$ of $\{1,\dots,2p+d\}$, which can be decomposed into a subset $I \subseteq \{1,\dots,d\}$ and two disjoint subsets $I_\alpha$ and $I_\beta$ of $\{1,\dots,p\}$ (they have to be disjoint so that the corresponding columns of $A$ are linearly independent). We denote by
$K = ( I_\beta \cup I_\alpha)^{ \mathsf{c}} \subset \{1,\dots,p\}$. Since $|I| + |I_\alpha| +|I_\beta| = p + 1$, then $|I| = |K| + 1$.  

In order to compute the basic feasible solution (i.e., $x_J = -A_J^{-1} b$), we denote by $T$
the $|I| \times |I|$ square matrix 
$
T = \bigg(
\begin{array}{c}
1_{|I|}^\top  \\
S_{IK}^\top
\end{array}
\bigg)
$.
The basic feasible solution $x$ corresponds to $\eta_I$ defined as $\eta_I = T^{-1} (1,0_{|I|-1}^\top)^\top$, i.e., 
such that $ \eta_I^\top 1_{|I|} = 1$ and $S_{IK}^\top \eta_I = 0$ (as many equations as unknowns). Then, $\alpha_{I_\alpha} = S_{I I_\alpha  }^\top \eta_I$, $\beta_{I_\beta} = - S_{I I_\beta }^\top \eta_I$. All of these have to be non-negative, while all others are set to zero.

We can now compute the candidate dual variable (i.e., $y = A_J^{-\top} c_J$), as  $y = (-v, w^\top)^\top$, with $v \in \rb$ and $w \in \rb^p$, with the following equations: 
$s_i^\top w = v$ for all $i \in I$, and $w_{I_\alpha} = 0_{|I_\alpha|}$ and $w_{I_\beta} = 1_{|I_\beta|}$. This may be obtained by computing $(-v,w_K^\top)^\top = - T^{-\top} S_{I I_\beta} 1_{|I_\beta|}$.
We then obtain the following vector of reduced cost (i.e., $\bar{c} = c - A^\top y$):
$$
\bar{c} = \Bigg(
\begin{array}{c}
0_d \\
0_p \\
1_p 
\end{array}
\Bigg) - 
\Bigg(
\begin{array}{ccc}
1_d & S \\
0_p & - I_p \\
0_p & + I_p
\end{array}
\Bigg)
\bigg(
\begin{array}{c}
 -v \\
 w
 \end{array}
\bigg)
=
\Bigg(
\begin{array}{c}
v 1 - S w \\
w  \\
1 - w
\end{array}
\Bigg).
$$
If $\bar{c} \geqslant 0$, we recover optimality conditions for the original problem. There are several possibilities for lack of optimality. Indeed, we need to check which elements of $J^{\sc c}$ is violating the constraint. It could be $j \in I^{\sc c}$ such that $ s_j^\top w \geqslant v$, or $j \in \{1,\dots,p\}$ such that $w_j \notin [0,1]$.
In our algorithm, we may choose which violated constraints to treat first. We will always check first $w \in [0,1]^p$ (since it does not require to run the greedy algorithm), then, in cases where we indeed have $w \in [0,1]^p$, we run the greedy algorithm to obtain $j \in I^{\sc c}$ such that $ s_j^\top w \geqslant v$.
We thus consider the following situations:

\begin{list}{\labelitemi}{\leftmargin=1.1em}
   \addtolength{\itemsep}{-.2\baselineskip}

\item[--]
If there exists $i \in K$ such that $w_i<0$. The descent direction $d_J = - A_J^{-1} A_j$ is thus equal to
$\Bigg(
\begin{array}{c}
-T^{-1}(0,\delta_i^\top)^\top \\
-S_{I I_\alpha}^\top T^{-1}(0,\delta_i^\top)^\top \\
+ S_{I I_\beta}^\top T^{-1}(0,\delta_i^\top)^\top
\end{array}
\Bigg)$, and by keeping only the negative component, we find the largest $u$ such that $(x+ud)_J$ hits a zero, and remove the corresponding index.


\item[--] Similarly, if there exists $i \in K $ such that $w_i>1$, we have the same situation.

\item[--] If there exists $j$ such that $(v 1 - S w )_j<0$, the descent direction is
$\Bigg(
\begin{array}{c}
- T^{-1}(1,  (s_j)_K)^\top \\
- S_{I I_\alpha}^\top T^{-1}(1,   (s_j)_K)^\top \\
+ S_{I I_\beta}^\top T^{-1}(1,   (s_j)_K)^\top
\end{array}
\Bigg)$,
then we add a new index   to $I$ and remove one from $I_\alpha$ or $I_\beta$. Note that if we assume $w \in [0,1]^p$, then we have an optimal solution of the problem constrained to vectors $s_j$ for $j \in I$.
\end{list} 
Using the proper linear algebra tools common in simplex methods~\cite{bertsimas1997introduction}, each iteration has complexity $O(p^2)$.

In summary, if we consider a pivot selection strategy such that we always consider first the violation of the constraints $w \geqslant 0$ and $w \leqslant 1$, then, every time these two constraints are satisfied,  $w$ is a global minimum of $\max_{i \in I} s_i^\top w$ over $w \in [0,1]^p$, that is a piecewise affine lower-bound obtained from subgradients of the \lova extension are certain points. Thus,  we are actually ``almost'' using Kelley's method described in \mysec{simplicial} (almost, because, like for active-set methods for quadratic programming in \mysec{activeQPLS}, extreme points $s_j$ may  come in and out of the set~$I$).
Moreover, the global minimum $w$ mentioned above is not unique, and the simplex method selects an extreme point of the polytope of solutions. This is to be contrasted with the next section, where interior-point will be considered, leading to much improved performance in our experiments.

\section{Analytic center cutting planes}
\label{sec:accpm-sfm}
In this section, we consider the application of the method presented in \mysec{accpm} to the problem $\min_{w \in [0,1]^p} f(w)$. Given the set of already observed points $w_{0}, \dots, w_{t-1}$ (and the corresponding outcomes of the greedy algorithm at these points $s_0,\dots,s_{t-1}$), and the best function values $F_{t-1}$ for $F$ obtained so far, then the next point is obtained by finding a weighted analytic center, i.e., by minimizing
\BEAS
\!\!\!\!\!& & \min_{w \in \rb^p, u \in \rb} - \alpha \log (F_{t-1} - u) - \sum_{j=0}^{t-1} \log( u - s_j^\top w) - \frac{1}{2} \sum_{i=1}^p \log w_i ( 1 - w_i)  .
\EEAS
Using strong convex duality, this is equivalent to
\BEAS
\!\!\!\!\!&  & \min_{w \in \rb^p, u \in \rb}- \alpha \log (F_{t-1} - u)- \frac{1}{2} \sum_{i=1}^p \log w_i ( 1 - w_i)  \\[-.3cm]
& & \hspace*{5.5cm}
+ \sum_{j=0}^{t-1} \max_{\eta_j \in \rb} \eta_j(   s_j^\top w - u)  + \log \eta_j \\
\!\!\!\!\!&   =\!\! & \max_{\eta \in \rb^{t}} 
\min_{w \in \rb^p, u \in \rb} 
- \alpha \log (F_{t-1} - u)- \frac{1}{2} \sum_{i=1}^p \log w_i ( 1 - w_i) 
\\[-.3cm]
& & \hspace*{5.5cm}
+ w^\top S^\top \eta - u \eta^\top 1
+ \sum_{j=0}^{t-1} \log \eta_j \\[-.2cm]
\!\!\! \!\!& =\!\!  & \max_{\eta \in \rb^{t}}  \ 
\alpha \log \eta^\top 1-  \eta^\top 1
+ \sum_{j=0}^{t-1} \log \eta_j  + \sum_{j=1}^p \varphi( (S^\top \eta)_j) + \mbox{cst},
\EEAS
with $\varphi(x) = \min_{w \in [0,1] } w x - \frac{1}{2} \log w (1-w)=\frac{x}{2} + \frac{ 1 - \sqrt{1+x^2}}{2} + \frac{1}{2} \log \frac{ 1 + \sqrt{1+x^2}}{2}$.
The minimization may be done using Newton's method~\cite{boyd}, since a feasible start is easy to find from previous iterations.

The generic cutting-plane method considers $\alpha=1$. When $\alpha$ tends to infinity, then every analytic center subproblem is equivalent to minimizing $\max_{j \in \{0,\dots,t-1\}} s_j^\top w$ such that $w \in [0,1]^p$ and selecting among all minimizers the analytic center of the polytope of solutions. Thus, we obtain an instance of Kelley's method. The selection of an interior-point leads in practice to a choice of the next subgradient $s_t$ which is much better than with an extreme point (which the simplex method described above would give).

\section{Minimum-norm point algorithm}
\label{sec:minnorm}

 From \eq{proxalpha} or Prop.~\ref{prop:subfromprox}, we obtain that if we know how to minimize $f(w) + \frac{1}{2} \| w\|_2^2$, or equivalently, minimize $\frac{1}{2} \| s\|_2^2$ such that $s \in B(F)$, then we get all minimizers of $F$ from the negative components of $s$.
 
We can then apply the  minimum-norm point algorithm detailed in \mysec{minnorm-prox} to the vertices of $B(F)$, and notice that step (5) does not require to list all extreme points, but simply to maximize (or minimize) a linear function, which we can do owing to the greedy algorithm.
The complexity of each step of the algorithm is essentially $O(p)$ function evaluations and operations of order $O(p^3)$. However, there are no known upper bounds on the number of iterations. Finally, we obtain $s\in B(F)$ as a convex combination of extreme points.

Note that
once we know which values of the optimal vector $s$ (or $w$) should be equal, greater or smaller, then, we obtain in closed form all values.
Indeed, let $v_1 > v_2 > \cdots > v_m$ the $m$ different values taken by $w$, and $A_i$ the corresponding sets such that $w_{k} = v_j$ for $k \in A_j$. Since we can express $f(w) + \frac{1}{2} \|w\|_2^2 = \sum_{j=1}^m \big\{ v_j 
[ F( A_1 \cup \cdots \cup A_{j}) -  F( A_1 \cup \cdots \cup A_{j-1}) ] + \frac{|A_j|}{2} c_j^2 \big\} $, we then have:
\BEQ
\label{eq:cminnorm}
v_j = \frac{ -F( A_1 \cup \cdots \cup A_{j}) +  F( A_1 \cup \cdots \cup A_{j-1}) }{ | A_{j}| },
\EEQ
which allows to compute the values $v_j$ knowing only the sets $A_j$ (i.e., the ordered partition of constant sets of the solution). This shows in particular that minimizing $f(w) + \frac{1}{2} \| w\|_2^2$ may be seen as a certain search problem over ordered partitions.

\section{Approximate minimization through convex optimization}

\label{sec:approxsfm}
In this section, we consider two approaches to submodular function minimization based on iterative algorithms for convex optimization: a \emph{direct} approach, which is based on minimizing the  \lova  extension directly on $[0,1]^p$ (and thus using Prop.~\ref{prop:minsub}), and an \emph{indirect approach},  which is based on quadratic separable optimization problems (and thus using Prop.~\ref{prop:quadprox}). All these algorithms will access the submodular function through the greedy algorithm, once per iteration, with minor operations inbetween.

\paragraph{Restriction of the problem.} Given a submodular function $F$, if $F(\{k\})<0$, then $k$ must be in any minimizer of $F$, since, because of submodularity, if it is not, then adding it would reduce the value of~$F$. Similarly,
if $F(V)-F(V\backslash \{k\})>0$, then $k$ must be in the complement of any minimizer of $F$. Thus, 
if we denote $A_{\min}$ the set of $k \in V$ such that  $F(\{k\})<0$ and $A_{\max}$ the complement of the set of $ k \in V$ such that $F(V)-F(V\backslash \{k\})>0$, then we may restrict the minimization of $F$ to subset $A$ such that $A_{\min} \subseteq A \subseteq A_{\max}$. This is equivalent to minimizing the submodular function $A\mapsto F(A \cup A_{\min} ) - F( A_{\min})$ on $A_{\max} \backslash A_{\min}$.  

From now on, (mostly for the convergence rate described below) we assume that we have done this restriction and that we are now minimizing a function $F$ so that for all $k \in V$, $F(\{k\}) \geqslant 0$ and $F(V)-F(V\backslash \{k\})\leqslant 0$. We denote by $\alpha_k =F(\{k\}) +  F(V\backslash \{k\}) - F(V)$, which is non-negative by submodularity. Note that in practice, this restriction can be seamlessly done by starting regular iterative methods from specific starting points.

\paragraph{Direct approach.}
From Prop.~\ref{prop:minsub}, we can use any convex optimization algorithm to minimize $f(w)$ on $w \in [0,1]^p$. Following~\cite{hazan2009online}, we consider subgradient descent with step-size $\gamma_t = \frac{D \sqrt{2}}{\sqrt{pt}}$ (where $D^2 =\sum_{k \in V} \alpha_k^2$), i.e., (a) starting from any $w_0 \in [0,1]^p$,  we iterate (a) the computation of a maximiser $s_{t-1}$ of $w_{t-1}^\top s$ over $s \in B(F)$, and (b) the update $w_{t} = \Pi_{ [0,1]^p} \big [ w_{t-1} - \frac{D \sqrt{2}}{\sqrt{p t}} s_{t-1} \big]$, where $\Pi_{ [0,1]^p} $ is the orthogonal projection onto the set $[0,1]^p$ (which may done by thresholding the components independently).

The following proposition shows that in order to obtain a certified $\varepsilon$-approximate set $B$, we need at most  $\frac{4 p D^2}{\varepsilon^2}$ iterations of subgradient descent (whose complexity is that of the greedy algorithm to find a base $s  \in B(F)$).

\begin{proposition} \textbf{(Submodular function minimization by subgradient descent)}
After $t$ steps of projected subgradient descent,  among the $p$ sup-level sets of $w_t$, there is a set $B$ such that
$F(B) - \min_{A \subseteq V} F(A) \leqslant  \frac{ D p^{1/2}}{\sqrt{2t}}$. Moreover, we have a certificate of optimality
$
\bar{s}_t = \frac{1}{t+1} \sum_{u=0}^t s_u
$, so that $F(B) - (\bar{s}_t )_-(V) \leqslant \frac{ D p^{1/2}}{\sqrt{2t}}$, 
with $D^2 = \sum_{k=1}^p \alpha_k^2$.
\end{proposition}
\begin{proof}
Given an approximate solution $w$ so that $0 \leqslant f(w) - f^\ast \leqslant \varepsilon$, with
$f^\ast = \min_{A \subseteq V} F(A) = \min_{w \in [0,1]^p} f(w)$, we can sort the elements of $w$ in decreasing order, i.e., $1 \geqslant w_{j_1} \geqslant \cdots \geqslant w_{j_p} \geqslant 0$. We then have, with $B_k = \{j_1,\dots,j_k\}$,
\BEAS
f(w) - f^\ast
 & = &  \sum_{k=1}^{p-1} ( F(B_k) - f^\ast) ( w_{j_k} - w_{j_{k+1}} ) \\
 & & + ( F(V) - f^\ast) (w_{j_p}-0) + (F(\varnothing) - f^\ast) 
(1-w_{j_1}).
\EEAS
Thus, as the sum of positive numbers, there must be at least one $B_k$ such that $F(B_k) - f^\ast \leqslant \varepsilon$. Therefore, given $w$ such that $0 \leqslant f(w) - f^\ast \leqslant \varepsilon$, there is at least on the sup-level set of $w$ which has values for $F$ which is $\varepsilon$-approximate.

The subgradients of $f$, i.e., elements $s$ of $B(F)$ are such that $ F(V) - F(V \backslash\{k\}) \leqslant s_k \leqslant F(\{k\})$. This implies that $f$ is Lipschitz-continuous with constant $D$, with $D^2 = \sum_{k=1}^p \alpha_k^2$.
Since $[0,1]^p$ is included in an $\ell_2$-ball of radius $\sqrt{p}/2$, results from \mysec{subgrad} imply that we may take $\varepsilon = \frac{ D p^{1/2}}{\sqrt{2t}}$.
Moreover, as shown in \cite{condgrad}, the average of all subgradients provides a certificate of duality with the same known convergence rate (i.e., if we use it as a certificate, it may lead to much better certificates than the bound actually suggests).

Finally, if we replace the subgradient iteration by $w_{t} = \Pi_{ [0,1]^p} \big [ w_{t-1} - \Diag(\alpha)^{-1}\frac{  \sqrt{2}}{\sqrt{ t}} s_{t-1} \big]$, then this corresponds to a subgradient descent algorithm on the function
$w \mapsto f(\Diag(\alpha)^{-1/2} w)$ on the set $\prod_{k \in V} [ 0, \alpha_k^{1/2}]$, for which the diameter of the  domain and the Lipschitz constant are equal to $\big(\sum_{k \in V} \alpha_k \big)^{1/2}$. We would  obtain the improved convergence rate of 
$
\frac{\sum_{k \in V}\alpha_k}{\sqrt{2t}}
$, but with few empirical differences.
\end{proof}

The previous proposition relies on one of the most simple algorithms for convex optimization, subgradient descent, which is applicable in most situations; however, its use is appropriate because the \lova extension is not differentiable, and the dual problem is also not differentiable.
We have considered a non-adaptive steps-size $\gamma_t = \frac{D \sqrt{2}}{\sqrt{pt}}$ in order to obtain a complexity bound. Another common strategy is to use an approximation Polyak's rule~\cite{Bertsekas}: given the function value $f(w_t)$, the gradient norm $\| s_t\|_2$ and the current best dual value $d_{t-1} =  \max_{u \in \{0,\dots, t-1\} } (\bar{s}_u)_- (V)$, the step-size is $\alpha_t = \frac{f(w_t) - d_{t-1}}{ \| s_t\|_2^2}$.
See \mysec{exp-sfm} for an experimental comparison.

 \paragraph{From separable problems to submodular function minimization.} 
 We now consider separable quadratic optimization problems whose duals are the maximization of a concave quadratic function on $B(F)$, which is smooth. We can thus use the conditional gradient algorithm described in \mysec{condgrad}, with a better convergence rate; however, as we show below, when we threshold the solution to obtain a set $A$, we get the same scaling as before (i.e., $O(1/\sqrt{t})$), with nevertheless an improved empirical behavior. See below and experimental comparisons in \mychap{experiments}. We first derive a bound bonding the duality gap for submodular function minimization when thresholding the iterates from the minimization of $\frac{1}{2} \| w\|_2^2 + f(w)$.
 
 \begin{proposition} \textbf{(Duality gap for submodular function minimization from proximal problem)}
 \label{prop:transfergap}
Let $(w,s) \in \rb^p \times B(F)$ be a pair of primal-dual candidates for the minimization of $\frac{1}{2} \| w\|_2^2 + f(w)$, with duality gap
$\varepsilon = \frac{1}{2} \| w\|_2^2 + f(w) + \frac{1}{2} \| s\|_2^2 $. Then if $A$ is the suplevel-set of $w$ with smallest value of $F$, then
$$
F(A) - s_-(V) \leqslant \sqrt{p \varepsilon / 2}.
$$
\end{proposition}
\begin{proof}
From \eq{dualitygapprox}, if we assume that $(F + \psi'(\alpha) )( \{ w \geqslant \alpha \}) - ( s + \psi'(\alpha) )_-(V) > \varepsilon / 2 \eta$ for all $\alpha \in [-\eta,\eta]$, then we obtain:
$$
\varepsilon \geqslant \int_{-\eta}^{+\eta}
\bigg\{
(F + \alpha 1_V )( \{ w \geqslant \alpha \}) - ( s + \alpha 1_V )_-(V)
\bigg\} d\alpha > \varepsilon,
$$
which is a contradiction. Thus, there exists $\alpha \in [-\eta,\eta]$ such that
$ 0 \leqslant (F + \alpha 1_V )( \{ w \geqslant \alpha \}) - ( s +  \alpha 1_V )_-(V)
 \leqslant  \varepsilon / 2 \eta$.
 This leads to
 \BEAS
F( \{ w \geqslant \alpha \}) - ( s )_-(V)
&  \!\!\!\leqslant  \!\!\! &  \frac{\varepsilon }{ 2 \eta} - \alpha |\{ w \geqslant \alpha \}|
 - ( s )_-(V) + ( s +  \alpha 1_V )_-(V)
\\
& \!\!\! \leqslant \!\!\! &  \frac{\varepsilon }{ 2 \eta} + |\alpha| p \leqslant \frac{\varepsilon }{ 2 \eta} + \eta p.
 \EEAS
 The last inequality may be derived using monotonicity arguments and considering two cases for the sign of $\alpha$. 
 By choosing $\eta = \sqrt{ \varepsilon  / 2p}$, we obtain the desired bound.
\end{proof}

\paragraph{Conditional gradient.}
We now consider the set-up of \mychap{prox}
with $\psi_k(w_k) = \frac{1}{2   } w_k^2$, and thus $\psi_k^\ast(s_k) = \frac{1}{2} s_k^2$. That is, e consider the conditional gradient algorithm studied in \mysec{approx-prox} and \mysec{condgrad}, with the smooth function $g(s) = \frac{1}{2} \sum_{k \in V}    s_k^2 $: (a) starting from any base $s_0 \in B(F)$, iterate (b) the greedy algorithm to obtain a minimizer $\bar{s}_{t-1}$ of $  s_{t-1}  ^\top s$ with respect to $s \in B(F)$, and (c) perform a line search to minimize with respect to $\rho \in [0,1]$,  $
[ s_{t-1} + \rho ( \bar{s}_{t-1} - s_{t-1}) ]^\top  [ s_{t-1} + \rho ( \bar{s}_{t-1} - s_{t-1}) ]$.

Let $\alpha_k = F(\{k\}) + F( V \backslash \{k\}) - F(V)$, $k=1,\dots,p$, be the widths of the hyper-rectangle enclosing $B(F$). The following proposition shows how to obtain an approximate minimizer of $F$.

\begin{proposition} \textbf{(Submodular function minimization by conditional gradient descent)}
After $t$ steps of the conditional gradient method described above,  among the $p$ sub-level sets of $  s_t$, there is a set $B$ such that
$F(B) - \min_{A \subseteq V} F(A) \leqslant \frac{1}{ \sqrt{t} } \sqrt{   \frac{ \sum_{k=1}^p \alpha_k^2  }{2} \sum_{k=1}^p  }  $. Moreover, $s_t$ acts as a certificate of optimality,
 so that $F(B) - (s_t)_-(V) \leqslant  \frac{1}{ \sqrt{t} } \sqrt{   \frac{ \sum_{k=1}^p \alpha_k^2  }{2} p  }  $.
\end{proposition}
\begin{proof}
The convergence rate analysis of the conditional gradient method leads to an $\varepsilon$-approximate solution with  $\varepsilon \leqslant \frac{ \sum_{k=1}^p \alpha_k^2  }{t+1}$. From Prop.~\ref{prop:transfergap}, then we obtain by thresholding the desired gap for submodular function minimization.
\end{proof}

Here the convergence rate is the same as for subgradient descent. See \mychap{experiments} for an empirical comparison, showing a better behavior for the conditional gradient method. As for subgradient descent, this algorithm provides certificates of optimality. Moreover, when offline (or online) certificates of optimality ensures that we an approximate solution, because the problem is strongly convex, we obtain also a bound $ \sqrt{2 \varepsilon}$ on $\| s_t - s^\ast\|_2$ where $s^\ast$ is the optimal solution. This in turn allows us to ensure that all indices $k$ such that $ s_t >  \sqrt{2 \varepsilon}$ cannot be in a minimizer of $F$, while those indices $k$ such that $ s_t < - \sqrt{2 \varepsilon}$ have to be in a minimizer, which can allow efficient reduction of the search space (although these have not been implemented in the simulations in \mychap{experiments}).

Alternative algorithms for the same separable optimization problems may be used, i.e., conditional gradient without line search~\cite{jaggi,dunn1980convergence}, with similar convergence rates and behavior, but sometimes worse empirical peformance. Another alternative is to consider projected subgradient descent in $w$, with the same convergence rate (because the objective function is then strongly convex). Note that as shown before (\mysec{approx-prox}), it is equivalent to a conditional gradient algorithm with no line search.

\section{Using special structure}
\label{sec:specialstructure}

For some specific submodular functions, it is possible to use alternative optimization algorithms with either improved complexity bounds or numerical efficiency. The most classical structure is \emph{decomposability}: the  submodular function $F$ is assumed to be a sum of \emph{simple} submodular functions $F_i$, $i=1,\dots,r$, i.e., $\forall A \subseteq V$, $F(A) = \sum_{i=1}^r F_i(A)$. There are several notions of simplicity that may be considered and are compared empirically in~\cite{treesubmod}. All included functions of cardinality and restrictions thereof, as well as cuts in chain or tree-structured graphs.

In~\cite{komodakis2011mrf}, it is assumed that a minimizer of the set function $A\mapsto F_j(A) - s(A)$ may be computed efficiently for any vector $s \in \rb^p$. This leads naturally to consider a dual approach where projected subgradient ascent is used to optimize the dual function. While this approach exploits decomposability appropriately (in particular to derive parallel implementations), the number of iterations required by the projected subgradient methods is large.

In~\cite{stobbe}, it is assumed that one may compute 
 a convex smooth (with bounded Lipschitz-constant of the gradient) approximation of the \lova extension $f_i$ with uniform approximation error.  In this situation, the \lova extension of $F$ may be approximated  by a smooth function on which an accelerated gradient technique such as described in \mysec{prox-opt} may be used with convergence rate $O(1/t^2)$ after $t$ iterations. When  choosing a well-defined amount of smoothnees, this leads to an approximation guarantee for submodular function minimization of the form $O(1/t)$, instead of $O(1/\sqrt{t})$ in the general case.

In~\cite{treesubmod},  it is assumed that one may compute efficiently the unique minimizer of $\frac{1}{2} \| w- z\|^2 + f_i(w)$ for any $z \in \rb^p$, which is equivalent to efficient orthogonal projections on the base polytope $B(F_i)$. One may then use a decomposition approach for the problem $\min_{w \in \rb^p} f(w) + \frac{1}{2} \|w\|_2^2$ (from which we may obtain a minimizer of $F$ by thresholding the solution at $0$). As shown in~\cite{treesubmod}, the dual problem may be cast as finding the closest points between two polytopes, for which dedicated efficient algorithms are available~\cite{bauschke2004finding}. Moreover, these approaches are also efficiently parallelizable.

\chapter{Other Submodular Optimization Problems}
\label{chap:max-ds}

While submodular function \emph{minimization} may be solved in polynomial time (see \mychap{sfm}), submodular function \emph{maximization} (which includes the maximum cut problem) is NP-hard. However, for many situations, local search algorithms exhibit theoretical guarantees and the design and analysis of such algorithms is an active area of research, in particular due to the many applications where the goal is to maximize submodular functions (see, e.g., sensor placement in \mysec{sensor} and experimental design in \mysec{design}). Interestingly, the techniques used for maximization and minimization are rather different, in particular with less use of convex analysis for maximization. In this chapter, we review some classical and recent results for the maximization of submodular  (\mysec{max-card} and \mysec{max}), before presenting the problem of differences of submodular functions in  \mysec{ds}.

\section{Maximization with cardinality constraints}
\label{sec:max-card}

In this section, we first consider the classical instance of a submodular maximization problem, for which the greedy algorithm leads to the optimal approximation guarantee.

\paragraph{Greedy algorithm for non-decreasing functions.}
Submodular function maximization provides a classical example where greedy algorithms do have  performance guarantees. We now consider a non-decreasing submodular function $F$ and the problem of minimizing $F(A)$ subject to the constraint $|A| \leqslant k$, for a certain $k$. The greedy algorithm will start with the empty set $A = \varnothing$ and iteratively add the element $k \in V \backslash A$ such that $F(A \cup \{k \})-F(A)$ is maximal. As we show below, it has an $(1-1/e)$-performance guarantee~\cite{nemhauser1978analysis}. Note that this guarantee cannot be improved in general, as it cannot for Max $k$-cover (assuming $P\neq NP$, no polynomial algorithm can provide better approximation guarantees; see more details in \cite{feige1998threshold}).

\begin{proposition}\textbf{(Performance guarantee for submodular function maximization)}
Let $F$ be a non-decreasing submodular function. The greedy algorithm for maximizing $F(A)$ subset to $|A| \leqslant k$ outputs a set $A$ such that 
$$F(A) \geqslant [ 1 -(1-1/k)^k  ]\max_{B \subseteq V,\ |B| \leqslant k} F(B)
\geqslant  (1-1/e) \max_{B \subseteq V,\ |B| \leqslant k} F(B).$$
\end{proposition}
\begin{proof}
We follow the proof of \cite{nemhauser1978analysis,wolsey}. Let $A^\ast $ be a maximizer of $F$ with $k$ elements, and $a_j$ the $j$-th element selected during the greedy algorithm. 
We consider 
$\rho_j = F(\{a_1,\dots,a_j\}) - F(\{a_1,\dots,a_{j-1}\})$.  For a given $j \in \{1,\dots,k\}$,
we denote by $\{b_1,\dots,b_m\}$ the elements of $A^\ast \backslash A_j$ (we must have $k \geqslant m$).
We then have:
\BEAS
 & & F(A^\ast) \\
 & \leqslant \!\!\! & F(A^\ast \cup A_{j-1}) \mbox{ because } F \mbox{ is non-decreasing,} \\
& =  \!\!\!  & F(A_{j-1}) + \sum_{i=1}^m \big[ F( A_{j-1} \cup \{b_1,\dots,b_i\}) - F( A_{j-1} \cup \{b_1,\dots,b_{i-1}\}) \big] \\
& \leqslant  \!\!\! & F(A_{j-1}) + \sum_{i=1}^m \big[ F( A_{j-1} \cup \{ b_i\}) \!-\! F( A_{j-1}  ) \big] \mbox{ by submodularity}, \\
& \leqslant  \!\!\! & F(A_{j-1}) +m  \rho_j \mbox{ by definition of the greedy algorithm}, \\
& \leqslant  \!\!\! & F(A_{j-1}) +k  \rho_j \mbox{ because } m \leqslant k, \\
& =   \!\!\! & \sum_{i=1}^{j-1} \rho_i + k \rho_j \mbox{ by definition of } \rho_i, \ i \in \{1,\dots,j-1\}.
 \EEAS
 Since $F(A_k) = \sum_{i=1}^k \rho_i$, in order to have a lower bound on $F(A_k)$, we can minimize $\sum_{i=1}^k \rho_i$ subject to the $k$ constraints defined above (plus pointwise positivity), i.e.,
 $\sum_{i=1}^{j-1} \rho_i + k \rho_j \geqslant F(A^\ast)$. This is a linear programming problem with $2k-1$ inequality constraints. Let $M$ be the $k \times k$ matrix such that for all $j \in \{1,\dots, k\}$, $(M \rho)_j = \sum_{i=1}^{j-1} \rho_i + k \rho_j $. We need to minimize $\rho^\top 1$ such that $M \rho \geqslant F(A^\ast)$ and $\rho \geqslant 0$. We may define a convex dual optimization problem by introducing Lagrange multipliers $\lambda \in \rb_+^p$:
 \BEAS
 \min_{M \rho \geqslant F(A^\ast), \ \rho \geqslant 0} \rho^\top 1
&  = &    \min_{\rho \geqslant 0} \max_{ \lambda \geqslant 0} \rho^\top 1 + \lambda ^\top ( M \rho - F(A^\ast) 1 ) \\
&  = &    \max_{ \lambda \geqslant 0}  \min_{\rho \geqslant 0} \rho^\top 1 + \lambda ^\top ( M \rho - F(A^\ast)  1) \\
&  = &    \max_{ \lambda \geqslant 0, M^\top \lambda \leqslant 1}   - F(A^\ast)  \lambda^\top 1.
 \EEAS
Since $M$ is lower triangular, we may compute the vector $ \rho = F(A^\ast) M^{-1} 1$ iteratively and easily show by induction that $\rho_j = F(A^\ast)  ( k- 1)^{j-1} k^{-j} $. Similarly, $M^\top$ is upper-triangular and we may compute
$\lambda =    M^{-\top} 1$ as $\lambda_{k-j+1} =  ( k- 1)^{j-1} k^{-j}$.

Since these two vectors happen to be non-negative, they are respectively primal and dual feasible. Since they respectively lead to the same primal and dual objective, this shows that the optimal value of the linear program is equal to
 $F(A^\ast)  1^\top M^{-1} 1 = F(A^\ast)  \sum_{i=1}^k ( 1-1/k)^{i-1} k^{-1} = F(A^\ast)  ( 1 - 1/k)^{k}$, hence the desired result since
 $(1-1/k)^k = \exp( k \log( 1 -1/k) ) \leqslant \exp( k \times (-1/k) ) = 1/e$.
\end{proof}

\paragraph{Extensions.} Given the previous result on cardinality constraints, several extensions have been considered, such as knapsack constraints or matroid constraints~(see~\cite{chekuri2011submodular} and references therein).
Moreover, fast algorithms and improved \emph{online} data-dependent bounds can be further derived~\cite{minoux1978accelerated}.

\section{General submodular function maximization}
\label{sec:max}
In this section, we consider a submodular function and the maximization problem:
\BEQ
\max_{A \subseteq V} F(A).
\EEQ
This problem is known to be NP-hard (note that it includes the maximum cut problem)~\cite{feige2007maximizing}. 
In this section, we present general local optimality results as well as a review of existing approximation guarantees available for non-negative functions.

\paragraph{Local search algorithm.}
Given any set $A$, simple local search algorithms simply consider all sets of the form $A \cup \{k\}$ and $A \backslash \{k\}$ and select the one with largest value of $F$. If this value is lower than $F$, then the algorithm stops and we are by definition at a local minimum. While these local minima do not lead to any global guarantees in general, there is an interesting added guarantee based on submodularity, which we now prove (see more details in~\cite{goldengorin1999data}).

\begin{proposition}\textbf{(Local maxima for submodular function maximization)}
Let $F$ be a submodular function and $A \subseteq V$ such that for all $k \in A$, $ F(A \backslash \{k\}) \leqslant F(A)$ and for all $k \in V \backslash A$, $F(A \cup \{k\}) \leqslant F(A)$. Then for all $B \subseteq A$ and all $B \supset A$, $F(B) \leqslant F(A)$.
\end{proposition}
\begin{proof}
If $B = A \cup \{ i_1,\dots,i_q \}$,
then
\BEAS
F(B) - F(A) &  = &  \sum_{j=1}^q F(A \cup \{ i_1,\dots,i_j \}) -F( A \cup \{ i_1,\dots,i_{j-1} \}) \\[-.1cm]
& \leqslant &  \sum_{j=1}^q F(A \cup \{ i_j \}) - F(A) \leqslant 0,
\EEAS
which leads to the first result. The second one may be obtained from the first one applied to $A \mapsto F(V \backslash A) - F(V)$.
\end{proof}
Note that branch-and-bound algorithms (with worst-case exponential time complexity) may be designed that specifically take advantage of the property above~\cite{goldengorin1999data}.

\paragraph{Formulation using base polyhedron.} Given $F$ and its \lova extension $f$, we  have (the first equality is true since maximization of convex function leads to an extreme point~\cite{rockafellar97}):
\BEAS
\max_{A \subseteq V} F(A)
& = & \max_{ w \in [0,1]^p} f(w) ,\\
& = & \max_{w \in [0,1]^p} \max_{s \in B(F)} w^\top s \mbox{ because of Prop.~\ref{prop:greedy},}\\
& = & \max_{s \in B(F)} s_+(V)  = \max_{s \in B(F)} \frac{1}{2}( s + |s|)(B) \\
& = &  \frac{1}{2} F(V) + \frac{1}{2} \max_{s \in B(F) }\| s \|_1 .
\EEAS
Thus submodular function maximization may be seen as finding the \emph{maximum} $\ell_1$-norm point in the base polyhedron (which is not a convex optimization problem). See an illustration in Figure~\ref{fig:geom}.

\paragraph{Non-negative submodular function maximization.}
When the function is known to be non-negative (i.e., with non-negative values $F(A)$ for all $A \subseteq V$), then simple local search algorithm have led to theoretical guarantees~\cite{feige2007maximizing,chekuri2011submodular,focs2012}.
It has first been shown in \cite{feige2007maximizing} that a $1/2$ relative bound  could not be improved in general if a polynomial number of queries of the submodular function is used.  Recently,~\cite{focs2012} has shown that a simple strategy that maintains two solutions, one starting from the empty set and one starting from the full set, and updates them using local moves, achieves a ratio of 1/3, while a randomized local move leads to the optimal approximation ratio of 1/2.
 
 However, such theoretical guarantees should be considered with caution, since in this similar setting of maximizing a non-negative submodular function, 
selecting a random subset already achieves at least $1/4$ of the optimal value (see the simple 
argument outlined by~\cite{feige2007maximizing} that simply uses the convexity of the \lova extension and conditioning): having theoretical guarantees do not necessarily imply that an algorithm is doing anything subtle.

 Interestingly, in the analysis of submodular function maximization, a new extension from $\{0,1\}^p$ to $[0,1]^p$ has emerged, the \emph{multi-linear extension}~\cite{chekuri2011submodular}, which is equal to 
 $$\tilde{f}(w) = \sum_{A \subseteq V} F(A) \prod_{i \in A} w_i \prod_{i \in V \backslash A} ( 1- w_i).$$
 It is equal to the expectation of $F(B)$ where $B$ is a random subset where the $i$-th element is selected with probability $w_i$ (note that this interpretation allows the computation of the extension through sampling), and this is to be contrasted with the \lova extension, which is equal to the expectation of $F( \{ w \geqslant u \})$ for $u$  a random variable with uniform distribution in $[0,1]$. The multi-linear extension is neither convex nor concave but has marginal convexity properties that may be used for the design and analysis of algorithms for maximization problems~\cite{chekuri2011submodular,focs2011}.

\section{Difference of submodular functions$^\ast$}
\label{sec:ds}
In regular continuous optimization, differences of convex functions play an important role, and appear in various disguises, such as DC-programming~\cite{horst1999dc}, concave-convex procedures~\cite{yuille2003concave}, or majorization-minimization algorithms~\cite{hunter2004tutorial}. They allow the expression of any continuous optimization problem with natural descent algorithms based on upper-bounding a concave function by its tangents.

In the context of combinatorial optimization, \cite{narasimhan2006submodular}  has shown that a similar situation holds for differences of submodular functions. We now review these properties.

\begin{figure}
\begin{center}
\includegraphics[scale=.9]{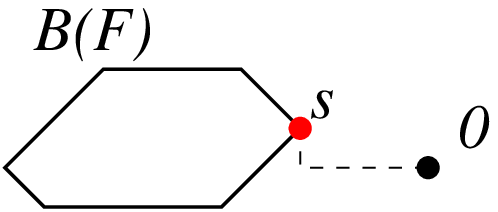} \hspace*{1cm}
\includegraphics[scale=.9]{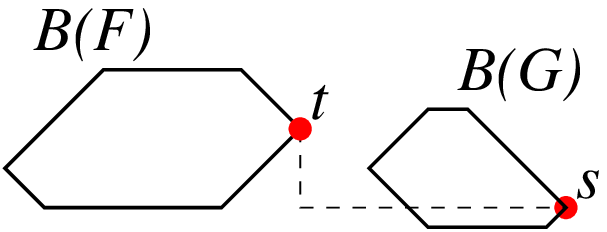}
\end{center}
\caption{Geometric interpretation of submodular function maximization (left) and optimization of differences of submodular functions (right). See text for details.}
\label{fig:geom}
\end{figure}

\paragraph{Formulation of any combinatorial optimization problem.} 
Let $F: 2^V \to \rb$ be any set-function, and $H$ a \emph{strictly} submodular function, i.e., a function such that 
$$\alpha = \min_{ A \subseteq V} \min_{ i,j \in V \backslash A}
 - H(A \cup \{i,j\} ) + H( A \cup \{i\})+  H( A \cup \{j\}) - H(A) > 0. $$
A typical example would be $H(A) = -\frac{1}{2}|A|^2$, where $\alpha = 1$. If 
$$
\beta = \min_{ A \subseteq V} \min_{ i,j \in V \backslash A}
 - F(A \cup \{i,j\} ) + F( A \cup \{i\})+  F( A \cup \{j\}) - F(A) 
$$
is non-negative,
then $F$ is submodular (see Prop.~\ref{prop:second}). If $\beta < 0$, 
 then
$F(A) - \frac{\beta}{\alpha} H(A)$ is submodular, and thus, we have $F(A) = [ F(A) - \frac{\beta}{\alpha} H(A) ]  - 
[ -   \frac{\beta}{\alpha} H(A) ]$, which is a difference of two submodular functions. Thus any combinatorial optimization problem may be seen as a difference of submodular functions (with of course non-unique decomposition). However, some problems, such as subset selection in \mysec{subset}, or more generally discriminative learning of graphical model structure may naturally be seen as such~\cite{narasimhan2006submodular}.

\paragraph{Optimization algorithms.}
Given two submodular set-functions $F$ and $G$, we consider the following iterative algorithm, starting from a subset~$A$:
\BNUM
\item Compute modular lower-bound $B \mapsto s(B)$, of $G$ which is tight at $A$: this might be done by using the greedy algorithm of Prop.~\ref{prop:greedy} with $w = 1_A$. Several orderings of components of $w$ may be used (see~\cite{narasimhan2006submodular} for more details).
\item Take $A$ as any minimizer of $B \mapsto F(B) - s(B)$, using any algorithm of \mychap{sfm}.
\ENUM
It converges to a local minimum, in the sense that at convergence to a set $A$, all sets $A \cup \{k\}$ and $A \backslash \{k\}$ have smaller function values. 

\paragraph{Formulation using base polyhedron.} We can give a similar geometric interpretation than for submodular function maximization; given $F,G$ and their \lova extensions $f$, $g$, we  have:
\BEAS
\min_{A \subseteq V} F(A) - G(A) 
& = & \min_{A \subseteq V } \min_{s \in B(G)} F(A) - s(A) \mbox{ because of Prop.~\ref{prop:greedy},} \\
& = & \min_{w  \in [0,1]^p } \min_{s \in B(G)} f(w) - s^\top w   \mbox{ because of Prop.~\ref{prop:minlova},} \\
& = & \min_{s \in B(G)} \min_{w  \in [0,1]^p }  f(w) - s^\top w  \\
& = & \min_{s \in B(G)} \min_{w  \in [0,1]^p }  \max_{t \in B(F)} t^\top w  - s^\top w   \\
& = & \min_{s \in B(G)}  \max_{t \in B(F)} \min_{w  \in [0,1]^p }   t^\top w  - s^\top w \mbox{ by strong duality,} \\
& = &  \min_{s \in B(G)}  \max_{t \in B(F)} (t-s)_-(V) \\
&= & 
  \frac{F(V) - G(V)}{2}  - \frac{1}{2}  \min_{s \in B(G)}  \max_{t \in B(F)} \| t - s \|_1.
  \EEAS
Thus optimization of the difference of submodular functions is related to the Hausdorff distance between $B(F)$ and $B(G)$: this distance is equal to $
\max \big\{\min_{s \in B(G)}  \max_{t \in B(F)} \| t - s \|_1
,\min_{s \in B(G)}  \max_{t \in B(F)} \| t - s \|_1 \big\}
$ (see, e.g.,~\cite{munkres1984elements}). See also an illustration in Figure~\ref{fig:geom}.

\chapter{Experiments}
\label{chap:experiments}

In this chapter, we provide illustrations of the optimization algorithms described earlier, for submodular function minimization (\mysec{exp-sfm}), as well as for convex optimization problems: quadratic separable ones such as the ones used for proximal methods or within submodular function minimization (\mysec{exp-prox}),  an application of sparsity-inducing norms to wavelet-based estimators (\mysec{exp-wavelet}), and some simple illustrative experiments of recovery of one-dimensional signals using structured regularizers (\mysec{exp-graph}). The Matlab code for all these experiments may be found at \url{http://www.di.ens.fr/~fbach/submodular/}.

\section{Submodular function minimization}
\label{sec:exp-sfm}
We compare several   approaches to submodular function minimization described in \mychap{sfm}, namely:

\begin{list}{\labelitemi}{\leftmargin=1.1em}
   \addtolength{\itemsep}{-.2\baselineskip}

\item[--] \textbf{MNP}: the minimum-norm-point algorithm to maximize $-\frac{1}{2} \| s\|_2^2$ over $s \in B(F)$, described in \mysec{minnorm}.
\item[--] \textbf{Simplex}: the simplex algorithm described in \mysec{simplexsfm}.

\item[--] \textbf{ACCPM}: the analytic center cutting plane technique from  \mysec{accpm-sfm}.
\item[--] \textbf{ACCPM-Kelley}: the analytic center cutting plane technique presented in \mysec{accpm-sfm}, with a large weight $\alpha = 1000$, which emulates the simplicial method from~\mysec{simplicial}.

\item[--] \textbf{Ellipsoid}: the ellipsoid algorithm described in \mysec{sfm_ellipsoid}.

\item[--] \textbf{SG}: the projected gradient descent algorithm to minimize $f(w)$ over $w \in [0,1]^p$, described in \mysec{approxsfm}, with two variants, a step-size proportional to $1/\sqrt{t}$ (denoted ``SG-1/$t^{1/2}$'') and using the approximation of Polyak's rule (``SG-Polyak'') described in \mysec{approxsfm}.

\item[--] \textbf{CG-LS}: the conditional gradient algorithm to maximize $-\frac{1}{2} \| s\|_2^2$ over $s \in B(F)$, with line search, described in \mysec{approxsfm}.
\item[--] \textbf{CG-2/(t+1)}: the conditional gradient algorithm to maximize $-\frac{1}{2} \| s\|_2^2$ over $s \in B(F)$, with step size $2/(t+1)$, described in \mysec{approxsfm}.
\end{list}
From all these algorithms, we may obtain sets $A \subseteq V$ and dual certificates $s \in B(F)$; the quantity $F(A) - s_-(V)$ (see Prop.~\ref{prop:dualmin}) then serves as a certificate of optimality. In order to distinguish primal and dual approximate optimality we report $F(A) - {\rm Opt } $ and ${\rm Opt } +s_-(V)$, where ${\rm Opt }  = \min_{A \subseteq V}F(A)$ is the optimal value of the problem.

\begin{figure}
\begin{center}
\includegraphics[scale=.5]{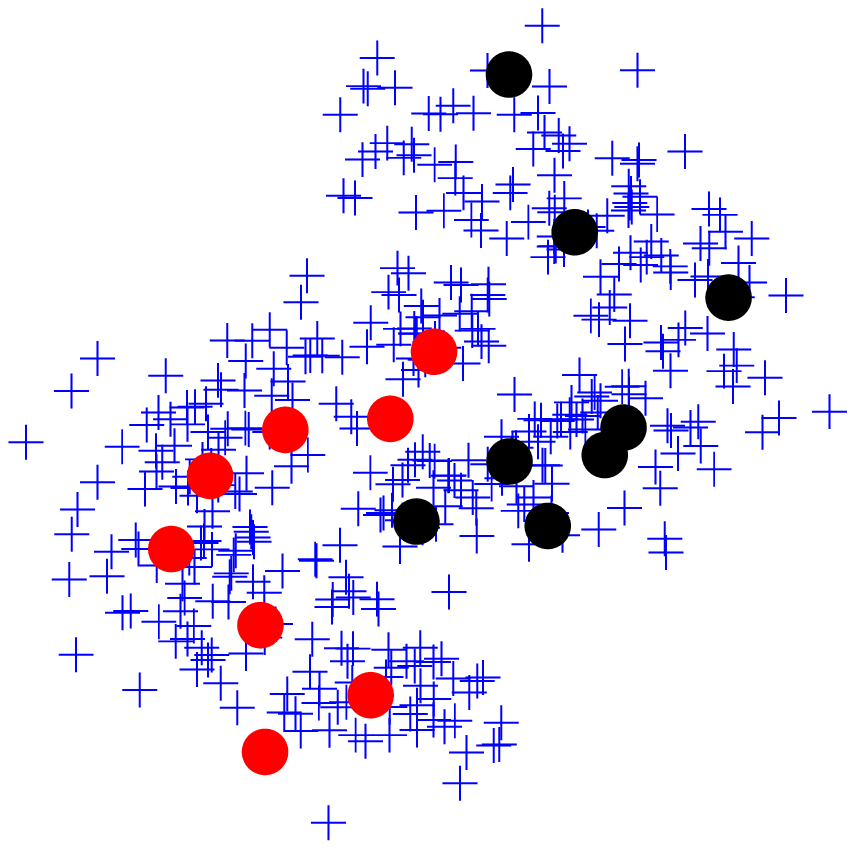} \hspace*{1cm}
\includegraphics[scale=.5]{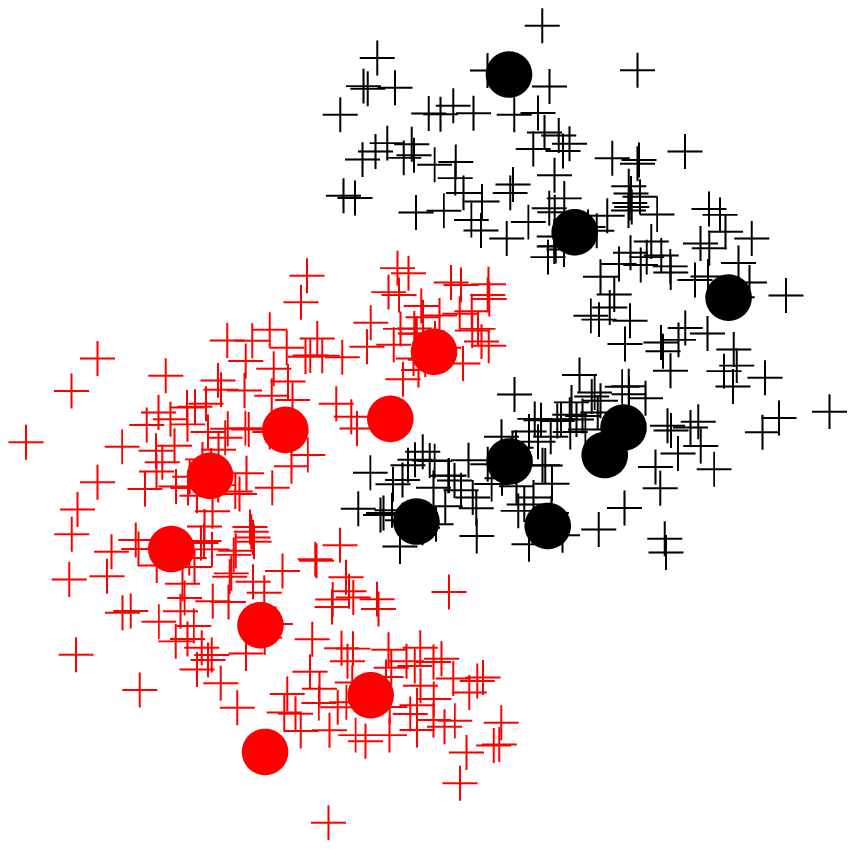}

\vspace*{-.75cm}

\end{center}
\caption{Examples of semi-supervised clustering : (left) observations, (right) results of the semi-supervised clustering algorithm based on submodular function minimization, with eight labelled data points. Best seen in color.}
\label{fig:twomoon-exp}
\end{figure}

\begin{figure}
\begin{center}
\hspace*{-1.1cm}
\includegraphics[scale=.5]{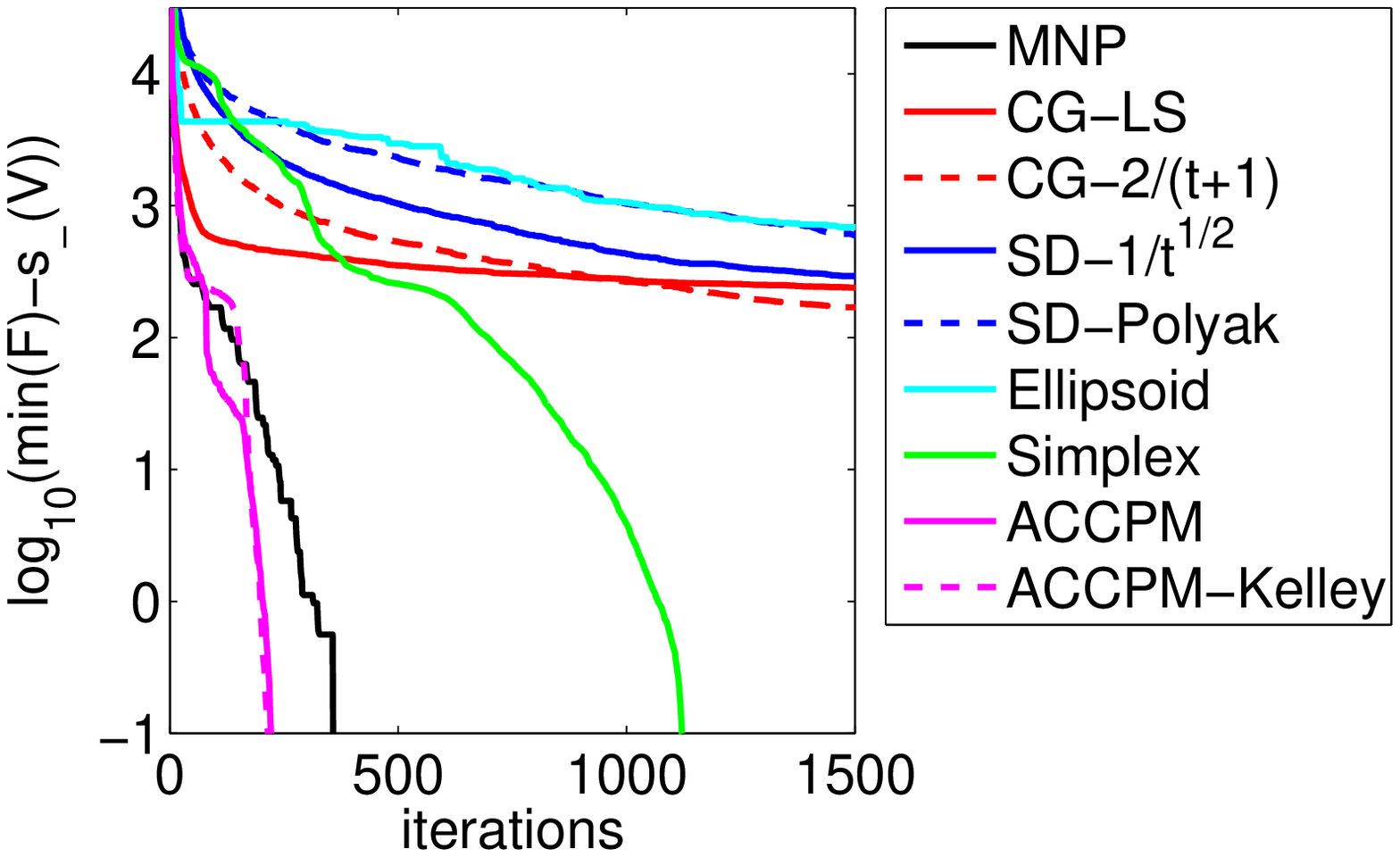} \hspace*{-.1cm}
\includegraphics[scale=.5]{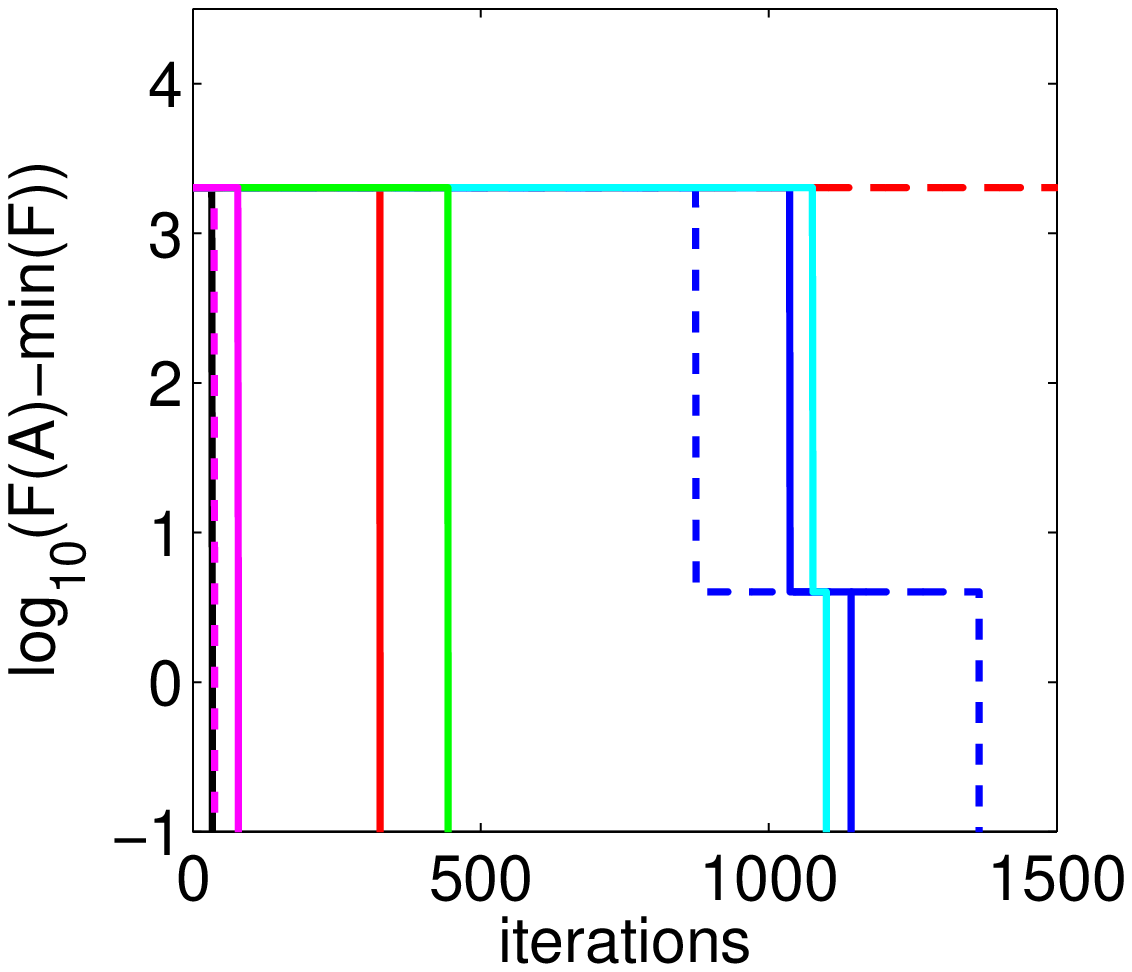}
\hspace*{-1.1cm}
\end{center}

\vspace*{-.5cm}

\caption{Submodular function minimization results for ``Genrmf-wide'' example: (left) 
optimal value minus dual function values in log-scale vs.~number of iterations vs.~number of iteration. (Right) Primal function values minus optimal value in log-scale vs.~number of iterations. Best seen in color.}
\label{fig:genrmf-wide}
\end{figure}

We test these algorithms on five data sets:
\begin{list}{\labelitemi}{\leftmargin=1.1em}
   \addtolength{\itemsep}{-.2\baselineskip}

\item[--] \textbf{Two-moons} (clustering with mutual information criterion): we generated data from a standard synthetic examples in semi-supervised learning (see Figure~\ref{fig:twomoon-exp}) with $p=400$ data points, and 16 labelled data points, using the method presented in \mysec{entropies}, based on the mutual information between two Gaussian processes (with a Gaussian-RBF kernel).

 \item[--] \textbf{Genrmf-wide} and  \textbf{Genrmf-long} (min-cut/max-flow standard benchmark): following~\cite{fujishige2006minimum}, we generated cut problem using the generator GENRMF available from DIMACS challenge\footnote{ The First DIMACS international algorithm implementation challenge: The core
experiments (1990), available at \url{ftp://dimacs.rutgers.edu/pub/netßow/generalinfo/core.tex}.}. Two types of network were generated, ``long'' and ``wide'', with respectively $p=575$ vertices and 2390 edges, and $p=430$ and 1872 edges (see~\cite{fujishige2006minimum} for more details).

\item[--] \textbf{Speech}: we consider a dataset used by~\cite{lin2011optimal,jegelka2011-fast-approx-sfm},
in order to solve the  problem of  finding a maximum size speech corpus with bounded vocabulary ($p=800$). The submodular function is of the form $F(A) = | V \backslash A | + \lambda \sqrt{ G(A) }$, where $G(A)$ is a set-cover function from~\mysec{covers} (this function is submodular because of Prop.~\ref{prop:compo}).

\item[--] \textbf{Image segmentation}: we consider the minimum-cut problem used in \myfig{imageseg} for segmenting a $ 50 \times 50$ image, i.e., $p=2500$.
\end{list}

In Figures~\ref{fig:genrmf-wide},~\ref{fig:genrmf-long},~\ref{fig:moons},~\ref{fig:speech} and~\ref{fig:image}, we compare the algorithms on the five datasets. We denote by ${ \rm Opt}$ the optimal value of the optimization problem, i.e.,
${ \rm Opt} = \min_{w \in \rb^p} f(w)  = \max_{s \in B(F)} s_-(V)$.
On the left plots, we display the dual suboptimality, i.e, $\log_{10}( { \rm Opt} - s_-(V))$. In the right plots we display the primal suboptimality $\log_{10}(F(B)-{ \rm Opt} )$. Note that in all the plots, we plot the best values achieved so far, i.e., we make all curves non-increasing.

\begin{figure}
\begin{center}
\hspace*{-1.1cm}
\includegraphics[scale=.5]{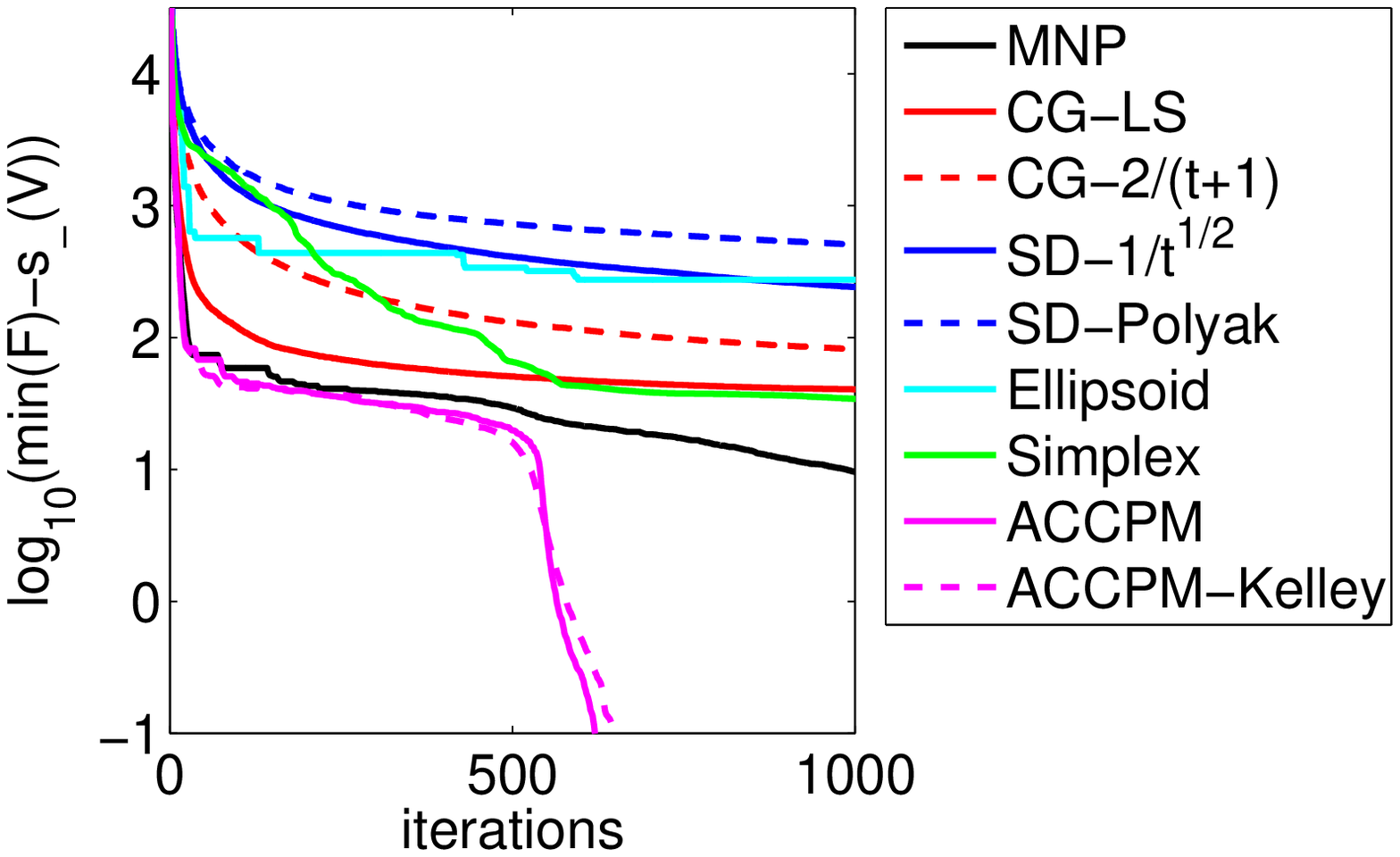} \hspace*{-.1cm}
\includegraphics[scale=.5]{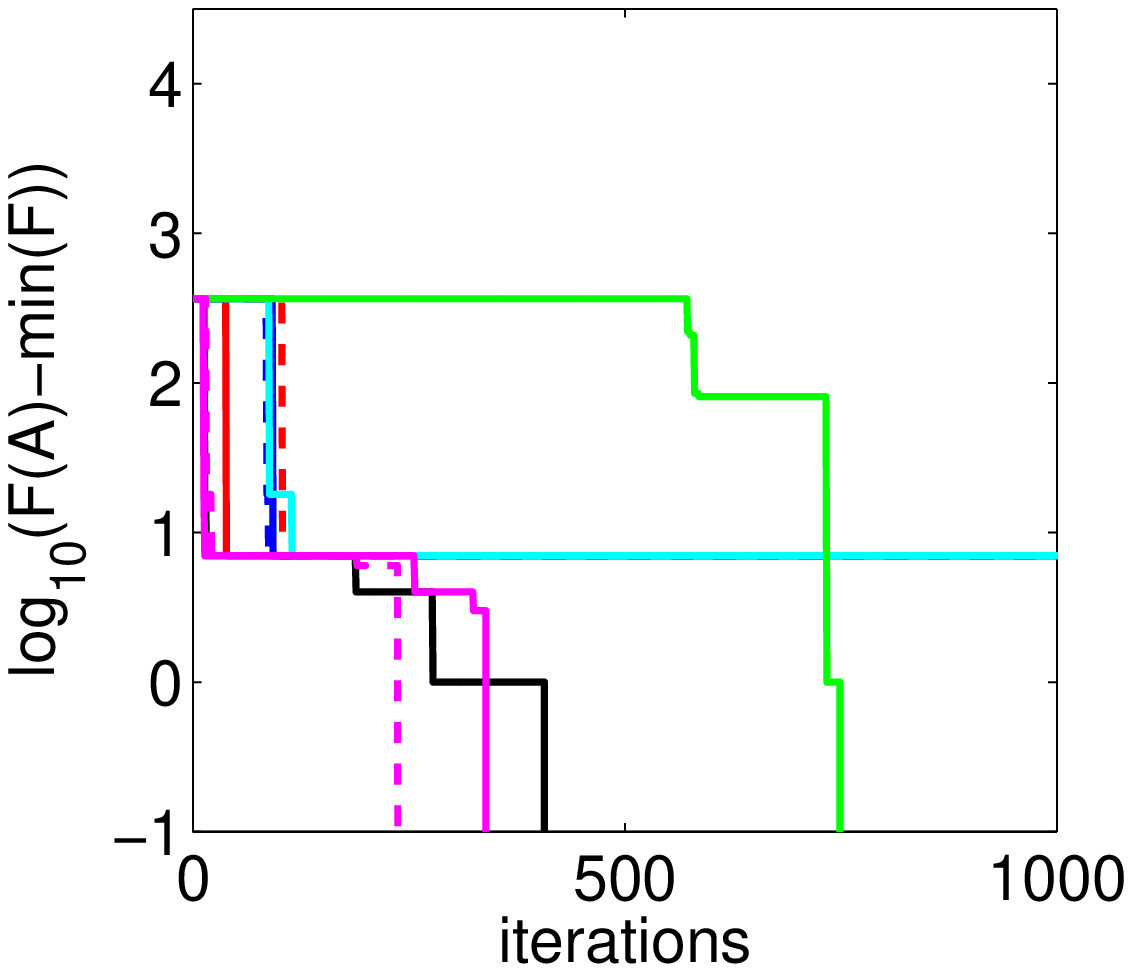}
\hspace*{-1.1cm}
\end{center}

\vspace*{-.5cm}

\caption{Submodular function minimization results for ``Genrmf-long'' example: (left) 
optimal value minus dual function values in log-scale vs.~number of iterations vs.~number of iteration. (Right) Primal function values minus optimal value in log-scale vs.~number of iterations. Best seen in color.}
\label{fig:genrmf-long}
\end{figure}

Since all algorithms perform a sequence of greedy algorithms (for finding maximum weight bases), we measure performance in numbers of iterations in order to have an implementation-independent measure. Among the tested methods, some algorithms (subgradient and conditional gradient) have no extra cost, while others involved solving linear systems.

On all datasets, the achieved primal function values are in fact much lower than the certified values (i.e., primal suboptimality converges to zero much faster than dual suboptimality). In other words, primal values $F(A)$ are quickly very good and iterations are just needed to sharpen the certificate of optimality.

Among the algorithms, only ACCPM has been able to obtain both  close to optimal primal and dual solutions in all cases, while the minimum-norm-point does so in most cases, though more slowly. Among methods that may find the global optimum, the simplex exhibit poor performance. The simpler methods (subgradient and conditional gradient)  perform worse, with an advantage for the conditional gradient with line search, which is able to get good primal values quicker (while dual certificates converge slowly).

\begin{figure}
\begin{center}
\hspace*{-1.1cm}
\includegraphics[scale=.5]{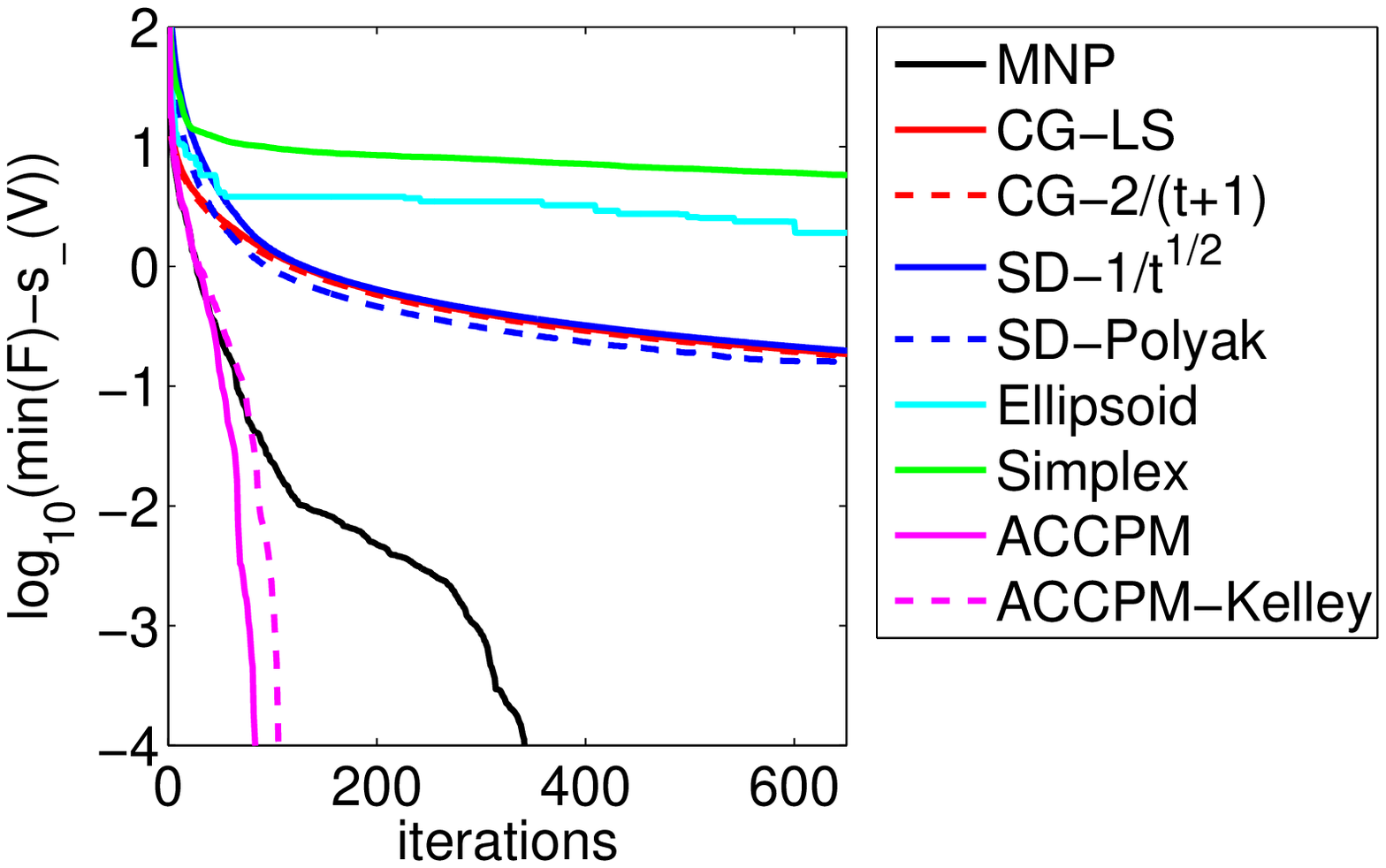} \hspace*{-.1cm}
\includegraphics[scale=.5]{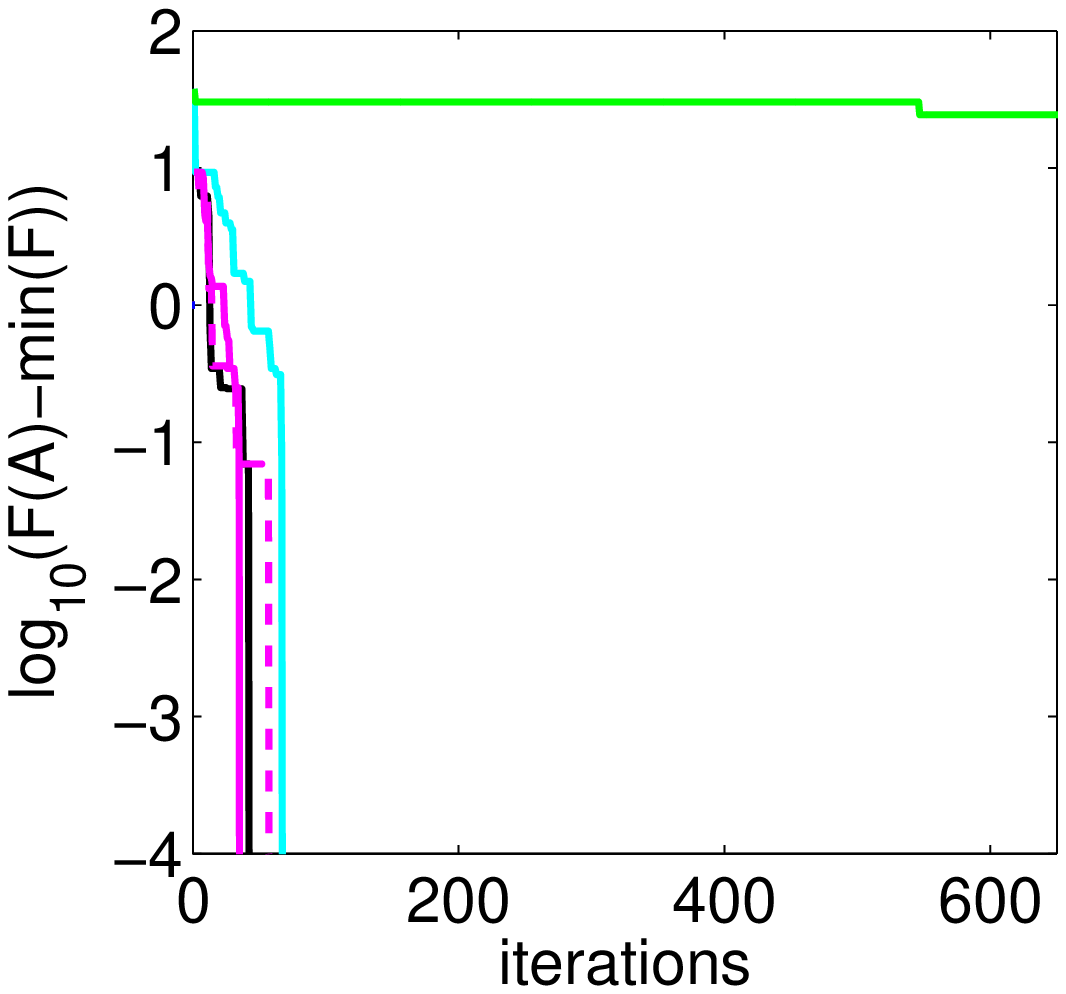}
\hspace*{-1.1cm}
\end{center}

\vspace*{-.5cm}

\caption{Submodular function minimization results for ``Two-moons'' example: (left) 
optimal value minus dual function values in log-scale vs.~number of iterations vs.~number of iteration. (Right) Primal function values minus optimal value in log-scale vs.~number of iterations. Best seen in color.}
\label{fig:moons}
\end{figure}

%
%
%
%
%
%
%

\begin{figure}
\begin{center}
\hspace*{-1.1cm}
\includegraphics[scale=.5]{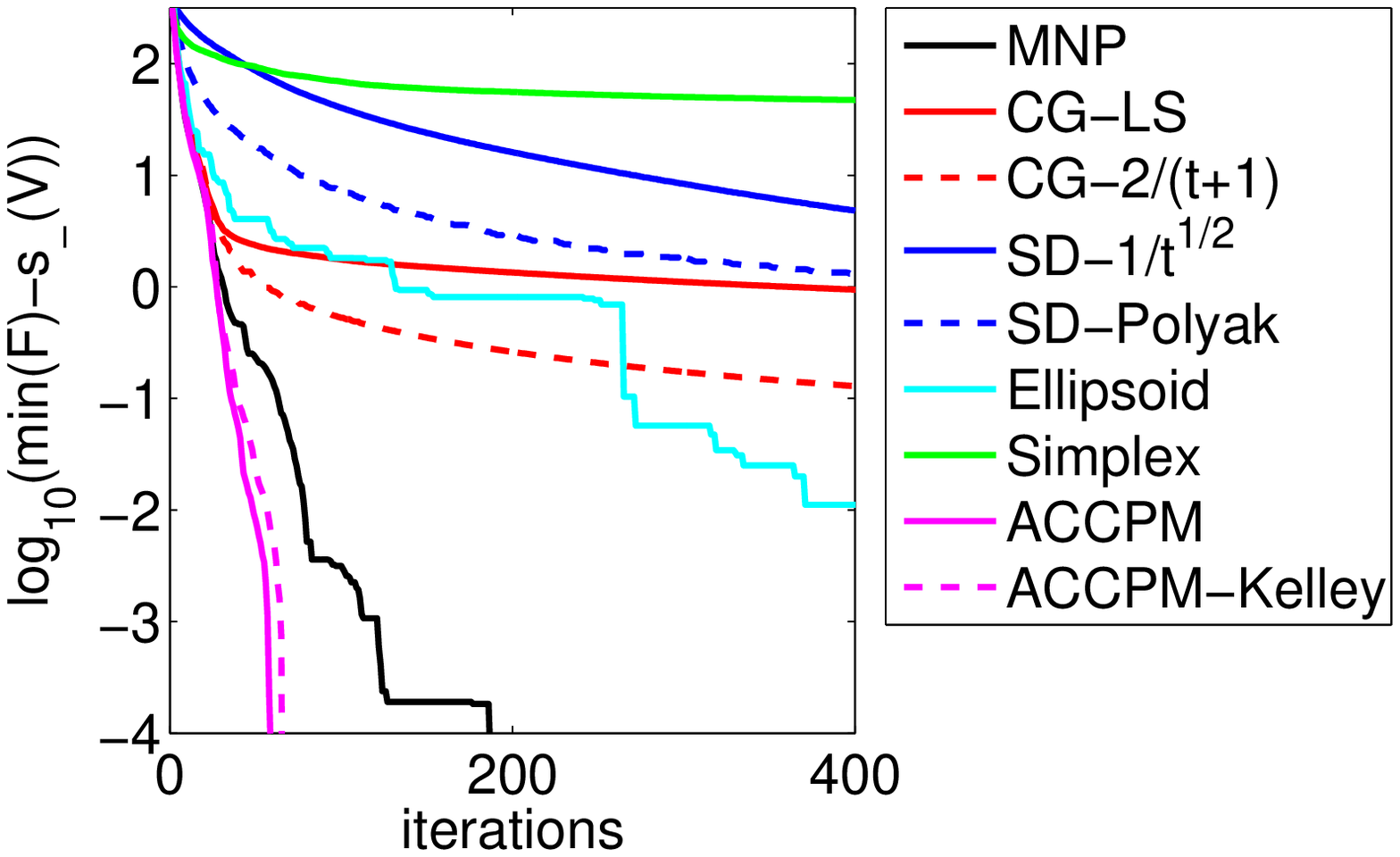} \hspace*{-.1cm}
\includegraphics[scale=.5]{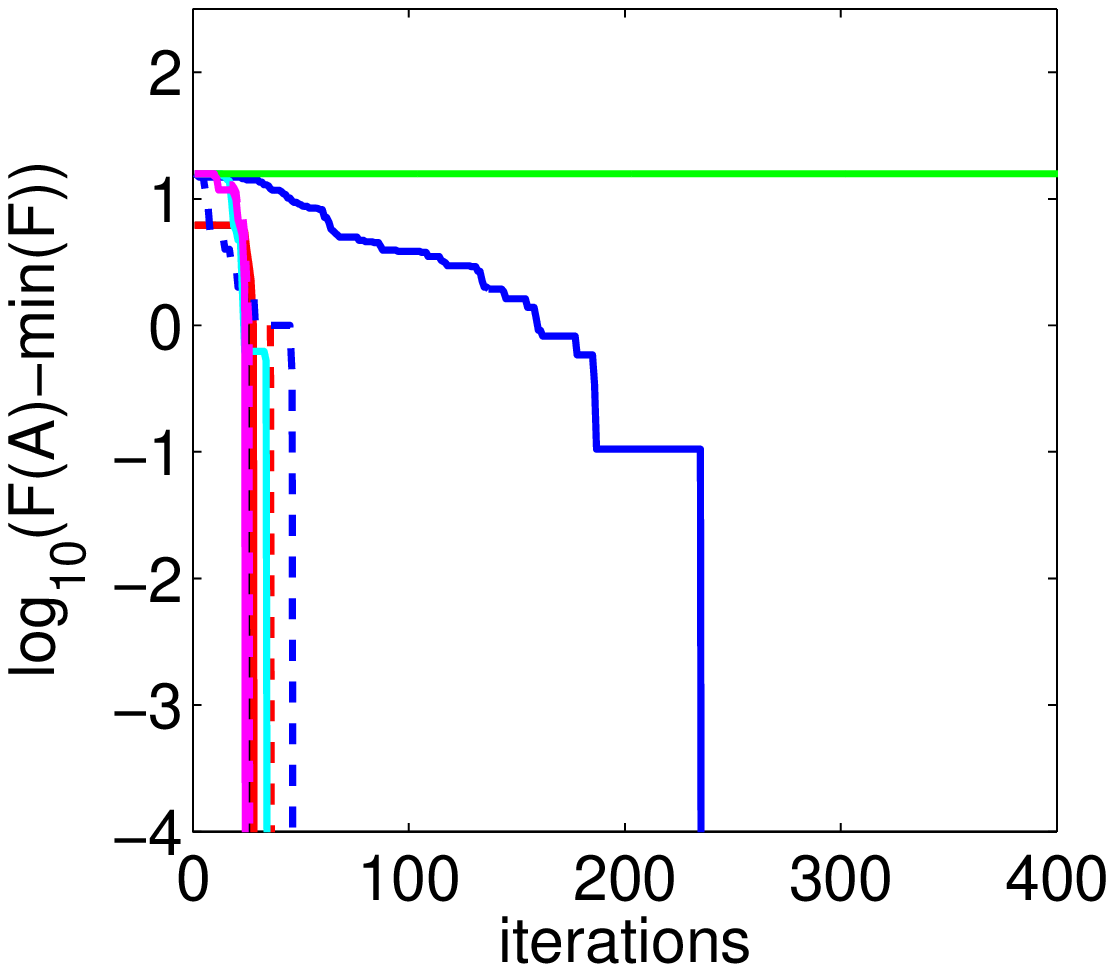}
\hspace*{-1.1cm}
\end{center}

\vspace*{-.5cm}

\caption{Submodular function minimization results for ``Speech'' example: (left) 
optimal value minus dual function values in log-scale vs.~number of iterations vs.~number of iteration. (Right) Primal function values minus optimal value in log-scale vs.~number of iterations. Best seen in color.}
\label{fig:speech}
\end{figure}

%
%
%
%
%
%

\begin{figure}
\begin{center}
\hspace*{-1.1cm}
\includegraphics[scale=.5]{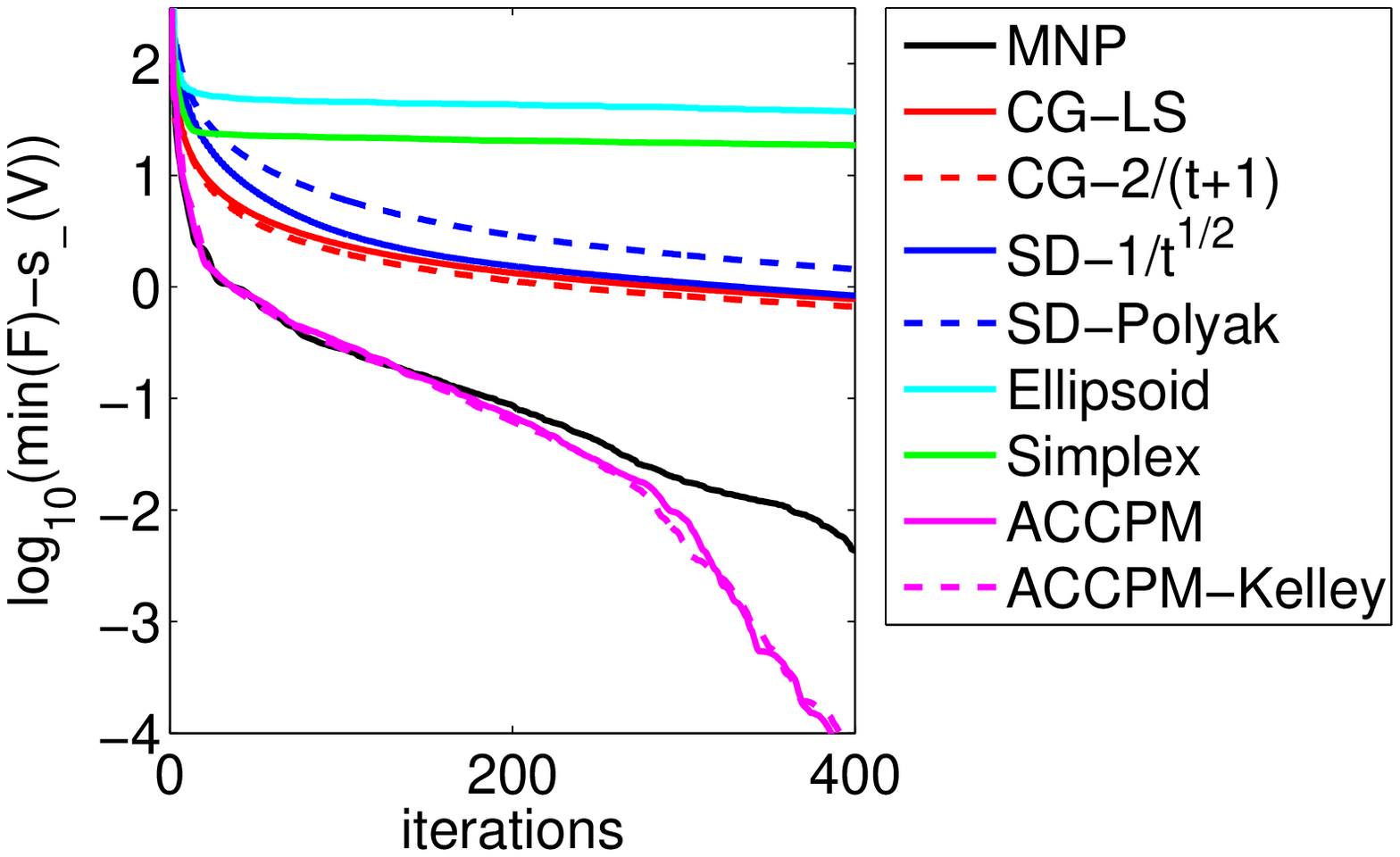} \hspace*{-.1cm}
\includegraphics[scale=.5]{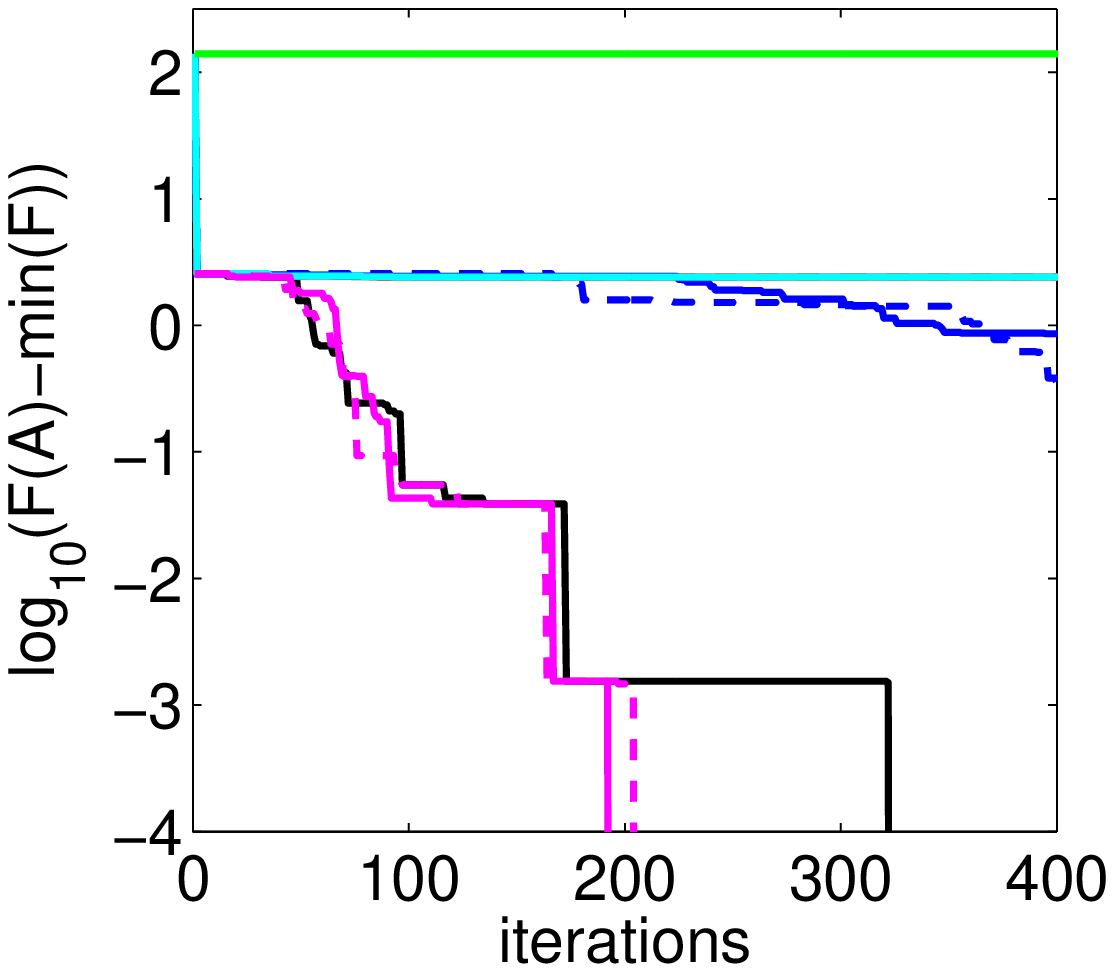}
\hspace*{-1.1cm}
\end{center}

\vspace*{-.5cm}

\caption{Submodular function minimization results for ``image'' example: (left) 
optimal value minus dual function values in log-scale vs.~number of iterations vs.~number of iteration. (Right) Primal function values minus optimal value in log-scale vs.~number of iterations.
Best seen in color.}
\label{fig:image}
\end{figure}

\section{Separable optimization problems}
\label{sec:exp-prox}

In this section, we compare the iterative algorithms outlined in \mychap{prox-algo} for minimization of quadratic separable optimization problems, on two of the problems related to submodular function minimization from the previous section (i.e., minimizing $f(w) + \frac{1}{2} \| w\|_2^2$). In Figures~\ref{fig:genrmf-wide-prox} and \ref{fig:genrmf-long-prox}, we compare three algorithms on two datasets, namely the mininum-norm-point algorithm, and two versions of conditional gradient (with and without line search). On the left plots, we display the primal suboptimality $\log_{10}(f(w) + \frac{1}{2} \|w\|^2_2 - 
\min_{v \in \rb^p} f(v) + \frac{1}{2} \|v\|_2^2)$ while in the right plots we display dual suboptimality, for the same algorithms.   As in \mysec{exp-sfm}, on all datasets, the achieved primal function values are in fact much lower than the certified values, a situation common in convex optimization. On all datasets, the min-norm-point algorithm achieved quickest small duality gaps. On all datasets, among the two conditional gradient algorithms, the version  with line-search perform significantly better than the algorithm with decaying step sizes. Note also, that while the conditional gradient algorithm is not finitely convergent, its performance is not much worse than the minimum-norm-point algorithm, with smaller running time complexity per iteration.
Moreover, as shown on the right plots, the ``pool-adjacent-violator'' correction is crucial in obtaining much improved primal candidates.

\begin{figure}

\vspace*{-.5cm}

\begin{center}
\hspace*{-1cm}
\includegraphics[scale=.5]{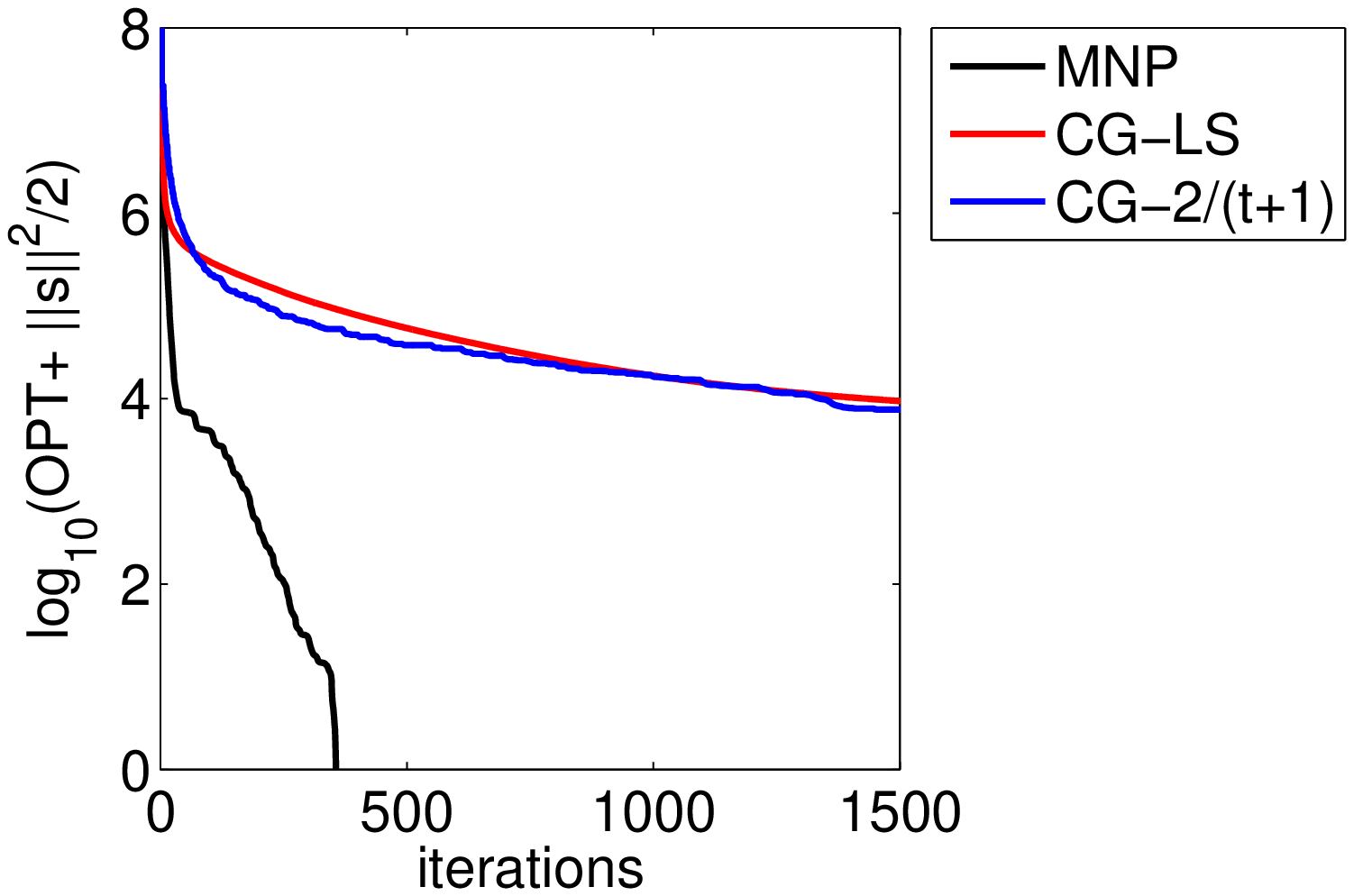} \hspace*{-.1cm}
\includegraphics[scale=.5]{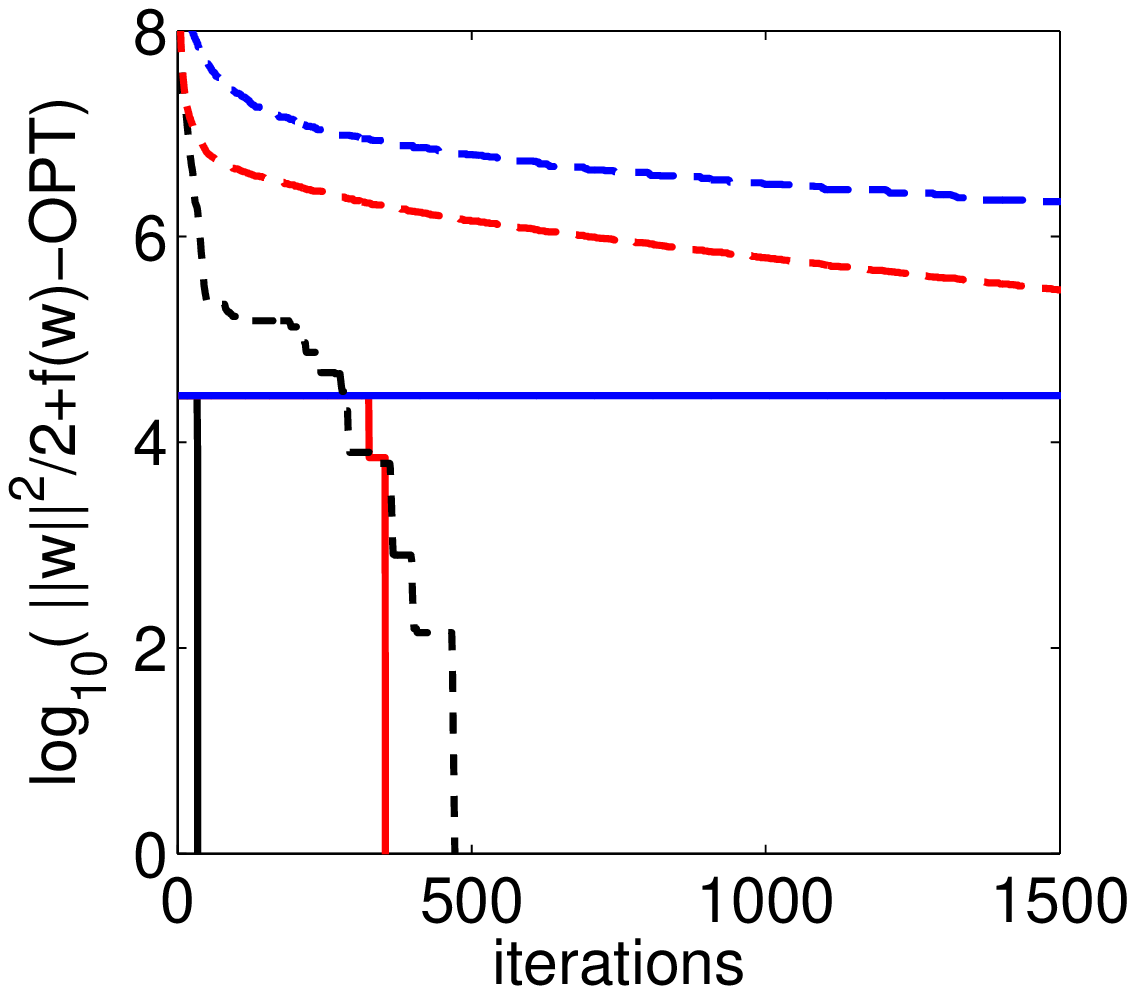}
\hspace*{-1cm}
\end{center}

\vspace*{-.5cm}

\caption{Separable optimization problem  for ``Genrmf-wide'' example. (Left) optimal value minus dual function values in log-scale vs.~number of iterations.
(Right) Primal function values minus optimal value in log-scale vs.~number of iterations, in dashed, before the ``pool-adjacent-violator'' correction. Best seen in color.}
\label{fig:genrmf-wide-prox}
\end{figure}

\begin{figure}
\begin{center}
\hspace*{-1cm}
\includegraphics[scale=.5]{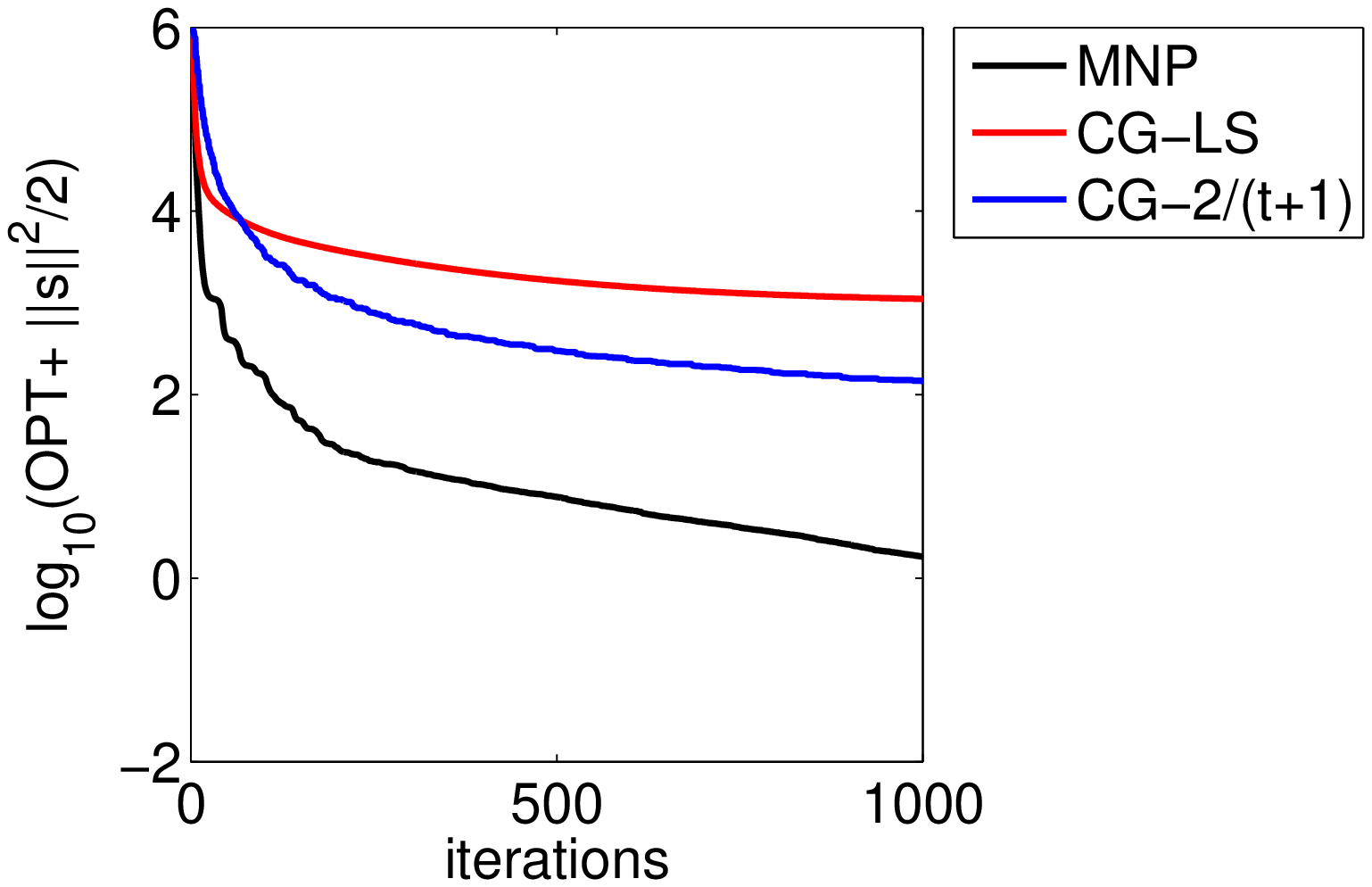} \hspace*{-.1cm}
\includegraphics[scale=.5]{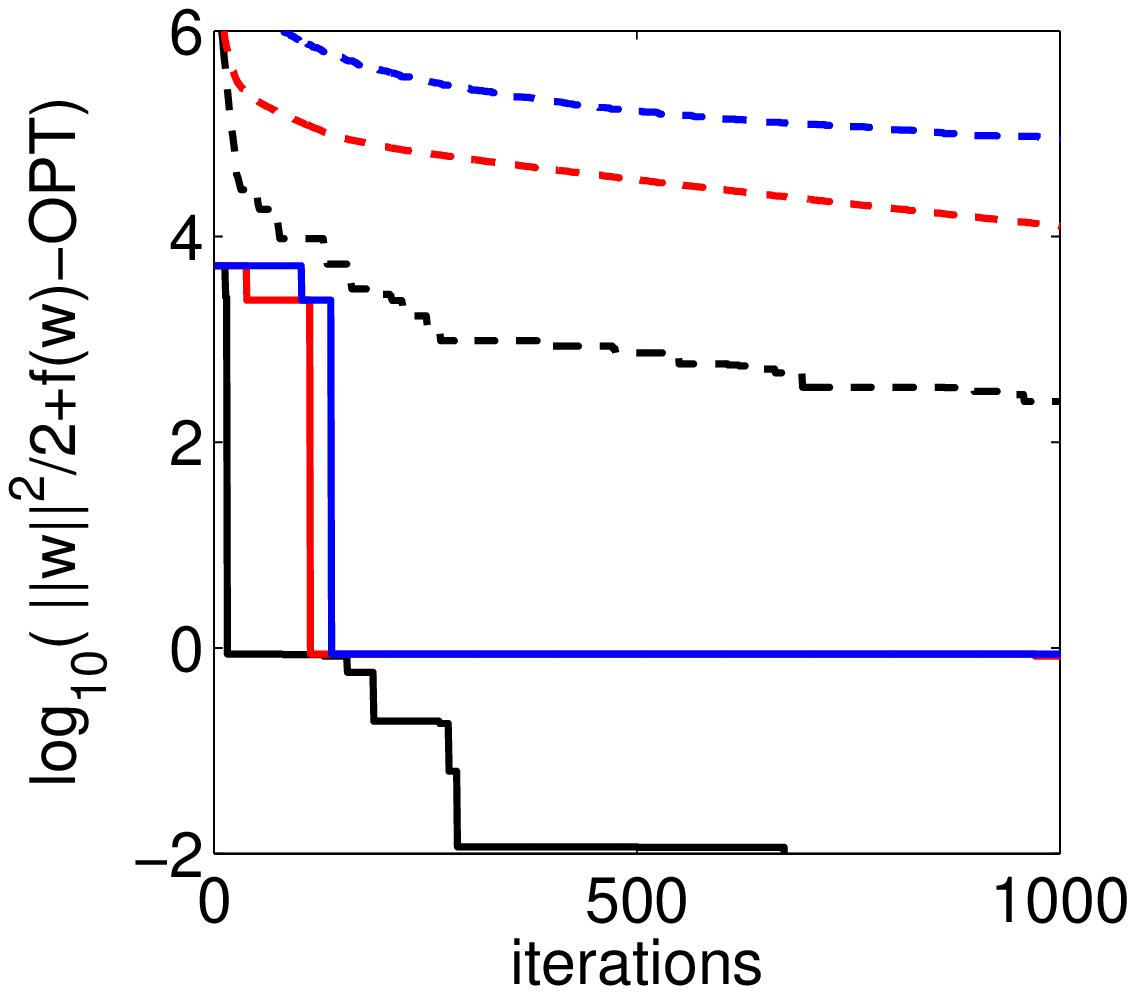}
\hspace*{-1cm}
\end{center}

\vspace*{-.5cm}

\caption{Separable optimization problem  for ``Genrmf-long'' example. (Left) optimal value minus dual function values in log-scale vs.~number of iterations.
(Right) Primal function values minus optimal value in log-scale vs.~number of iterations, in dashed, before the ``pool-adjacent-violator'' correction. Best seen in color.}
\label{fig:genrmf-long-prox}
\end{figure}

\section{Regularized least-squares estimation}
\label{sec:exp-wavelet}
In this section, we illustrate the use of the \lova extension in the context of sparsity-inducing norms detailed in \mysec{sparse}, with the submodular function defined in Figure~\ref{fig:tree}, which is based on a tree structure among the $p$ variables, and encourages variables to be selected after their ancestors. We do not use any weights, and thus $F(A)$ is equal to the cardinality of the union of all ancestors $\rm{Anc}(A)$ of nodes indexed by elements of $A$.

 Given a probability distribution $(x,y)$ on $[0,1] \times \rb$, we aim to estimate $g(x) = \mathbb{E} ( Y | X\! =\!  x )$, by a piecewise constant function. Following~\cite{cap}, we consider a Haar wavelet estimator with maximal depth $d$. That is, given the Haar wavelet, defined on $\rb$ as
$\psi(t) = 1_{[0,1/2)}(t) - 1_{[1/2,1)}(t)$, we consider the functions $\psi_{ij}(t)$ defined as
$\psi_{ij}(t) = \psi( 2^{i-1} t - j  )$, for $i=1,\dots,d$ and $j \in \{0,\dots, 2^{i-1}\!-\!1\}$, leading to $p = 2^d \! - 1$ basis functions. These functions come naturally in a binary tree structure, as shown in Figure~\ref{fig:wavelet-tree} for $d=3$. Imposing a tree-structured prior enforces that a wavelet with given support is selected only after all larger supports are selected; this avoids the selection of isolated wavelets with small supports.

\begin{figure}
\begin{center}
\includegraphics[scale=.65]{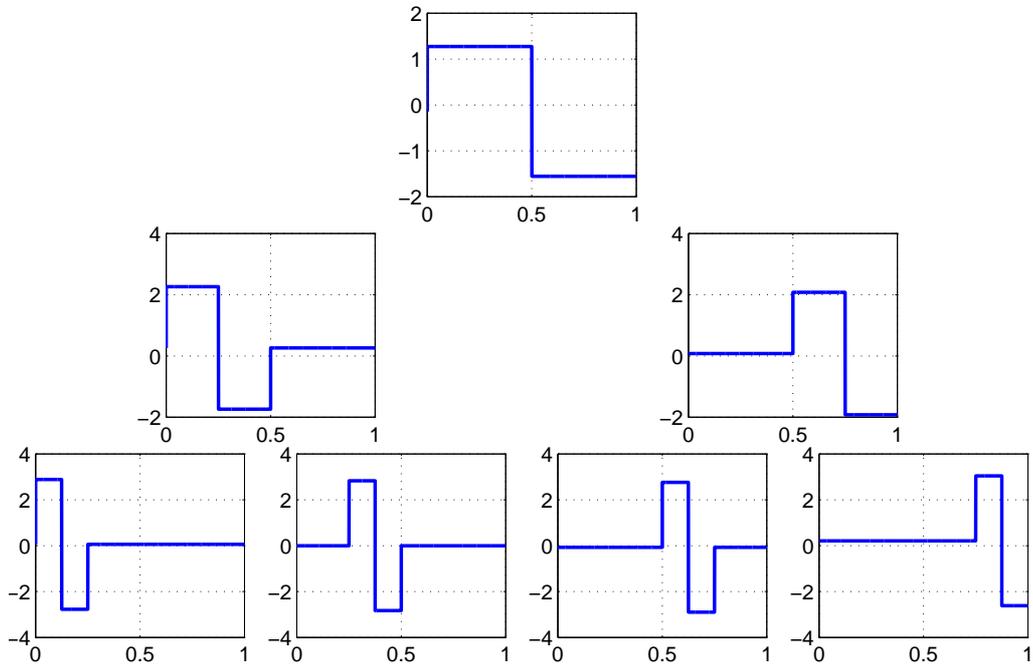}

\vspace*{-.5cm}

\end{center}
\caption{Wavelet binary tree ($d=3$). See text for details.}
\label{fig:wavelet-tree}
\end{figure}

We consider random inputs $x_i \in [0,1]$, $i=1,\dots,n$, from a uniform distribution and compute $y_i = \sin (20 \pi x_i^2) + \varepsilon_i$, where $\varepsilon_i$ is Gaussian with mean zero and standard deviation $0.1$. We consider the optimization problem
\BEQ
\label{eq:exp}
\min_{w \in \rb^p,  b \in \rb } \frac{1}{2n} \sum_{k=1}^n \bigg(
y_k - \sum_{i=1}^d \sum_{j=0}^{2^{i-1}-1} w_{ij} \psi_{ij}(x_k) - b
\bigg)^2 + \lambda R(w),
\EEQ
where $b$ is a constant term and $R(w)$ is a regularization function.
 In  Figure~\ref{fig:wavelet-results}, we compare several regularization terms, namely $R(w) = \frac{1}{2} \| w\|_2^2$ (ridge regression), $R(w) = \| w\|_1$ (Lasso) and $R(w) = \Omega(w) = f(|w|)$ defined from the hierarchical submodular function  $F(A) = { \rm Card}( {\rm Anc}(A) )$. For all of these, we select $\lambda$ such that the generalization performance is maximized, and compare the estimated functions. The hierarchical prior leads to a lower estimation error with fewer artefacts.

\begin{figure}
\begin{center}
\hspace*{-.75cm}
\includegraphics[scale=.57]{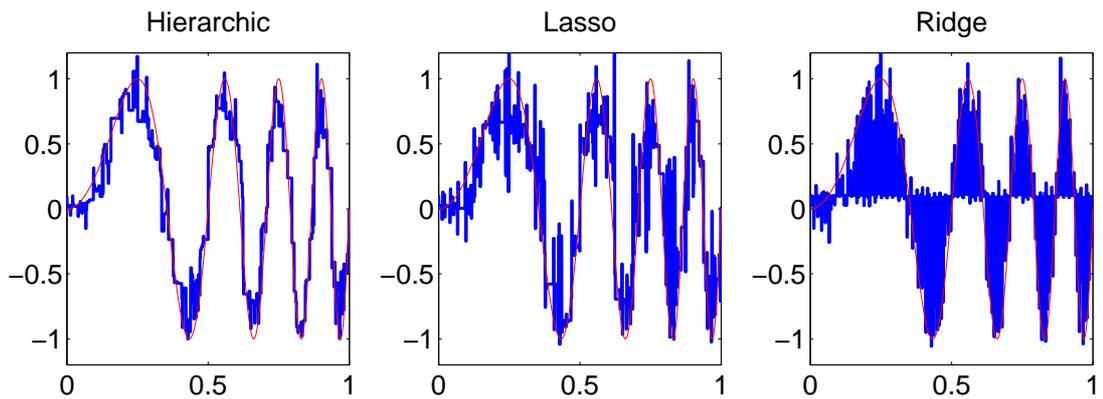} 
\end{center}

\vspace*{-.5cm}

\caption{Estimation with wavelet trees: (left) hierarchical penalty (MSE=0.04), (middle) Lasso (MSE=0.11),  (right) ridge regression (MSE=0.33). See text for details. }
\label{fig:wavelet-results}
\end{figure}

In this section, our goal is also to compare several optimization schemes to minimize \eq{exp} for this particular example (for more simulations on larger-scale examples with similar conclusions, see~\cite{fot,mairal2011b,Jenatton2010b,shapinglevelsets}).
We compare in Figure~\ref{fig:wavelet-results-speed} several ways of solving the regularized least-squares problem:
\begin{list}{\labelitemi}{\leftmargin=1.1em}
   \addtolength{\itemsep}{-.2\baselineskip}

\item[--] \textbf{Prox. hierarchical}: we use a dedicated proximal operator based on the composition of local proximal operators~\cite{Jenatton2010b}. This strategy is only applicable for this submodular function.
\item[--] \textbf{Prox. decomposition}: we use the algorithm of \mysec{decomp} which uses the fact that for any vector $t$, $F-t$ may be minimized by dynamic programming~\cite{Jenatton2010b}. We also consider a modification (``abs''), where the divide-and-conquer strategy is used directly on the symmetric submodular polyhedron.
\item[--] \textbf{Prox-MNP}: we use the generic method which does not use any of the structure, with the modification (``abs'') that operates directly on $|P|(F)$ and not on $B(F)$. For these two algorithms, since the min-norm-point algorithm is used many times for similar inputs, we use warm restarts to speed up the algorithm. We also report results without such warm restarts (the algorithm is then much slower).
\item[--] \textbf{subgrad-descent}: we use a generic method which does not use any of the structure, and minimize directly \eq{exp} by subgdradient descent, using the best possible (in hindsight) step-size sequence proportional to $1/\sqrt{t}$.
\item[--] \textbf{Active-primal}: primal active-set method presented in \mysec{activeQPLS}, which require to be able to minimize the submodular function efficiently (possible here). When $F$ is the cardinality function, this corresponds to the traditional active-set algorithm for the Lasso.
\item[--] \textbf{Active-dual}: \emph{dual} active-set method presented in \mysec{activeQPLS}, which simply requires to access the submodular function through the greedy algorithm. 
Since our implementation is unstable (due to the large linear ill-conditioned systems that are approximately solved), it is only used for the small problem where $p=127$.

\end{list}

As expected, in Figure~\ref{fig:wavelet-results-speed}, we see that the most efficient algorithm is the dedicated proximal algorithm (which is usually not available except in particular cases like the tree-structured norm), while the methods based on submodular functions fare correctly, with an advantage for methods using the structure (i.e., the decomposition method, which is only applicable when submodular function minimization is efficient) over the generic method based on the min-norm-point algorithm (which is always applicable). Note that the primal active-set method (which is only applicable when minimizing $F$ is efficient) is  competitive while the dual active-set method takes many iterations to make some progress and then converge quickly.

Interestingly,applying the divide-and-conquer strategies directly to $|P|(F)$ and not through $B(F)$ is more efficient while this is the opposite when the min-norm-point algorithm is used.

\begin{figure}
\begin{center}
 \includegraphics[scale=.57]{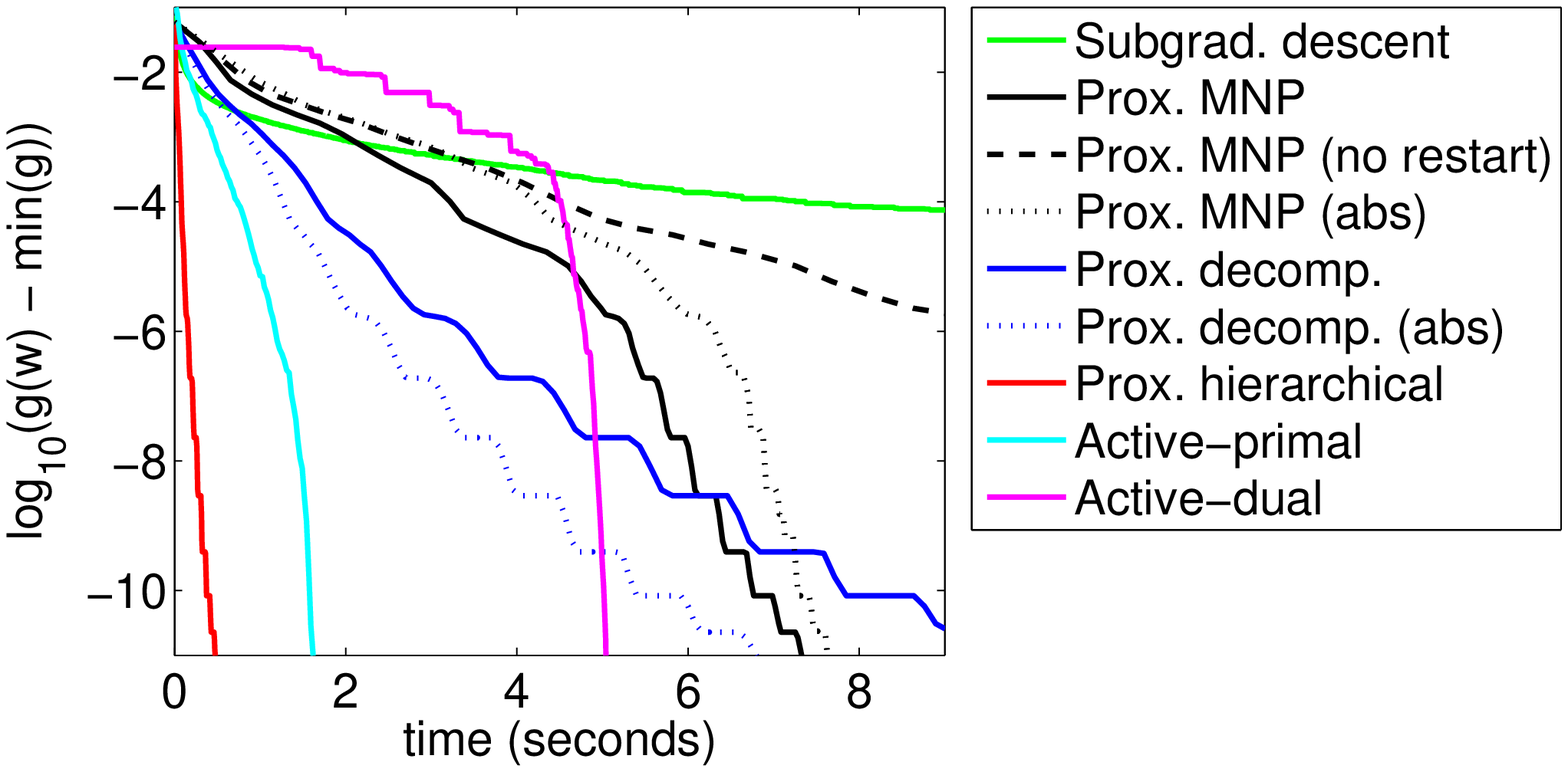} 
\\
\includegraphics[scale=.57]{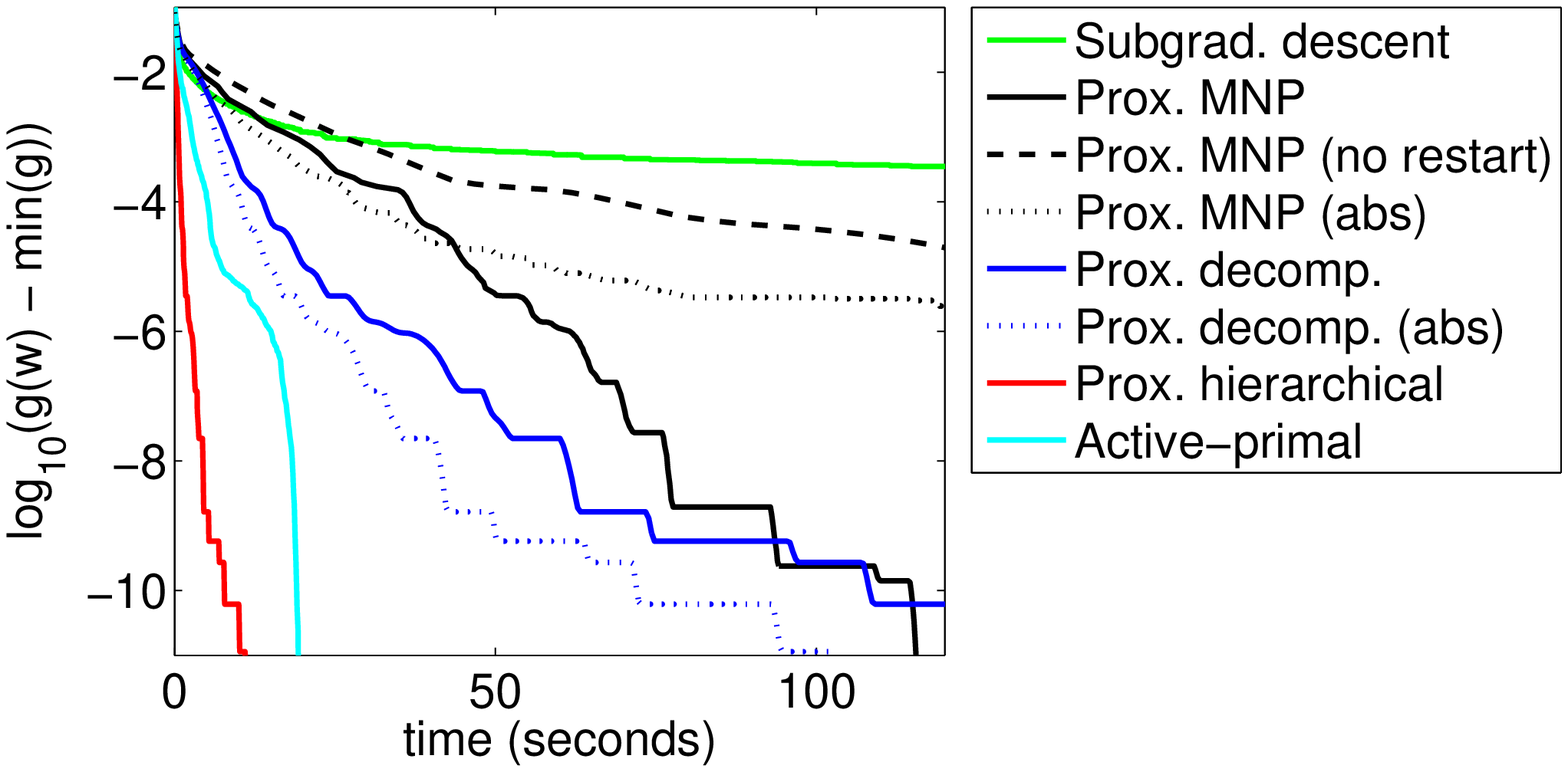} 
 \end{center}
 
 \vspace*{-.5cm}
 
\caption{Running times for convex optimization for a regularized problem: several methods are compared; top: small-scale problems $(p=127)$, bottom: medium-scale problem $p=511)$. See text for details. Best seen in color.}
\label{fig:wavelet-results-speed}
\end{figure}

\section{Graph-based structured sparsity}
\label{sec:graph}
\label{sec:exp-graph}

In this monograph, we have considered several sparsity-inducing norms related to graph topologies. In this section, we consider a chain graph, i.e., we are looking for sparsity-inducing terms that take into account the specific  ordering of the components of our vector $w$. We consider the following regularizers $r(w)$:

\begin{list}{\labelitemi}{\leftmargin=1.1em}
   \addtolength{\itemsep}{-.2\baselineskip}
\item[--] Total variation + $\ell_1$-penalty: $r(w) = \sum_{k=1}^{p-1} | w_k - w_{k+1} | + \frac{1}{2} \|w\|_1$. This regularizer aims at finding piecewise-constant signals with the additional prior that the 0-level-set is large.

 \item[--] Laplacian quadratic form + $\ell_1$-penaly: $r(w) = \frac{1}{2}\sum_{k=1}^{p-1} | w_k - w_{k+1} |^2 +  \| w\|_1$. This regularizer aims at finding smooth and sparse signals.
 
\item[--] $\ell_\infty$-relaxation of the function $F(A) = |A| + {\rm range}(A) + \mbox{cst}$. This function aims at selecting contiguous elements in a sequence, but may suffer from additional biases due to the extra extreme points of the $\ell_\infty$-norm.
\item[--] $\ell_2$-relaxation of the function $F(A) = |A| + {\rm range}(A) + \mbox{cst}$. This function aims at selecting contiguous elements in a sequence. We consider the relaxation outlined in \mysec{lprelax}.

\end{list}

In \myfig{denoising1}, we compare these four regularizers on three signals that clearly exhibit the different behaviors. We selected the largest regularization parameter that leads to optimal sparsity pattern selection. We can make the following observations:  (a) the $\ell_\infty$ and $\ell_2$ relaxation are invariant by sign flips of the inputs (hence the results between the first two rows are simply flipped), and are adapted to signals with contiguous non-zero elements but with potentially different signs; (b) the total variation is particularly well adapted to the first row, while the Laplacian smoothing is best adapted to the third row; (c) in the third row, the $\ell_\infty$-relaxation adds an additional bias while the $\ell_2$-relaxation does not.

\begin{figure}
\begin{center}
  \includegraphics[scale=.65]{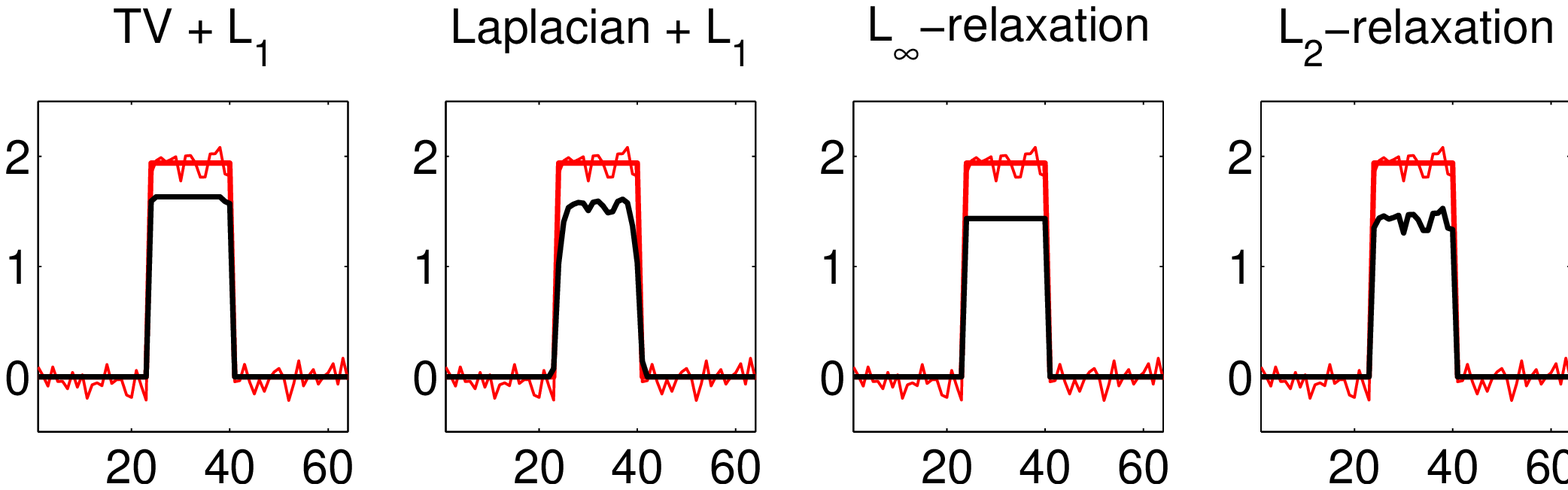} 
\\
  \includegraphics[scale=.65]{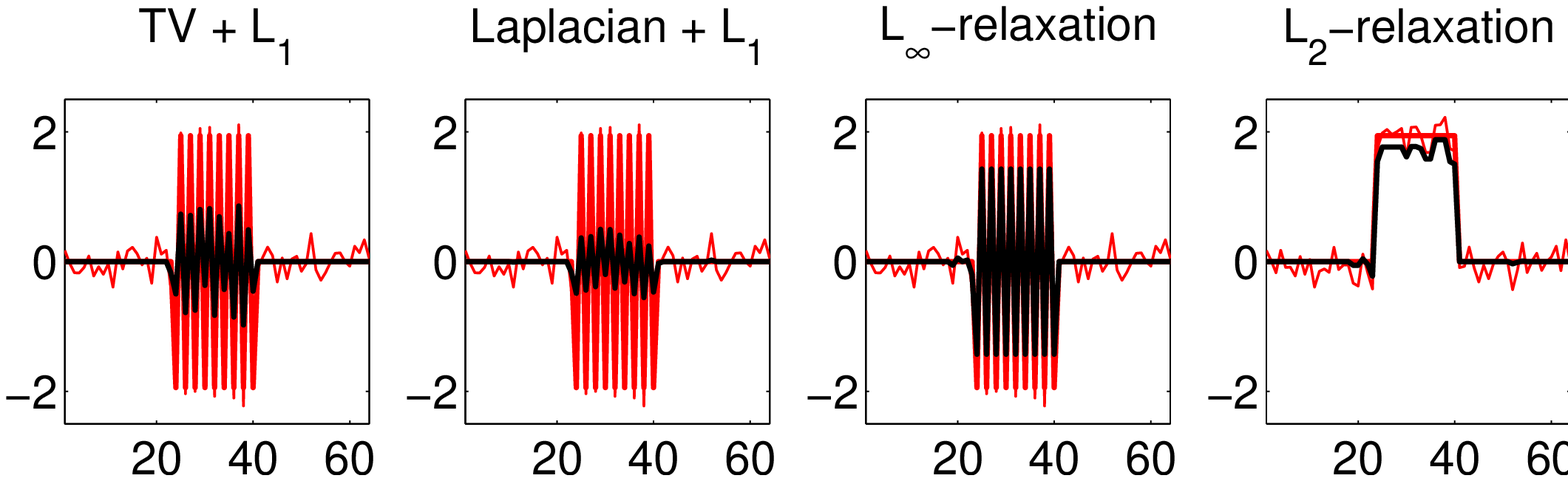} 
  \\
  \includegraphics[scale=.65]{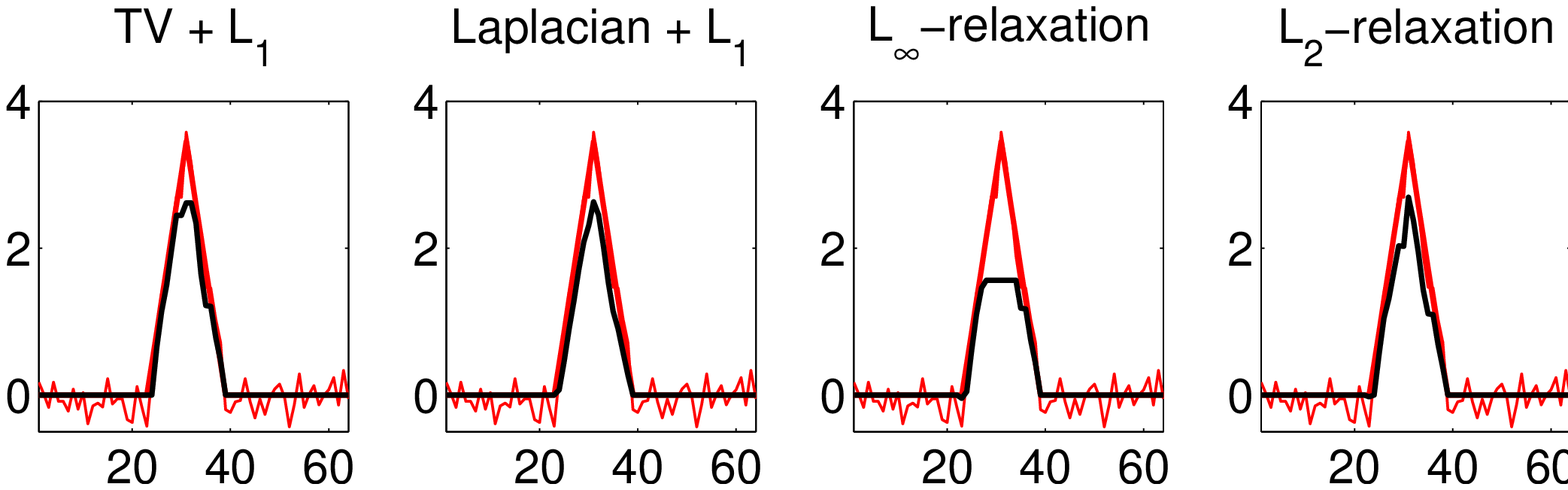} 
   \end{center}
 
 \vspace*{-.5cm}
 
\caption{Denoising results for three different signals. Comparison of several sparsity-inducing norms. Red: signal (bold) and noisy version (plain). Black: recovered signal. See text for details. }
\label{fig:denoising1}
\end{figure}

 \chapter{Conclusion}

In this monograph, we have explored various properties and applications of submodular functions. We have emphasized primarily on a convex perspective, where the key concepts are the \lova extension and the associated submodular and base polyhedra. 

Given the numerous examples involving such functions, the analysis and algorithms presented in this monograph allow the unification of several results in convex optimization, in particular in situations where combinatorial structures are considered. 

\paragraph{Related work on submodularity.}
In this monograph, we have focused primarily on the relationships between convex optimization and submodular functions. However, submodularity is an active area of research in computer science, and more generally in algorithms and machine learning, which goes beyond such links:

\begin{list}{\labelitemi}{\leftmargin=1.1em}
   \addtolength{\itemsep}{-.3\baselineskip}
\item[--] \textbf{Online optimization}: In this monograph, we have focused on offline methods for the optimization problem: at any given iteration, the submodular function is accessed through the value oracle. In certain situations common in machine learning, the submodular function is a sum of often simple submodular functions, and online learning techniques can be brought to bear to either speed up the optimization or provide online solutions. This can be done both for the maximization of submodular functions~\cite{streeter2007online,guillory2011online} or their minimization~\cite{hazan2009online}

\item[--] \textbf{Learning submodular functions}: We have assumed that the set-functions we were working with were given, and hence manually built for each given application. It is of clear interest to learn the submodular functions directly from data. See, e.g.,~\cite{balcan2011learning,goemans2009approximating,stobbe2012learning} for several approaches.
 
 \item[--] \textbf{Discrete convex analysis}: We have presented a link between combinatorial optimization and convex optimization based on submodular functions. The theory of discrete convex analysis goes beyond submodularity and the minimization or maximization of submodular functions. See, e.g.,~\cite{murota1987discrete}.
  
  \item[--] \textbf{Beyond submodular minimization or maximization}: concepts related to submodularity may also be used for sequential decision problems, in particular through the development of adaptive submodularity~(see~\cite{golovin2011adaptive} and references therein).

\end{list}

\paragraph{Open questions.}

Several questions related to submodular analysis are worth exploring, such as:

\begin{list}{\labelitemi}{\leftmargin=1.1em}
   \addtolength{\itemsep}{-.3\baselineskip}
\item[--] \textbf{Improved complexity bounds and practical performance for submodular function minimization}:
the currently best-performing algorithms (min-norm-point and analytic center cutting-planes) do not come with convergence bounds. Designing
 efficient  optimization algorithms for submodular function minimization, with both good computational complexity bounds and practical performance, remains a challenge. On a related note, active-set methods such that the simplex and min-norm-point algorithms may in general take exponentially many steps~\cite{klee:1972}. Are submodular polyhedra special enough so that the complexity of these algorithms (or variations thereof) is polynomial for submodular function minimization?

\item[--] \textbf{Lower bounds}:
We have presented algorithms for approximate submodular function minimization with convergence rate of the form $O(1/\sqrt{t})$ where $t$ is the number of calls to the greedy algorithm; it would be interesting to obtain better rates or show that this rate is optimal, like in non-smooth convex optimization (see, e.g.,~\cite{Nesterov2004}).

\item[--] \textbf{Multi-way partitions}:
Computer vision applications have focused also on multi-way partitions, where an image has to be segmented in more than two regions~\cite{boykov2001fast}. The problem cannot then be solved in polynomial-time~\cite{chekuri2011approximation}, and it would be interesting to derive frameworks with good practical performance on large graphs with millions of nodes and attractive approximation guarantees.

\item[--] \textbf{Semi-definite programming}:
The current theory of submodular functions essentially considers links between combinatorial optimization problems and linear programming, or linearly constrained quadratic programming; it would be interesting to extend submodular analysis using more modern convex optimization tools such as semidefinite programming.

\item[--] \textbf{Convex relaxation of submodular maximization}: while submodular function minimization may be seen naturally as a convex problem, this is not the case of maximization; being able to provide convex relaxation would notably allow a unified treatment of differences of submodular functions, which then include all set-functions.
 \end{list}

\appendix

\chapter{Review of Convex Analysis and Optimization}
	\label{app:convex}

In this appendix, we review relevant concepts from convex analysis in Appendix~\ref{app:convexana}. For more details, see~\cite{boyd,Bertsekas,borwein2006caa,rockafellar97}.
We also consider
a detailed convexity-based proofs for the max-flow min-cut theorem in Appendix~\ref{app:maxflow}, and a derivation of the pool-adjacent-violators algorithm in Appendix~\ref{app:pava}.

\section{Convex analysis}
\label{app:convexana}
In this section, we review extended-value convex functions, Fenchel conjugates, Fenchel duality, dual norms, gauge functions and polar sets.

\paragraph{Extended-value convex functions.}
In this monograph, we consider  functions defined on $\rb^p$ with values in $\rb \cup \{ +\infty\}$; the domain of $f$ is defined to be the set of vectors in $\rb^p$ such that $f$ has finite values. Such an ``extended-value'' function is said to be convex if its domain is convex and $f$ restricted to its domain (which is a real-valued function) is convex.

Throughout this monograph, we denote by $w \mapsto I_C(w)$ the indicator function of the convex set $C$, defined as $0$ for $w \in C$ and $+\infty$ otherwise; this defines a convex function and allows constrained optimization problems to be treated as unconstrained optimization problems. In this monograph, we always assume that $f$ is a \emph{proper} function (i.e., has non-empty domain). A function is said closed if for all $\alpha \in \rb$, the set $\{ w \in \rb^p, \ f(w) \leqslant \alpha \}$ is a closed set. We only consider \emph{closed} proper functions in this monograph.

\paragraph{Fenchel conjugate.}
For any function $f: \rb^p \to \rb \cup \{ +\infty\}$, we may define the \emph{Fenchel conjugate} $f^\ast$ as the extended-value function from $\rb^p$ to $\rb \cup \{ +\infty\}$ defined as
\BEQ
f^\ast(s) = \sup_{ w \in \rb^p} w^\top s - f(w).
\EEQ
For a given direction $s \in \rb^p$, $f^\ast(s)$ may be seen as minus the intercept of the tangent to the graph of $f$ with slope $s$, defined as the hyperplane $s^\top w + c = 0$ which is below $f$ and closest (i.e., largest possible $ -c =  \max_{w \in \rb^p} w^\top s - f(w)$). See \myfig{fenchel}.

As a pointwise supremum of linear functions, $f^\ast$ is always convex (even if $f$ is not), and it is always closed. By construction, for all $s \in \rb^p$ and $w \in \rb^p$, $f(w) \geqslant w^\top s - f^\ast(s)$, i.e., $f$ is lower-bounded by a set of affine functions. When $f$ is convex, it is geometrically natural that the envelope of these affine functions is equal to $f$. Indeed, when $f$ is convex and closed, then the biconjugate of $f$ (i.e., $f^{\ast \ast}$) is equal to $f$, i.e., for all $w \in \rb^p$,
 $$
 f(w) = \sup_{s \in \rb^p} w^\top s  - f^\ast(s).
 $$
 In \myfig{fenchel}, we provide an illustration of the representation of $f$ as the maximum of affine functions.
 
 \begin{figure}
\begin{center}
 \includegraphics[scale=1]{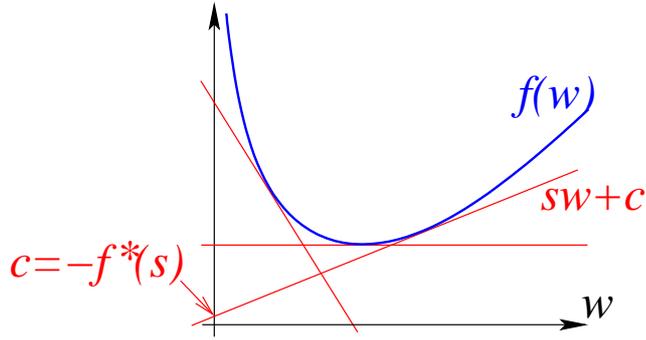}
 \end{center}
\caption{Convex function $f(w)$ as the maximum of affine functions. For any slope $s$, $c = - f^\ast(s)$ is the intercept of the tangent of direction $s$. }
\label{fig:fenchel}
\end{figure}

  If $f$ is not convex and closed, then the bi-conjugate is always a lower-bound on $f$, i.e., for all $w \in \rb^p$, $f^{\ast \ast}(w) \leqslant f(w)$, and it is the tightest such convex closed lower bound, often referred to as the \emph{convex envelope} (see examples in \mychap{relaxation}).

When $f$ is convex and closed, many properties of $f$ may be seen from $f^\ast$ and vice-versa:
\begin{list}{\labelitemi}{\leftmargin=1.1em}
   \addtolength{\itemsep}{-.2\baselineskip}
\item[--] $f$ is strictly convex if and only if $f^\ast$ is differentiable in the interior of its domain,
\item[--] $f$ is $\mu$-strongly convex (i.e., the function $w \mapsto f(w) - \frac{\mu}{2} \|w\|_2^2$ is convex) if and only if $f^\ast$ has Lipschitz-continuous gradients (with constant $1/\mu$) in the interior of its domain.
\end{list}

In this monograph, we   often consider \emph{separable} functions $f: \rb^p \to \rb$, that is, functions which may be written as $f(w) = \sum_{k=1}^p f_k(w_k)$ for certain functions $f_k: \rb \to \rb$. We then have, for all $s \in \rb^p$, $f^\ast(s) = \sum_{k=1}^p f_k^\ast(s_k)$. In Table~\ref{tab:fenchel}, we give classical Fenchel dual pairs (see also~\cite{borwein2006caa}). Moreover, for a real number $a \neq 0$ and $b \in \rb^p$, if $f^\ast$ is the Fenchel dual of~$f$, then $ s \mapsto f^\ast(s/a) - b^\top s  / a$ is the Fenchel dual of $w \mapsto f(aw+b)$.

\begin{table}
$\!\!\!\!$
\begin{tabular}{|ll|ll|l|}
\hline
dom($f$) $\!\!\!\!\!$ & $f(w)$ & dom($f^\ast$)  $\!\!\!\!\!$ & $f^\ast(s)$ & conditions \\
\hline
$\rb$ &  $\frac{1}{2} w^2$ & $\rb$ & $\frac{1}{2} s^2$ & \\
$\rb$ &  $\frac{1}{q} |w|^q$ & $\rb$ & $\frac{1}{r} |s|^r$ & $\! q,r \in (1,\infty)\!\!$\\
$\rb$ & $\log(1+ e^{w} )$ & $[0,1] $ & $(1\!-\!s) \log(1\!-\!s)+ s \log s $ & \\
$ \rb$ & $ (w)_+$ & $[0,1]$ & $0$ & \\
$ \rb_+$ & $- w^q/ q $ & $ \rb_+^\ast$ & $- ( - s)^r / r$ & $q\in (0,1)$\\
\hline
\end{tabular}
\caption{Fenchel dual pairs for one-dimensional functions. The real numbers $q$ and $r$ are linked through $\frac{1}{q}+\frac{1}{r}=1$. }
\label{tab:fenchel}
\end{table}

\paragraph{Fenchel-Young inequality.}
We have seen earlier that for any pair $(w,s) \in \rb^p$, then we have
$$
f(w) + f^\ast(s) - w^\top s \geqslant 0.
$$
There is equality above, if and only if $(w,s)$ is a Fenchel-dual pair for $f$, i.e., if and only
if $w$ is maximizer of $w^\top s - f(w)$, which is itself equivalent to $s$ is maximizer of $w^\top s - f^\ast(s)$. When both $f$ and $f^\ast$ are differentiable, this corresponds to $f'(w) = s$ and $w = (f^\ast)'(s)$.

\paragraph{Fenchel duality.}
Given two convex functions $f,g: \rb^p \to \rb \cup \{+ \infty\}$, then  we have:
\BEAS
\min_{w \in \rb^p} g(w) + f(w) 
& = & \min_{w \in \rb^p} \max_{s \in \rb^p} g(w) + s^\top - f^\ast(s) \\
& = & \max_{s \in \rb^p} \min_{w \in \rb^p}  g(w) + s^\top - f^\ast(s) \\
& = & \max_{s \in \rb^p}- g^\ast(-s) - f^\ast(s),
\EEAS
which defines two convex optimization problems dual to each other. Moreover, given any candidate pair $(w,s)$, the following difference between primal and dual objectives provides certificate of optimality:
$$
{\rm gap}(w,s) = \big[ f(w)  + f^\ast(s) - w^\top s \big]   +  \big[  g(w) + g^\ast(-s) - w^\top (-s) \big].
$$
By Fenchel-Young inequality, the ${\rm gap}(w,s) $ is always non-negative (as the sum of two non-negative parts), and is equal to zero if and only the two parts are equal to zero, i.e., $(w,s)$ is a Fenchel dual pair for $f$ and   $(w,-s)$ is a Fenchel dual pair for $g$.

\paragraph{Support function.}
Given a convex closed set $C$, the support function of $C$ is the Fenchel conjugate of $I_C$, defined as:
$$
\forall s \in \rb^p, \ I_C^\ast(s) = \sup_{ w \in C} w^\top s.
$$
It is always a positively homogeneous proper closed convex function. Moreover, if $f$ is a positively homogeneous proper closed convex function, then $f^\ast$ is the indicator function of a closed convex set.

\paragraph{Proximal problems and duality.}
In this monograph, we will consider minimization problems of the form
$$
\min_{w \in \rb^p} \frac{1}{2}\| w - z\|_2^2 + f(w),
$$
where $f$ is a positively homogeneous proper closed convex function (with $C$ being a convex closed set such that $f^\ast = I_C$). We then have
\BEAS
\min_{w \in \rb^p} \frac{1}{2}\| w - z\|_2^2 + f(w)
& = & 
\min_{w \in \rb^p} \max_{s \in C} \frac{1}{2}\| w - z\|_2^2 + w^\top s \\
& = & 
\max_{s \in C}  \min_{w \in \rb^p}  \frac{1}{2}\| w - z\|_2^2 + w^\top s \\
& = & 
\max_{s \in C}  \frac{1}{2} \| z\|_2^2 - \frac{1}{2} \| s - z \|_2^2,
\EEAS
where the unique minima of the two problems are related through $w = z - s$. Note that the inversion of the maximum and minimum were made possible because strong duality holds in this situation ($f$ has domain equal to $\rb^p$). Thus the original problem is equivalent to an orthogonal projection on $C$. See applications and extensions to more general separable functions (beyond quadratic) in \mychap{prox}.

\paragraph{Norms and dual norms.}
A norm $\Omega$ on $\rb^p$ is a convex positively homogeneous function such that $\Omega(w)=0$ if and only if $w=0$. One may define its dual norm $\Omega^\ast$ as follows:
$$
\forall s \in \rb^p, \ \Omega^\ast(s) = \sup_{ \Omega(w) \leqslant 1} s^\top w.
$$
The dual norm is a norm, and should not be confused with the Fenchel conjugate of $\Omega$, which is the indicator function of the unit dual ball. Moreover, the dual norm of the dual norm is the norm itself. Classical examples include $\Omega(w) = \| w\|_q$, for which $\Omega^\ast(s) = \| s \|_r$, with $1/q+1/r=1$. For example, the $\ell_2$-norm is self-dual, and the dual of the $\ell_1$-norm is the $\ell_\infty$-norm.

  \paragraph{Gauge functions.} 
  Given a closed convex set $C \subset \rb^d$, the gauge function $\gamma_C$ is the function
   $$
 \gamma_C(x) = \inf \{ \lambda \geqslant 0, \ x \in \lambda C \}.
 $$
 The domain ${\rm dom}(\gamma_C)$ of $\gamma_C$ is the cone generated by $C$, i.e., $\rb_+ C$ (that is, $\gamma_C(x) < +\infty$ if and only if $x \in \rb_+ C$).  The function $\gamma_C$ is equivalently defined as the homogeneized version of the indicator function $I_C$ (with values $0$ on $C$ and $+\infty$ on its complement), i.e., 
$
\gamma_C(x) = \inf_{\lambda \geqslant 0} \lambda I_C\Big(\frac{x}{\lambda} \Big).
$
From this interpretation, $\gamma_C$ is therefore a convex function. Moreover, it is positively homogeneous and has non-negative values. Conversely, any function $\gamma$ which satisfies these three properties is the gauge function of the set $\{ x \in \rb^d, \ \gamma(x) \leqslant 1\}$.

Several closed convex sets $C$ lead to the same gauge function. However the unique closed convex set containing the origin is $\{ x \in \rb^d, \ \gamma_C(x) \leqslant 1\}$. In general, we have for any closed convex set $C$, $\{ x \in \rb^d, \ \gamma_C(x) \leqslant 1\} = {\rm hull}( C \cup \{0\})$.

  Classical examples are norms, which are gauge functions coming from their unit balls: norms are gauge functions $\gamma$ which (a) have a full domain, (b) are such that $\gamma(x) = 0 \Leftrightarrow x = 0$, and (c) are even, which corresponds to sets $C$  which (a) have 0 in its interior, (b) are compact and (c) centrally symmetric. In general, the set $C$ might neither be compact nor centrally symmetric, for example, when $C = \{ x \in \rb_d^+, \ 1_d^\top x \leqslant  1\}$. Moreover, a gauge function may take infinite values (such as in the previous case, for any vector with a strictly negative component).

  \paragraph{Polar sets and functions.}
  Given any set $C$ (not necessarily convex), the polar of~$C$ is the set $C^\circ$ defined as
  $$
  C^\circ = \{ y \in \rb^d, \ \forall x \in C, \ x^\top y \leqslant 1\}.
  $$
  It is always closed and convex. Moreover, the polar of $C$ is equal to the polar of the closure of  ${\rm hull}( C \cup \{0\})$.
  
  When $C$ is the unit ball of a norm $\Omega$, $C^\circ$ is the unit ball of the dual norm which we denote $\Omega^\ast$; note here the inconsistency in notation, because $\Omega^\ast$ is not the Fenchel-conjugate of $\Omega$ (the Fenchel conjugate of $\Omega$ is not the dual norm, but the indicator function of the dual unit ball). Moreover, if $C$ is a cone, then the \emph{polar cone} is equal to $C^\circ = \{ y \in \rb^d, \ \forall x \in C, \ x^\top y \leqslant 0\}$, which is exactly the negative of the \emph{dual cone} often used also in convex analysis.
  
    If $C$ is a closed convex set containing the origin, then $C^{\circ \circ} = C$---more generally, for any set~$C$, $C^{\circ \circ}$ is the closure of 
    ${\rm hull}( C \cup \{0\})$.
The polarity is a one-to-one mapping from closed convex sets containing the origin to themselves.

  The Fenchel conjugate of $\gamma_C$ is the indicator function of $C^\circ$, i.e.,
  $
  \gamma_C^\ast = I_{C^\circ},
  $
  which is equivalent to $\gamma_C = I_{C^\circ}^\ast$, i.e., $\forall x \in \rb^d, \ \gamma_C(x) = \sup_{y \in C^\circ} x^\top y$.
  Given a gauge function $\gamma_C$, we define its polar as the function $\gamma_C^\circ$ given by
  $$
  \gamma_C^\circ(y) = \inf \big\{ \lambda \geqslant 0, \  \forall x \in \rb^d, \ x^\top y \leqslant \lambda \gamma_C(x) \big\}
  = \sup_{ x \in \rb^d} \frac{ x^\top y}{ \gamma_C(x)},
  $$
  the last inequality being true only if $\gamma_C(x) = 0 \Leftrightarrow x =0 $ (i.e., $C$ compact).
  It turns out that 
  $
   \gamma_C^\circ =  \gamma_{C^\circ}.
  $
  This implies that $\gamma_{C^\circ} = I_{ C^{\circ \circ} }^\ast$, i.e., $\gamma_{C^\circ}(y) = \sup_{ x \in C^{\circ \circ}} x^\top y  = \sup_{ x \in C} (x^\top y)_+$. 
  For example, the polar of a norm is its dual norm. We have for all $x,y \in \rb^d$, the inequality that is well known for forms: $x^\top y \leqslant \gamma_C(x) \gamma_{C^\circ}(y)$. Finally, the Fenchel-conjugate of $x \mapsto \frac{1}{2} \gamma_C(x)^2$ is $y \mapsto \frac{1}{2} \gamma_C^\circ(y)^2$.

\paragraph{Operations on gauge functions.} For two closed convex sets $C$ and $D$ containing the origin, then for all $x \in \rb^d$, $ \max \{ \gamma_C(x),\gamma_D(x) \} = \gamma_{C \cap D}(x)$. Another combination, is the ``inf-convolution'' of $\gamma_C$ and $\gamma_D$, i.e.,
$x \mapsto \inf_{x = y+ z} \gamma_C(z) + \gamma_D(y)$, which is equal to $\gamma_{ {\rm hull}(C \cup D) }$. Moreover, $\gamma_{C \cap D}^\circ = \gamma_{ {\rm hull}(C^\circ \cup D^\circ) }$, or equivalently, $ ( C \cap D)^\circ =  {\rm hull}(C^\circ \cup D^\circ) $.

\paragraph{Links with convex hulls.}

Given a compact set $P$ and its compact convex hull $C$ (for example, $P$ might be the set of extreme points of $C$), we have
$
P^\circ = C^\circ,
$
since maxima of linear functions on $C$ or $P$ are equal.
 An alternative definition of $\gamma_C$ is then  
 $$  \gamma_C(x) = \min \bigg\{ \sum_{i \in I} \eta_i, \ (\eta_i)_{i \in I} \in \rb_+^I, \ (x_i)_{i \in I} \in P^I,   \ I \mbox{ finite}, \ x = \sum_{i \in I} \eta_i x_i \bigg\}. $$
 Moreover, in the definition above, by Caratheodory's theorem for cones, we may restrict the cardinality of $I$ to be less than or equal to $d$.

  \label{app:gauge}

\section{Max-flow min-cut theorem}
\label{app:maxflow}
We consider  a set $W$ of vertices, which includes a set $S$ of sources and a set $V$ of sinks (which will be the set on which the submodular function will be defined). We assume that we are given capacities, i.e., a function $c$ from $W \times W$ to $\rb_+$. For all functions $\varphi: W \times W \to \rb$, we use the notation $\varphi(A,B) = \sum_{k \in A, \ j \in B} \varphi(k,j)$.

A flow is a function $\varphi: W \times W \to \rb_+$ such that:

\begin{list}{\labelitemi}{\leftmargin=2.1em}
   \addtolength{\itemsep}{-.2\baselineskip}

\item[(a)]  capacity constaints: $\varphi \leqslant c$ for all arcs,
\item[(b)]  flow conservation: for all $w \in W \backslash (S \cup V )$, the net-flow at $w$, i.e., $\varphi(W,\{w\}) - \varphi( \{w\},W)$, is zero,
\item[(c)] positive incoming flow: for all sources $s \in S$,  the net-flow at $s$ is non-positive, i.e., $\varphi(W,\{s\}) - \varphi( \{s\},W) \leqslant 0$,
\item[(d)] positive outcoming flow:   for all sinks $t \in V$, 
the net-flow at $t$ is non-negative, i.e., $\varphi(W,\{t\}) - \varphi( \{t\},W) \geqslant 0$.
\end{list}
We denote by $\mathcal{F}$ the set of flows, and $\varphi(w_1,w_2)$ is the flow going from $w_1$ to $w_2$. This set $\mathcal{F}$ is a polyhedron in $\rb^{W \times W}$ as it is defined by a set of linear equality and inequality constraints

\begin{figure}
\begin{center}
 \includegraphics[scale=.7]{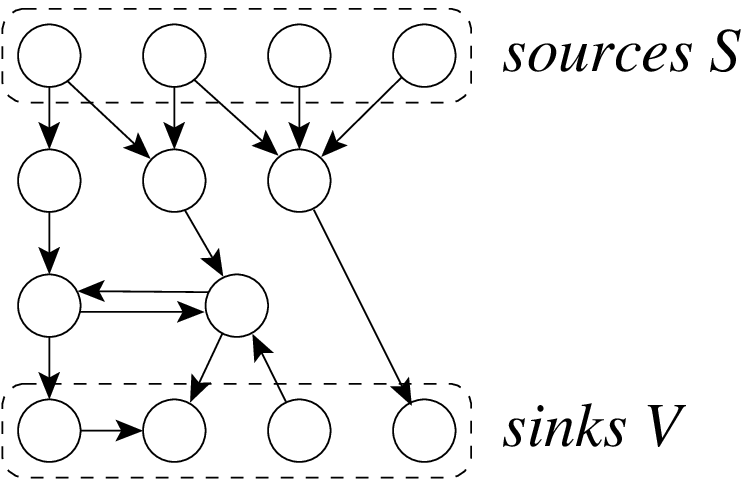}
 \end{center}
 
 \vspace*{-.25cm}

\caption{Flows. Only arcs with strictly positive capacity are typically displayed. Flow comes in by the sources and gets out from the sinks.}
\label{fig:flow-app}
\end{figure}

For $A \subseteq V$ (the set of sinks), we define 
$$ \displaystyle F(A) = \max_{ \varphi \in \mathcal{F}} \ \varphi(W,A) - \varphi(A,W),$$
which is the maximal net-flow getting out of $A$. We now prove the max-flow/min-cut theorem, namely that
$$
F(A) = \min_{X \subseteq W, \ S \subseteq X, \ A \subseteq W \backslash X } c(X, W \backslash X).
$$

The maximum-flow $F(A)$ is the optimal value of a linear program. We therefore introduce Lagrange multipliers for the constraints in   (b), (c) and (d)---note that we do not dualize constraint (a). These corresponds to $\lambda_w$ for $w \in W$, with the constraint that $\lambda_s \leqslant 0$ for $ s\in S$ and $\lambda_t \geqslant 0$ for $ t\in V$. We obtain the dual problem as follows (strong duality holds because of the Slater condition):
\BEAS
 & & F(A) \\
 & = & \max_{ 0 \leqslant \varphi \leqslant c}  \min_{\lambda}
\varphi(W,A) - \varphi(A,W) +
\sum_{w \in W} \lambda_w \big[ \varphi(W,\{w\}) - \varphi( \{w\},W) \big] 
\\
& = &   \min_{\lambda} \max_{ 0 \leqslant \varphi \leqslant c}
\varphi(W,A) - \varphi(A,W) +
\sum_{w \in W} \lambda_w \big[ \varphi(W,\{w\}) - \varphi( \{w\},W) \big] 
\\
& = &   \min_{\lambda} \max_{ 0 \leqslant \varphi \leqslant c}
\sum_{w \in W, v \in W} \varphi(v,w)
\big(
\lambda_w - \lambda_v + 1_{w \in A} - 1_{v \in A}
\big)
\\
& = &   \min_{\lambda}  
\sum_{w \in W, v \in W}   c(v,w)
\big(
\lambda_w - \lambda_v + 1_{w \in A} - 1_{v \in A}
\big)_+.
\EEAS

For any set $X \subseteq W$ such that $S \subseteq X$ and $A \subseteq W \backslash X$, then we may define $x \in \{0,1\}^W$ as the indicator vector of $X$. The cut $c(X,W \backslash X)$ is then equal
to $\sum_{w,v \in W} c(v,w) ( x_v - x_w)_+$. Given our constraints on $X$, we have $x_S = 1$ and $x_A = 0$,  $\lambda_w$ defined as $ \lambda_w = 1 - x_w - 1_{w \in A}$ is such that
for $w \in S$, $\lambda_w  =   0$ and for
$ w \in T$, $\lambda_w = 1 - x_w - 1_{w \in A} = 1_{w \in W \backslash X} - 1_{w \in A} \geqslant 0$ because $A \subseteq W \backslash X$. Thus $\lambda$ is a dual-feasible vector and thus the
cut $c(X,W \backslash X)$ is larger than $F(A)$.

In order to show equality, we denote by $\lambda$ an optimal dual solution of the convex program above,
and consider a uniform random variable  $\rho$ in $(0,1)$,  and define $x \in \{0,1\}^W$ as follows:
$x_w = 1$ if $-\lambda_w  - 1_{w \in A} \geqslant -\rho$, and zero otherwise. Denote by $X$ the set of $w$ such that $x_w=1$.
If $w \in S$, then $-\lambda_w  - 1_{w \in A} \geqslant 0$ and thus $x_w$ is almost surely equal to one, i.e, $S \subseteq X$ almost surely. Moreover, if $w \in A$, $-\lambda_w  - 1_{w \in A}  \leqslant -1$ and thus 
$x_w$ is almost surely equal to zero, i.e, $A \subseteq W \backslash X$ almost surely. 
Finally, the
expectation of the cut $c(X,W \backslash X)$ is 
equal to  $\sum_{w \in W, v \in W}   c(v,w) \mathbb{P}( -\lambda_v - 1_{v \in A} \geqslant
- \rho \geqslant   -\lambda_w - 1_{w \in A} )
\leqslant \sum_{w \in W, v \in W}   c(v,w)
\big(
\lambda_w - \lambda_v + 1_{w \in A} - 1_{v \in A}
\big)_+ = F(A)$. This implies that almost surely, the cut $c(X,W \backslash X)$ is equal to $F(A)$ and hence the result.

\section{Pool-adjacent-violators algorithm}
\label{app:pava}
In this section we consider $z \in \rb^p$ and the following minimization problem, which is the isotonic regression problem with a chain constraint (see more general formulations in \mysec{isotonic}).
\BEQ
\min_{w \in \rb^p} \frac{1}{2}\| w - z \|_2^2 \mbox{ such that } \forall t \in \{1,\dots,p-1\}, \  w_t \geqslant w_{t+1}.
\EEQ

We may find the dual problem as follows:
\BEAS
&& \min_{w \in \rb^p} \frac{1}{2} \| w - z \|_2^2 \mbox{ such that } \forall t \in \{1,\dots,p-1\}, \  w_t \geqslant w_{t+1} \\
& \!\!\!= &   \min_{w \in \rb^p} \max_{ s \in \rb^p  } -s^\top ( w - z) - \frac{1}{2} \|s\|_2^2 
\ \mbox{ s.t. } \forall t \in \{1,\dots,p-1\}, \  w_t \geqslant w_{t+1}\\
& & \hspace*{2cm} \mbox{ with } w = z-s \mbox{ at optimality, } \\
&\!\! \!= &   \max_{ s \in \rb^p  }  \min_{w \in \rb^p}  -s^\top ( w - z) - \frac{1}{2} \|s\|_2^2 
\ \mbox{ s.t. } \forall t \in \{1,\dots,p-1\}, \  w_t \geqslant w_{t+1}.
\EEAS
We have, with $S_t = s_1 + \cdots + s_t$ (with the convention $S_0=0$), by summation by parts,
\BEAS
w^\top s& = &   \sum_{t=1}^p (S_t - S_{t-1}) w_t = 
\sum_{t=1}^{p-1} S_t ( w_t - w_{t+1}) + w_p S_p.
\EEAS
This implies that the maximum value of $w^\top s$ such that $w$ is a non-increasing sequence is equal to zero, if $S_t \leqslant 0$ for all $t \in \{1,\dots,P-1\}$, and $S_p = 0$, and equal to $+\infty$ otherwise.
Thus we obtain the dual optimization problem:
\BEA
\label{eq:dual-pav}  & & \max_{ s \in \rb^p  }   s^\top z - \frac{1}{2} \| s \|_2^2 \\
\nonumber & & 
\mbox{ such that } s(V) = 0 \mbox{ and } \forall k \in \{1,\dots,p-1\}, \ 
s(\{1,\dots,k\}) \leqslant 0.
\EEA

We now consider an active-set method such as described in \mysec{QP} for the problem in \eq{dual-pav}, starting from all constraints saturated, i.e., $s(\{1,\dots,k\}) = 0$, for all $k \in \{1,\dots,p-1\}$, leading to $s = 0$, this corresponds to a primal candidate $w = z$ and the active set $J = \varnothing$.

Given an active set $J$, we may compute the dual feasible $s \in \rb^p$ that maximizes
$ s^\top z - \frac{1}{2} \| s \|_2^2$ such that $s(V) = 0$ and for all
$j \in J^{\sf c}$, $s(\{1,\dots,j\}) = 0$. If   $j_1 < \cdots < j_k$ are the $k=| J^{\sf c}|$ elements of $ J^{\sf c}$ and we consider the sets $A_1 = \{1,\dots,j_1\}$, 
$A_i = \{ j_{i-1}+1, \dots, j_{i} \}$ for $i \in \{2,\dots,k\}$, and $A_{k+1} = \{ j_{k}+1, \dots, p \}$,
then $w = z - s$ has to be constant on each $A_i$, and its  value equal to $v_i =  {z(A_i)}/{|A_i|}$. Indeed, we have
\BEAS
& & \max_{ s_{A_i}(A_i) = 0 } s_{A_i} ^\top z_{A_i} - \frac{1}{2} \| s_{A_i}\|_2^2 \\
& = &  \min_{v_i} \max_{ s_{A_i}  } s_{A_i} ^\top z_{A_i} - \frac{1}{2} \| s_{A_i}\|_2^2 - v_i   s_{A_i}(A_i)  \\
& = & \min_{ v_i} \frac{1}{2} \| z_{A_i} - v_i 1_{A_i} \|_2^2,
\EEAS
with solution $v_i = z(A_i)/|A_i|$.

As shown in \mysec{QP}, in an active-set method, given a new candidate active   set, there are two types of steps: (a) when the corresponding optimal value $s$ obtained from this active-set is feasible, we need to check dual-feasibility (i.e., here that the sequence $(v_i)$ is non-increasing). If it is not, then the active set is augmented and here this exactly corresponds to  taking any violating adjacent pair $(v_i,v_{i+1})$, and merge the corresponding sets $A_i$ and $A_{i+1}$, i.e., add $j_{i}$ to $J$. The other type of steps is (b) when the corresponding optimal value is not feasible: this never occurs in this situation~\cite{grotzinger1984projections}.

Thus the algorithm updates the constant values of $w$, and averages adjacent ones when they violate the ordering constraint.
By starting from the first indices, it is possible to choose the ordering of the pooling operations so that the overall complexity is $O(p)$. See~\cite{grotzinger1984projections,best1990active} for details. Moreover, replacing the squared $\ell_2$-norm $\| w - z \|^2$ by a weighted squared $\ell_2$-norm leads to the same running-time complexity.

\chapter{Operations that Preserve Submodularity}
\label{app:misc}
 
\label{app:ope}
\label{app:preserve}

In this appendix, we present several ways of building submodular functions from existing ones. For all of these, we describe how the \lova extensions and the submodular polyhedra are affected. Note that in many cases, operations are simpler in terms of submodular and base polyhedra.
Many operations such as projections onto subspaces may be interpreted in terms of polyhedra corresponding to other submodular functions. 

We have seen in \mysec{entropies} that given any submodular function $F$, we may define $G(A) = F(A) + F(V\backslash A) - F(V)$. Then $G$ is always submodular and symmetric (and thus non-negative, see \mysec{posi}). This symmetrization can be applied to any submodular function and in the example of \mychap{examples}, they often lead to interesting new functions. We now present other operations that preserve submodularity.

\begin{proposition}\textbf{(Restriction of a submodular function)}
\label{prop:restriction}
let $F$ be a submodular function such that $F(\varnothing)=0$ and $A \subseteq V$. The restriction of $F$ on $A$, denoted $F_A$ is a set-function on $A$ defined as $F_A(B) = F(B)$ for $B \subseteq A$. The function $F_A$ is submodular. Moreover, if we can write the \lova extension of $F$ as $f(w) = f( w_A , w_{V \backslash A})$, then the \lova extension of $F_A$ is $f_A(w_A) = f(w_A,0)$. Moreover, the submodular polyhedron $P(F_A)$ is simply the projection of $P(F)$ on the components indexed by $A$, i.e., $s \in P(F_A)$ if and only if $\exists t $ such that $(s,t) \in P(F)$.
 \end{proposition}
\begin{proof} Submodularity and the form of the \lova extension  are straightforward from definitions. To obtain the submodular polyhedron, notice that we have
$f_A(w_A) = f(w_A,0)  = \max_{(s,t) \in P(F)} w_A^\top s + 0^\top t$, which implies the desired result.
\end{proof}

\begin{proposition}\textbf{(Contraction of a submodular function)}
\label{prop:contraction}
let $F$ be a submodular function such that $F(\varnothing)=0$ and $A \subseteq V$. The contraction of $F$ on $A$, denoted $F^A$ is a set-function on $V \backslash A$ defined as $F^A(B) = F(A \cup B) - F(A)$ for $B \subseteq V \backslash A$. The function $F^A$ is submodular. Moreover, if we can write the \lova extension of $F$ as $f(w) = f( w_A , w_{V \backslash A})$, then the \lova extension of $F^A$ is $f^A(w_{V \backslash A}) = f(1_A,w_{V \backslash A}) -  F(A)$. Moreover, the submodular polyhedron $P(F^A)$ is simply the projection of $P(F) \cap \{ s(A) = F(A) \}$ on the components indexed by $V \backslash A$, i.e., $t \in P(F^A)$ if and only if $\exists s \in P(F) \cap \{ s(A) = F(A) \}$, such that $s_{V \backslash A} =t$.
 \end{proposition}
\begin{proof} Submodularity and the form of the \lova extension  are straightforward from definitions. Let $t \in \rb^{| V \backslash A|}$. If  $\exists s \in P(F) \cap \{ s(A) = F(A) \}$, such that $s_{V \backslash A} =t$, then we have for all $B \subseteq V \backslash A$, $t(B) = t(B) + s(A) - F(A) \leqslant F(A \cup B) - F(A)$, and hence $t \in P(F^A)$. If $t \in P(F^A)$, then take any $v \in B(F_A)$ and concatenate $v$ and $t$ into $s$. Then, for all subsets $C \subseteq V$,
$s(C) = s(C \cap A) + s( C \cap ( V \backslash A) ) = 
 v(C \cap A) + t( C \cap ( V \backslash A) ) \leqslant F(C \cap A) + F( A \cup ( C \cap ( V \backslash A)  ) ) - F(A)
 = F(C \cap A) + F( A    \cup C ) - F(A) \leqslant F(C)$ by submodularity. Hence $s \in P(F)$.

\end{proof}

The next proposition shows how to build a new submodular function from an existing one, by partial minimization. Note the similarity (and the difference) between the submodular polyhedra for a partial minimum (Prop.~\ref{prop:partial}) and for the restriction defined in Prop.~\ref{prop:restriction}. 
 
Note also that contrary to convex functions, the pointwise maximum of two submodular functions is not in general submodular (as can be seen by considering functions of the cardinality from \mysec{card}).

\begin{proposition}\textbf{(Partial minimum of a submodular function)}
\label{prop:partial}
We consider a submodular function $G$ on $V \cup W$, where $V \cap W = \varnothing$ (and $|W|=q$), with \lova extension $g:\rb^{p+q} \to \rb$. We consider, for $A \subseteq V$, $F(A) = \min_{B \subseteq W} G(A \cup B) -    \min_{B \subseteq W} G(  B)  $. The set-function $F$ is submodular and such that $F(\varnothing)=0$. 
Moreover, if $ \min_{B \subseteq W} G( B) =0$, we have for all $w \in \rb_+^p$,    $f(w) = \min_{ v\in \rb_+^q} g(w,v)  $, and the submodular polyhedron $P(F) $ is the set of $s \in \rb^p$ for which there exists $t \in \rb_+^q$, such that $(s,t) \in P(G)$.
 \end{proposition}
\begin{proof} Define $c = \min_{B \subseteq W} G( B) \leqslant 0 $, which is independent of $A$. We have, for $A,A' \subseteq V$, and any $B,B' \subseteq W$, by definition of $F$:
\BEAS
 & & F(A \cup A') + F( A\cap A') \\
& \leqslant & -2c +  G([ A \cup A'] \cup [ B  \cup B' ] ) + G( [ A \cap A' ] \cup [ B \cap B' ] )  \\
& = &  -2c + G([ A \cup B ] \cup [ A' \cup B' ] ) + G( [ A \cup B ] \cap [ A' \cup B' ] )  \\
& \leqslant &  -2c + G( A \cup B ) + G ( A' \cup B'  )
\mbox{ by submodularity}
.
\EEAS
Note that the second equality is true because $V$ and $W$ are disjoint.
Minimizing with respect to $B$ and $B'$ leads to the submodularity of $F$.

Assuming that $c=0$, we now show that $P(F) = \big\{ s \in \rb^p, \ \exists t \in \rb^q, \ t\geqslant 0, \ (s,t) \in P(\widetilde{G}) \big\}$. We have $s \in P(F)$, if and only if for all $A \subseteq V$, $B \subset W$, then $s(A) \leqslant G( A \cup B)$: (a) for $s \in P(F)$, we define $t =0$, and thus $(s,t) \in P(G)$; (b) if there exists $t \geqslant 0$ such that $(s,t) \in P(G)$, then for all $A, B$, $s(A) \leqslant s(A) + t(B) \leqslant G(A \cup B)$.

We may now use Prop.~\ref{prop:support}, to get for $w \in \rb^p_+$:
\BEAS
f(w) & = & \max_{s \in P(F)} s^\top w = \max_{(s,t) \in P(G), \ t \geqslant 0} s^\top w \\
& = & \min_{v \geqslant 0} \max_{(s,t) \in P(G) } s^\top w + t^\top v \mbox{ by Lagrangian duality}, \\
& = & \min_{v \geqslant 0} g(w,v) \mbox{ by Prop.~\ref{prop:support}}.
\EEAS
 \end{proof}

The following propositions give an interpretation of the intersection between the submodular polyhedron and sets of the form $\{s \leqslant z\}$ and $\{s \geqslant z\}$. Prop.~\ref{prop:conv} notably implies that for all $z \in \rb^p$, we have:
$ \min_{ B \subseteq V } F(B) + z(V \backslash B) = \max_{ s \in P(F), \ s \leqslant z } s(V)$, which implies the second statement of Prop.~\ref{prop:dualmin} for $z=0$.

\begin{proposition}\textbf{(Convolution of a submodular function and a modular function)}
\label{prop:conv}
Let $F$ be a submodular function such that $F(\varnothing)=0$ and $z \in \rb^p$. Define
$G(A) = \min_{ B \subseteq A} F(B) + z(A \backslash B)$. Then $G$ is submodular, satisfies $G(\varnothing)=0$, and the submodular polyhedron $P(G)$ is equal to $P(F) \cap \{ s \leqslant z \}$. Moreover, for all $A\subseteq V$,  $G(A) \leqslant F(A)$ and $G(A)\leqslant z(A)$.
\end{proposition}
\begin{proof} Let $A,A' \subseteq V$, and $B,B'$ the corresponding minimizers defining $G(A)$ and $G(A')$. We have:
\BEAS
& & G(A) + G(A')\\
 & \!\!\!=\!\!\! & F(B) + z(A \backslash B) +  F(B') + z(A' \backslash B') \\
& \!\!\!\geqslant \!\!\!& F(B \cup B') + F(B \cap B') +  z(A \backslash B) + z(A' \backslash B')  \mbox{ by submodularity}, \\
& \!\!\!= \!\!\!& F(B \cup B') + F(B \cap B') +  z([A \cup A'] \backslash [B \cup B'])
+  z([A \cap A'] \backslash [B \cap B']) \\
&\!\!\! \geqslant \!\!\! & G(A \cup A') + G(A \cap A') \mbox{ by definition of } G,
\EEAS
hence the submodularity of $G$. If $s \in P(G)$, then  $\forall B \subseteq A \subseteq V$, $s(A) \leqslant 
G(A) \leqslant  F(B) + z(A \backslash B)$. Taking $B=A$, we get that $s \in P(F)$; from $B = \varnothing$, we get $s \leqslant z$, and hence $s \in P(F) \cap \{ s \leqslant z \}$. If $s \in P(F) \cap \{ s \leqslant z \}$, for all $\forall B \subseteq A \subseteq V$,
$s(A) = s(A \backslash B) + s(B) \leqslant  z(A \backslash B) + F(B)$; by minimizing with respect to $B$, we get that $s \in P(G)$.

We get $G(A) \leqslant F(A)$ by taking $B=A$ in the definition of $G(A)$, and we get $G(A) \leqslant z(A)$ by taking $B  = \varnothing$.
\end{proof}

\begin{proposition}\textbf{(Monotonization of a submodular function)}
\label{prop:monotone}
Let $F$ be a submodular function such that $F(\varnothing)=0$. Define
$G(A) = \min_{ B \supset A} F(B) - \min_{ B \subseteq V} F(B) $. Then $G$ is submodular such that $G(\varnothing)=0$, and the base polyhedron $B(G)$ is equal to $B(F) \cap \{ s \geqslant 0 \}$. Moreover, $G$ is non-decreasing, and for all $A\subseteq V$,  $G(A) \leqslant F(A)$.
\end{proposition}
\begin{proof} Let $c = \min_{ B \subseteq V} F(B) $. Let $A,A' \subseteq V$, and $B,B'$ the corresponding minimizers defining $G(A)$ and $G(A')$. We have:
\BEAS
G(A) + G(A') & = & F(B)   +  F(B')   - 2c  \\
& \geqslant & F(B \cup B') + F(B \cap B')  - 2c  \mbox{ by submodularity} \\
& \geqslant & G(A  \cup A') + G(A \cap A')  \mbox{ by definition of } G,
\EEAS
hence the submodularity of $G$. It is obviously non-decreasing. We get $G(A) \leqslant F(A)$ by taking $B=A$ in the definition of $G(A)$. Since $G$ is increasing, $B(G) \subseteq \rb_+^p$ (because all of its extreme points, obtained by the greedy algorithm, are in $\rb_+^p$). By definition of $G$, $B(G) \subseteq B(F)$. Thus $B(G) \subseteq B(F) \cap \rb_+^p$. The opposite inclusion is trivial from the definition.

\end{proof}

The final result that we present in this appendix is due to~\cite{lin2011-class-submod-sum} and resembles the usual compositional properties of convex functions~\cite{boyd}.

\begin{proposition} \textbf{(Composition of concave and submodular non-decreasing functions)}
\label{prop:compo}
Let $F: 2^V \to \rb$ be a non-decreasing submodular function with values in a convex subset $K \subseteq \rb$, and $\varphi: K \to \rb$ a non-decreasing concave function. Then $A \mapsto \varphi(F(A))$ is submodular.
\end{proposition}
  \begin{proof}
  Let $A \subseteq V$ and two disjoints elements $j$ and $k$ of $V \backslash A$. We have from the submodularity of $F$ and Prop.~\ref{prop:second},
  $F(A \cup \{j,k\})  \leqslant F(A \cup \{j\})   + F( A \cup \{k\}) - F(A)$. We thus get by monotonicity of $\varphi$:
  \BEAS
  & & \varphi( F(A \cup \{j,k\}) \\
  &\!\!\!\! \leqslant \!\!\!\!& \varphi( F(A \cup \{j\})   - F(A) + F( A \cup \{k\}) ) - \varphi( F( A \cup \{k\}) ) + \varphi( F( A \cup \{k\}) )  \\
  &\!\!\!\! \leqslant \!\!\!\!& \varphi( F(A \cup \{j\})   - F(A)  + F( A )  ) - \varphi( F( A) ) + \varphi( F( A \cup \{k\}) ) ,
  \EEAS
  because $F( A \cup \{k\}) \geqslant F(A)$, $F( A \cup \{j\}) \geqslant F(A)$ and $\varphi$ is concave (using the property used in the proof of Prop.~\ref{prop:card}).
  \end{proof}

\subsection*{Acknowledgements}
This monograph was partially supported by  the European Research Council (SIERRA Project). The author would like to
thank Rodolphe Jenatton, Armand Joulin, Simon Lacoste-Julien, Julien Mairal and Guillaume Obozinski for discussions related to submodular functions and convex optimization. The suggestions of the reviewers were greatly appreciated and have significantly helped improve the manuscript.

\bibliographystyle{plain} 
\bibliography{submodular}

\end{document}